\documentclass[preprint,authoryear,3p,times,twocolumn]{elsarticle}

\usepackage{amssymb,amsmath,amsthm}
\usepackage{graphicx}
\usepackage{url}
\usepackage{tablefootnote}
\usepackage{multirow}
\usepackage{subfigure}
\usepackage{caption}
\usepackage{algorithm}  
\usepackage{algorithmic}   
\usepackage[dvipsnames]{xcolor}
\usepackage{natbib}

\newtheorem{thm}{Theorem}[section]

\newtheorem{lem}[thm]{Lemma}

\biboptions{square}

\journal{Some Journal}

\begin{document}

\begin{frontmatter}



\title{Gradient Distribution Priors for Biomedical Image Processing}


\author{Yuanhao~Gong, Ivo~F.~Sbalzarini}

\address{MOSAIC Group, Center for Systems Biology Dresden (CSBD),\\ 
Max Planck Institute of Molecular Cell Biology and Genetics, 
Pfotenhauerstr.~108, D--01307 Dresden, Germany.}

\begin{abstract}
Ill-posed inverse problems are commonplace in biomedical image processing. Their solution typically requires imposing prior knowledge about the latent ground truth. While this regularizes the problem to an extent where it can be solved, it also biases the result toward the expected. With inappropriate priors harming more than they use, it remains unclear what prior to use for a given practical problem. Priors are hence mostly chosen in an {\em ad hoc} or empirical fashion. We argue here that the gradient distribution of natural-scene images may provide a versatile and well-founded prior for biomedical images. We provide motivation for this choice from different points of view, and we fully validate the resulting prior for use on biomedical images by showing its stability and correlation with image quality. We then provide a set of simple parametric models for the resulting prior, leading to straightforward (quasi-)convex optimization problems for which we provide efficient solver algorithms. We illustrate the use of the present models and solvers in a variety of common image-processing tasks, including contrast enhancement, noise level estimation, denoising, blind deconvolution, zooming/up-sampling, and dehazing. In all cases we show that the present method leads to results that are comparable to or better than the state of the art; always using the same, simple prior. We conclude by discussing the limitations and possible interpretations of the prior.
\end{abstract}


\end{frontmatter}

% \linenumbers

\setlength{\tabcolsep}{4pt}

\section{Introduction}
\label{sec:int}
Image processing has become a central element of many workflows in biology and medicine. While the type and source of images varies greatly from light microscopy to magnetic resonance imaging to electron microscopy, the image-processing tasks are often similar. Frequent tasks include image denoising, image deconvolution (beblurring), image zooming (super-resolution), scatter light removal (dehazing), noise level estimation, image quality assessment, and contrast enhancement for image visualization. All of these are inverse problems, as one attempts to reconstruct an unknown ``perfect image'' from the given imperfect (noisy, blurry, hazy, etc.) observation. Inverse problems are almost always ill-posed or at least ill-conditioned, especially if the transformation that is to be undone is non-linear or unknown.

In order to be able to solve such problems, additional knowledge about the unknown perfect image has to be assumed. Conceptually, there are two approaches to estimating the perfect image: interpolation (smoothing or filtering) and model fitting (Bayesian inference). In the former approach, the additionally assumed knowledge is encoded in the choice of the interpolation method, or in the filter kernels used. These choices typically impose certain geometric properties of the perfect image, such as connectivity, smoothness, sparsity, or curvature. In the Bayesian approach, one attempts to reconstruct a perfect image such that it resembles as much as possible the observed image when run though the (blurring, noise, etc.) transformation. Bayesian inference requires prior knowledge in the form of a {\em prior} that sufficiently constrains the reconstruction problem to render it well-posed. Frequently used priors in biomedical image processing include sparsity in the spatial and/or frequency domain~\citep{PPAHybrid}, total variation (TV)~\citep{TV1992,chan2000,chantas2010variational}, mean curvature (MC)~\citep{mean1997,Zhu2007,Liu2011}, Gaussian curvature (GC)~\citep{lee2005,Zhu2007,Lu2011,gong2013a}, and hybrid priors~\citep{kim2006pde,TGV2010}. 

While the prior knowledge regularizes the inverse problem to an extent where it can be solved, it also biases the result toward the expected. It has repeatedly been shown that inappropriate priors may obscure features in the image, or lead to wrong results altogether. Choosing the ``right'' prior, however, is as hard as solving the original problem, since the underlying perfect image is unknown. The main drawback of frequently used priors is that they bear no relation to the image contents. They merely postulate certain geometric or spectral properties of the image signal, which may mean little. The very popular TV prior~\citep{TV1992,chan2000,chantas2010variational}, for example, presupposes that the unknown perfect image be a collection of uniformly colored or uniformly bright regions, i.e., be piece-wise constant. Imposing this prior leads to removal of image detail and processing artifacts if this presumption is not justified. 

Spectral priors have been introduced in order to relax some of the constraints. They do not directly impose knowledge about a property of the perfect image, but only about the histogram (or probability distribution) of that property. As such, they are weaker priors and bias the result less. A particularly popular spectral prior is the Gradient Distribution Prior (GDP), which presupposes a certain statistical distribution of the gradients in the image, i.e., a certain gradient histogram. It has been shown to lead to better results than the TV prior in many image-processing tasks~\citep{zhu:1997,NaturePrior,shan:2008a,cho:2009,chen:2010,krishnan2009fast,cho:2012}. In Section~\ref{sec:why}, we provide some reasoning about why and when GDP should be preferred. Conceptually, GDPs are related to histogram equalization. However, while the latter presupposes a uniform distribution in the intensity domain, the former operates in the gradient domain, and the presupposed distribution is not uniform, but learned from examples.

Despite their success in signal processing, GDPs learned from natural-scene images have to the best of our knowledge never been adopted in biomedical image processing, nor have they been validated on biomedical images. Validation as a prior entails showing the following two properties:
\begin{itemize}
\item {\bf Stability:} The prior should be independent of the image content. It should be stable against variations in the imaged objects (organs, cells, tissues, etc.) and against different imaging modalities (fluorescence microscopy, electron microscopy, X-ray imaging, etc.). Here, we validate this property for the GDP in Section~\ref{sec:GDP}.
\item {\bf Correlation with image quality:} The prior should be correlated with subjectively perceived image quality. Only then, imposing the prior is expected to improve image quality. We show this here for the GDP prior in Section~\ref{sec:correlation}.
\end{itemize} 
We provide here a complete validation of the natural-scene GDP for biomedical images. 
The concept is illustrated in Fig.~\ref{fig:flow}: The GDP is constructed from natural-scene images and is validated on biomedical images. Then, this prior can be used for biomedical image-processing tasks. 

\begin{figure}[tb]
  \centering
  \includegraphics[width=\linewidth]{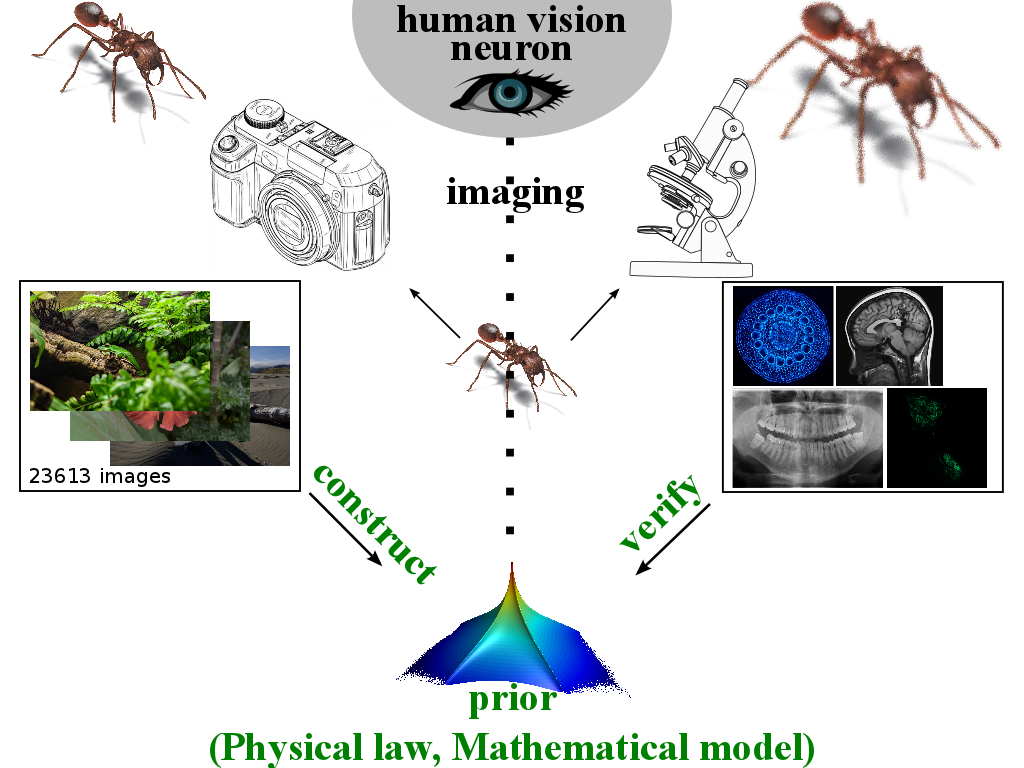}
  \caption{Concept of using natural-scene GDPs in biomedical imaging. The prior is learned from natural-scene images. Since they obey the same physical laws or mathematical model as biomedical images, we propose to use this prior also in that case. We hence learn the GDP from natural-scene images and then validate it on biomedical images, enabling its later application. The human vision system provides the unifying link between the two, as it has evolved to detect and process natural-scene images, but is also used to look at biomedical images.}
  \label{fig:flow} %% label for entire figure
\end{figure}

The second contribution made here is the generalization of GDPs to two (and higher) dimensions, and the introduction of novel parametric models for gradient distributions, which includes the correlation between $x$ and $y$ component. So far, gradient distributions have been modeled under the assumption that the gradient components along $x$ and $y$ are statistically independent in log scale. This, for example, leads to the well-known hyper-Laplacian model~\cite{krishnan2009fast,cho:2012}. As we show here, this assumption is not always justified. Often, the gradient components in an image are correlated. Ignoring these correlations may not only lead to artifacts in the result, but also complicates solving the resulting inference problem, frequently requiring alternating optimization over the gradient components. Here, we propose new parametric models for GDP in two dimensions, hence accounting for all correlations. We show that these models not only lead to more accurate results, but also allow simpler and more efficient computational solution of the resulting inference problem. We also highlight the commonalities and differences between the popular TV prior, the hyper Laplace model, and our novel models. In particular, we shown that the popular TV prior can be interpreted as a linear approximation of a GDP in log scale (Section~\ref{sec:link}).

Before presenting our results, however, we formalize the problem in a mathematical framework in Section \ref{sec:variation} and provide a detailed motivation of our method in Sections~\ref{sec:why} and \ref{sec:whynot}. Stability of the GDP on biomedical images is studied in Section~\ref{sec:GDP}. In Section~\ref{sec:GDPModel} we provide novel parametric models for the GDP in one and two dimensions. Correlation with image quality is shown in Section~\ref{sec:correlation}. In Section~\ref{sec:remap}, we demonstrate how to impose the present GDP in a variational framework. In Section~\ref{sec:app}, we illustrate several applications from biomedical image processing. We conclude and discuss this work in the closing Section~\ref{sec:DFW}.   

\section{Mathematical Framework}
\label{sec:variation}
We aim at computing an estimate $\hat{U}$ of the unknown, latent perfect image $U(\vec{x})$ from the observed discrete samples $S = \{ s_i(\vec{x}): i=1\dots N\}$ ($\vec{x}$ is the spatial coordinate), which are the pixels of the data image, or more generally a cloud of data points. The data $S$ have been generated from the underlying truth $U$ by imaging the latter, introducing blur, noise, scattering, down-sampling, etc. This lossy image-formation process is generally modeled as a non-linear map $S=F(U)$. The reconstruction problem can then be expressed in variational form as:
\begin{equation}
\label{eq:general}
\hat{U} = \arg\min_{U'\in F_s}\left\{ \int_{\vec{x}\in\Omega}\!\!\!\Phi_1(U',S)\,\mathrm{d}\vec{x} + \lambda \!\! \int_{\vec{x}\in\Omega}\!\!\!\Phi_2(U')\,\mathrm{d}\vec{x}\right\}\, , 
\end{equation}
where $\Phi_1$ is a data-fitting cost function, measuring how well the estimated image approximates the observed image when run through the transformation considered. $\Phi_1$ hence models the (generally unknown) imaging transformation $F$. $\Phi_2$ constitutes the prior, i.e., a regularization function on $U$. It is common to include a scalar weighting coefficient $\lambda$, called the regularization coefficient, which tunes the trade-off between the data and the prior. The optimal value of $\lambda$ is only known for certain special models (Section~\ref{sec:link}). $\Omega$ is the image domain, and $F_s$ is the postulated function space in which $U$ lives.

The choice of function space $F_s$ (e.g., $C^2(\Omega)$, $C^4(\Omega)$, or Bounded Variation space over $\Omega$) defines the image model. All  solutions of the reconstruction problem are members of this space. A popular choice is the space $L^q(\Omega)$ ($0\leq q \leq \infty$). When $\frac{1}{q}+\frac{1}{q^{*}}=1$, $L^q(\Omega)$ and $L^{q^*}(\Omega)$ are dual spaces. From the H\"older inequality, one concludes that $L^{q_+}(\Omega)\subseteq L^{q}(\Omega)$, for $\forall q_+\geq q \geq 1$. This implies that $L^{1}(\Omega)$ is the most general (least restrictive) space among all $L^q(\Omega)$.
 
$\Phi_1$ measures how well a certain hypothetical reconstruction $U'$ fits the data $S$. This generally involves a model $\hat{F}$ of the unknown image-formation process $F$. This model is typically built from prior knowledge about the optics of the imaging equipment. In order to quantify the distance between $\hat{F}(U')$ and $S$, $\Phi_1$ uses any metric or semi-metric, such as the Euclidean distance, the Hausdorff distance, an $\ell_p$ distance, tangent distance, or a Bregman divergence~\citep{Paul:2013}. The choice of the data-fitting function $\Phi_1$ depends on how the data were obtained, on the noise type and magnitude, on the tolerated reconstruction error, and on considerations of computational efficiency. The $\ell_2$ norm is commonly used because it filters Gaussian noise on the data. Another frequent choice is the $\ell_1$ norm, because it filters outliers.    

In most of models, $\Phi_2$ is a regularization term, such as Tikhonov, the $\ell_2$ norm of the gradient, TV, MC, or GC. This term imposes prior knowledge (sparsity, smoothness, etc.) about the unknown perfect image $U$.
 
\subsection{Spectrally Regularized Models}
When using a spectrally regularized model, the regularization term does not directly act on $U'$, but on a distribution or histogram $p(U')$:
\begin{equation}
\label{eq:spectral}
\begin{split}
\hat{U} = \arg\min_{U'\in F_s} & \left\{ \int_{\vec{x}\in\Omega}\!\!\!\Phi_1(U',S)\,\mathrm{d}\vec{x} \right\}\\
& s.t.~~p(\mathcal{J}(U')) = p^{\mathrm{pr}}_{\mathcal{J}}\, ,
\end{split}
\end{equation}
where $\mathcal{J}$ is a filter (map, feature, differential operator, etc.) and $p^{\mathrm{pr}}_{\mathcal{J}}$ is the corresponding spectral prior. In GDPs, the filter $\mathcal{J}=\nabla$, and $p(\cdot)$ is the gradient distribution.

The spectral constraint can be relaxed by introducing an auxiliary variable $\tilde{U}$ for decoupling:
\begin{equation}
\label{eq:spectralSplit}
\begin{split}
\hat{U} = \arg \min_{U'\in F_s} & \left\{ \int_{\vec{x}\in\Omega}\!\!\!\,\left[\Phi_1(U',S) +\lambda \Phi_2(U',\tilde{U})\right]\,\mathrm{d}\vec{x}\right\},\\
& s.t.~~p(J(\tilde{U}))= p^{\mathrm{pr}}_{J}\, .
\end{split}
\end{equation}
The type of decoupling is generic to all variational models with hard constraints. It has previously been used, for example, in split-Bregman~\citep{Paul:2013}, TGV~\citep{TGV2010}, and Hyper-Laplacian~\citep{krishnan2009fast} models. From an optimization point of view, Eq.~\ref{eq:spectral} and Eq.~\ref{eq:spectralSplit} are problems with hard and soft constraint, respectively (details see Section~\ref{sec:remap}).

\section{Motivation: Why the Gradient Distribution?}
\label{sec:why}
Rather than postulating {\it ad hoc} properties of the perfect image to be reconstructed, spectral priors are typically learned or identified from large image collections. Given a sufficiently diverse collection of images, the histogram or probability distribution of a spectral prior is estimated by computing some features over all images. There are many features that can be computed, including color and texture features, but the image gradient is particularly interesting. This is first because it is remarkably stable (invariant) across images. Second, it is easy to compute and can hence be learned from large image collections. Third, the gradient has a simple intuitive meaning as the first-order approximation to the perfect image. Furthermore, reconstructing an image from a given gradient field is computationally simple and efficient. In the following, we use the term {\em gradient field} whenever we mean the gradient image, i.e., an image that has the same size as the original data image, but where each pixel stores two values that are the two components of the gradient of the original image at that location. The {\em gradient distribution} is the histogram or probability distribution of these values across all pixels, and/or across multiple images. We restrict our discussion to two-dimensional images where the gradient has two components. Extensions to higher-dimensional images are readily possible by adding additional gradient vector components. 

GDP have been exploited in Bayesian frameworks for image denoising~\citep{zhu:1997}, deblurring~\citep{shan:2008a}, restoration~\citep{cho:2012}, super resolution~\citep{zhang:2012}, and others~\citep{NaturePrior,shan:2008a,cho:2009,chen:2010,krishnan2009fast,cho:2012}. As shown in Ref.~\citep{cho:2009}, deblurring in the gradient domain is more efficient than working with the original pixel values~\citep{shan:2008a}. This can be explained by the reduced correlation in the gradient domain (cf.~Fig.~\ref{fig:example}), which is favorable for blurring kernel estimation.

It is well known by now that the gradient distribution of any image has a heavy tail in log scale.  
In previous works, however, the 2D gradient distribution is assumed to be the product of two independent 1D distributions along $x$ and $y$~\citep{shan:2008a,cho:2009,chen:2010,krishnan2009fast,cho:2012}. As we show below, this is not necessarily the case and there can be significant correlations between the $x$ and $y$ components of the gradient. The full joint 2D gradient distribution can be estimated from image databases and we provide parametric models for it that enable us to use it similarly to the 1D case.

The gradient distribution has a number of different interpretations that provide additional arguments for its use.

\subsection{A statistical argument}
Statistical interpretations of images have led to a wealth of powerful reconstruction methods based on Bayesian estimation theory. A very famous example are Markov Random Fields (MRF), as used for example in {\it{Fields of Experts}} models~\citep{roth:2009,xu:2009,dong:2012}. They are based on assuming a Gibbs distribution for the pixel intensities $I$:
\begin{equation}
\label{FOE}
p(I(\vec{x}))=\frac{1}{Z}\prod_{r}\prod_{i=1}^{N}\Psi_i((J_i*I)_{r};\alpha_i)\, .
\end{equation}
$Z$ is a normalization constant, $J_i$ is a linear filter, $*$ is the convolution operator, $r$ a local window size, and $\Psi_i$ is modeled as a heavy-tailed (often Student-T) distribution with parameter $\alpha_i$ (the number of degrees of freedom). The perfect image is then estimated by a Bayesian {\it{Maximum A Posteriori}} (MAP) estimator. These models are powerful and can learn structural information, but there are several crucial parameters to be tuned: the image filter $J$, the shape of the potential function $\Psi$, the local window size $r$, etc. Of these, the window size $r$ is particularly hard to choose for training and inference.

While MRFs estimate the intensity in each pixel (hence taking a microscopic view of the image), spectral priors describe a global property of the whole image (macroscopic view), ignoring local geometry variations. This leads to more stable estimators from a statistical point of view. 

According to Boltzmann, the microscopic and the macroscopic view are connected through the concept of {\it entropy}. Interpreting each pixel as a random variable, an image can be seen as a state of a high-dimensional random process (pixels). This corresponds to an Ising model in statistical mechanics. Assuming that of all possible combinations of pixel values one could form, meaningful images are equilibrium states, i.e., states (pixel value combinations) of largest probability, the entropy of the image should also be largest. This is the basis of well-known maximum-entropy estimators, which have proven powerful in machine learning and image processing~\cite{Gull:1978,Gull:1984,Hu:1991}. Conversely, minimizing the  entropy leads to a more structured system, which is also useful in image processing~\citep{Awate:2006}. However, it is clear that simply maximizing the entropy does not only enhance the signal, but also the noise, while simply minimizing the entropy does not only reduce the noise, but also removes image detail. Therefore, a prior is again needed to trade this off. The entropy distribution of natural-scene images is shown in Fig.~\ref{fig:entropyInt}, which suggests that entropy should be kept on a certain level. GDPs achieve exactly this, as shown in Section~\ref{sec:convex}.

\begin{figure}[h]
  \centering
{\includegraphics[width=0.7\linewidth]{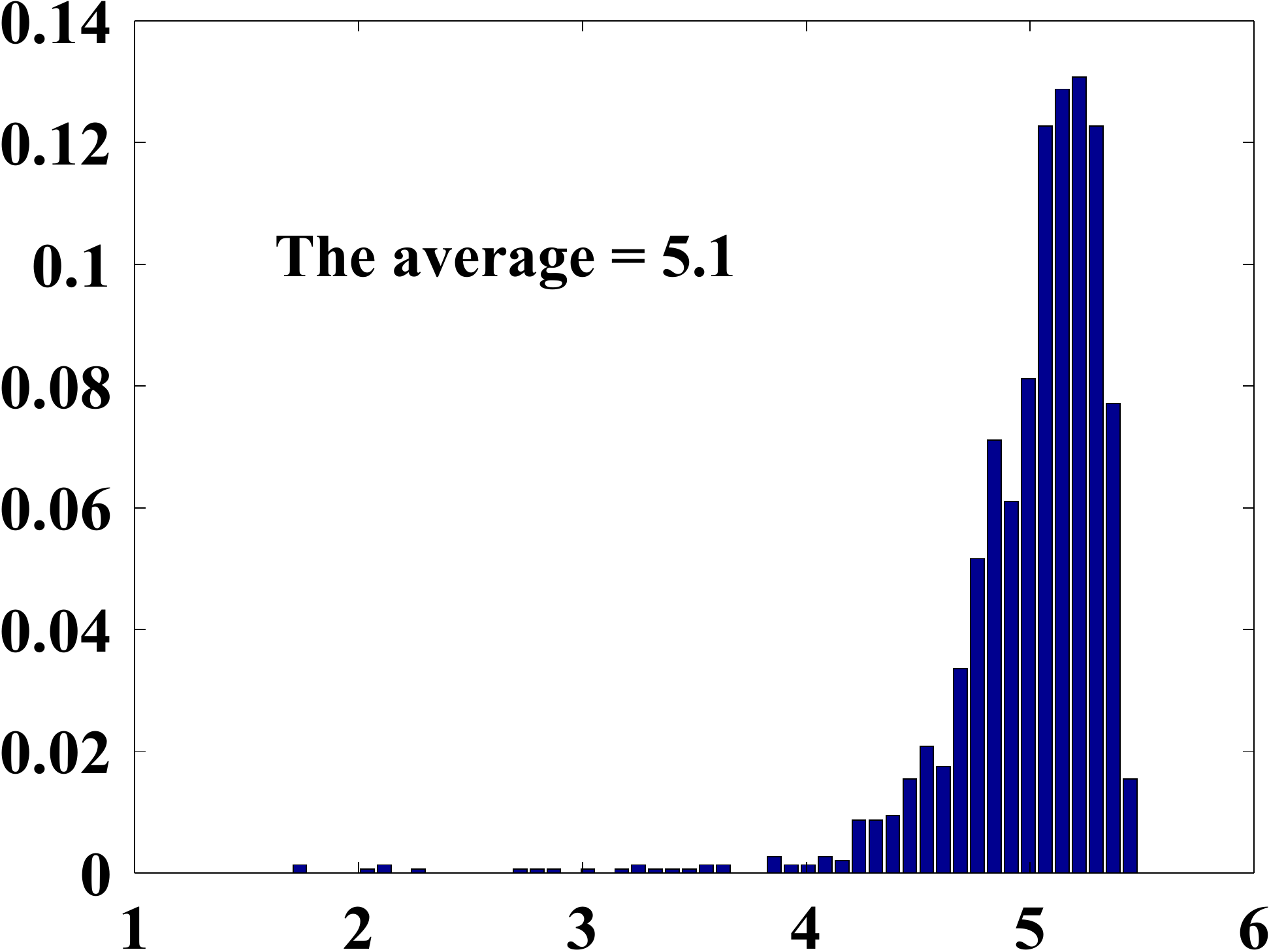}}
  \caption{Entropy distribution of natural-scene images.}
  \label{fig:entropyInt} %% label for entire figure
\end{figure}

A simple comparison of point-wise MRFs, Fields-Of-Experts, and GDP is shown in Fig.~\ref{fig:limit}. The figure shows images that were directly sampled from the different prior models. It is worth noticing the unique texture and structure of each model. The main difference comes from the gradient distribution being a macroscopic image description, whereas the other two are microscopic models. 

\begin{figure}[h]
  \centering
  \subfigure[Images generated by a point-wise Markov Random Field.]{
  {\includegraphics[width=0.22\linewidth,height=0.22\linewidth]{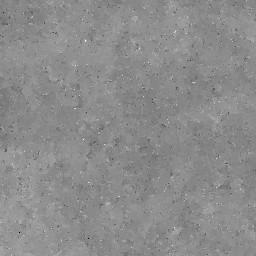}}
    {\includegraphics[width=0.22\linewidth,height=0.22\linewidth]{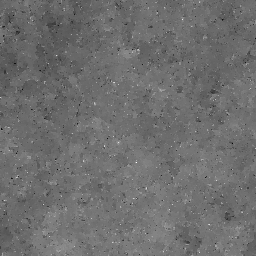}}
    {\includegraphics[width=0.22\linewidth,height=0.22\linewidth]{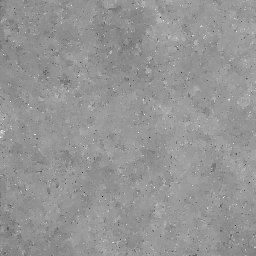}}
    {\includegraphics[width=0.22\linewidth,height=0.22\linewidth]{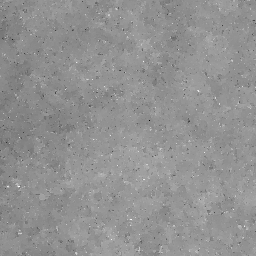}}
    {\includegraphics[width=0.07\linewidth,height=0.22\linewidth]{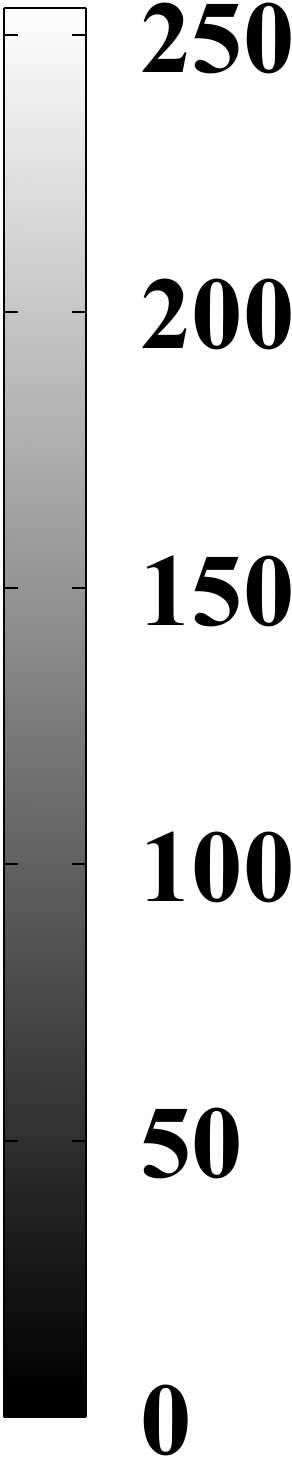}}
  \label{fig:pwmrf} 
  }
  
  \subfigure[Images generated by a gradient distribution prior.]{
  {\includegraphics[width=0.22\linewidth,height=0.22\linewidth]{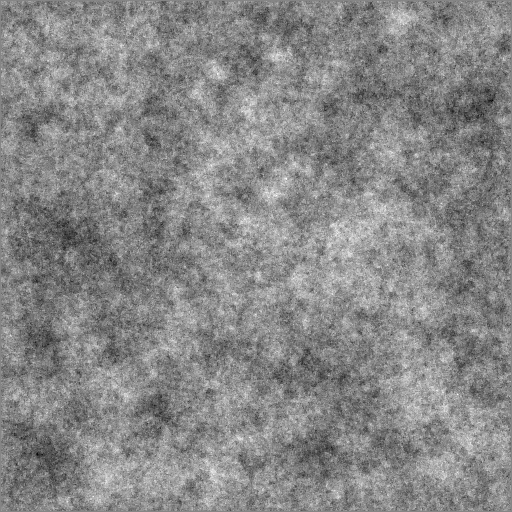}}
    {\includegraphics[width=0.22\linewidth,height=0.22\linewidth]{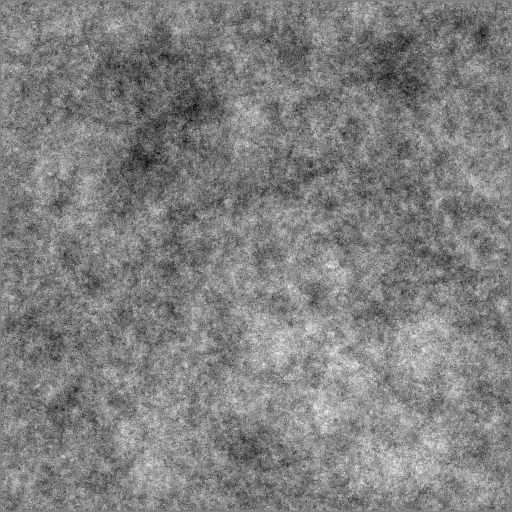}}
    {\includegraphics[width=0.22\linewidth,height=0.22\linewidth]{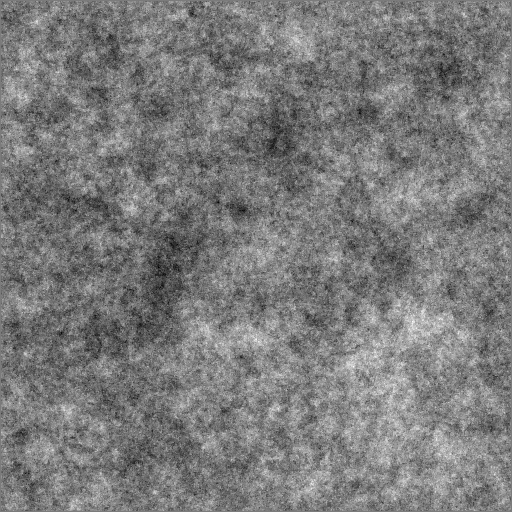}}
    {\includegraphics[width=0.22\linewidth,height=0.22\linewidth]{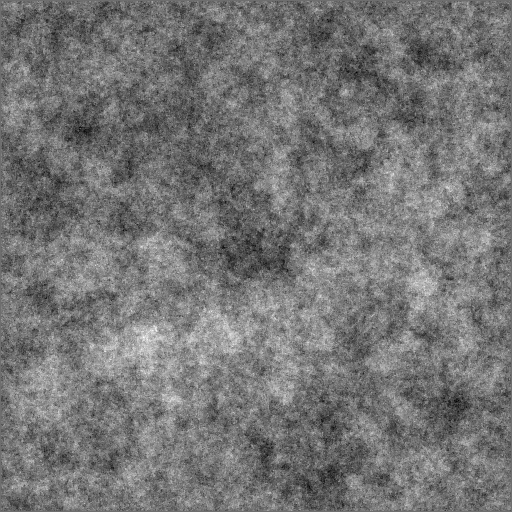}}
    {\includegraphics[width=0.07\linewidth,height=0.22\linewidth]{graymap.pdf}}
    \label{fig:Limit2} 
  }
  
  \subfigure[Images generated by a 3x3 Fields-Of-Experts.]{
  {\includegraphics[width=0.22\linewidth,height=0.22\linewidth]{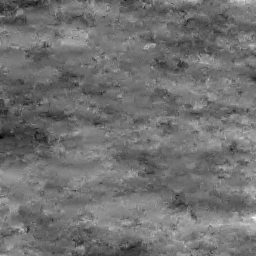}}
  {\includegraphics[width=0.22\linewidth,height=0.22\linewidth]{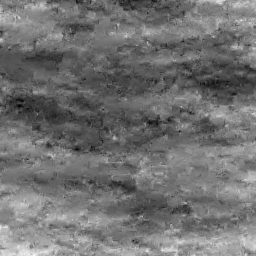}}
  {\includegraphics[width=0.22\linewidth,height=0.22\linewidth]{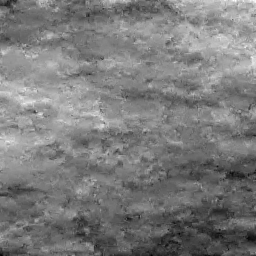}}
  {\includegraphics[width=0.22\linewidth,height=0.22\linewidth]{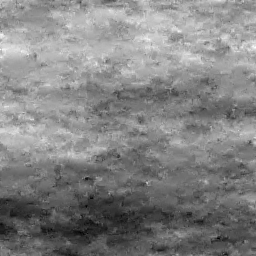}}
  {\includegraphics[width=0.07\linewidth,height=0.22\linewidth]{graymap.pdf}}
  \label{fig:FOE3} %% label for entire figure
  }
\caption{Comparing images sampled from different prior models.}
\label{fig:limit}
\end{figure}

\subsection{A variational calculus argument}
\label{sec:link}
MAP estimation using Bayes' rule is equivalent minimizing the negative log-likelihood:
\begin{equation}
\label{eq:bayes}
\begin{split}
\hat{U} = &\arg \max_{U'} \ \{\ p(U'|S)\propto p(S|U')p(U')\ \}\\
\Leftrightarrow &\arg \min_{U'} \ \{\ -\log(p(S|U')) -\lambda \log(p(U'))\ \},
\end{split}
\end{equation}
where the scalar regularization parameter $\lambda$ is introduced to balance the likelihood $p(S|U')$ and the prior $p(U')$. Assuming a Gaussian distribution for the likelihood naturally leads to an $\ell_2$ norm in the data fitting term and to an optimal $\lambda\propto\sigma^2$, where $\sigma$ is the standard deviation of the Gaussian.

When using GDP, the prior $p(U')$ is not over the hypothetical image $U'$, but over its gradient $\nabla U'$. A frequent assumption for this term is a Generalized Gaussian distribution, hence $p(\nabla U')=\exp^{-\alpha\|\nabla U'\|_*}$, where~\mbox{$\|\cdot\|_*$} is any proper norm. In the negative logarithm, this then leads to the TV regularization $-\log(p(\nabla U'))=\alpha\|\nabla U'\|_1$ for the $\ell_1$ norm, corresponding to a Laplace distribution model. The TV regularizer, or the Laplacian model, can hence be interpreted as a linear approximation in log-space to the GDP. This is also visually shown in Fig.~\ref{fig:map}(h,k). A hyper-Laplacian prior can be imposed in the same way for the $\ell_q$ norm ($0<q<1$)~\citep{krishnan2009fast}.

While MAP uses the posterior as an objective function, other choices are possible. When using the Minimum Mean-Squared Error (MMSE) 
\begin{equation}
\label{eq:MMSE}
\hat{U} = \arg \min_{U'} \int_U' p(U'|S)U'\mathrm{d}U'\,.
\end{equation}
as a cost function, for example, the prior is imposed analogously. This has been successfully used in conjunction with a TV prior to reduce staircasing artifacts~\citep{TVLSE}. As shown here in Section~\ref{sec:denoise}, our GDP model achieves similar results without making the objective function more complex.

\subsection{A functional analysis argument}
\begin{figure}[h]
\centering
\subfigure{\includegraphics[width=.45\linewidth]{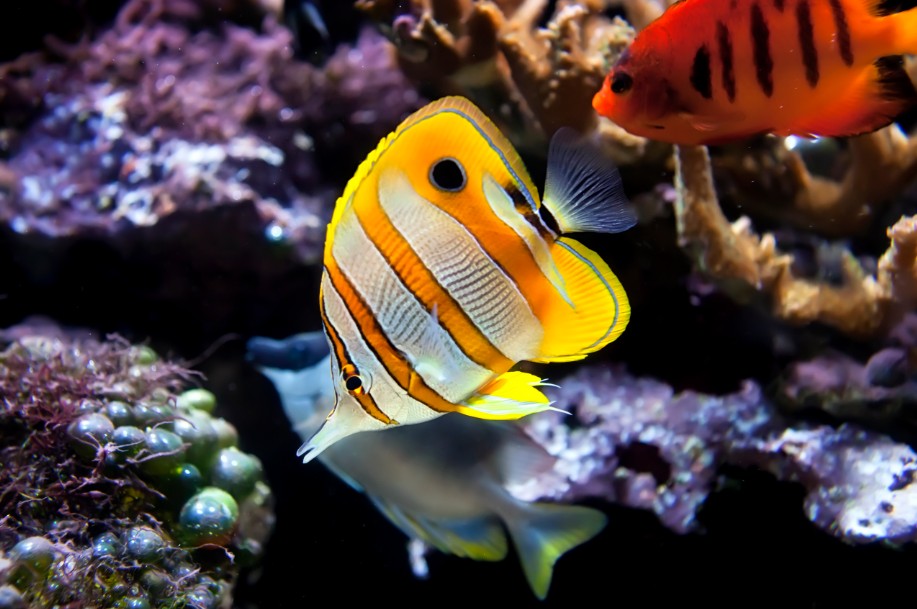}}
\subfigure{\includegraphics[width=.45\linewidth]{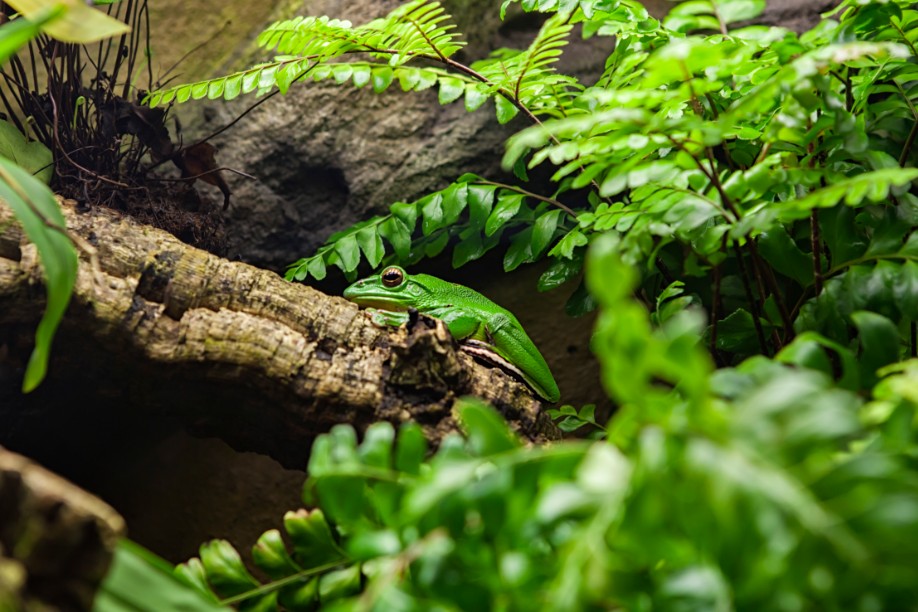}}
\subfigure{\includegraphics[width=.48\linewidth]{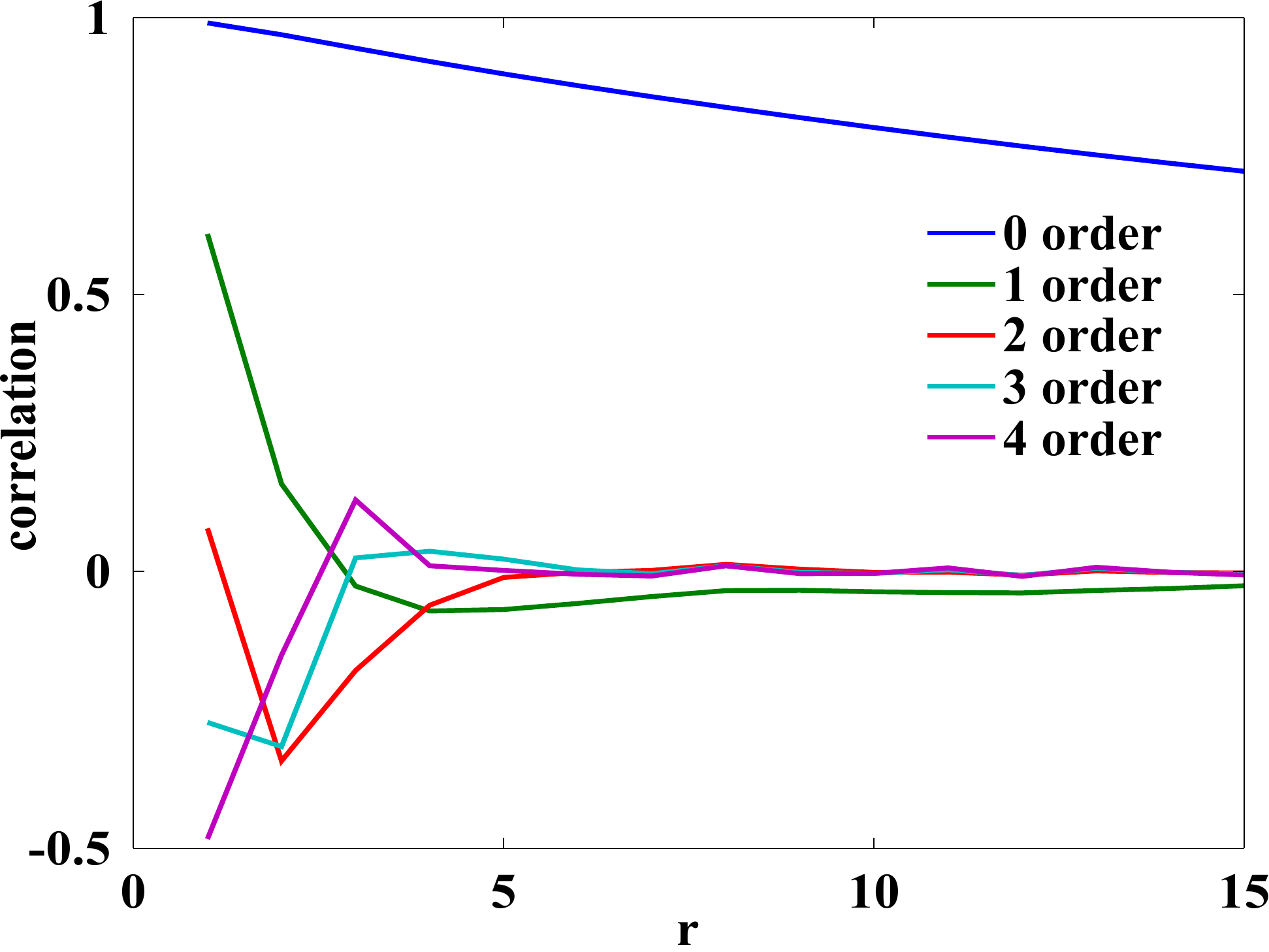}}
\subfigure{\includegraphics[width=.48\linewidth]{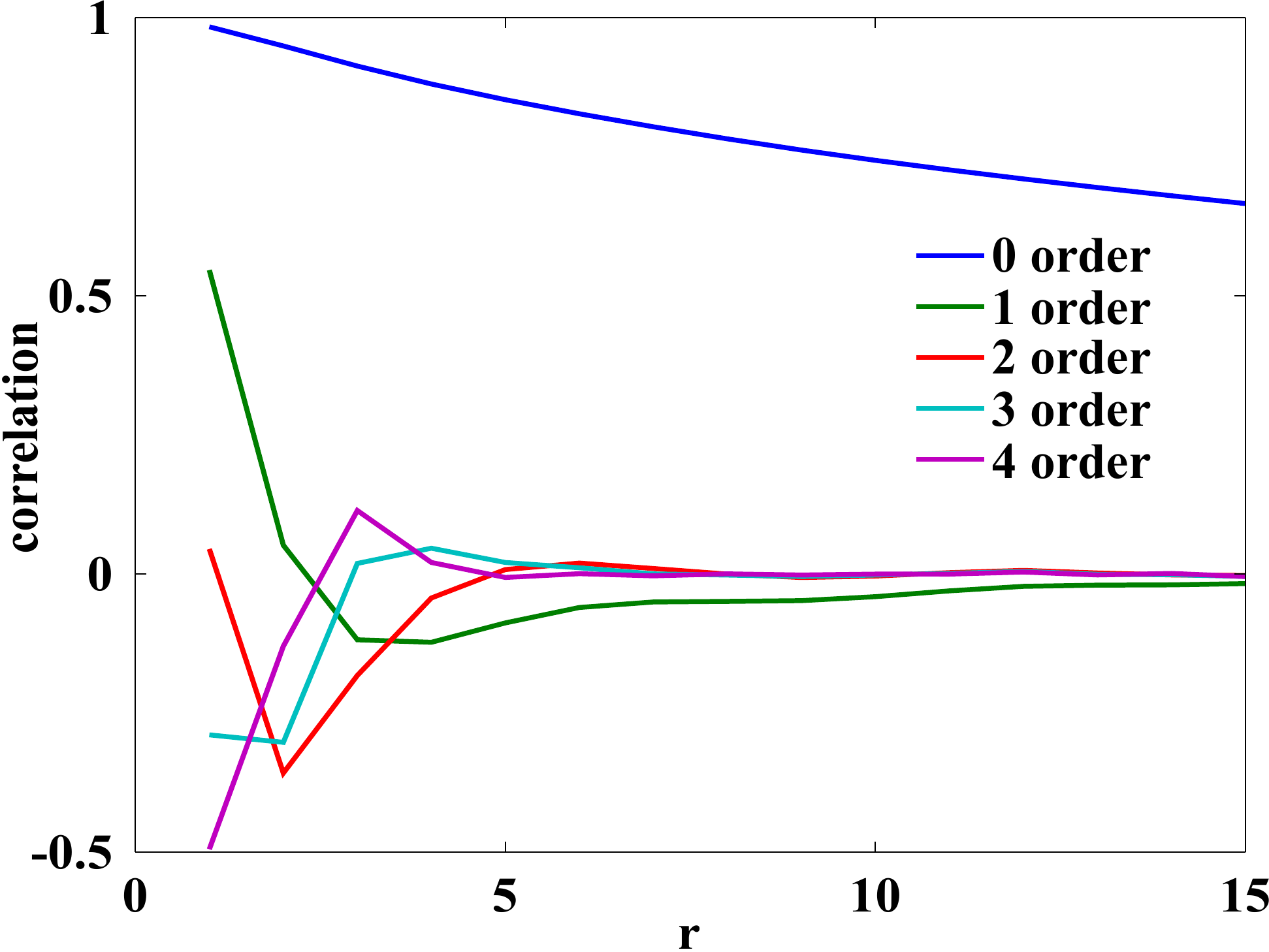}}
\caption{Two example images and their auto-correlation $AC(d,r)$.}
\label{fig:example}
\end{figure}

According to Taylor's theorem, the gradient is the first-order approximation to the perfect image. Higher-order approximations, however, do not necessarily improve the accuracy in image processing. The first reason for this is that the discrete image might not be high-order differentiable. 
A second reason is that the image auto-correlation decreases rapidly with increasing order of derivative. In Fig.~\ref{fig:example}, the auto-correlation between $I(\vec{x})$ and $I(\vec{x}+\vec{r})$ is shown for different orders $d$ of derivatives:
\begin{equation}
\label{COR}
AC(d,r)=\text{correlation}(\nabla^{d}I(\vec{x}),\nabla^{d}I(\vec{x}+\vec{r}))\, ,
\end{equation}
where $\vec{r}=(r,0)$. The correlation reduces significantly for $d>0$. This explains why gradient and curvature priors are so powerful for image processing, but higher-oder derivatives don't improve the result anymore. As seen in Fig.~\ref{fig:example}, the correlation reduces to zero more rapidly for higher $d$. For image-processing tasks, second order has repeatedly been shown to be enough~\citep{Zhu2007,gong2013a}. 

An alternative to using image derivatives could be to use Sobolev norms, which implicitly contain gradient information~\citep{sobolev2007}. This would mean that $F_s$ is a Sobolev space instead of $L^p(\Omega)$. When using Sobolev norms, however, there is no way to trade off the weight of the gradient information versus the data. This is why we prefer using GDPs instead of Sobolev norms, because the relative weight is an important parameters allowing us to control the noise level.

\subsection{A psychophysical argument}
The human vision system mainly detects gradient information~\citep{chichilnisky:2001, pillow:2005, cao:2011, gollisch:2010}. As shown in Fig.~\ref{fig:retina}, light entering the retina first arrives at the retinal ganglion cells. These cells are sensitive to gradient information, rather than to intensity, with a response that is described well by the error function~\citep{chichilnisky:2001,miller:2002}. Neighboring cells also interact with each other to amplify the gradient information, which explains the famous Mach band effect.  

During evolution, the neurons have adapted to the environment they were exposed to, and to process what is expected~\citep{chichilnisky:2001,simoncelli:2001,miller:2002}. This is the gradient distribution found in natural-scene images. The human vision systems is hence particularly well adapted to detect and process images that satisfy this distribution.

\begin{figure}[h]
\centering
\includegraphics[width=0.9\linewidth]{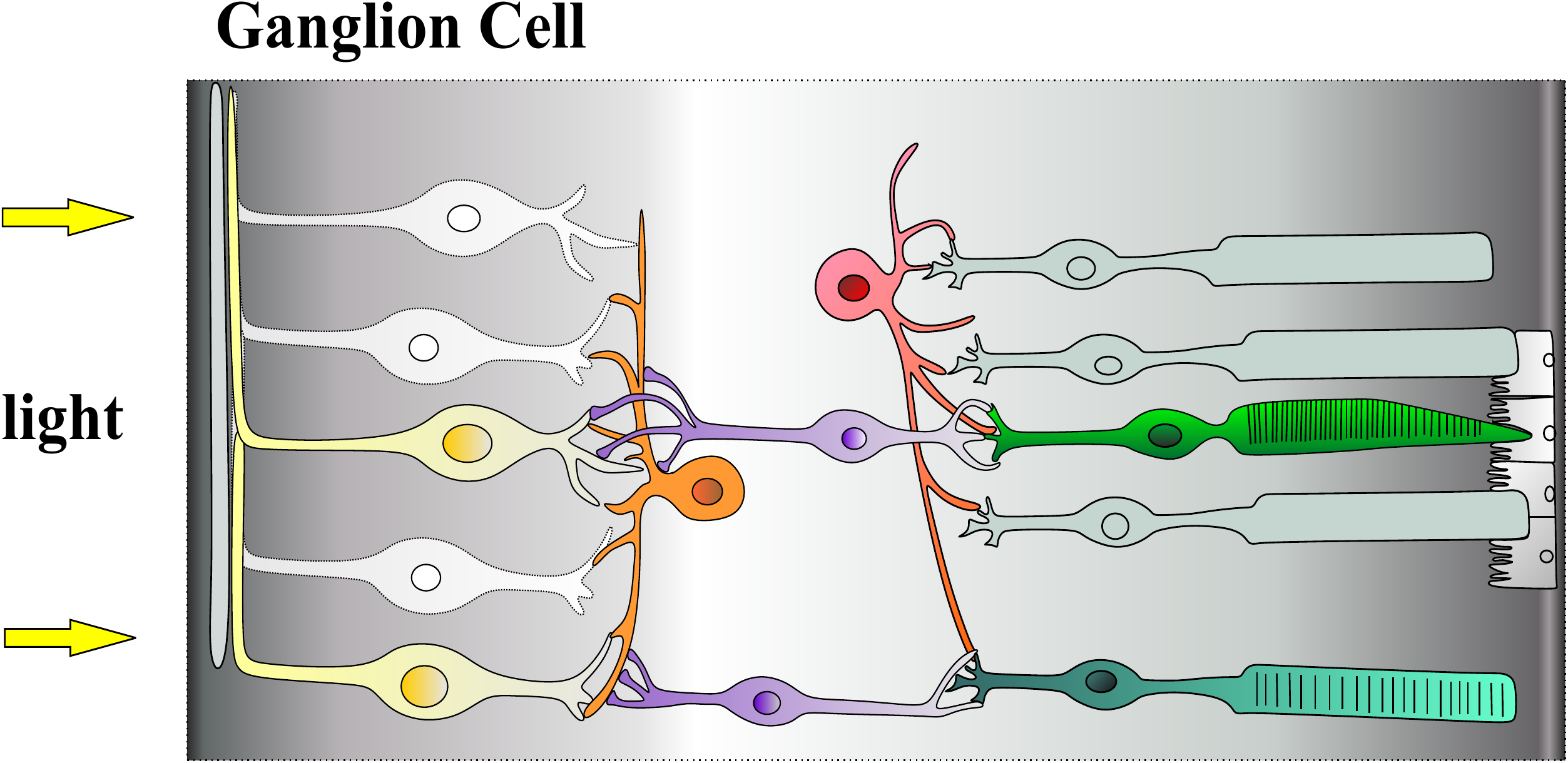}
\caption{Structure of the human retina. Light first hits the gradient detectors.}
\label{fig:retina}
\end{figure}

The fact that different people have almost the same visual perception is a consequence of the stability of the neuronal response to the gradient distribution. The vision system also suggests that coding an image by its gradient is an efficient way (sparse representation). We show coding efficiency and sparsity for the GDP in Section~\ref{sec:convex}. 

\subsection{Gradient field and original image}
Regardless of the stability of the gradient distribution, reconstructing an image from its gradient field is accurate and simple, as it constitutes an integration task with one point constraint~\citep{fattal2002gradient,perez2003poisson,agrawal2007gradient,gradientshop,kazhdan2006poisson,Xu2011,gong2012,nonlocal3D}. An excellent review about signal processing in the gradient domain can be found in Ref.~\citep{agrawal2007gradient}. Some recent advances in this area are described in Refs.~\citep{Xu2011,Xu2012,FATTAL09,Farbman2008,McCann2008}.  

Reconstructing an image from its gradient field can be done by solving a Poisson equation. With proper boundary conditions, the solution is unique, and there exists a wealth of stable, efficient, and accurate numerical solvers for this equation (details in Section.~\ref{sec:reconstruction}).  

\section{Why not directly learn GDP from biomedical images?}\label{sec:whynot}
%Remarkably, the gradient distribution is stable for large types of images~\citep{zhu:1997,NaturePrior}. One reason is that after millions years' revolution, human visual system is evolved to capture the gradient information instead of intensity in natural scene. This fact has already been confirmed by Mach band effect. The stability also comes from the physically stability of human visual system(different people have almost the same visual perception for the same scene). From another point of view, natural scene images have very high signal noise ratio(SNR), which also makes statistics of their gradient histogram stable(not affected much by noise). The stability is also confirmed in our experiments (Section~\ref{sec:GDP}). Thanks to the stability, gradient histogram has been adopted as a desired prior for many natural image processing tasks.  
Why are we proposing to learn the GDP from natural-scene images and then apply it to biomedical images (see Fig.~\ref{fig:flow})? Would it not be better to directly learn the GDP on biomedical images, for which its use is intended? 
The reason is two-fold: (1) biomedical images contain a variety of disturbances, such as noise, blur, and scattering. If the GDP is later to be used for denoising, deblurring, or dehazing, it must be estimated from images that do not already contain these disturbances. (2) Even if the disturbances should be part of the prior, it is not easy to learn a GDP directly from biomedical images. The main reason is that biomedical images are usually quite noisy, hampering gradient estimation (derivatives amplify noise). Therefore, it is important that the GDP is estimated from ``clean'' images. 
We here propose to learn the GDP from natural-scene images. The motivations have been given above. In the following, we fully validate the use of this prior for biomedical images. The few examples in Fig.~\ref{fig:showcase} are to illustrate that natural-scene and biomedical images have at least qualitatively similar gradient distributions.

\begin{figure}[h]
  \centering
  {\includegraphics[width=0.24\linewidth,height=0.24\linewidth]{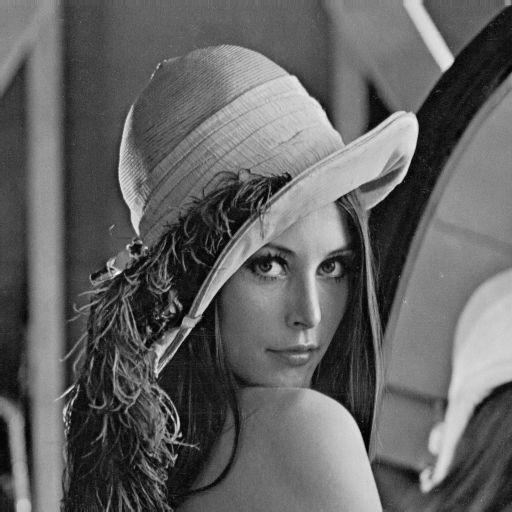}}
  {\includegraphics[width=0.24\linewidth,height=0.24\linewidth]{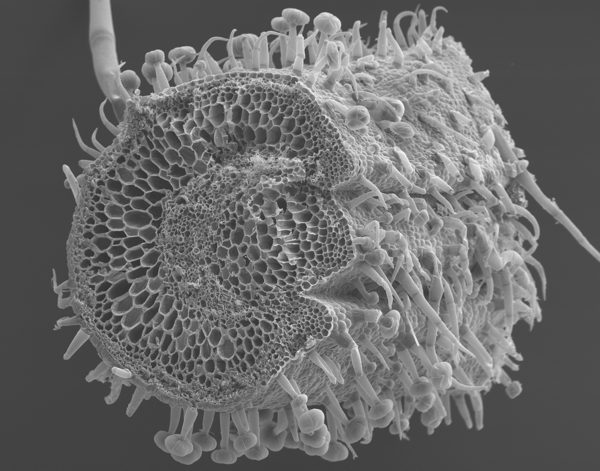}}
  {\includegraphics[width=0.24\linewidth,height=0.24\linewidth]{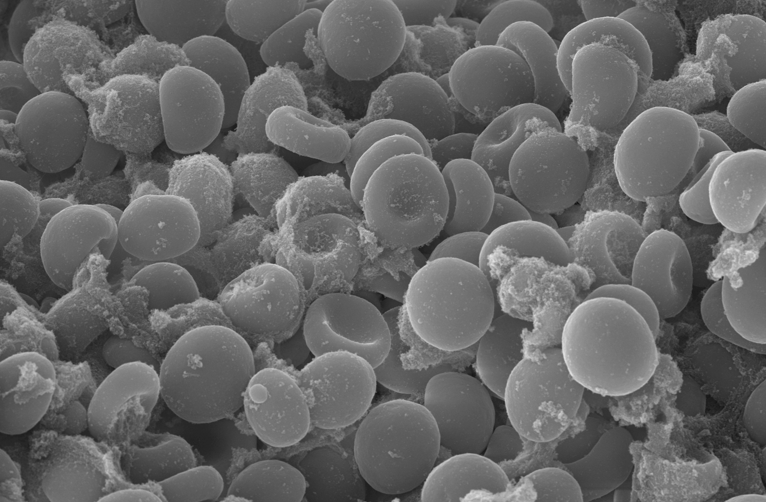}}
  {\includegraphics[width=0.24\linewidth,height=0.24\linewidth]{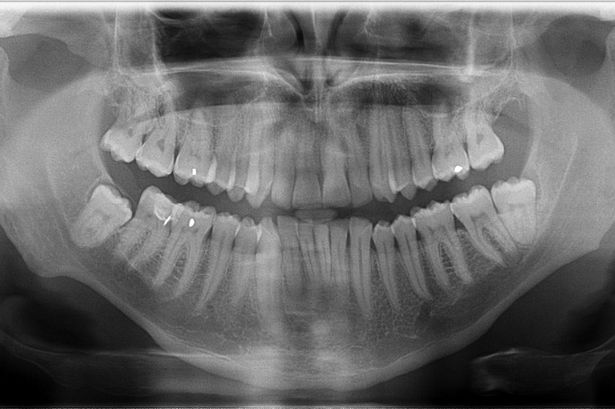}}
  {\includegraphics[width=0.96\linewidth,height=0.24\linewidth]{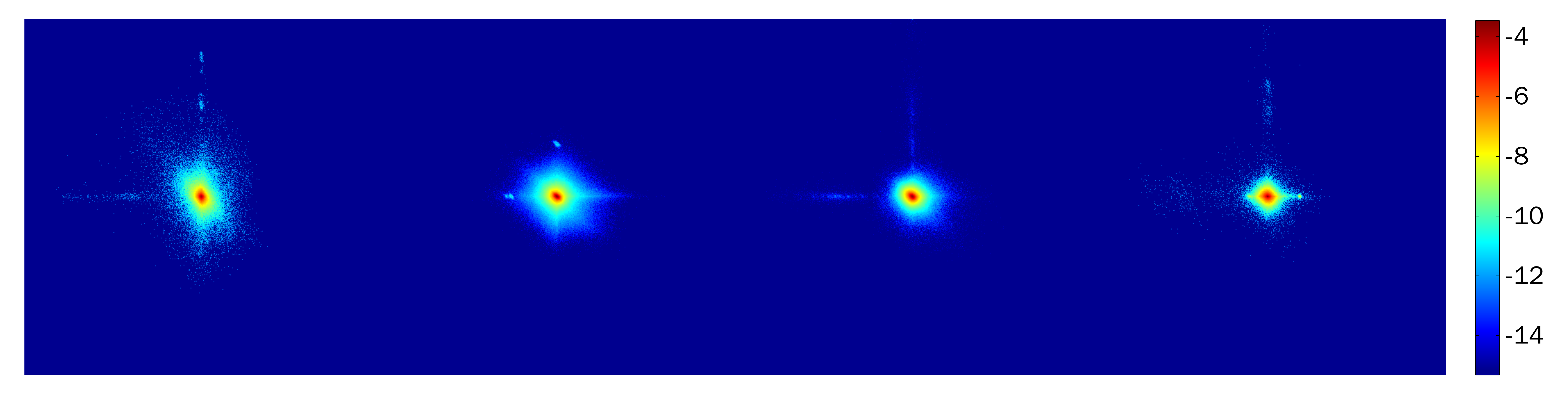}}
  \caption{Natural-scene and biomedical images have qualitatively similar gradient distributions and can hence be fitted with the same model. First row: natural-scene image (left) and biomedical images. Second row: corresponding gradient distributions in log scale.}
  \label{fig:showcase} %% label for entire figure
\end{figure} 

\section{The Gradient Distribution Prior (GDP)}
\label{sec:GDP}
In order to build the GDP, we learn the gradient distribution from a database of 23\,613 images of natural scenes. We analyze the resulting distributions and compute the variability of images around the mean distribution. We then validate the stability of the GDP on biomedical images.

\subsection{Datasets}
We collected seven datasets of natural-scene images as shown in Table~\ref{table:data}. Each image $I(x,y)$ was converted to 8-bit gray-scale. The gradient field is defined as:
\begin{equation}
\label{eq:gradient}
\vec{G}(x,y)=\nabla I(x,y) ,
\end{equation}
where here we use the first-order finite difference approximations $\nabla I\approx ( I(x+1,y)-I(x,y),I(x,y+1)-I(x,y))$. We use homogeneous Dirichlet boundary conditions at the image borders. Due to the use of 8-bit gray-scale images, possible gradients are in the discrete domain $[-255,255]\times [-255,255]$, where we can easily construct the 2D histogram of $\vec{G}$. We use $G^x$ and $G^y$ to denote respective components of $\vec{G}$.

In order to turn the histogram into a probability distribution, we divide all bins by the total number of pixels in the image, i.e., by $mn$ where $m$ and $n$ are the number of pixels along the $x$ and $y$ edges of the image. After aggregating data from all images in the database, we further normalize by the total number of images in the dataset. The resulting empirical distribution $p^{\mathrm{pr}}$ is shown in Fig.~\ref{fig:ThePrior}. 

%\centering  % used for centering table
\begin{table}
\scriptsize
\centering
\begin{tabular}{cccccccc|c} 
\hline\hline
 Footnote  & \tablefootnote{{http://www.vision.ee.ethz.ch/showroom/zubud/}} &	 \tablefootnote{{http://see.xidian.edu.cn/faculty/wsdong/Data/Flickr\_Images.rar}}&	 \tablefootnote{{http://www.robots.ox.ac.uk/\textasciitilde vgg/data/oxbuildings/}}&	 \tablefootnote{{http://www.comp.leeds.ac.uk/scs6jwks/dataset/leedsbutterfly/}}&	 \tablefootnote{{http://lear.inrialpes.fr/\textasciitilde jegou/data.php}}&	 \tablefootnote{{http://www.vision.caltech.edu/visipedia/CUB-200.html}}&	 \tablefootnote{{http://www.robots.ox.ac.uk/\textasciitilde vgg/data/flowers/102/index.html}} & all\\
\hline
\#images &  1005&	 1000&	 5063&	 832&	 1491&	 6033&	 8189 &  23613\\
\hline
%content & build&	natural&	build&	butterfly&	natural&	birds&	flower & \\
\hline
\end{tabular}

\caption{Natural-scene image datasets used to learn the prior. Source URLs are given in the footnotes.} % title of Table
\label{table:data} % 
\end{table}

\begin{figure}[h]
\centering
\includegraphics[width=0.7\linewidth]{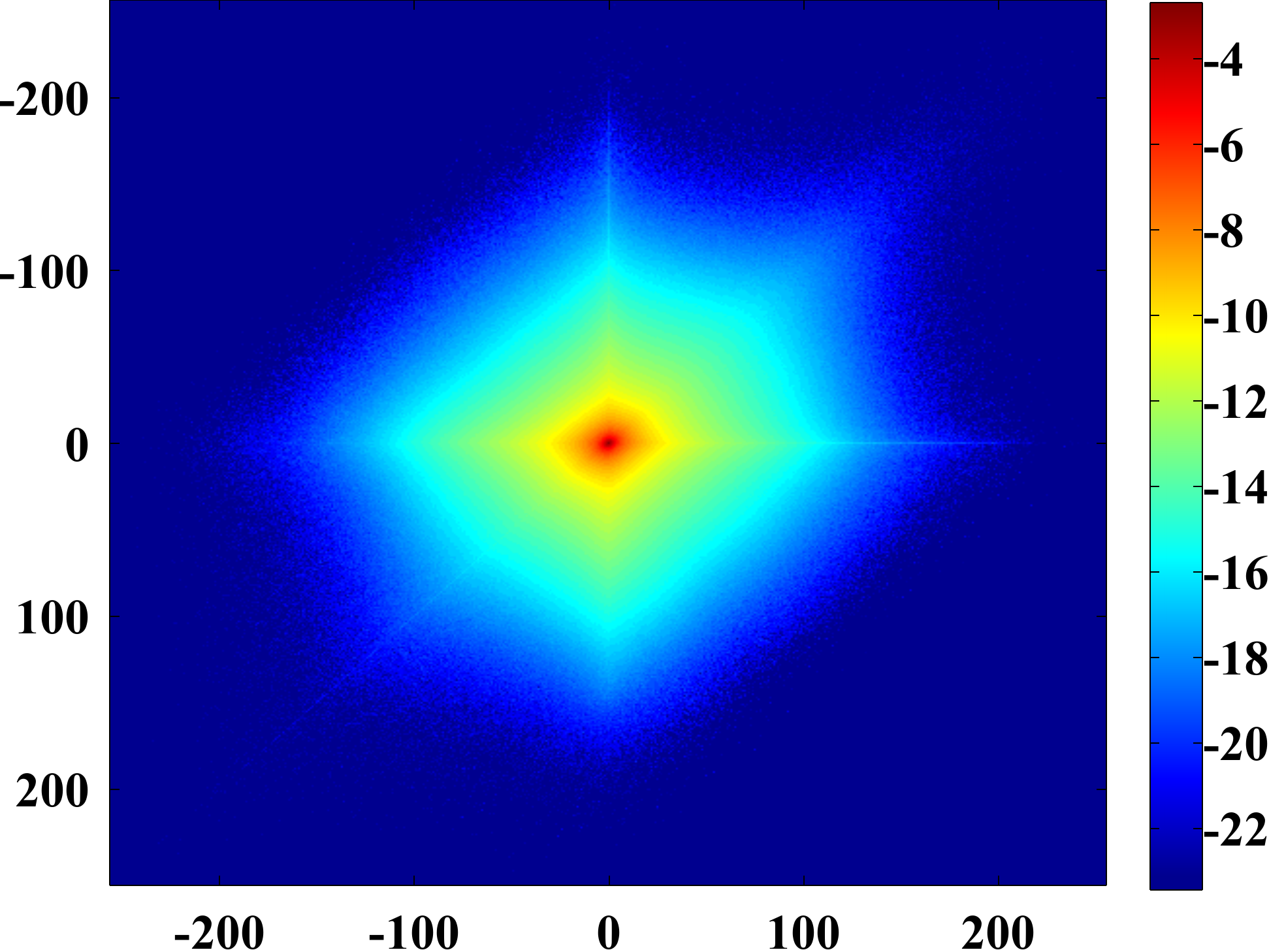}
\caption{Average gradient distribution of natural-scene images shown in log scale.}
 \label{fig:ThePrior}
\end{figure}

\subsection{Stability of the prior for natural-scene images}
We analyze how closely the natural-scene images in the training dataset match the average gradient distribution learned from them. Fig.~\ref{fig:Error} shows the histograms of several distances between the average prior and each image's individual gradient distribution. More than $95\%$ of the images have distributions that are closer to the prior than an RMS of $2\times10^{-4}$. Also when using other distance metrics, such as the Hellinger or KL distance, the training data are clustered around the prior (Fig.~\ref{fig:Error}b,c). 

\begin{figure}[h]
\centering
\subfigure[RMS distance]{\includegraphics[width=0.48\linewidth]{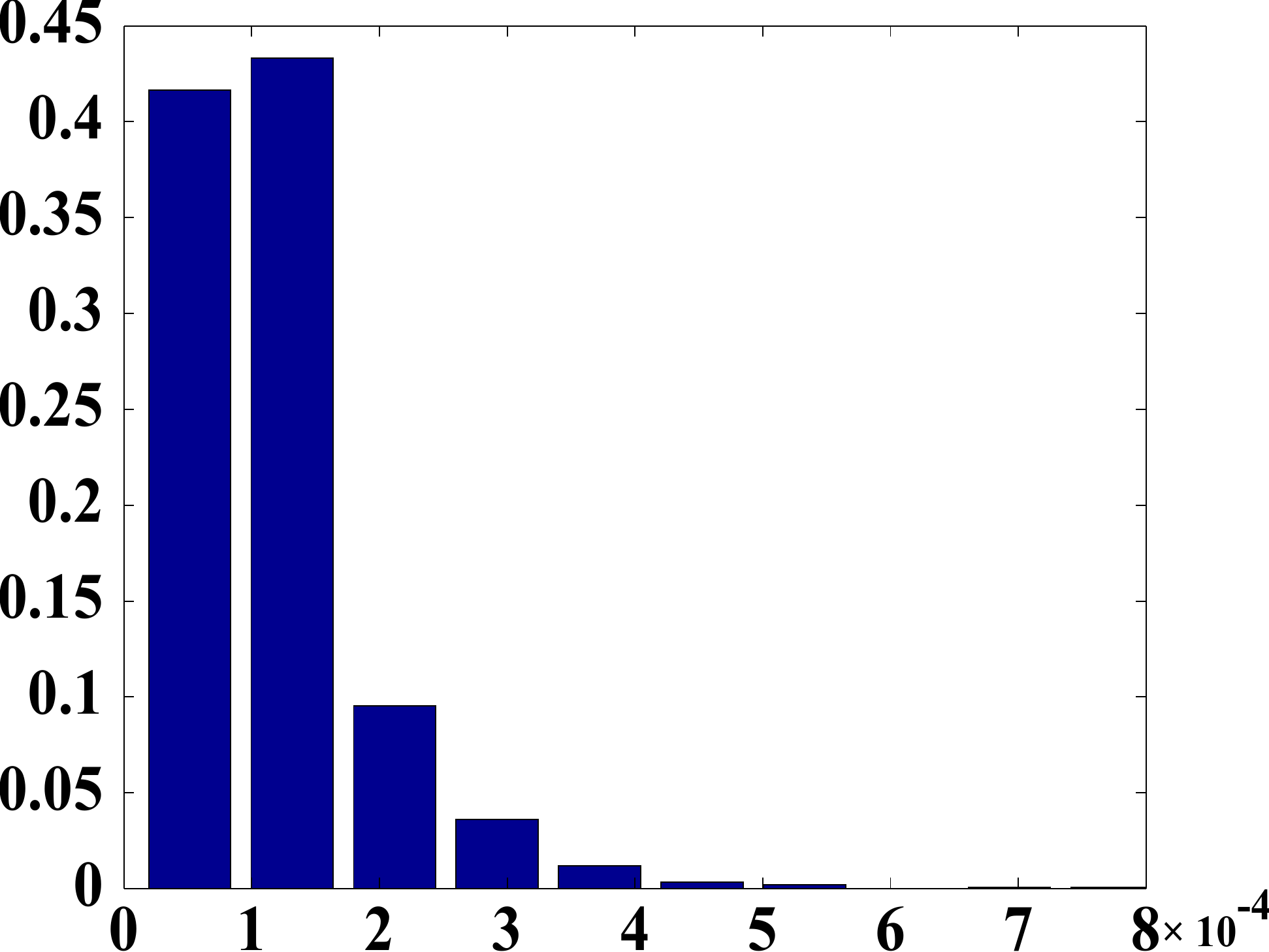}}
\subfigure[Hellinger distance]{\includegraphics[width=0.48\linewidth]{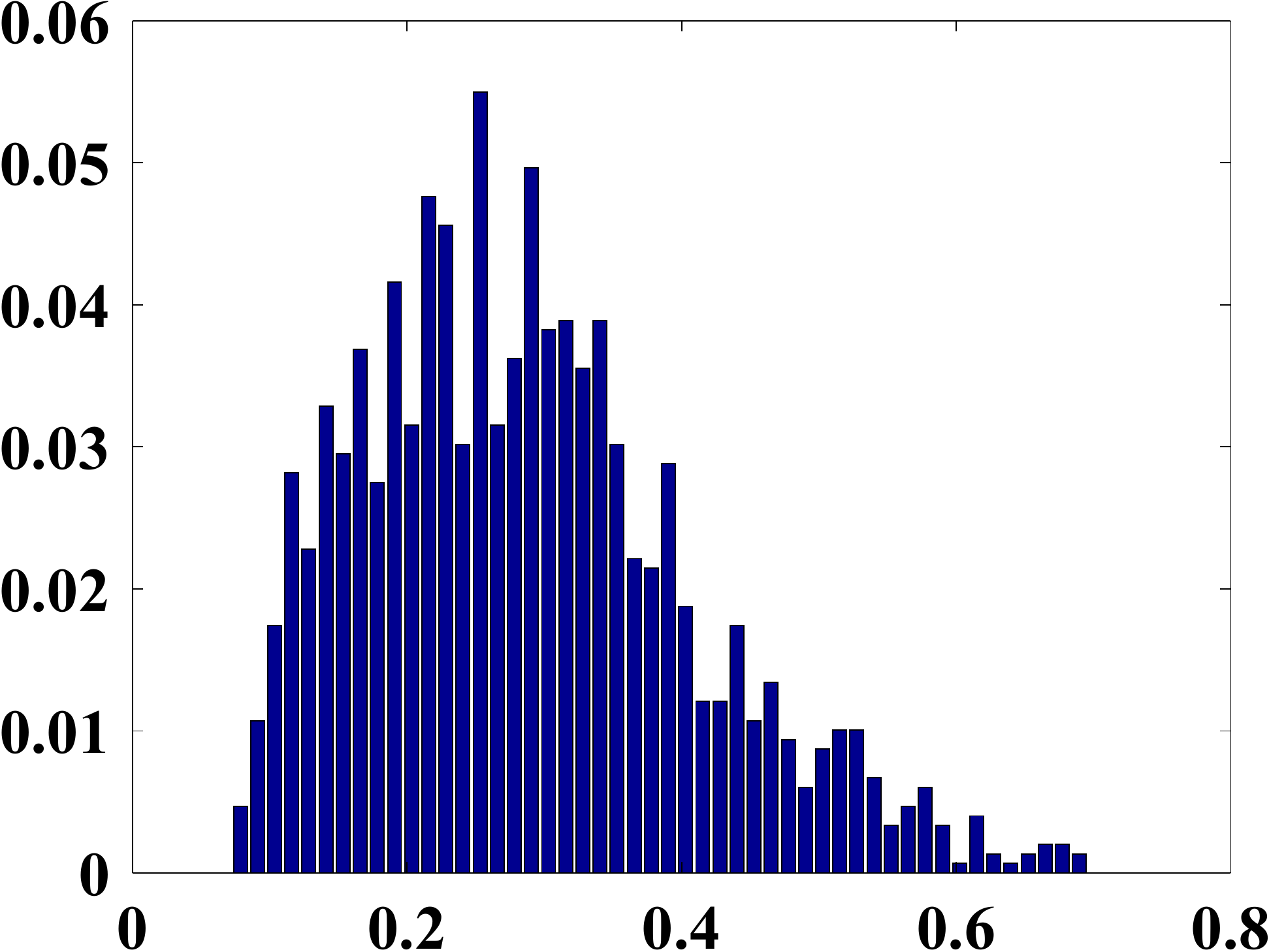}}
\subfigure[KL distance]{\includegraphics[width=0.48\linewidth]{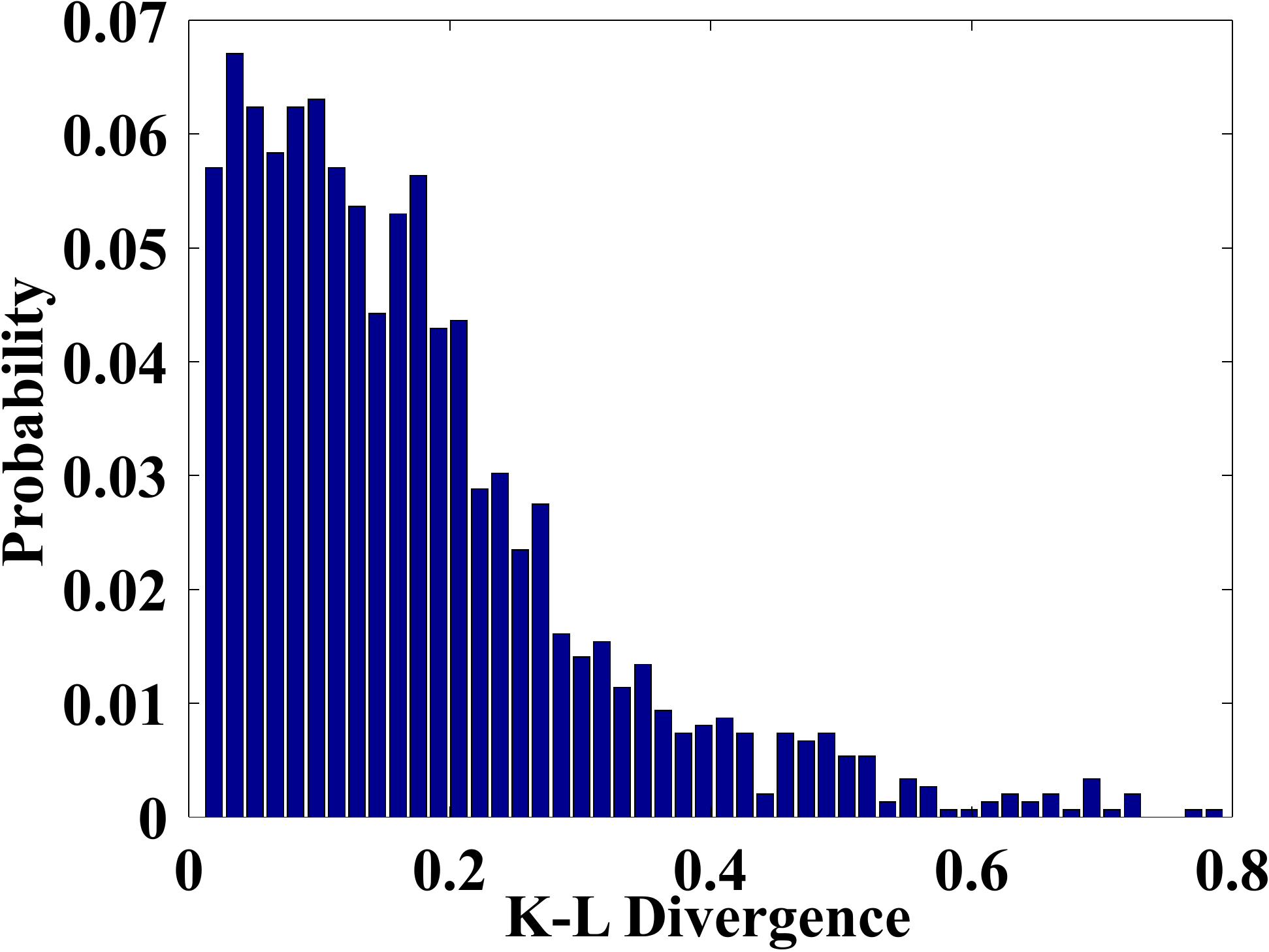}}
\caption{The prior is stable across natural-scene images. The distances (different metrics shown) between individual images and the average prior are mostly small, indicating that all images are clustered around the prior.}
 \label{fig:Error}
\end{figure}

\subsection{Scale stability of the prior}
In order for the GDP to be stable across images, it in particular has to be stable with respect to image scaling. We confirm this by down-sampling (down-scaling) all images in the training database by factors up to 0.5, and re-learning the average GDP from each scaled dataset. The scaling is applied to the whole image, rather than cropping a sub-image. The resulting average gradient distributions are shown in Fig~\ref{fig:Resize}. They are stable with scaling down to a scale factor of 0.5. Below 0.5, the average gradient distribution starts changing.

\begin{figure}[h]
\centering
\includegraphics[width=\linewidth]{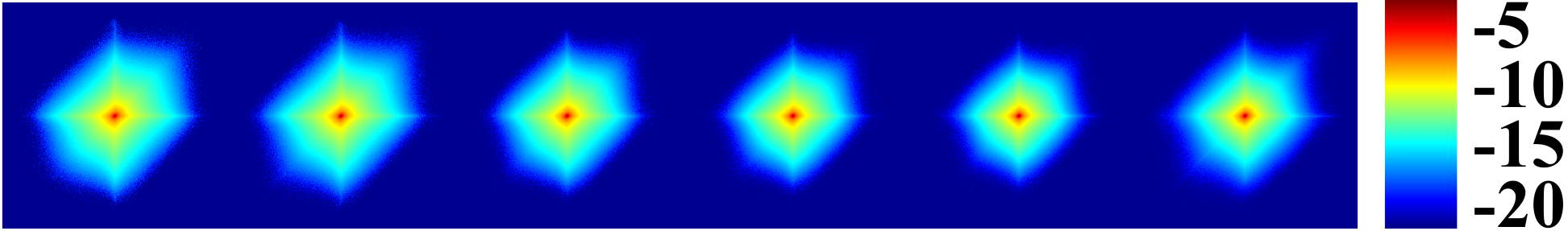}
\caption{The prior is stable with respect to image scaling. Shown is the average (across the entire training dataset) gradient distribution in log scale for images scaled by factors of (from left to right) 0.5, 0.6, 0.7, 0.8, 0.9, and 1 (i.e., the original, unscaled images).}
 \label{fig:Resize}
\end{figure}

\subsection{Stability of the prior on biomedical images}
\label{sec:test}
As mentioned above, it is hard (any potentially undesirable) to directly build a GDP from biomedical images. We hence learned the prior from natural-scene images, but validate it here on biomedical images. We first show stability of the present GDP for biomedical images. For this, we collected a small dataset of biomedical images, including X-ray, MRI, electron microscopy, and fluorescence microscopy images. Some examples are shown in Figs.~\ref{fig:Bio} and \ref{fig:naturalized}.

\begin{figure}[h]
\centering
\includegraphics[width=0.8\linewidth]{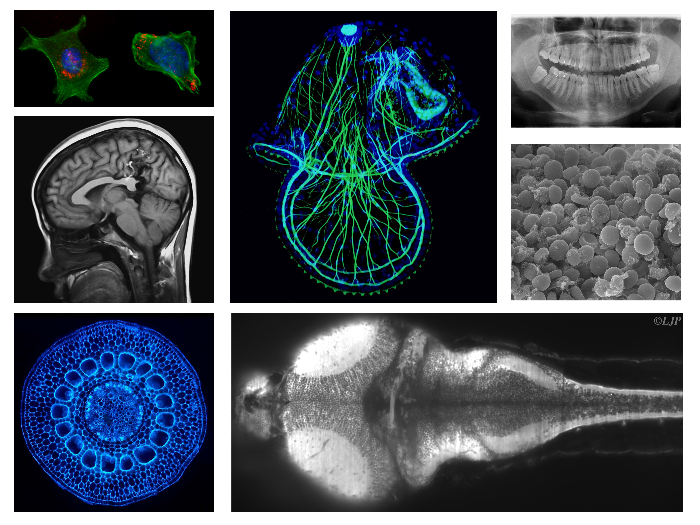}
\caption{Samples from our biomedical image collection.}
 \label{fig:Bio}
\end{figure} 

\begin{figure}[h]
\centering
\includegraphics[width=0.48\linewidth]{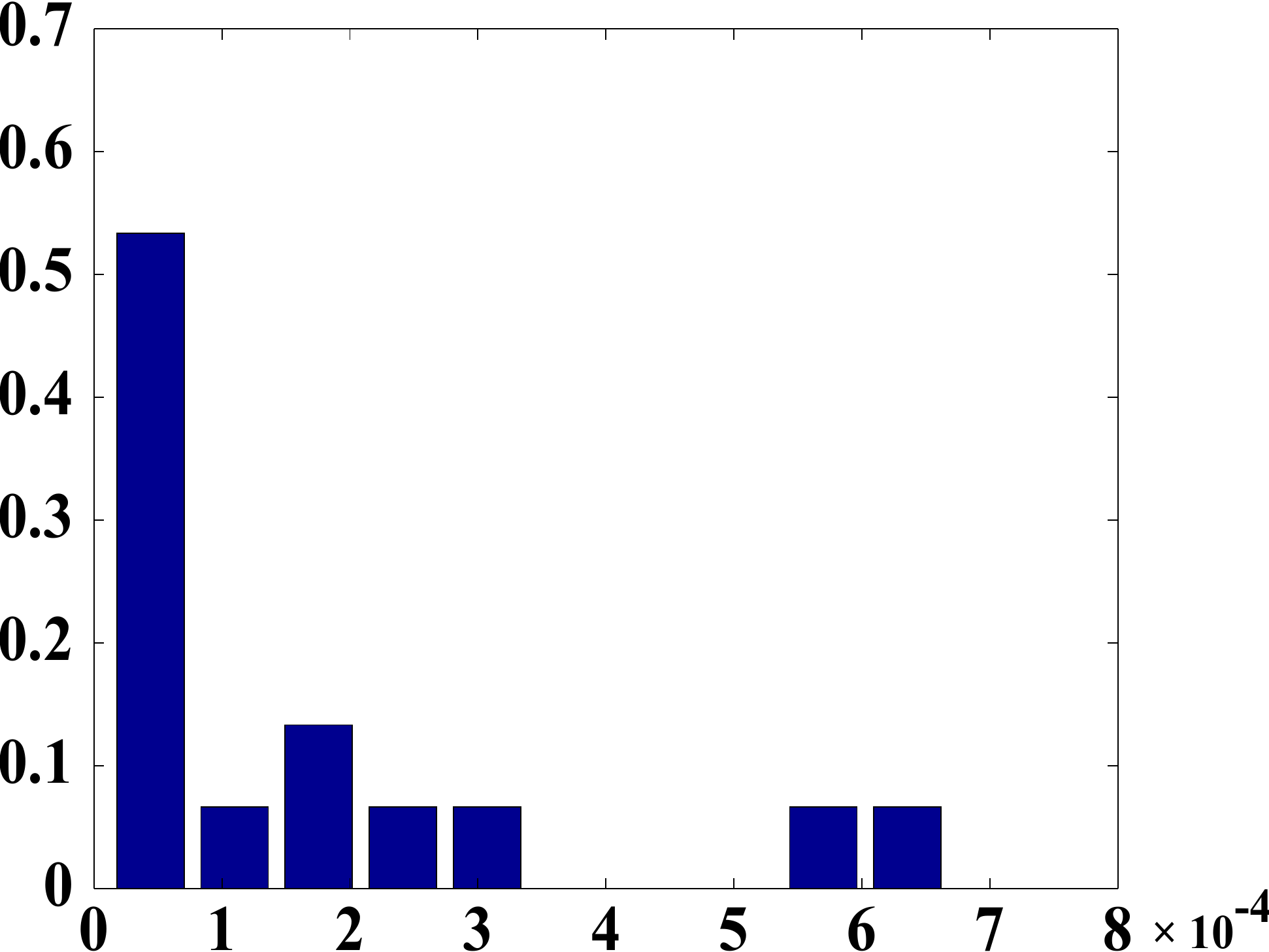}
\caption{RMS distance from the prior for biomedical images.}
 \label{fig:ValidError}
\end{figure}

We compute the RMS distance from the GDP for each image's gradient distribution. The distance histogram is shown in Fig.~\ref{fig:ValidError}. This confirms that most gradient distributions are close to the GDP learned from natural-scene images. As expected, the range of RMS for this test set of biomedical images is larger than the range of RMS for the training set of natural-scene images.

\section{Parametric Models for Gradient Distribution Priors}
\label{sec:GDPModel}
In order to efficiently use the GDP as a regularization term, and to formulate optimization schemed over it, a parametric model is desirable. 
We here provide two parametric models for the marginal and joint 2D distributions of our GDP, and we assess their approximation accuracy. We compare these new models with traditional gradient distribution models, such as Hyper Laplacian models, and with TV in 1D and 2D. We further propose a new model to approximate the cumulative distribution function (CDF) of the gradient instead of the PDF. This new model only has a single scale parameter, leading to effectively 1D problems in parameter inference.  We further discuss the convexity, sparsity, and entropy of these models. 

\subsection{1D Marginal Model}
Traditionally, image gradient distributions are modeled as Generalized Gaussian (Hyper Laplacian) Distributions:
\begin{equation}
\label{eq:GGD}
\log(P(G^{x}))=-a_0|G^{x}|^{b_0}+c_0 \, , 
\end{equation} 
where $a_0$, $b_0$, and $c_0$ are the model parameters.
This model includes Gaussian ($b_0=2$), Laplacian ($b_0=1$), and hyper-Laplacian distributions ($b_0=0.6$)~\citep{krishnan2009fast} and bears a close relationship with $\ell_q$ norms ($0<q<1$).

We instead propose to use the following models for the 1D marginal gradient distribution:

\textbf{Model 1:}
\begin{equation}
\label{eq:distr}
\log(P(G^{x}))=2a_1\left(\exp\!{\left\{-\frac{|G^{x}|^{b_1}}{a_1}\right\} }-1\right)+c_1(G^{x})^2 \, , 
\end{equation} 
where $a_1$, $b_1$, and $c_1$ are the parameters. The results of fitting this model (i.e., its parameters) to the average gradient distributions of our image sets are shown in Table~\ref{table:1D}. To the best of our knowledge, this Model 1 is the best-fitting model known so far (Fig.~\ref{fig:modelcompare} and Table~\ref{table:SSE}).
 
\begin{table}[H]
\scriptsize
\centering  % used for centering table
\begin{tabular}{c|cccccccc} 
\hline\hline
\multicolumn{2}{c}{Image set}  & 1&	2&	3&	4&	5&	6&	7 \\
\hline
\multirow{5}{*}{$G^x$} & $a_1$ & 3.66&	3.89&	3.93&	6.51&	7.83&	5.85&	6.34 \\
& $b_1$ & 0.58&	0.55&	0.58&	0.44&	0.44&	0.45&	0.50 \\
& $c_1\times 10^4$ & -2.4&	-1.2&	-1.5&	-0.48&	-1.9&	-0.57&	-1.3 \\
\cline{2-9}
& SSE & 40.5&	43.0&	67.7&	23.8&	34.1&	37.9&	25.4 \\
& $R^2$ & 0.99&	0.99&	0.99&	0.99&	0.99&	0.99&	0.99 \\
\hline
\multirow{5}{*}{$G^y$} & $a_1$ & 3.68&	3.87&	3.84&	7.29&	7.29&	5.55&	6.09 \\
& $b_1$ & 0.60&	0.55&	0.60&	0.42&	0.46&	0.47&	0.51 \\
& $c_1\times 10^4$ & -2.2&	-1.2&	-1.5&	-0.42&	-1.9&	-0.64&	-1.3 \\
\cline{2-9}
& SSE & 56.8&	41.4&	90.8&	19.9&	29.4&	32.6&	18.5 \\
& $R^2$ & 0.99&	0.99&	0.99&	0.99&	0.99&	0.99&	0.99 \\
\hline
\multicolumn{2}{c}{correlation}  & -0.12&	-0.23&	-0.19&	-0.22&	-0.12&	-0.25&	-0.11 \\
\hline
\multicolumn{2}{c}{(log scale)}  & 0.37&	0.31&	0.28&	0.37&	0.17&	0.39&	0.18 \\
\hline
\end{tabular}
\caption{Results for fitting the marginal with Model 1 to the average gradient distributions of all image datasets.} % title of Table
\label{table:1D} % is used to refer this table in the text
\end{table}

\textbf{Model 2:}
\begin{equation}
\label{eq:distr2}
\log(P(G^x))=-a_2(G^x)^2-\log(b_2 +|G^x|^2)+c_2,
\end{equation}
where $a_2$, $b_2$, and $c_2$ are the parameters. The results of fitting this model to the data are shown in Table~\ref{table:SSE}, compared with other models. Model 2 fits the data less well than the above Model 1, but has several advantages: 
\begin{itemize}
\item Integrability: Model 2 is integrable, which is convenient for use in optimization algorithms and for analytically computing the CDF. Let $T=\sqrt{a_2}$, $b_2=0$, and $c_2=0$, then the CDF of Model 2 is:
\begin{equation}
\label{eq:CDFModel2}
\widetilde{C}(G^x)=-\frac{e^{-(TG^x)^2}}{G^x}-T\sqrt{\pi}\,\text{erf}(TG^x)+H(G^x) ,
\end{equation} where $H$ is the Heaviside distribution and $\text{erf}$ is the error function. As shown in Fig.~\ref{fig:modelcompare}, the CDF version of this model still works when other models become invalid (Gaussian, Laplacian) or hard to integrate (Hyper-Laplacian, Model 1).
\item Computational efficiency: Model 2 has a simple mathematical form that can efficiently be evaluated on a computer. The effect is substantial, as shown in Fig~\ref{fig:CPU}(a) for the model evaluation, and in Fig~\ref{fig:CPU}(b) for evaluating the gradient of the model (e.g., in an optimization loop).
\item Optimization efficiency: Model 2 can be written as the difference of two convex functions. Optimization problems involving Model 2 can hence efficiently be solved using D.C.~programming, as shown in Section~\ref{sec:diffusion}. 
\end{itemize}
For these properties, we mainly consider Model 2 as a regularization term in Section~\ref{sec:app}.

\begin{figure}[H]
\centering
\subfigure[CPU time for evaluating the models.]{\includegraphics[width=0.7\linewidth]{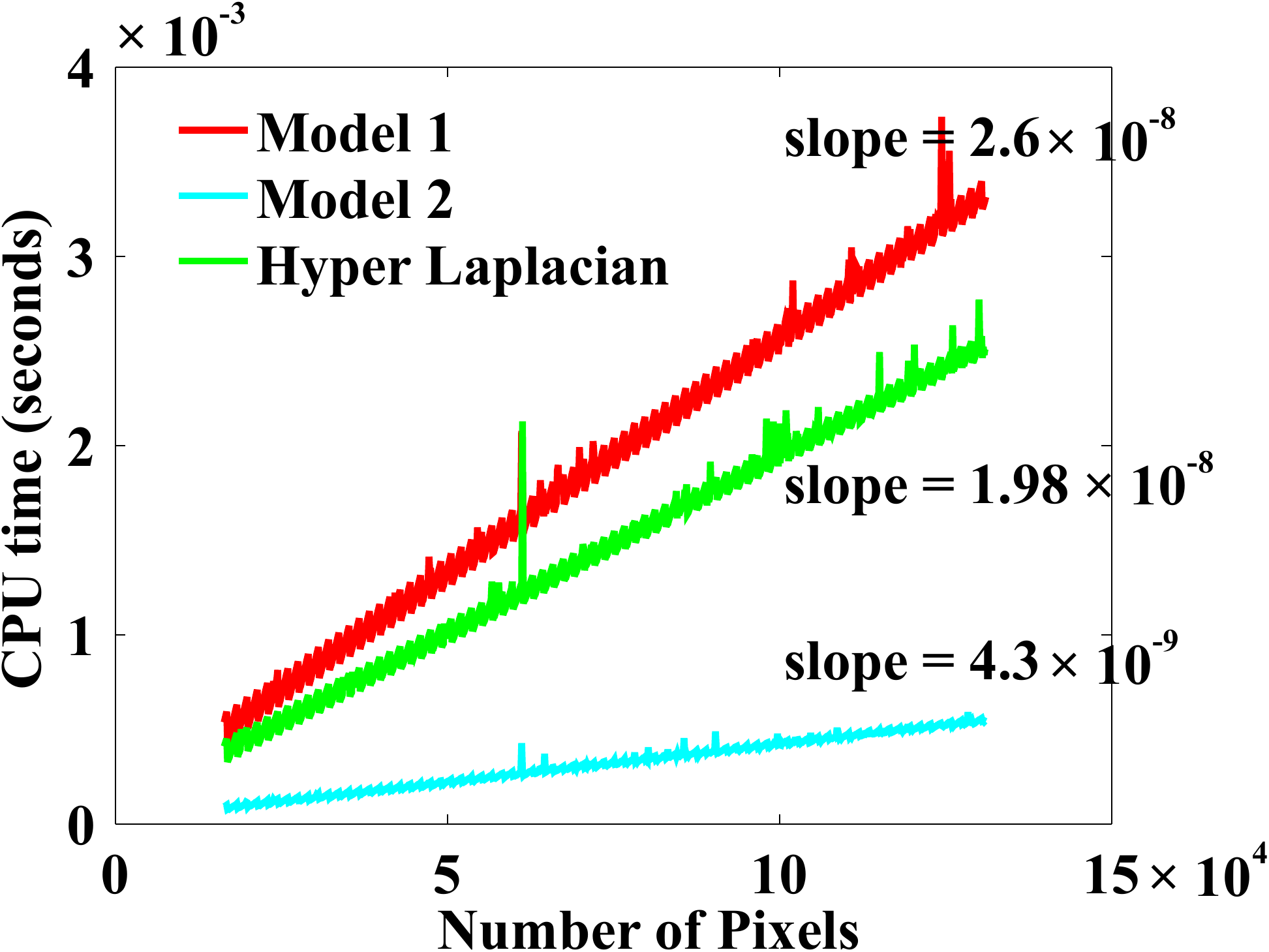}}
\subfigure[CPU time for evaluating the gradients of the models.]{\includegraphics[width=0.7\linewidth]{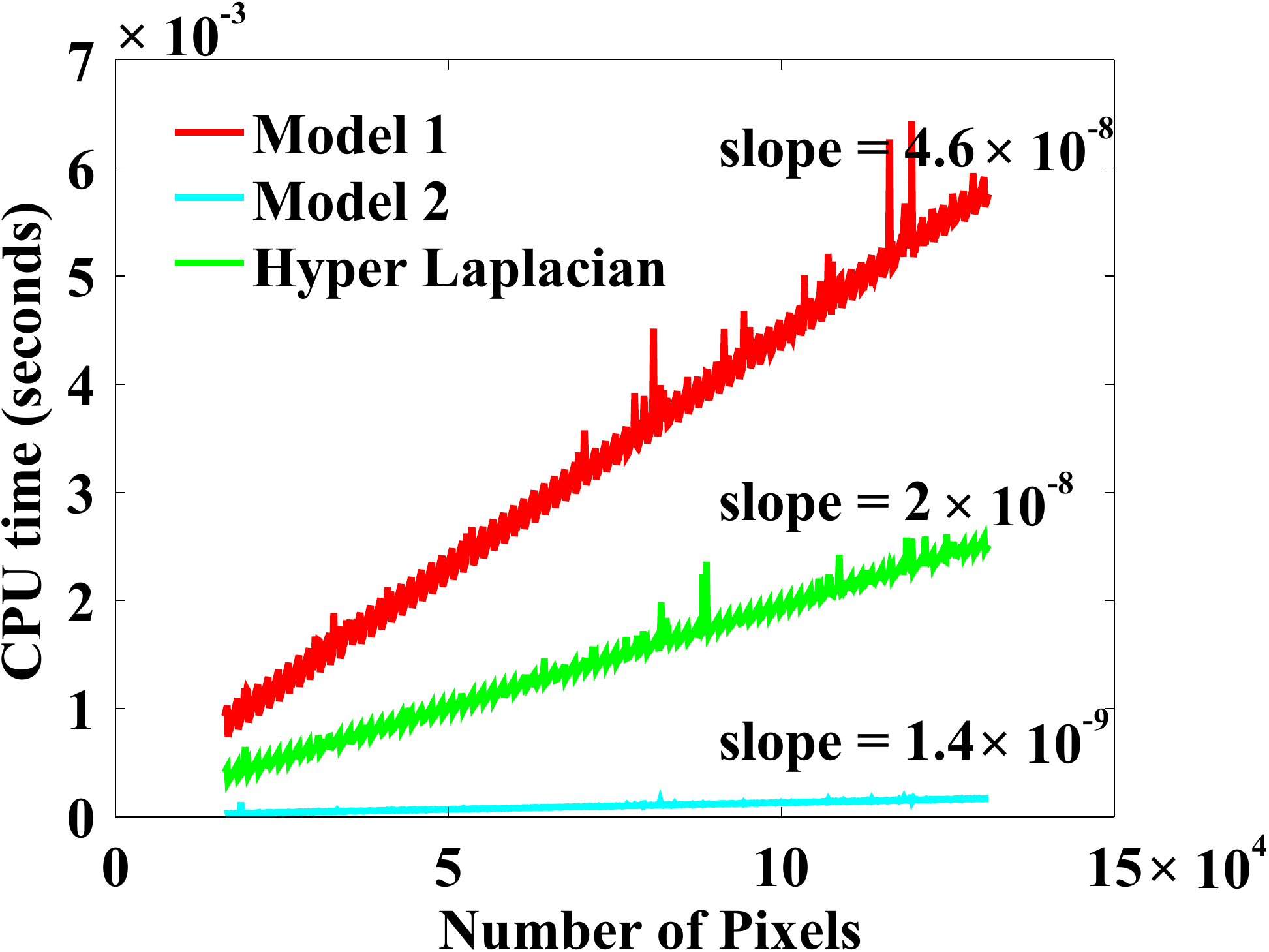}}
\caption{CPU time comparison for model and model gradient evaluations.}
\label{fig:CPU}
\end{figure}

We compare our two marginal models with other models in Fig.~\ref{fig:modelcompare} and Table.~\ref{table:SSE}. In all cases, they outperform the previously used Laplacian, Gaussian, and Hyper-Laplacian models. In Fig.~\ref{fig:CDFchange}, we and analyze the sensitivity of Model 2 as compared with the Hyper-Laplacian model. Model 2 fits the data better and shows good sensitivity (identifiability) with respect to parameter $a_2$.

\begin{figure}[!htb]
\centering
\subfigure[Marginal gradient distribution in log scale for each image. The color indicates scaled density.]{\includegraphics[width=.7\linewidth]{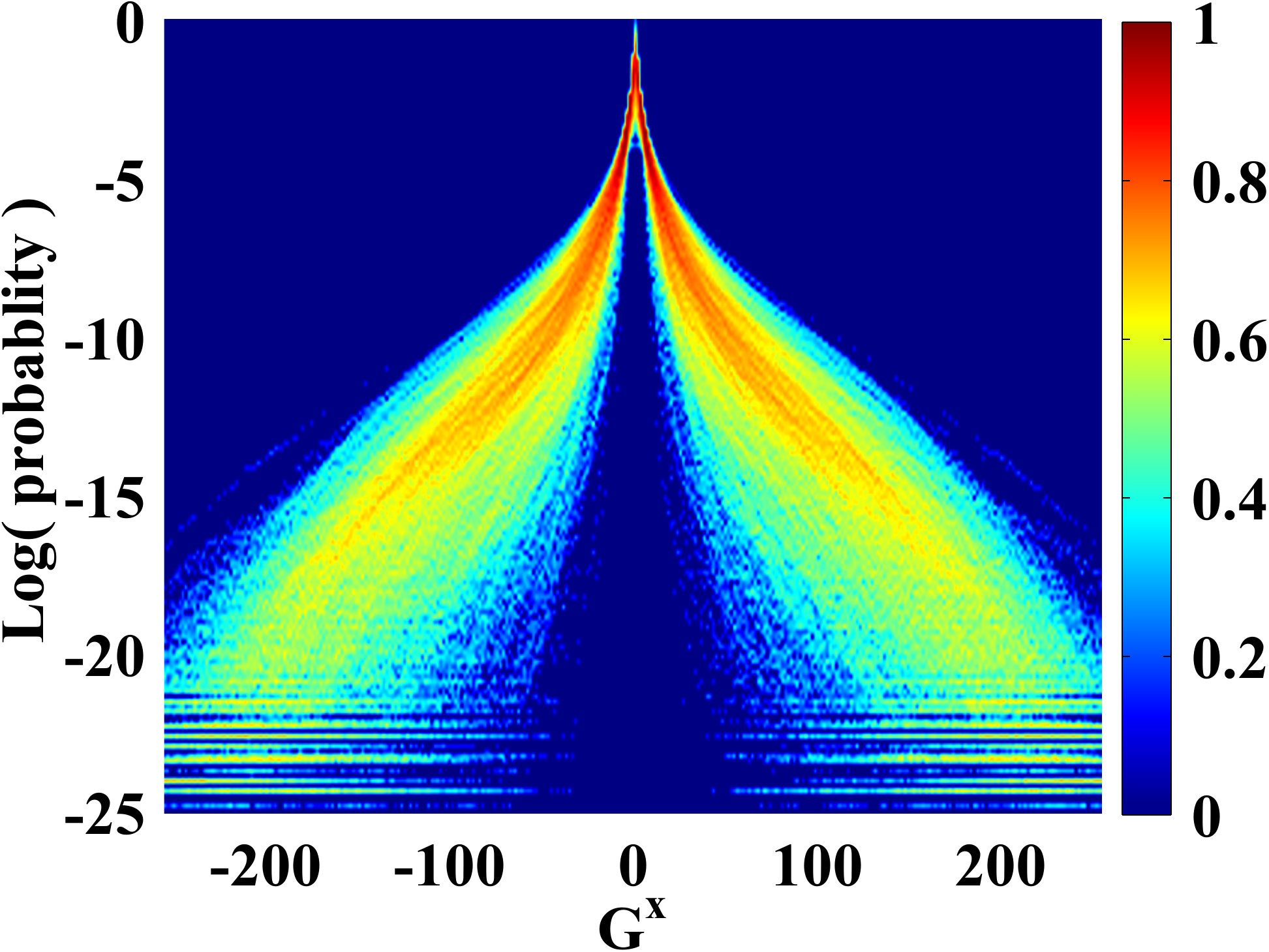}}
\subfigure[Model 2 with changing parameter $a_2$ (coded by color) and other parameters fixed at their best fit: $b_2=5.4$, $c_2=-0.266$.]{\includegraphics[width=.7\linewidth]{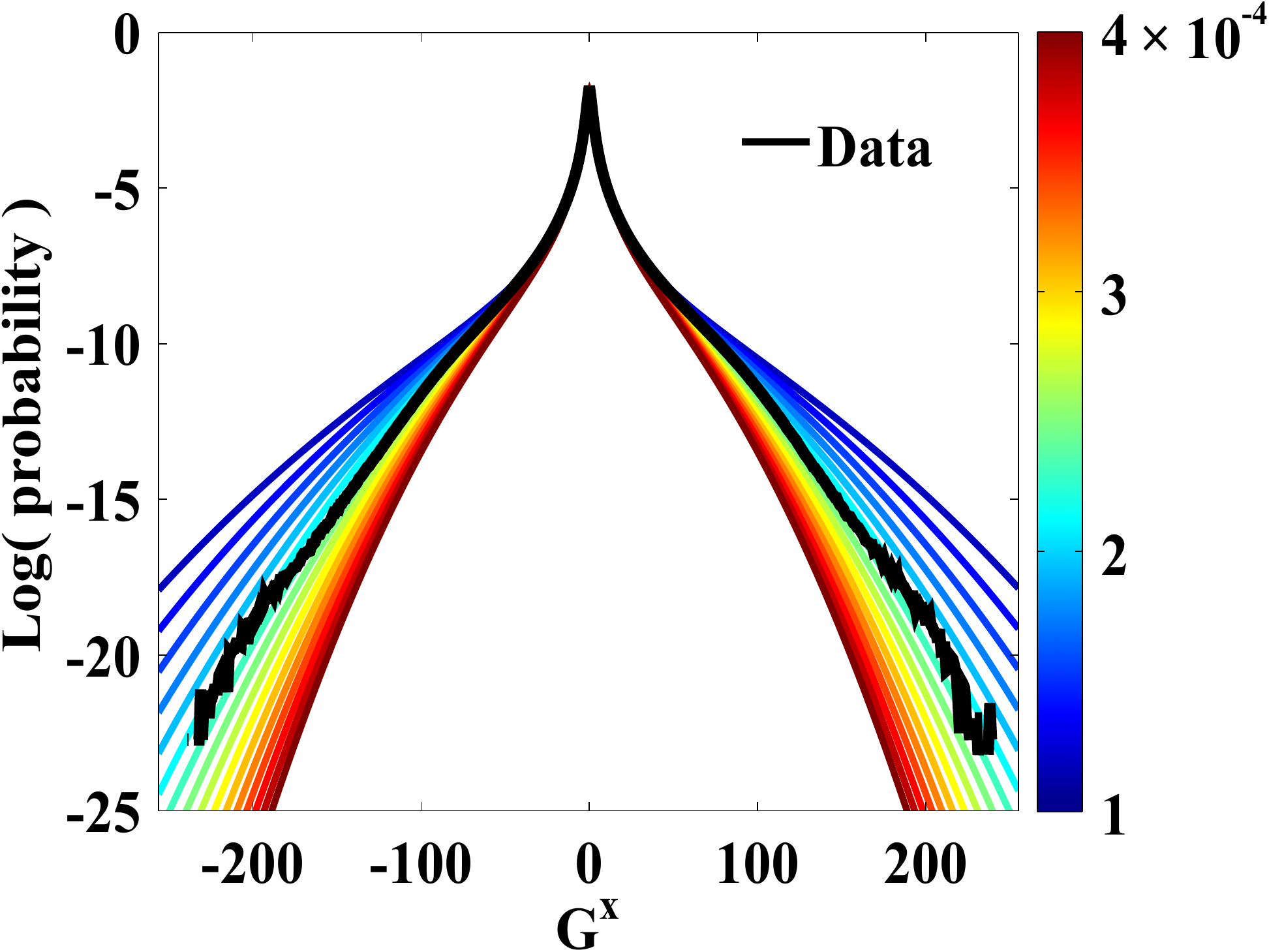}}
\subfigure[Hyper-Laplacian model with changing parameter $b_0$ (color) and all other parameters fixed at their best fit.]{\includegraphics[width=.7\linewidth]{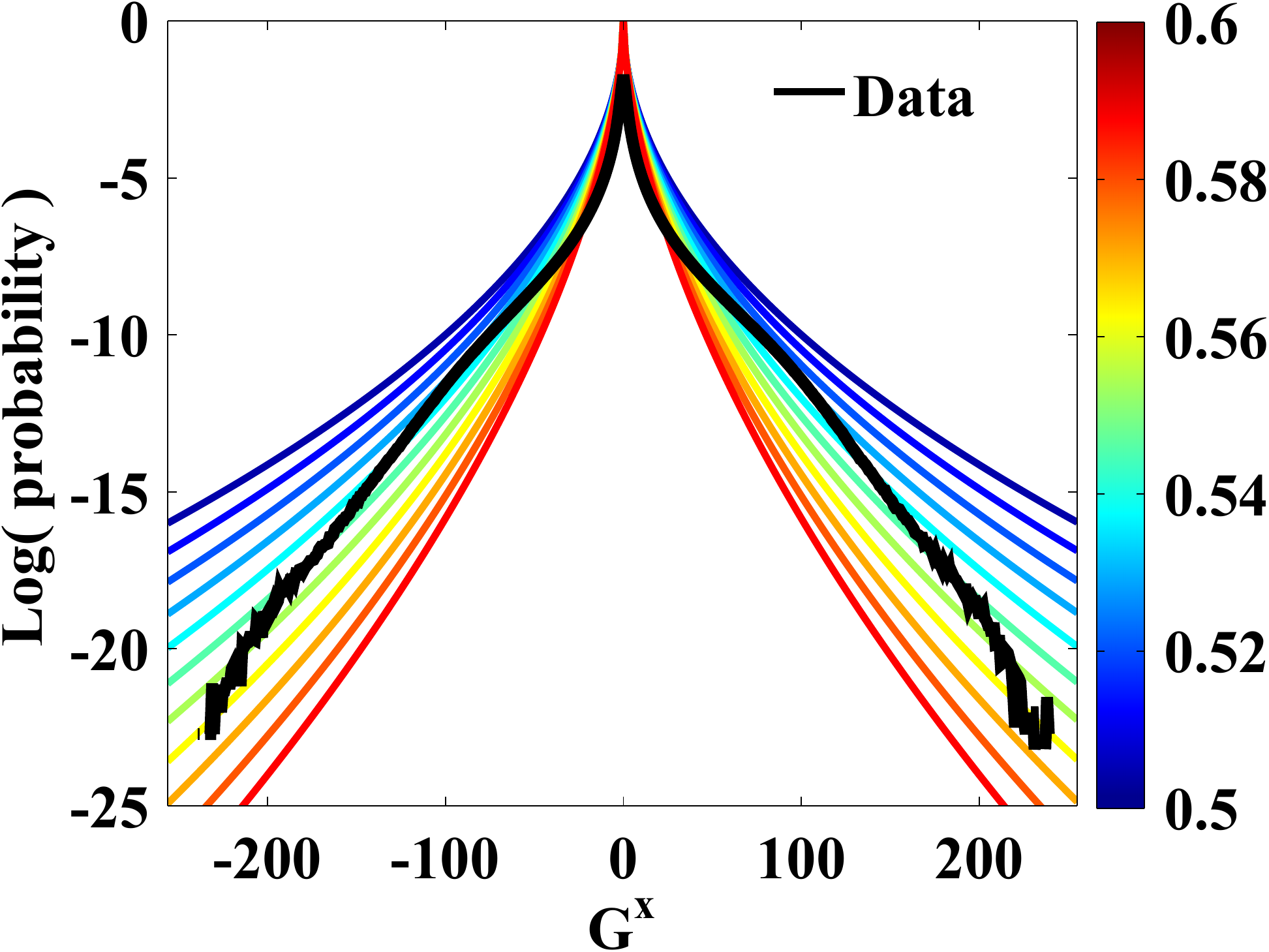}}
  \caption{Sensitivity analysis of Model 2 compared with the Hyper-Laplacian model. Parameter $a_2$ varies from $1\times 10^{-4}$ to $4\times 10^{-4}$ with step size $2\times 10^{-5}$. For the Hyper-Laplacian model, $b_0$ varies from 0.5 to 0.6 with step size 0.01.}
\label{fig:CDFchange} %% label for entire figure
\end{figure}

\begin{figure}[!htb]
  \centering
\subfigure[Log scale]{\includegraphics[width=0.7\linewidth]{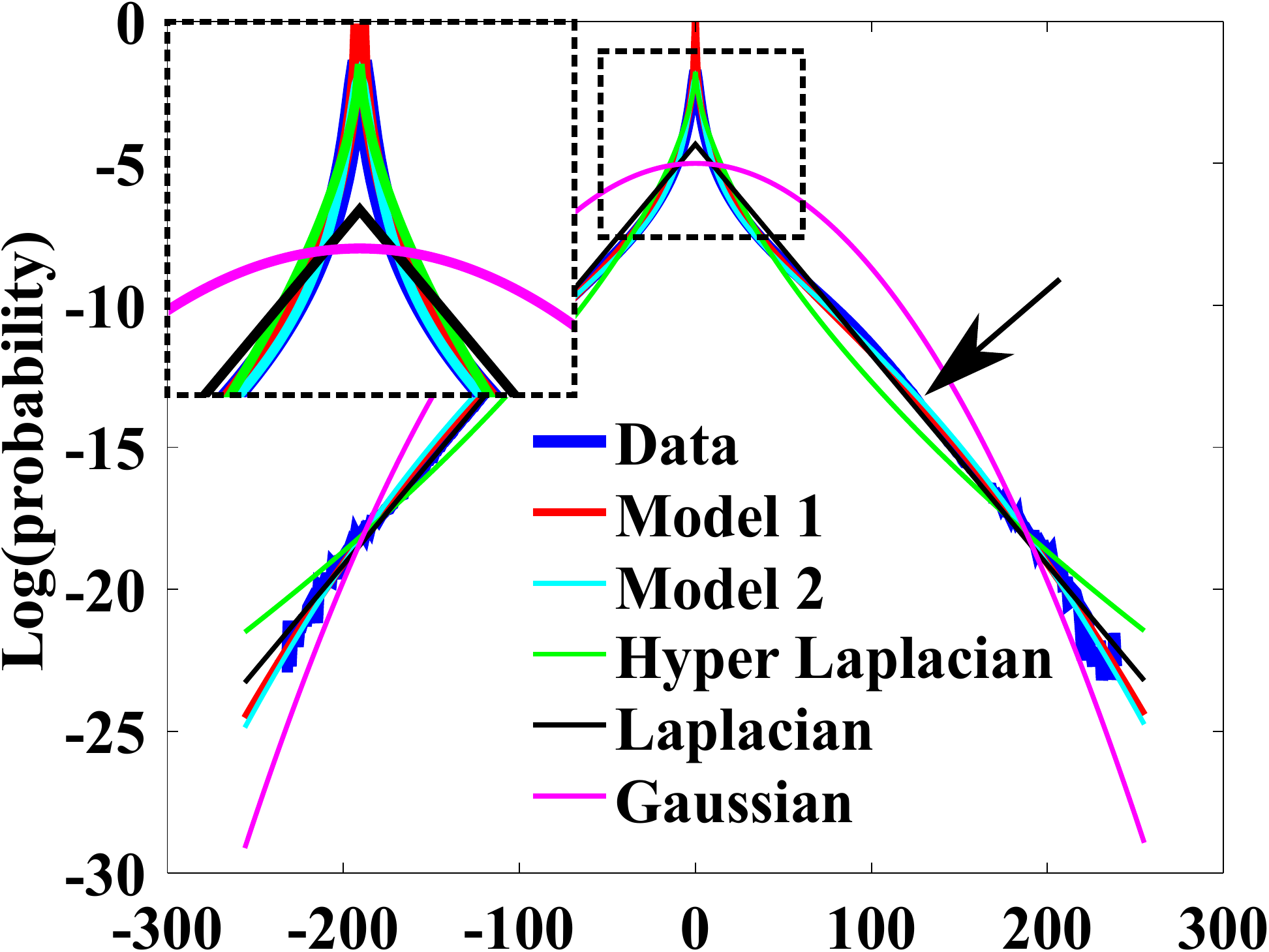}}
%\subfigure[Linear scale]{\includegraphics[width=0.7\linewidth]{compareP.pdf}}
\subfigure[cumulative sum in linear scale]{\includegraphics[width=0.7\linewidth]{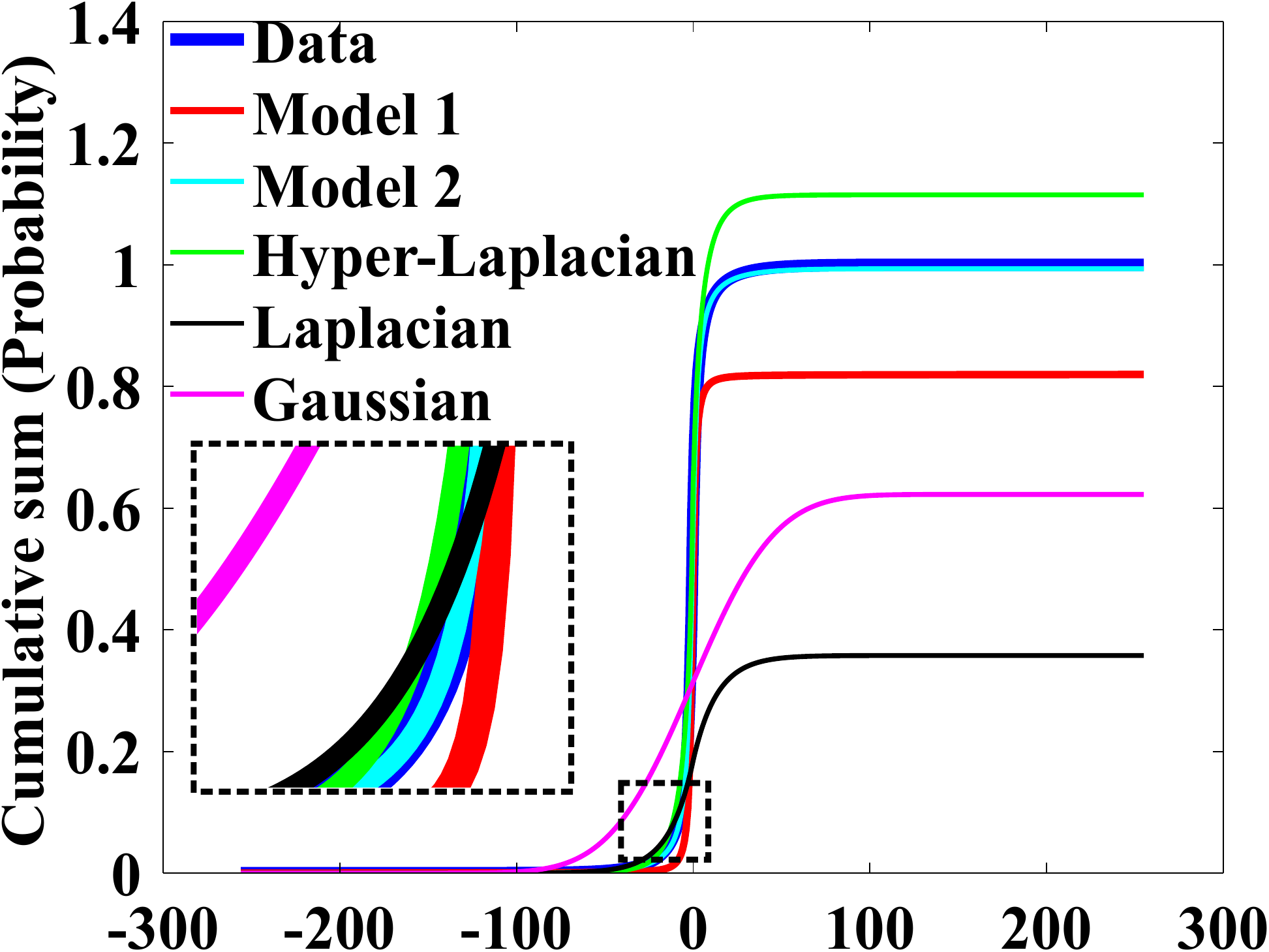}}
  \caption{Comparison of marginal models (log scale) and their cumulative sums (linear scale). Optimal parameters are used for each model. The quantitative differences are shown in Table~\ref{table:SSE}; sensitivity analysis of the model fits is shown in Fig.~\ref{fig:CDFchange}.}
  \label{fig:modelcompare} 
\end{figure}

\begin{table}[!htb]
\scriptsize
\centering  % used for centering table
\begin{tabular}{c|ccccccc} 
\hline\hline
Image set  & 1&	2&	3&	4&	5&	6&	7 \\
\hline
SSE & 40.5&	43.0&	67.7&	23.8&	34.1&	37.9&	25.4 \\
$R^2$ & 0.99&	0.99&	0.99&	0.99&	0.99&	0.99&	0.99 \\
\hline
 SSE & 271  &	324&	266&	44.4&	38.2&	62.8&	30.7 \\
$R^2$ & 0.96&	0.93&	0.96&	0.99&	0.99&	0.98&	0.99 \\
\hline
 SSE & 576&	301&	537&	45.4&	389&	70.5&	250 \\
$R^2$ & 0.92&	0.93&	0.91&	0.98&	0.96&	0.98&	0.97 \\
\hline
  SSE$\times 10^{-3}$ & 1.86&	3.01&	3.02&	3.95&	2.34&	3.90&	3.95 \\
$R^2$ & 0.74&	0.30&	0.52&	0.13&	0.81&	0.10&	0.57 \\
\hline
SSE$\times 10^{-4}$ & 0.83&	1.02&	1.10&	1.24&	1.32&	1.23&	1.64 \\
$R^2$ & -0.12&	-1.3&	-0.72&	-2.6&	-0.046&	-2.5&	-0.75 \\
\hline
\end{tabular}
\caption{Goodness of fit comparison for all models (from top to bottom): Model 1 (Eq.~\ref{eq:distr}), Model 2 (Eq.~\ref{eq:distr2}), Hyper-Laplacian, Laplacian, Gaussian.} % title of Table
\label{table:SSE} % is used to refer this table in the text
\end{table}

\subsection{2D Joint Model}
\label{sec:2DModel}
While the 1D marginal models approximate well the marginal distributions of the gradient, they are not statistically independent. As shown in the last two rows of Table~\ref{table:1D}, the two gradient components are weakly negatively correlated. This weak correlation between the gradient components explains why alternating optimization had to be used in all previous works that considered the marginal models independently, and why the results were still good even though the prior is not strictly correct.  

The traditional model (Eq.~\ref{eq:GGD}) can easily be extended to 2D:
\begin{equation}
\label{eq:GGD2}
\log(P(\vec{G}))=-a_0(|G^{x}|^{b_0}+|G^{y}|^{b_0})+c_0 \, , 
\end{equation} where $a_0$, $b_0$, and $c_0$ are the parameters. Such a model, including the Hyper Laplacian as a special case, however, treats the $x$ and $y$ components of the gradient as independent. 

\begin{figure*}[!tb]
  \centering
  \subfigure[data]{\includegraphics[width=0.25\linewidth]{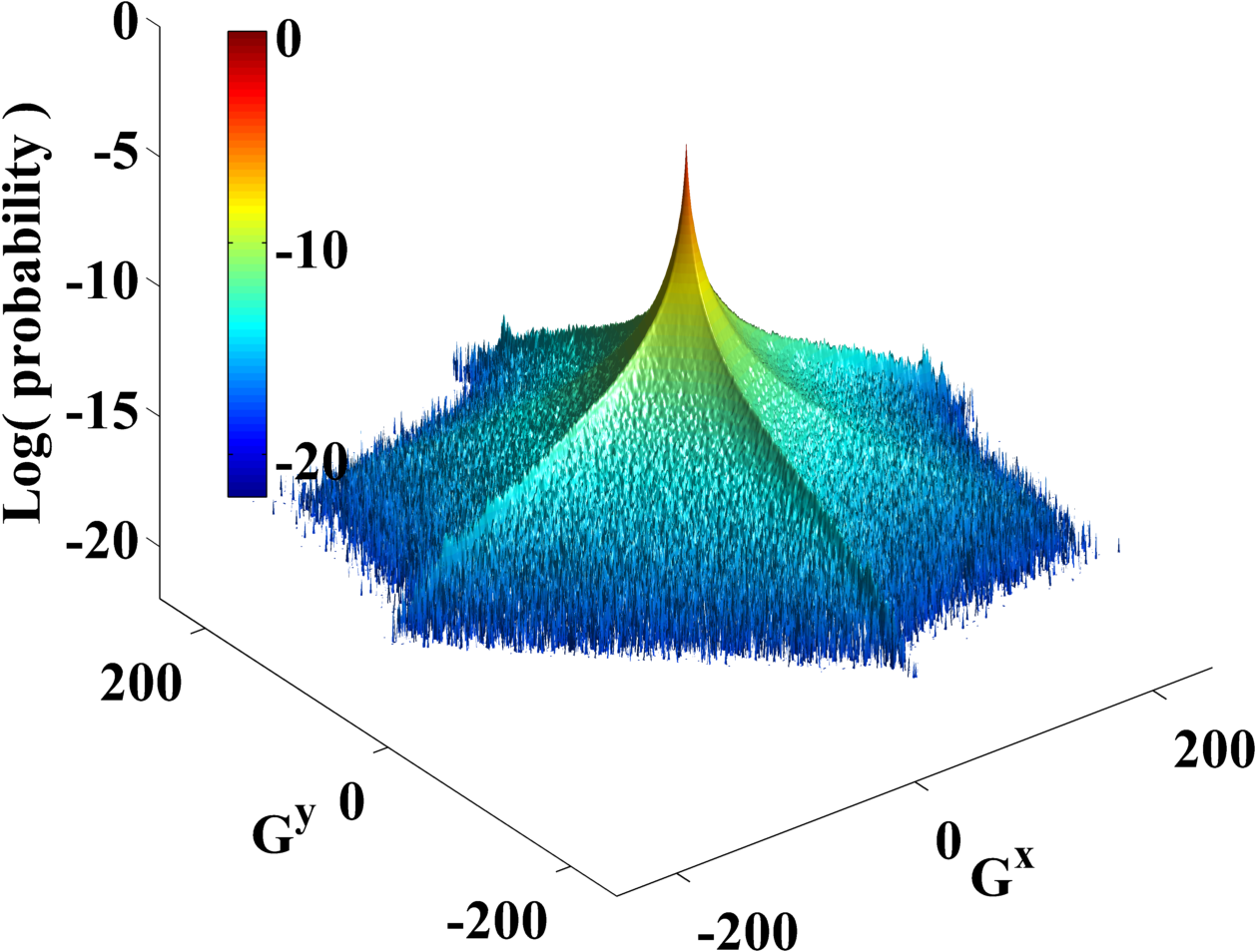}}
  \subfigure[Model 1]{\includegraphics[width=0.25\linewidth]{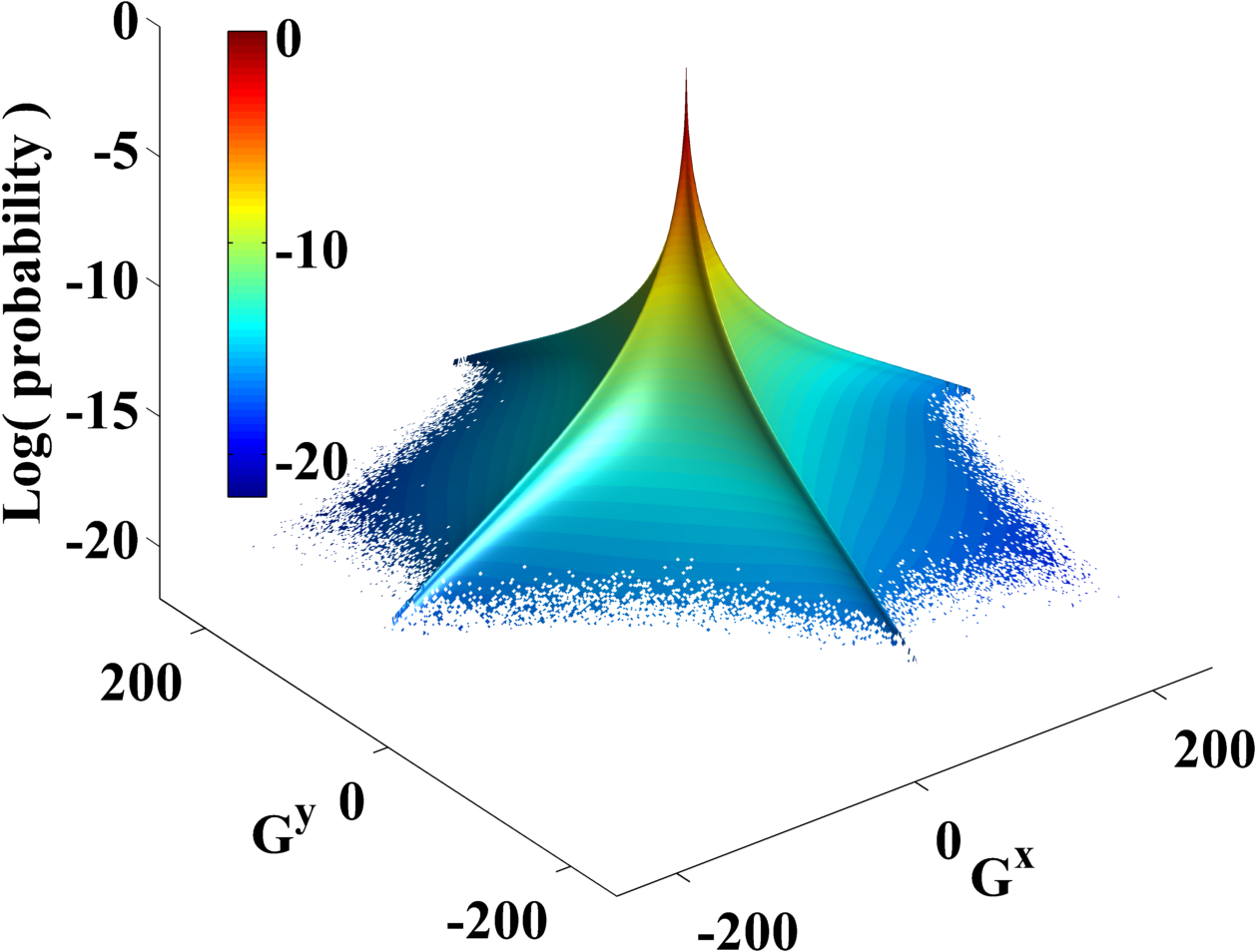}}
  \subfigure[Model 2]{\includegraphics[width=0.25\linewidth]{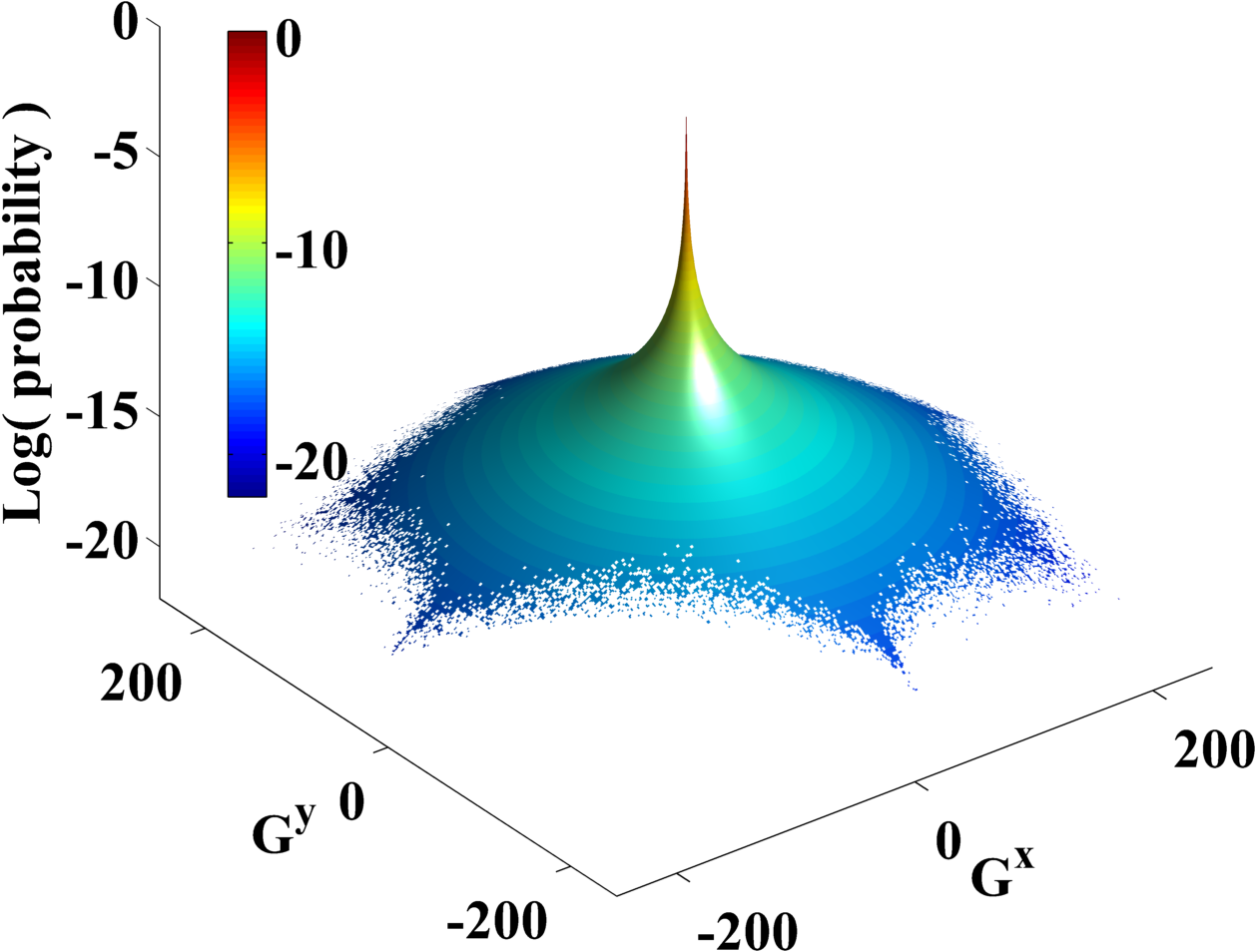}}
  
  \subfigure[data iso lines]{\includegraphics[width=0.25\linewidth]{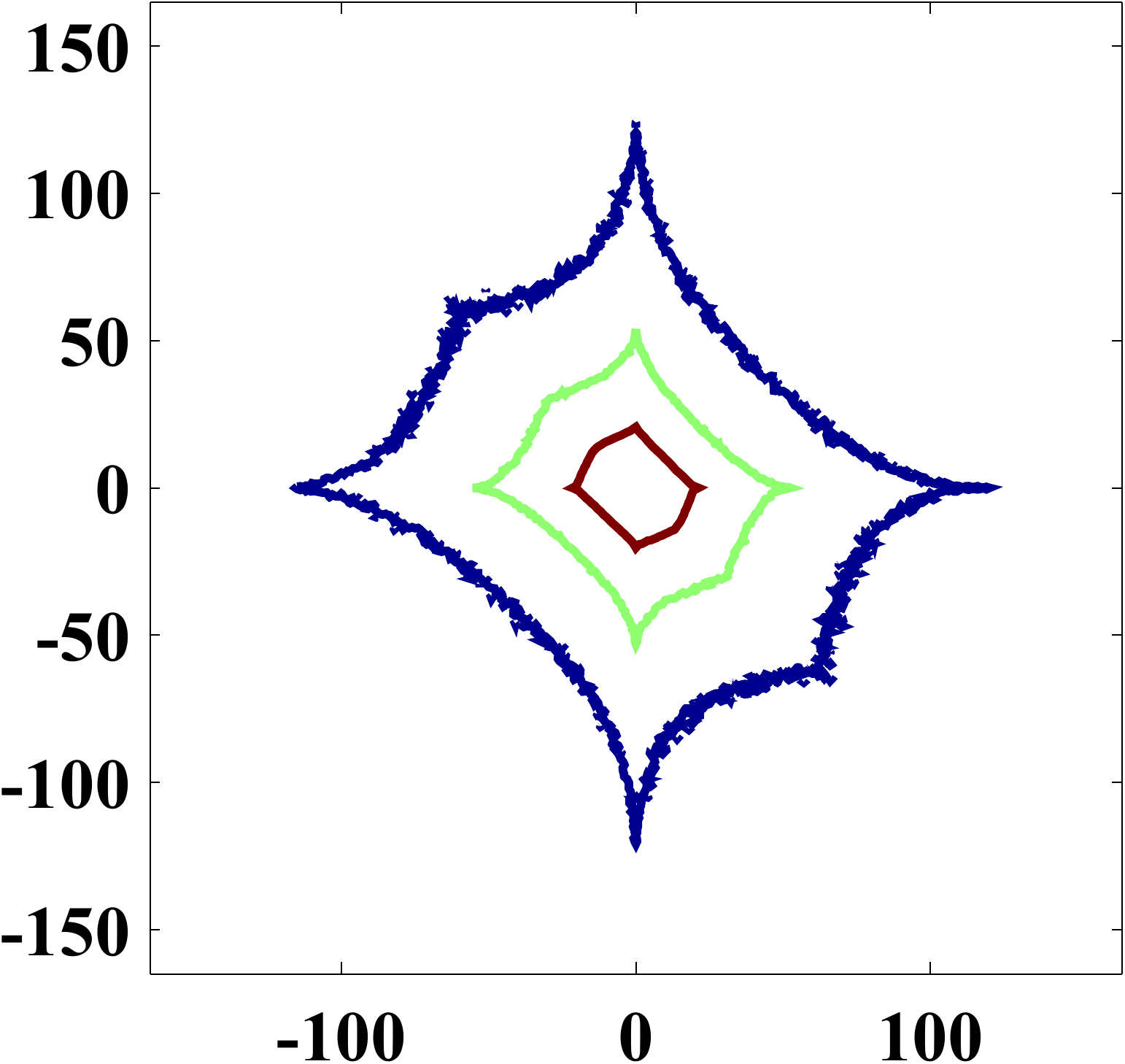}}
  \subfigure[Model 1 iso lines]{\includegraphics[width=0.25\linewidth]{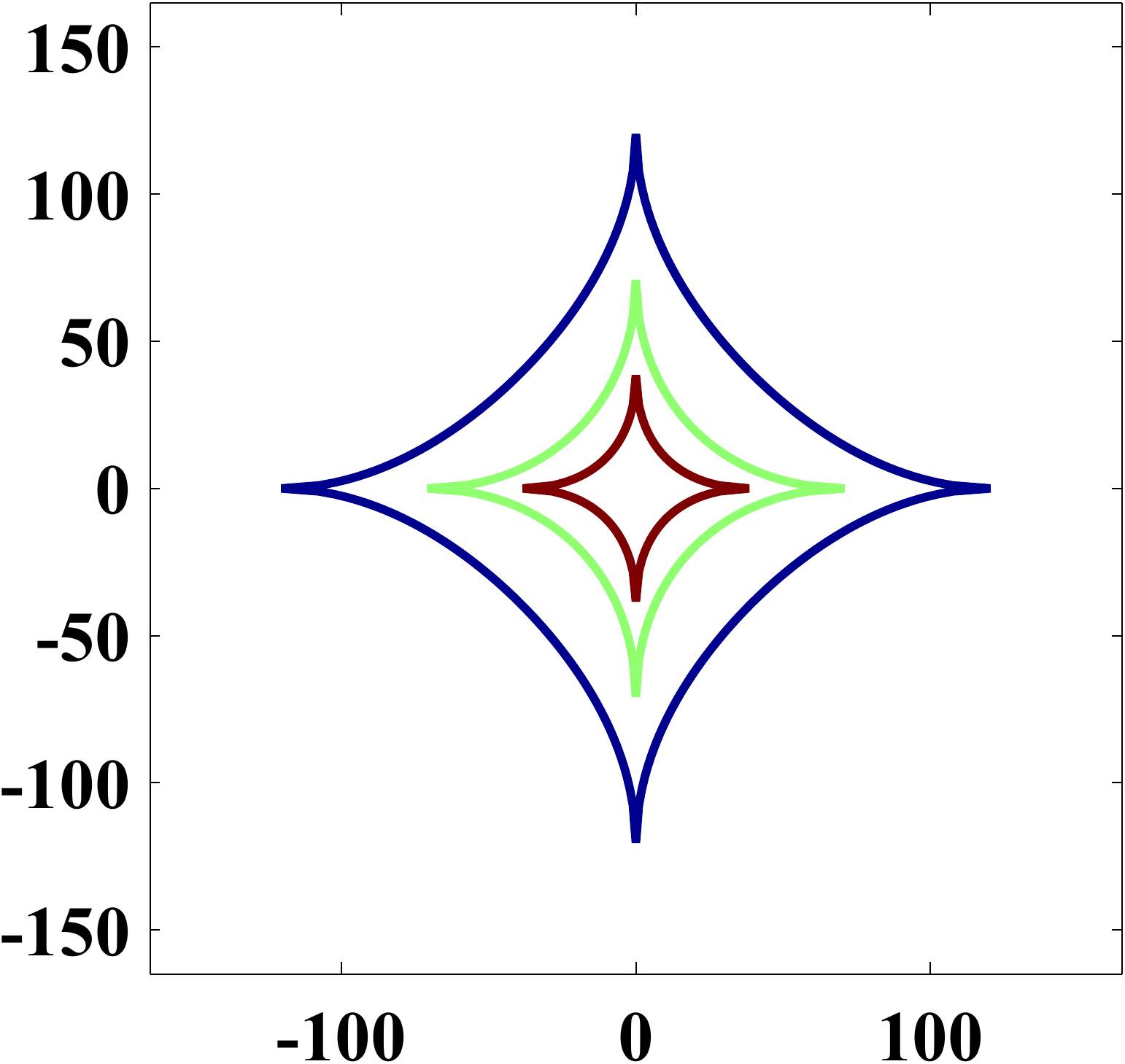}}
  \subfigure[Model 2 iso lines]{\includegraphics[width=0.25\linewidth]{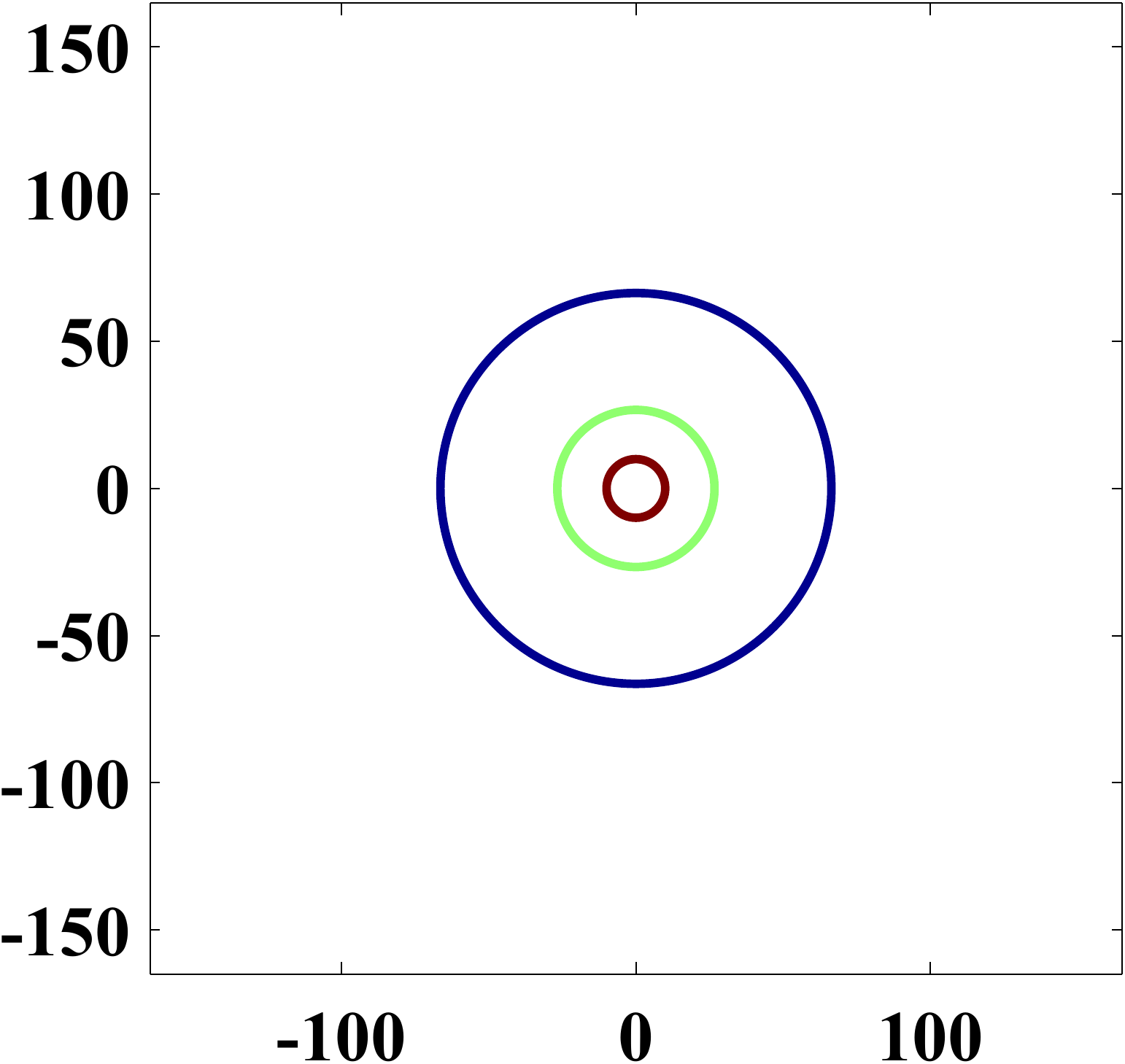}}
  
  \subfigure[Hyper Laplacian]{\includegraphics[width=0.25\linewidth]{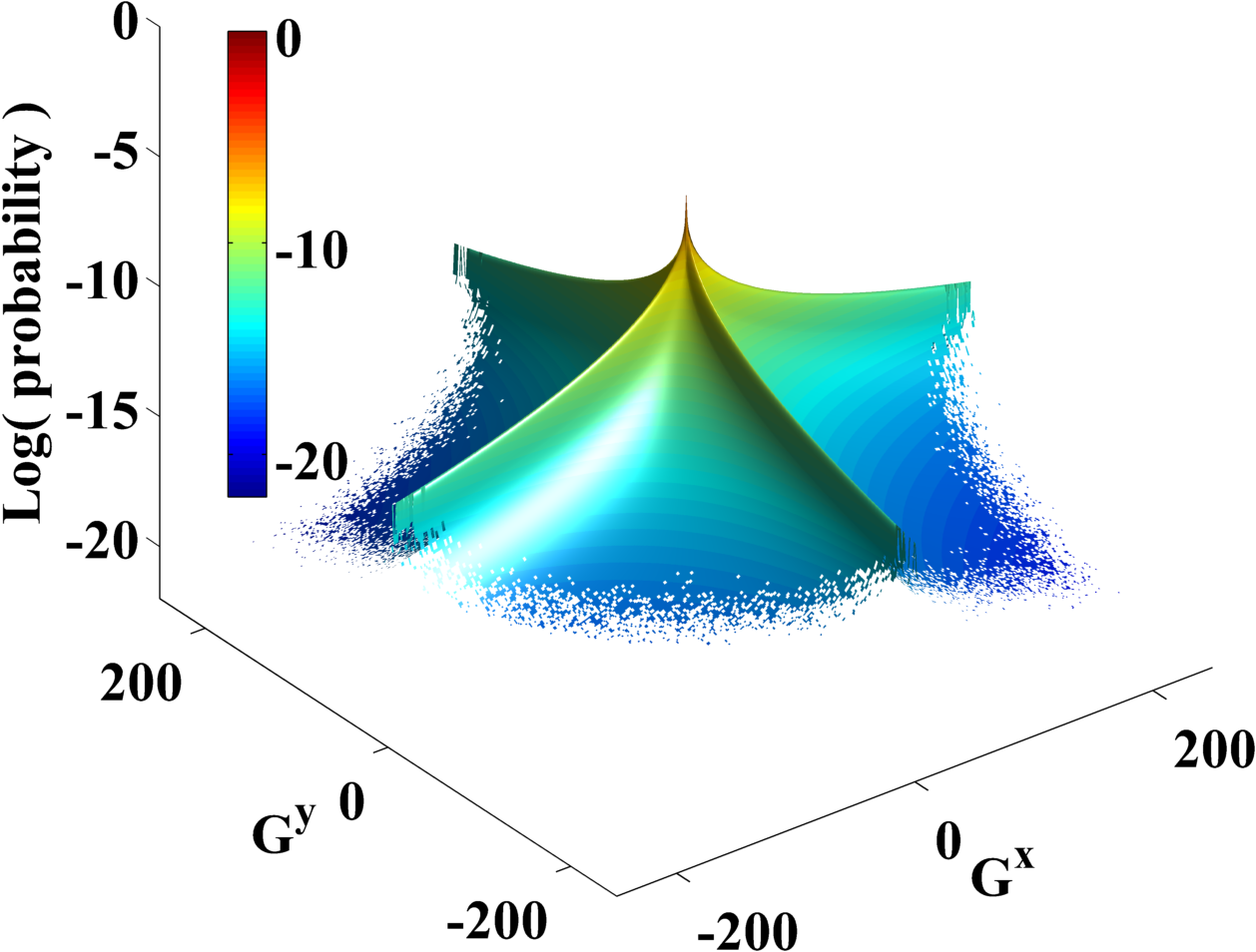}}
  \subfigure[Laplacian]{\includegraphics[width=0.25\linewidth]{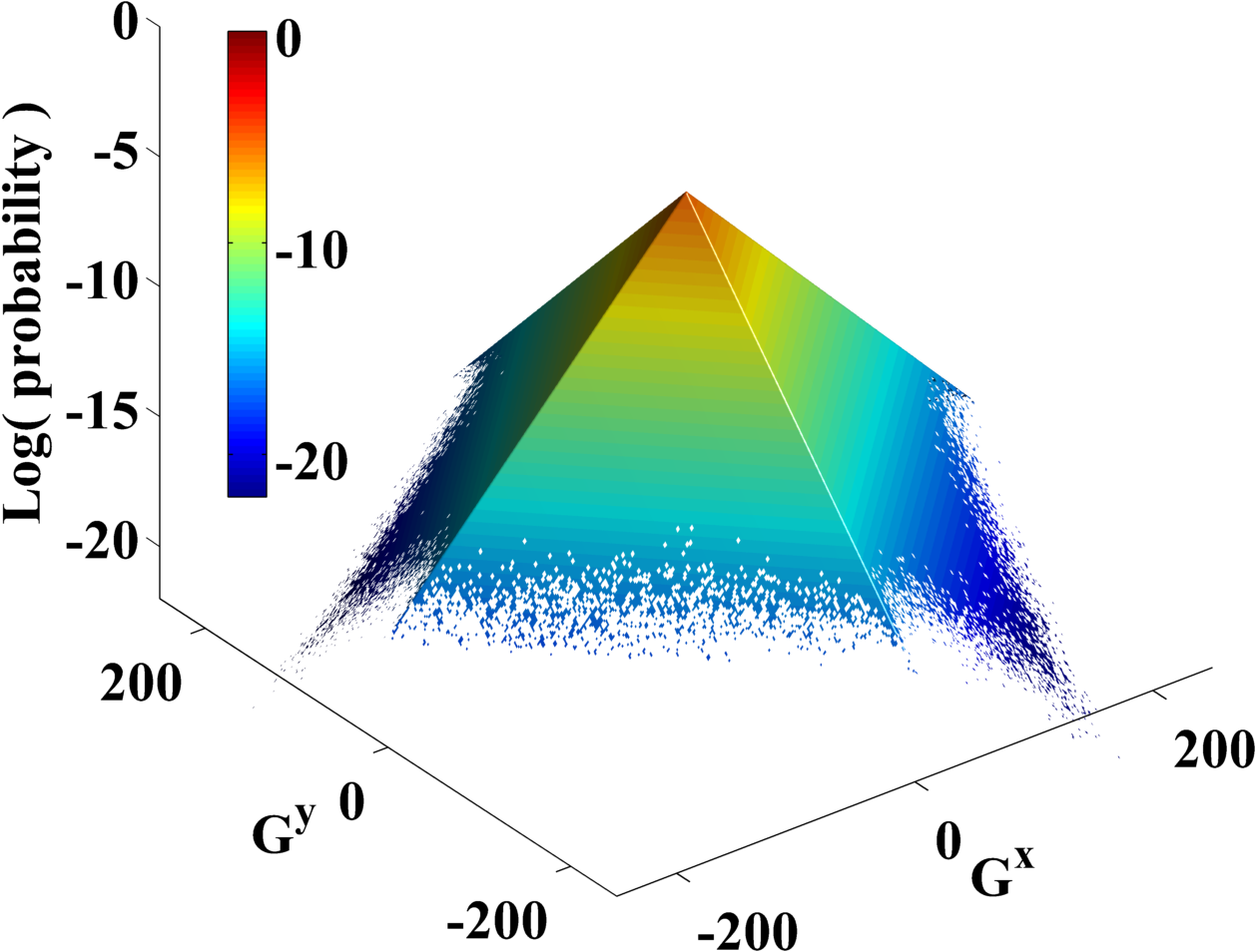}}
  \subfigure[Gaussian]{\includegraphics[width=0.25\linewidth]{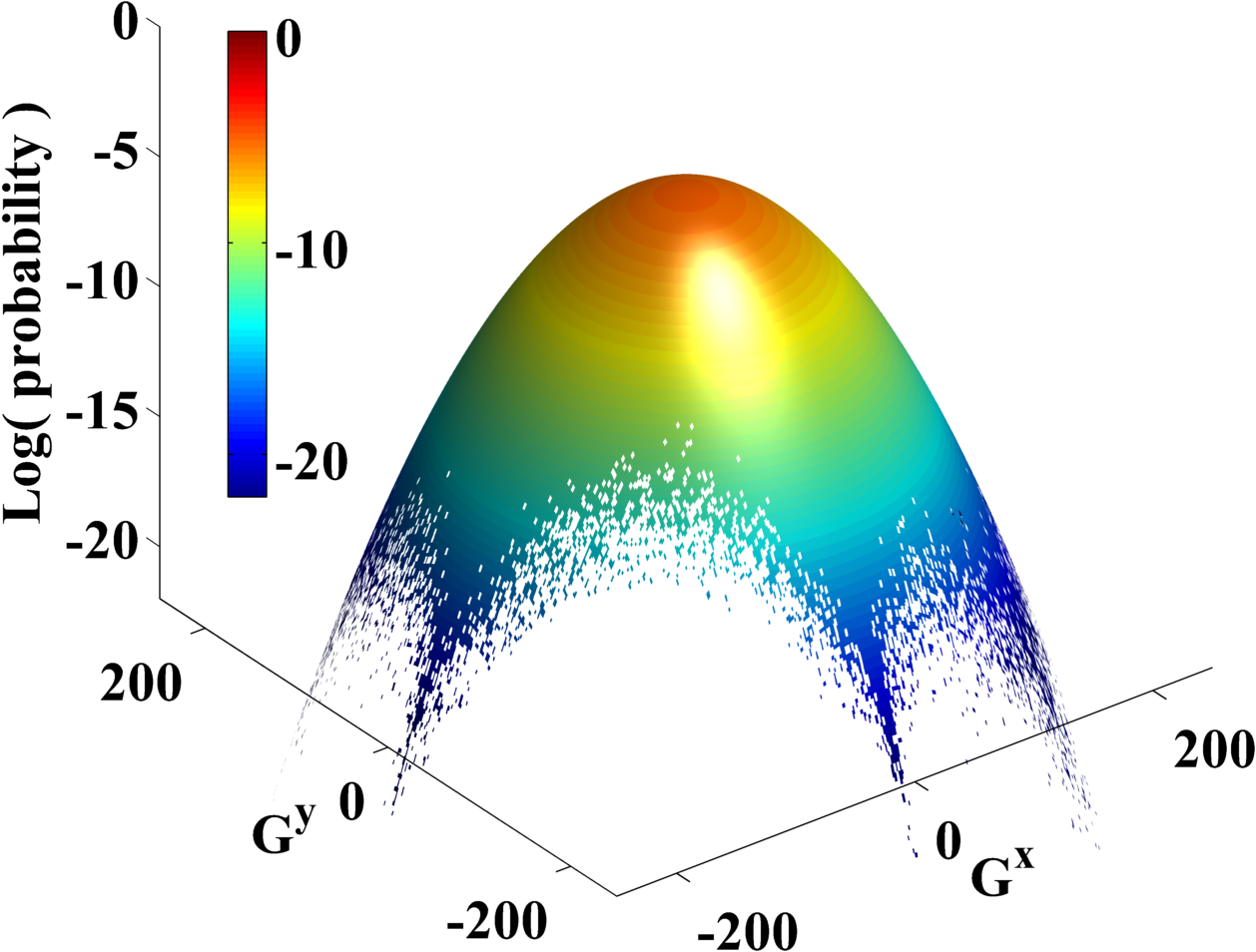}}

  \subfigure[HypLap iso lines]{\includegraphics[width=0.25\linewidth]{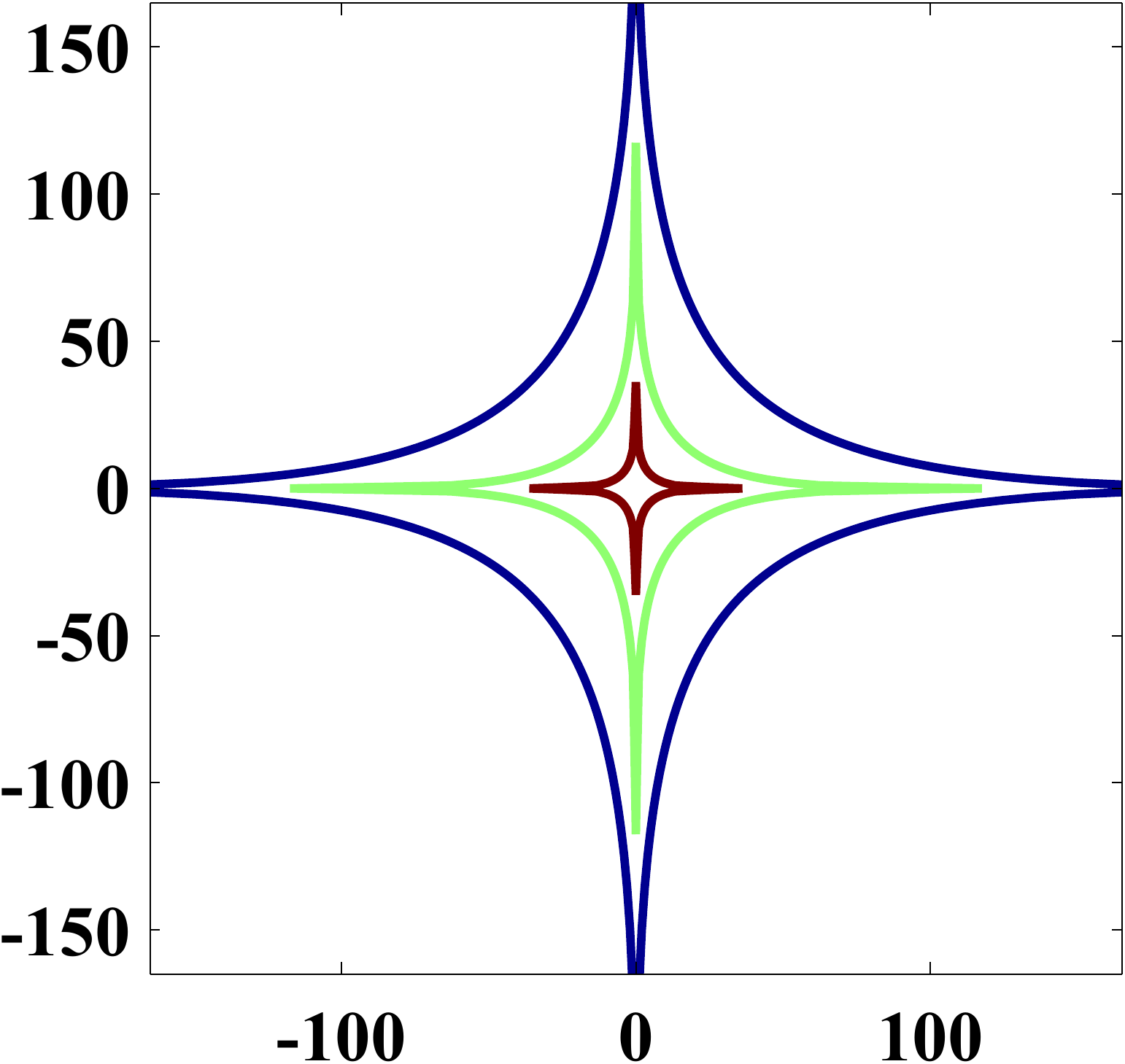}}
  \subfigure[Laplace iso lines]{\includegraphics[width=0.25\linewidth]{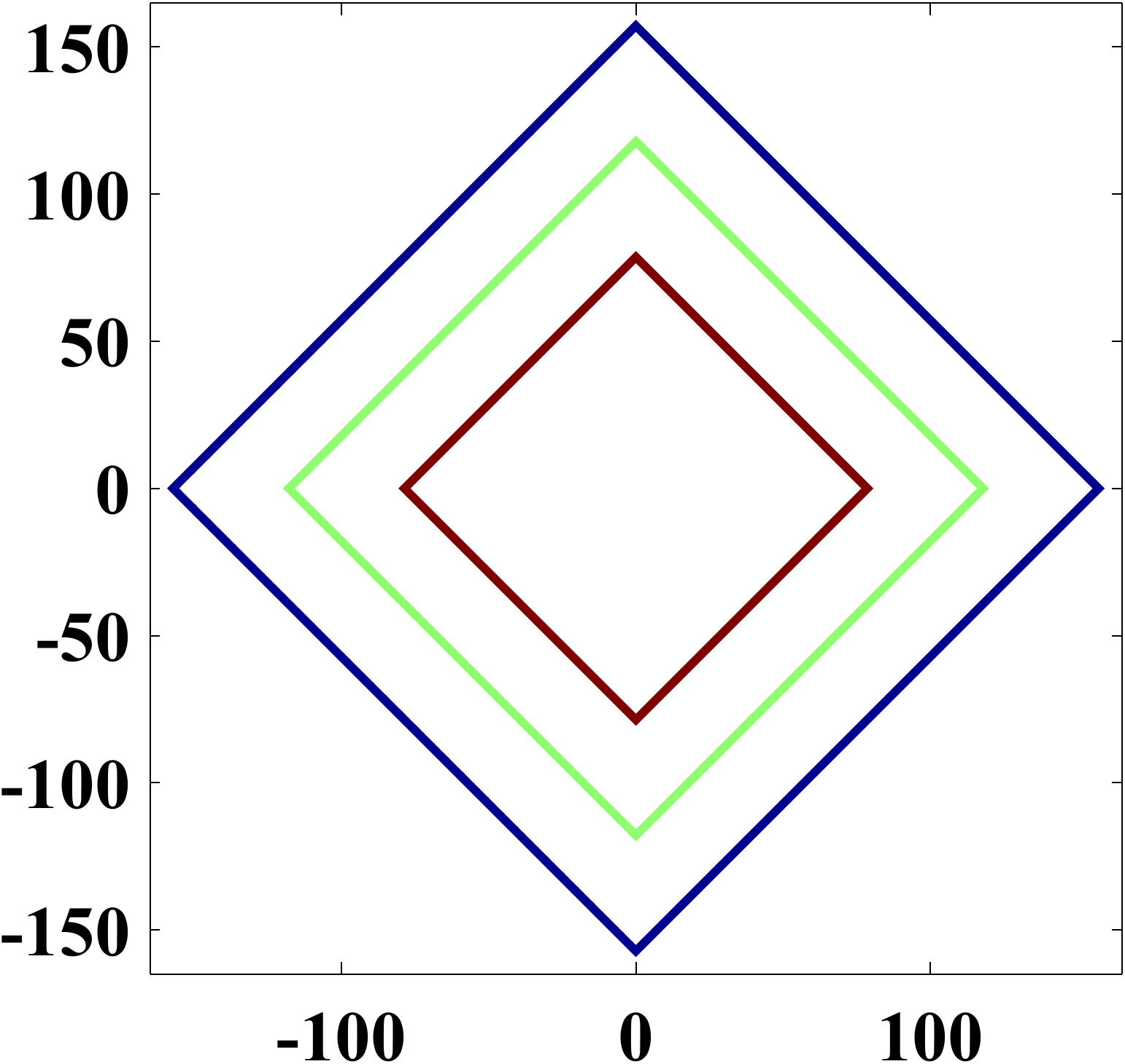}}
  \subfigure[Gaussian iso lines]{\includegraphics[width=0.25\linewidth]{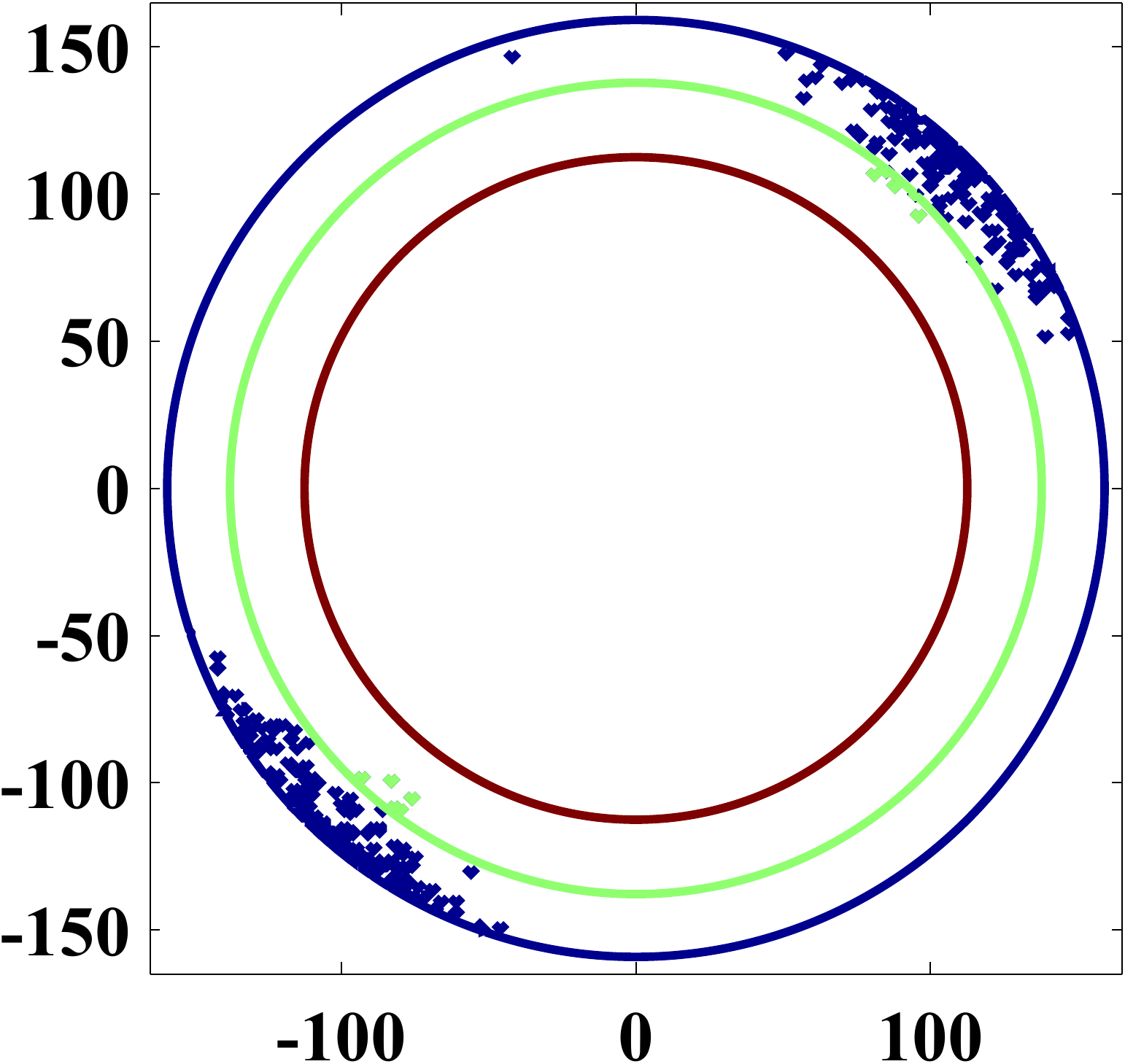}}
  \caption{Visual comparison of model fits for 2D joint gradient distribution models in log scale. First row: 2D gradient distribution of the data followed by plots of the best-fitting 2D models. Second row: iso contours of values -13, -11, and -9. The area included by isoline -13 of Model 2 is only $3\%$ of the whole domain, but the total probability mass in that area is $99\%$. More details about model sparsity are given in Section~\ref{sec:convex}.}
\label{fig:map} %% label for entire figure
\end{figure*}

Considering that a correlation between the gradient components may exist, we instead propose the following two models (corresponding to Model 1 and Model 2 above) for the 2D joint gradient distribution: 
{\small
\begin{equation}
\abovedisplayskip=3pt 
\belowdisplayskip=3pt
\label{eq:2Ddistr}
\log(P)=2a_1(\exp\!\left\{ {-\frac{|G^x|^{b_1}+|G^y|^{b_1}}{a_1}} \right\}-1) +c_1\lVert\vec{G}\rVert_2^2 \, .
\end{equation}
\begin{equation}
\abovedisplayskip=3pt 
\belowdisplayskip=3pt
\label{eq:2Ddistr2}
\log(P)=-a_2(|G^x|^2+|G^y|^2)- \log(b_2+|G^x|^2+|G^y|^2) +c_2\, .
\end{equation}
}

The fitted parameters and are shown in Table~\ref{table:2D}. Figure~\ref{fig:map} compares the model fits with previous models. The data histogram is shown in the left panel of Fig.~\ref{fig:map}, whereas the best-fit parametric models are plotted in the remaining panels. The area included by isoline -13 of Model 2 is only $3\%$ of the whole domain, but the total probability mass in that area is $99\%$. More details about model sparsity are given in Section~\ref{sec:convex}. From Table.~\ref{table:2D}, it is clear that Eq.~\ref{eq:2Ddistr2} fits the data almost as good as Eq.~\ref{eq:2Ddistr}, and both models fits much better than any previous model.

\begin{figure*}[!tb]
\centering
\setlength{\abovecaptionskip}{0cm}
\setlength{\belowcaptionskip}{0cm}
\setlength{\subfigbottomskip}{0cm}
\setlength{\subfigcapskip}{0cm}
\setlength{\subfigtopskip}{0cm}
\subfigure[original image]{\includegraphics[width=.13\linewidth]{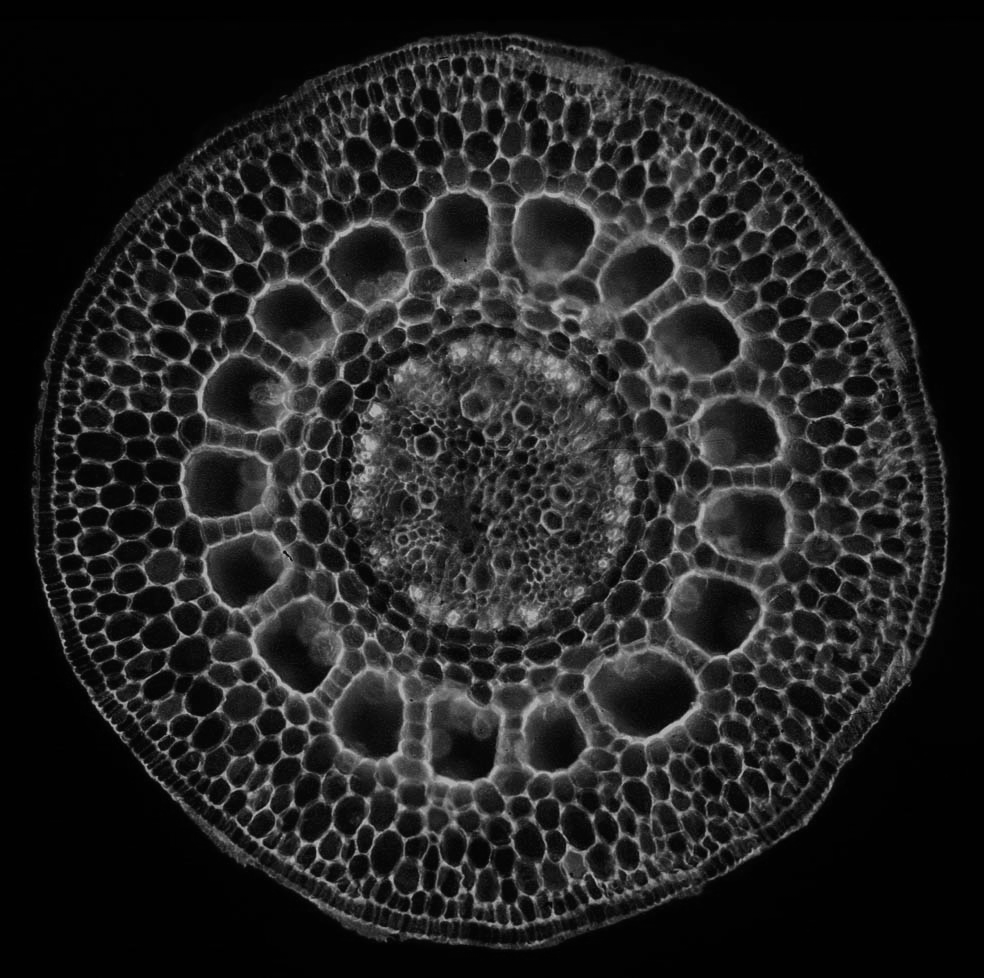}}
\subfigure[blurred ]{\includegraphics[width=.13\linewidth]{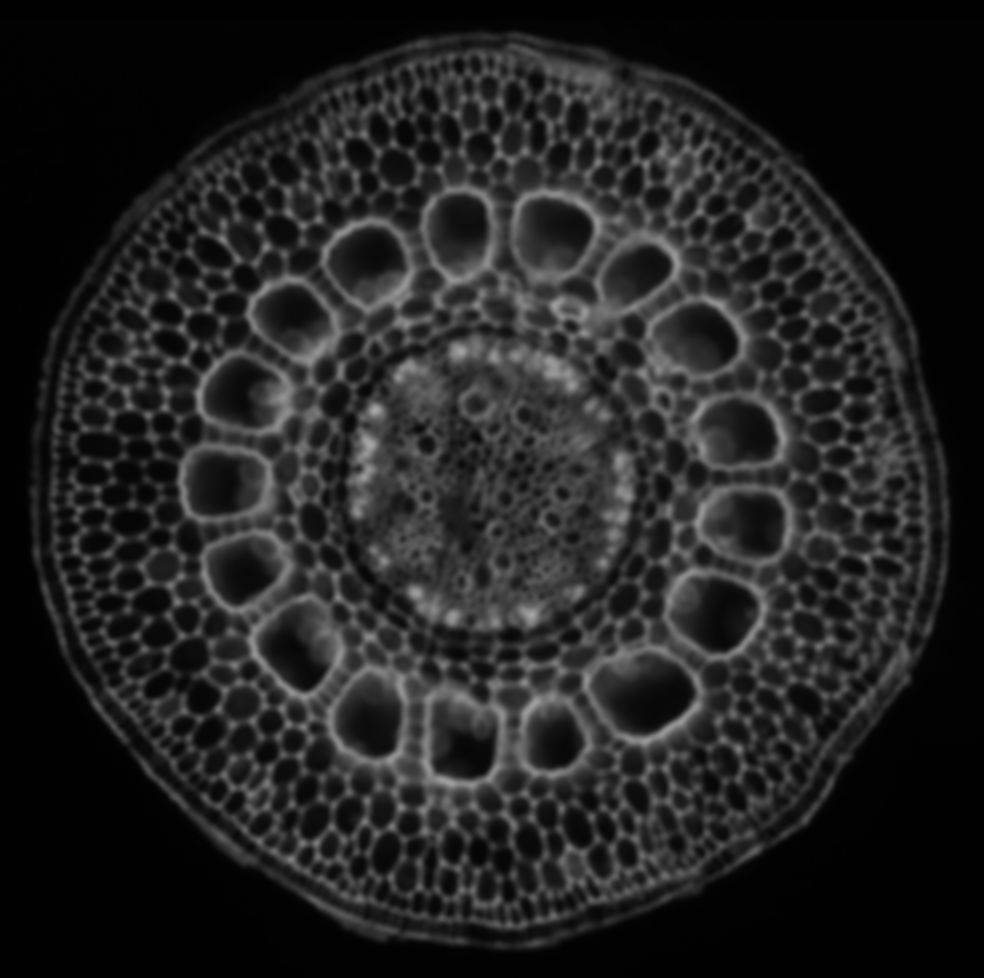}}
\subfigure[with noise]{\includegraphics[width=.13\linewidth]{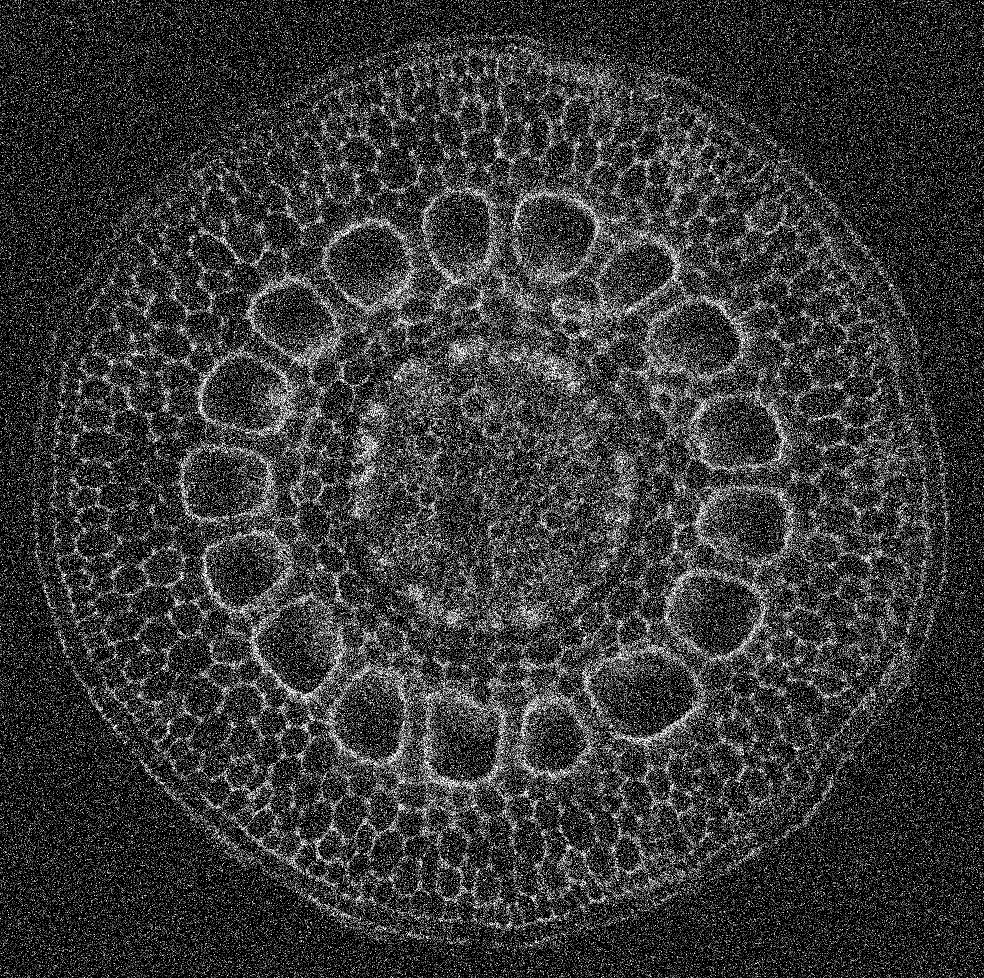}}
\subfigure[SR(nearest)]{\includegraphics[width=.13\linewidth]{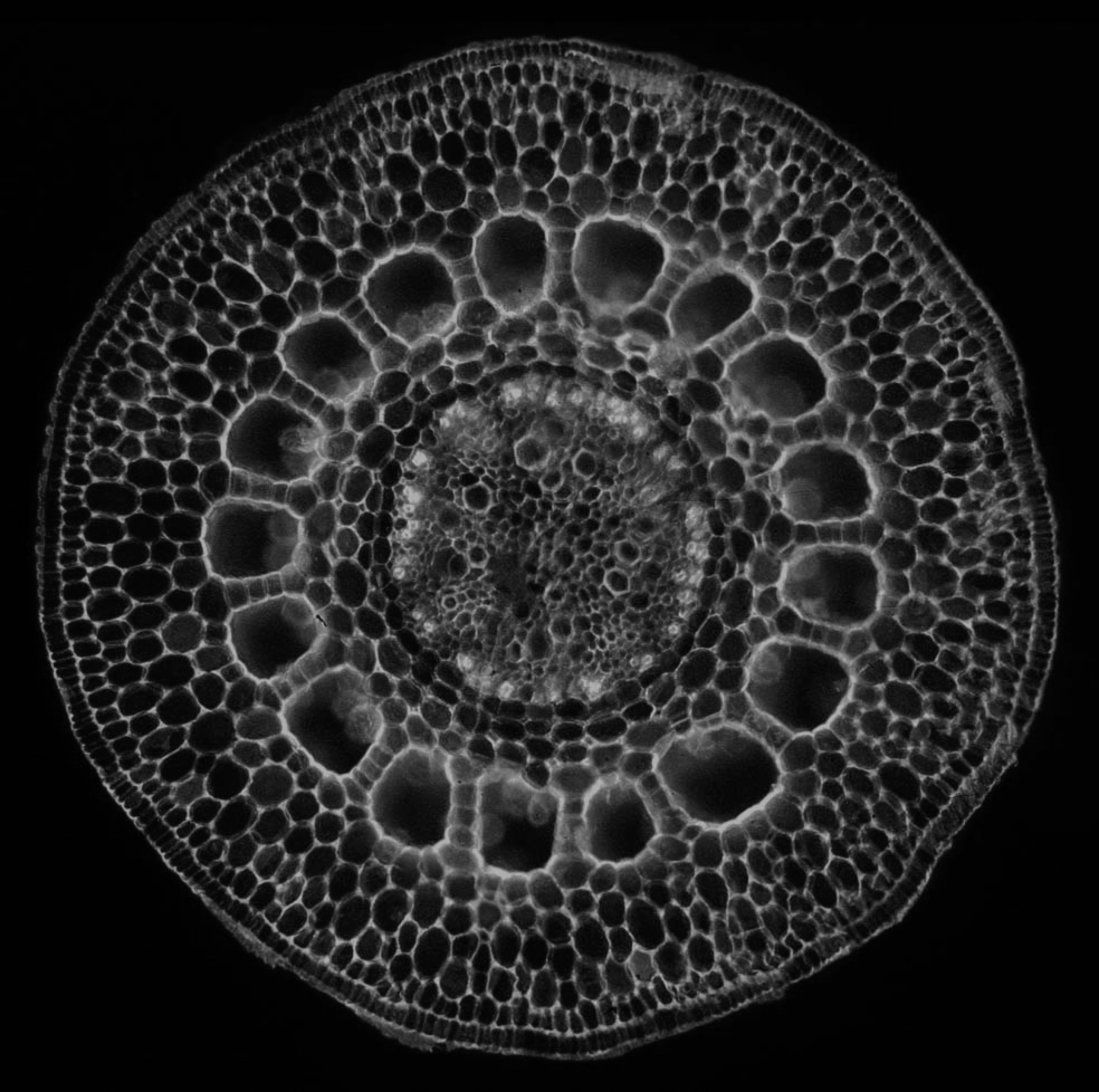}}
\subfigure[SR(bicubic)]{\includegraphics[width=.13\linewidth]{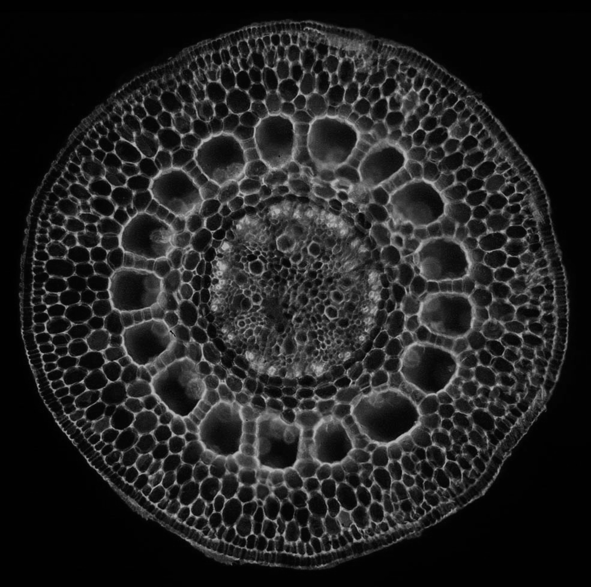}}
\subfigure[bilateral filter]{\includegraphics[width=.13\linewidth]{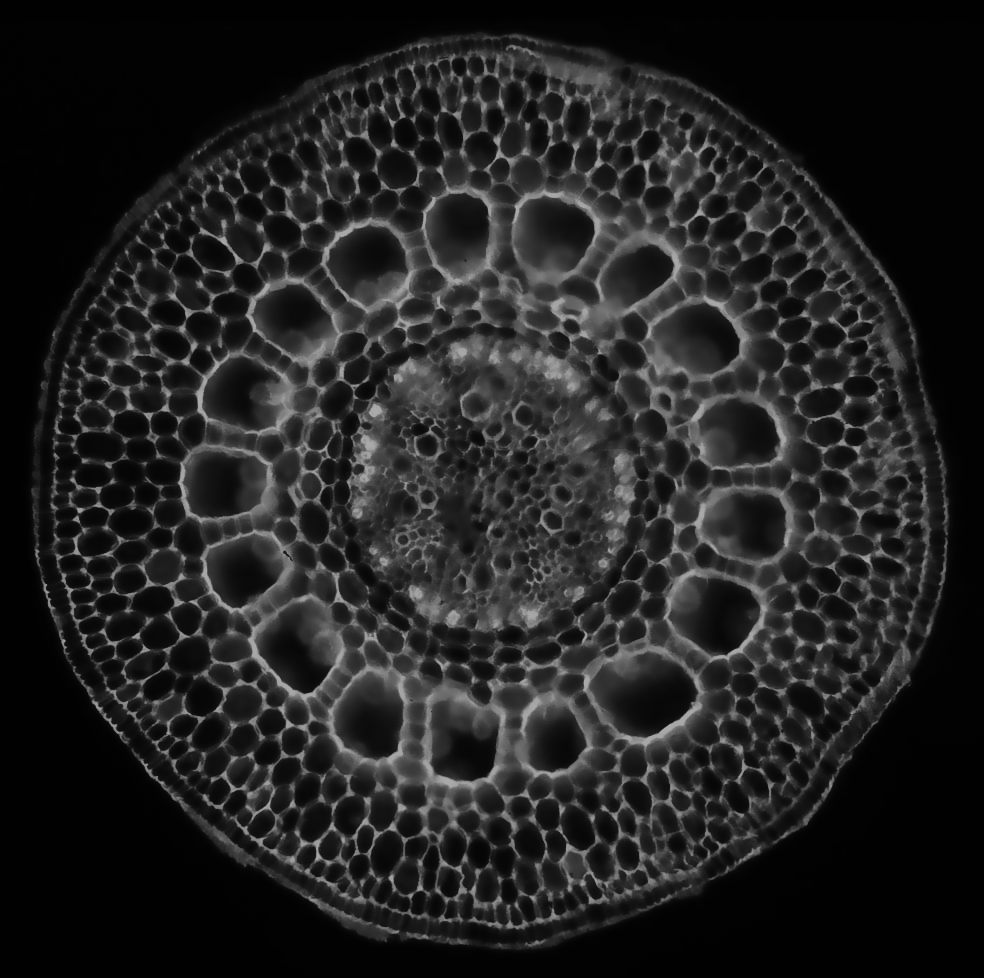}}
\subfigure[guided filter]{\includegraphics[width=.13\linewidth]{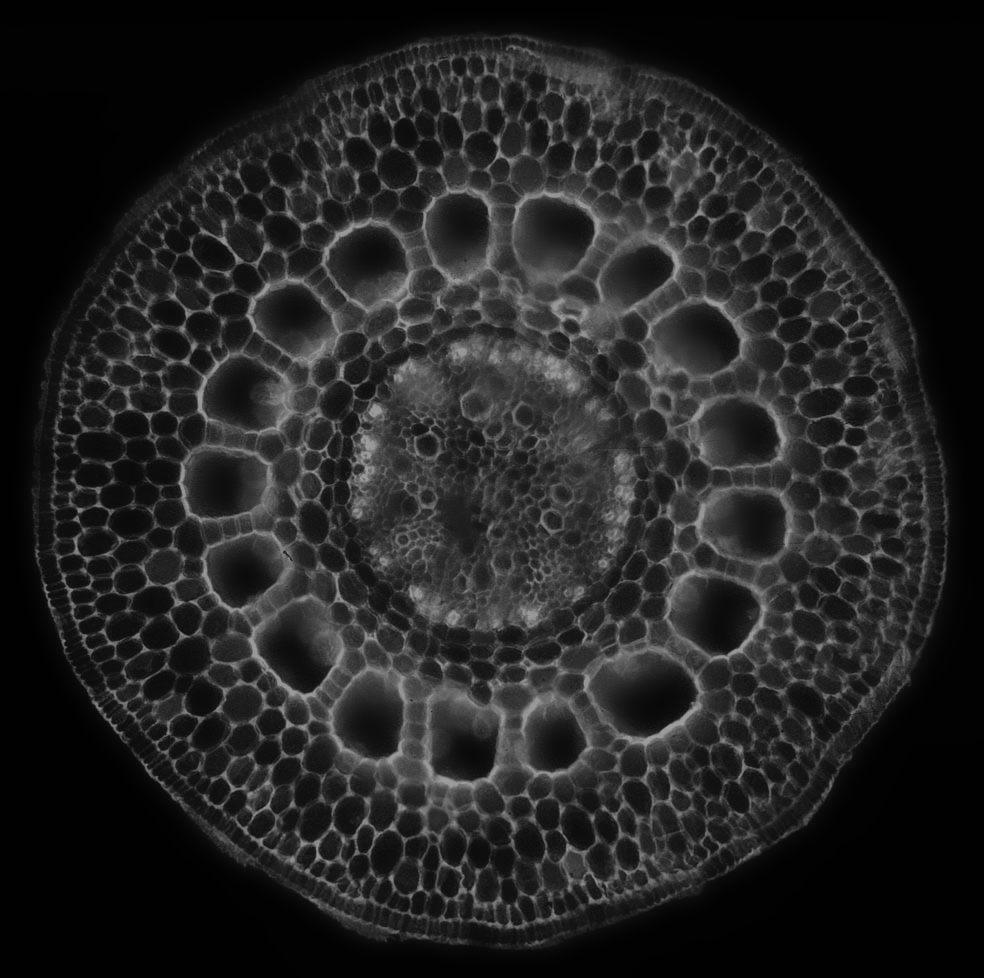}}
\subfigure[CDF($T$=0.40)]{\includegraphics[width=.13\linewidth]{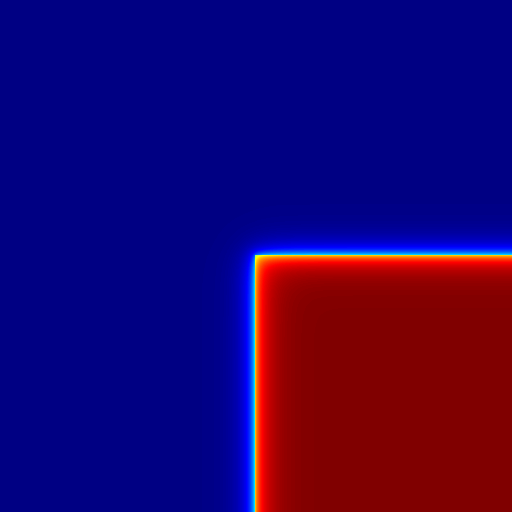}}
\subfigure[CDF($T$=0.72)]{\includegraphics[width=.13\linewidth]{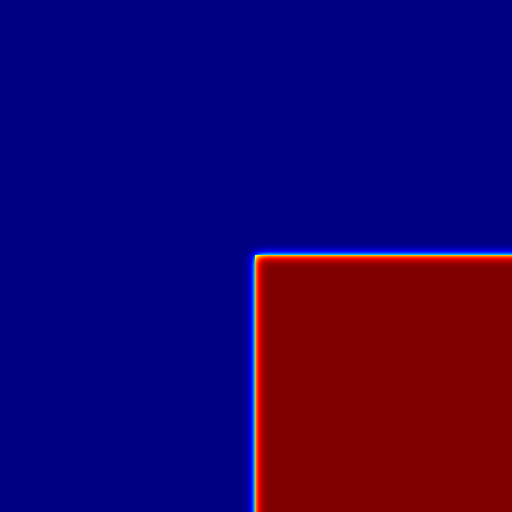}}
\subfigure[CDF($T$=0.03)]{\includegraphics[width=.13\linewidth]{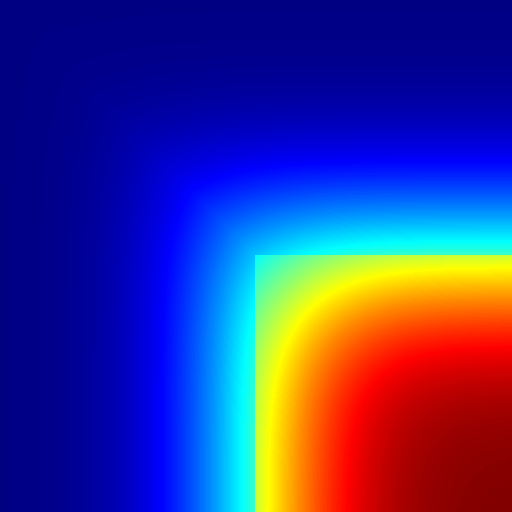}}
\subfigure[CDF($T$=0.84)]{\includegraphics[width=.13\linewidth]{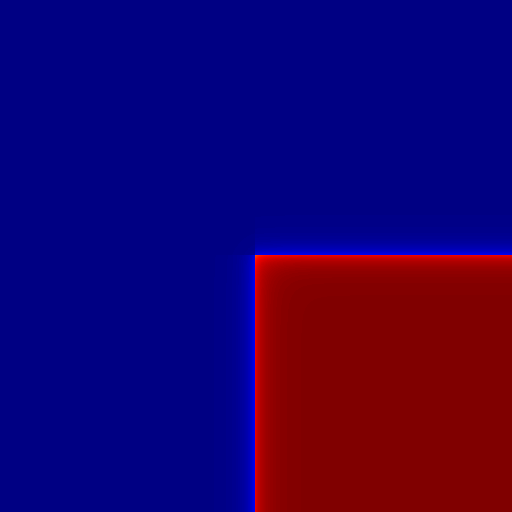}}
\subfigure[CDF($T$=0.80)]{\includegraphics[width=.13\linewidth]{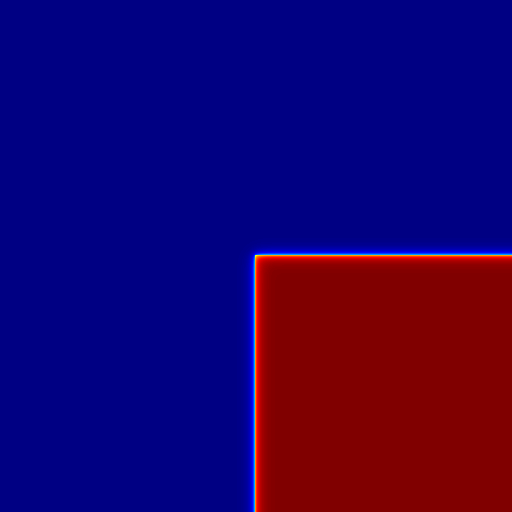}}
\subfigure[CDF($T$=0.60)]{\includegraphics[width=.13\linewidth]{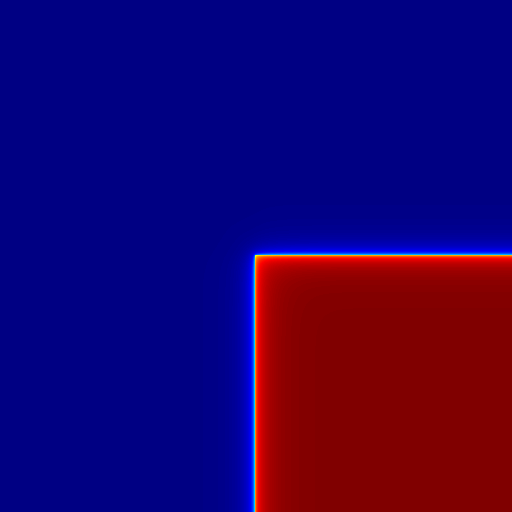}}
\subfigure[CDF($T$=0.58)]{\includegraphics[width=.13\linewidth]{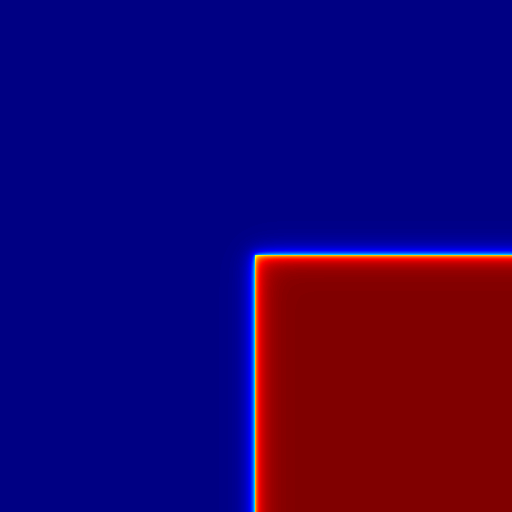}}
  \caption{Different images and their gradient CDFs (original image is from: beyondthehumaneye.blogspot.de).}
  \label{fig:CDF} %% label for entire figure
\end{figure*}

\begin{table}[!tb]
\scriptsize
\centering  % used for centering table
\begin{tabular}[width=\linewidth]{c|ccccccc|c} 
\hline\hline
Image set  & 1&	2&	3&	4&	5&	6&	7 & all \\
\hline
$a_1$ & 8.62 & 7.88&	7.76&	8.34&	9.13&	8.07&	8.89&	8.37 \\
$b_1$ & 0.51 & 0.52&	0.54&	0.52&	0.54&	0.53&	0.55&	0.53 \\
$c_1\times 10^5$ & -8.9 & -5.0&	-9.0&	-4.3&	-14&	-5.6&	-10&	-6.3 \\
\hline
SSE$\times 10^{-5}$ & 71 & 69&	1.1&	1.3&	1.3&	1.1&	8.7&	1.2 \\
$R^2$ & 0.89&	0.88&	0.92&	0.84&	0.89&	0.90&	0.92 & 0.91 \\
\hline\hline

$a_2\times 10^5$ & 10.9 & 4.81&	8.27&	4.42&	16.5&	5.21&	11.0&	6.21 \\
$b_2\times 10^2$ & 4.56 & 6.67&	4.22&	2.60&	1.85&	3.62&	1.01&	2.39 \\
$c_2$ & -4.74 & -4.38&	-4.60&	-4.99&	-5.79&	-4.91&	-6.03&	-5.24 \\
\hline
SSE$\times 10^{-5}$ & 1.13 & 0.892&	1.32&	1.47&	1.72&	1.26&	1.17&	1.48 \\
$R^2$ & 0.83&	0.85&	0.90&	0.82&	0.86&	0.89&	0.89 & 0.90 \\
\hline
\end{tabular}
\caption{Parameters and goodness of fit of the two-dimensional models: Model 1 (top) and Model 2 (bottom).} % title of Table
\label{table:2D} % is used to refer this table in the text
\end{table}

As in the 1D case, we can use the model CDF alternatively to the PDF: 
\begin{equation}
C(\vec{G})=\int\limits^{G^y}_{-\infty}\int\limits^{G^x}_{-\infty}P((u,v))\, \mathrm{d}u\mathrm{d}v \, . 
\label{eq:CDF}
\end{equation}

We approximate the 2D joint CDF by the parametric model:
\begin{equation}
%\widetilde{C}(\vec{G})=\left(\frac{\mathrm{atan}(TG^x)}{\pi}+\frac{1}{2}\right)\!\! \left(\frac{\mathrm{atan}(TG^y)}{\pi}+\frac{1}{2}\right) \, . 
\widetilde{C}(\vec{G})=\widetilde{C}(G^x)\widetilde{C}(G^y)\, ,
\label{eq:CDFM}
\end{equation} where $\widetilde{C}$ is defined in Eq.~\ref{eq:CDFModel2}.
The fitting results are shown in Table~\ref{table:CDF}. 

\begin{table}[h]
\scriptsize
\centering  % used for centering table
\begin{tabular}{c|ccccccc|c} 
\hline\hline
Image set  & 1&	2&	3&	4&	5&	6&	7 & all \\
\hline
T & 0.37 & 0.26&	0.38&	0.35&	0.56&	0.37&	0.7&	0.46 \\
\hline
SSE & 20.7 & 23.1&	19.1&	23.7&	22.9&	19.6&	23.0&	18.8 \\
$R^2$ & 0.99&	0.99&	0.99&	0.99&	0.99&	0.99&	0.99 & 0.99 \\
\hline
\end{tabular}
\caption{Fits of the parametric 2D CDF model.} % title of Table
\label{table:CDF} % is used to refer this table in the text
\end{table}

The 2D gradient CDF is sensitive to image transformations and potentially provides a powerful prior for the corresponding inverse problem. This is shown in Fig.~\ref{fig:CDF}, where an image is treated by different transformations and the corresponding CDFs are shown below. For the blurred image (Gaussian blur, $\sigma=3$), the frequency of small gradients is increased. For the noisy image ($10\%$ Gaussian noise), the frequency of large gradients is increased. For the super-resolution (SR) image (upsampling factor 9), the frequency of small gradients is increased. %but the result from bicubic interpolation increases more than the counterpart of nearest interpolation.     
For the bilateral filter ($w=5$, $\sigma_s=3$, $\sigma_c=0.1$) and the guided filter ($r=10$, $\epsilon=0.01$), the frequency of small gradients is increased.

The model in Eq.~\ref{eq:CDFM} has only a single scalar parameter: $T$. This parameter can easily be determined by solving the following convex  minimization problem:
\begin{equation}
\label{eq:parameter}
\min_T \int(\log(p)+T^2(G^x)^2+2\log(|G^x|))^2\mathrm{d}G^x,
\end{equation} which has the unique analytical solution:
\begin{equation}
\label{eq:parameterT}
T=\left(\frac{-
\int(2\log(|G^x|)+\log(p))(G^x)^2\mathrm{d}{G^x}}{\int (G^x)^4\mathrm{d}G^x}\right)^{\frac{1}{2}}.
\end{equation} 
Therefore, the parameter $T$ can directly be computed. The parameter $T$ leads to a linear gradient field remapping and guarantees an integrable gradient field. There is also an explicit relationship between parameter $a_2$ of the 1D marginal Model 2, and parameter $T$ of the 2D CFD model: $a_2=T^2$. This explains the good sensitivity of 1D Model 2 with respect to this parameter (see Fig.~\ref{fig:CDFchange}), and inspires the use of this parameter to define an image quality metric, as done next. 

\subsection{The Naturalness Factor}\label{sec:Nf}
Comparing any image's parameter $T$ with the expected value $T_\text{pr}$ from natural-scene images, i.e.~from the GDP, we define:

\noindent \fbox{%
  \parbox{\linewidth}{%
      {\bf Definition:} For any image $I$, the {\bf naturalness factor $N_{\!f}$} is defined as $N_{\!f}=\frac{T}{T_\text{pr}}$ and the image $I_n$ generated from $I$ such that $T_n\approx T_\text{pr}$ is called the {\bf naturalized image}.
  }%
}

Since $T_\text{pr}$ is obtained from the average gradient distribution of natural-scene images, the $N_f$ for a natural-scene image is expected to be one, as confirmed in Fig.~\ref{fig:NfDistr}. The range of possible values is $N_f\in [0.2,2.7]$, and the naturalness factor of natural-scene images satisfies a Log-Normally distribution (Fig.~\ref{fig:NfDistr}).

\begin{figure}[H]
  \centering
  \includegraphics[width=0.7\linewidth]{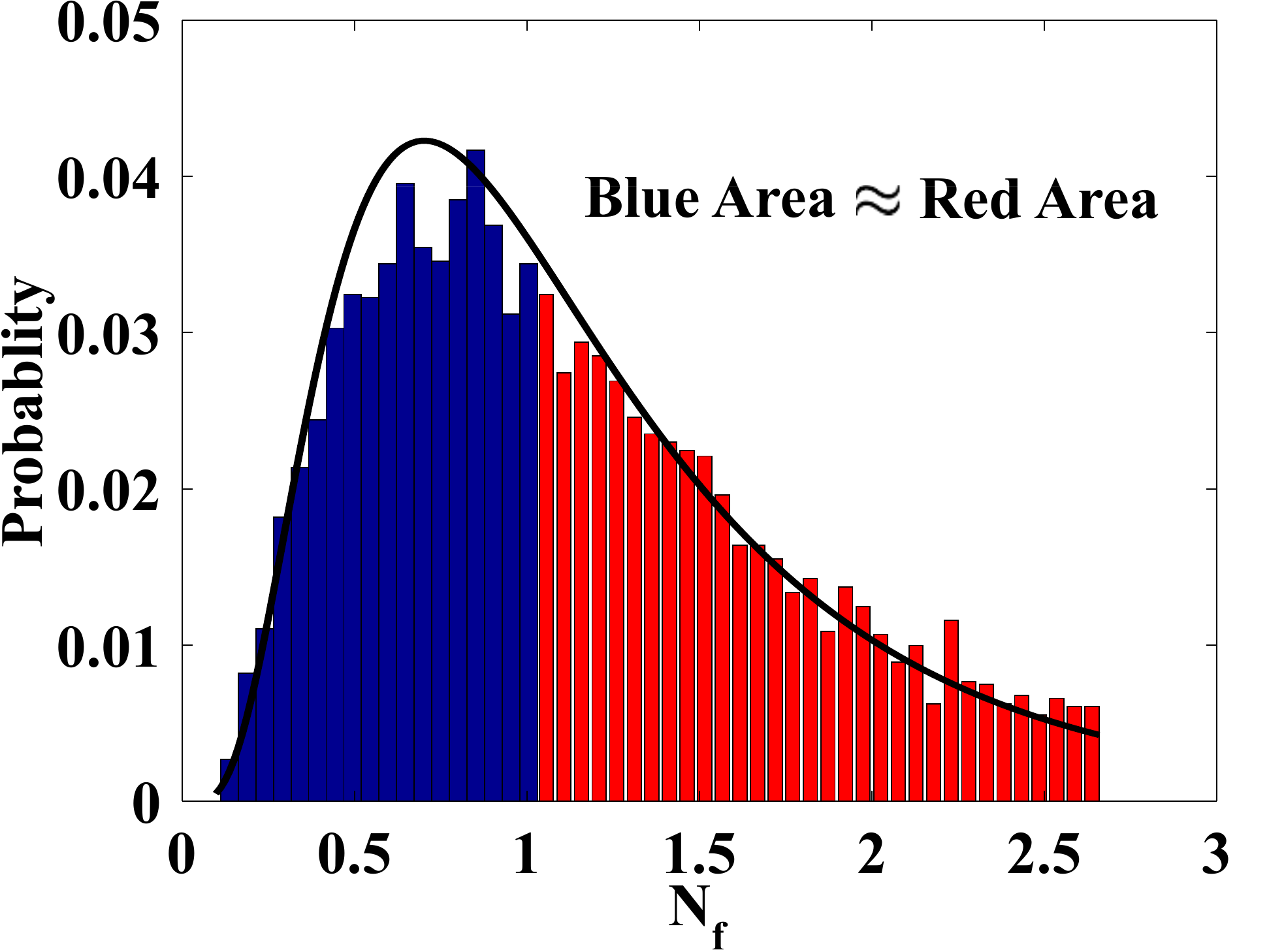}
  \caption{Naturalness Factor ($N_f$) distribution for the natural-scene images of the training set. Blue bars indicate $N_f<1$, red bars $N_f>1$. The black line is a Log-Normal distribution with parameters $\mu=0.039$ and $\sigma=0.613$. 
  %This distribution is consist with Fig.~\ref{fig:CDFchange}(b).
  }
  \label{fig:NfDistr} %% label for entire figure
\end{figure}

\subsection{Convexity, Sparsity, and Entropy of the GDP}
\label{sec:convex}
The TV (Laplacian) prior is so popular because it leads to convex variational models. Even though the Hyper-Laplacian fits the data better, it leads to a non-convex model, which is harder to solve. We show here that our {\textbf{Model 2}} (Eq.~\ref{eq:distr2}) and its 2D variant (Eq.~\ref{eq:2Ddistr2}) are quasi-concave, which means that all iso-sets are convex, hence simplifying optimization.
\begin{lem}
Eq.~\ref{eq:distr2} and Eq.~\ref{eq:2Ddistr2} are quasi-concave.
\end{lem}
\begin{proof}
For Eq.~\ref{eq:distr2}, we have:
\begin{equation}
\log (P(G_1^x)) < \log (P(G_2^x))~~ \mathrm{when}~~ G_1^x>G_2^x>0
\end{equation}
\begin{equation}
\log (P(G_1^x)) < \log (P(G_2^x))~~\mathrm{when}~~G_1^x<G_2^x<0.
\end{equation}
This monotonicity property with respect to 0 ensures that Eq.~\ref{eq:distr2} is quasi-concave.
Eq.~\ref{eq:2Ddistr2} is a rotation of Eq.~\ref{eq:distr2} with respect to the $y$ axis. Therefore, the set $\left\{ \vec{G}\,:\, \log (P(\vec{G}))\geq h\right\}$ is convex $\forall h$. 
\end{proof}
The overall energy function with Model 2 used as a prior, however, is not quasi-concave, 
but can be written as the difference of two convex functions. Such optimization problems are known as D.C.~problems (short for: different of convex), and efficient solvers are available for them. The present model hence leads to efficiently solvable variational problems while still fitting the data better than previous models. Table \ref{table:advantage} qualitatively compares different models.

\begin{table}[h]
\footnotesize
\centering  % used for centering table
\begin{tabular}{c|c|c|c} 
\hline\hline
Model  & convexity & accuracy & computation cost\\
\hline
Eq.~\ref{eq:distr}& quasi& high & medium\\
Eq.~\ref{eq:distr2}& quasi& high & low\\
Hyper-Laplacian & quasi& medium & low\\
Laplacian & yes& low & low\\
Gaussian & yes & very low & low \\
\hline
Eq.~\ref{eq:2Ddistr}& no& high & medium\\
Eq.~\ref{eq:2Ddistr2}& quasi& high & low\\
2D Hyper-Lap & no& medium & low\\
2D Laplacian & yes& low & low\\
2D Gaussian & yes & very low & low \\
\hline
Eq.~\ref{eq:CDFM}& no& medium & low\\
\hline
\end{tabular}
\caption{Comparison of different models.} % title of Table
\label{table:advantage} % is used to refer this table in the text
\end{table}

The gradient distribution balances sparsity and signal encoding. We define the sparsity of $p(\vec{G})$ as:
\begin{equation}
s_p(h) =\frac{\iint\chi_{p(\vec{G})>h}\mathrm{d}G^x\mathrm{d}G^y}{\iint\mathrm{d}G^x\mathrm{d}G^y} ,
\end{equation} where $\chi :R\rightarrow \{0,1\}$ is an indicator function. We further define:
\begin{equation}
C_h(h) =\iint \chi_{p(\vec{G})>h}p(\vec{G})\,\mathrm{d}G^x\mathrm{d}G^y \, ,
\end{equation}
which is the total probability mass on levels larger than $h$. The sparsity $s_p(h)$ measures how many words (of some dictionary) are needed to encode the information in $C_h(h)$ with an accuracy or tolerance $h$.  
The relationship between $s_p$ and $C_h$ is shown in Fig.~\ref{fig:sparsityD} for the image data from the training set. We observe that the gradient signal is sparse already at a low cutoff $h$. Figure~\ref{fig:sparsityM} shows the sparsity information curves for different parametric models of the GDP. Our Models 1 and 2 are about as sparse as the data, whereas all other models are less sparse and hence thrown away more information. This suggests that GDP could also be potentially interesting for compressed sensing~\citep{donoho:2006}. 

\begin{figure}[h]
  \centering
  {\includegraphics[width=0.7\linewidth]{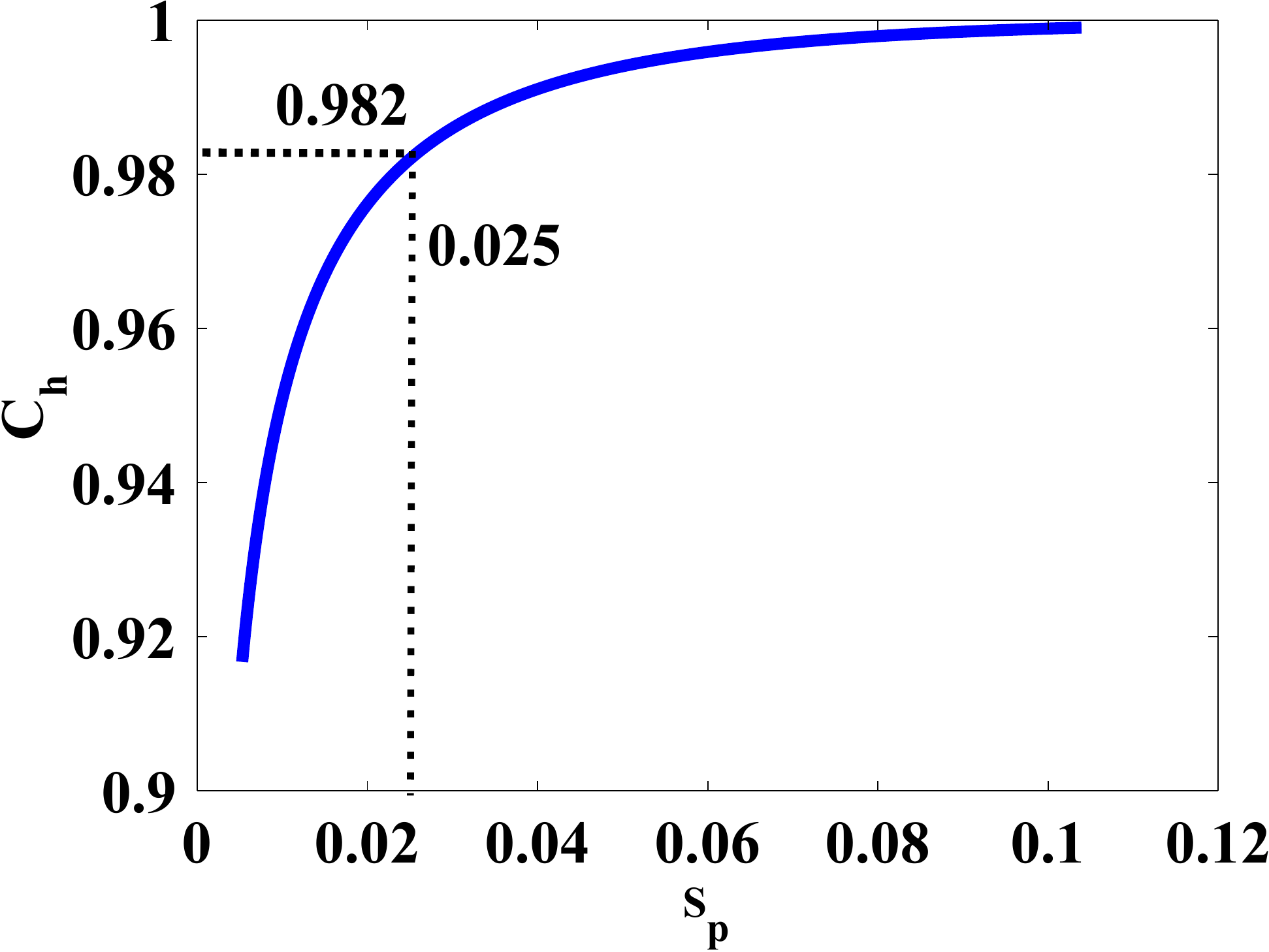}}
  \caption{Sparsity of the gradient distribution. $98.2\%$ of the information can be encoded with only $2.5\%$ of the dictionary at a cutoff level of $h=3.6\times 10^{-6}$.}
  \label{fig:sparsityD}
\end{figure}

\begin{figure}[h]
  \centering
  \subfigure[]{\includegraphics[width=0.7\linewidth]{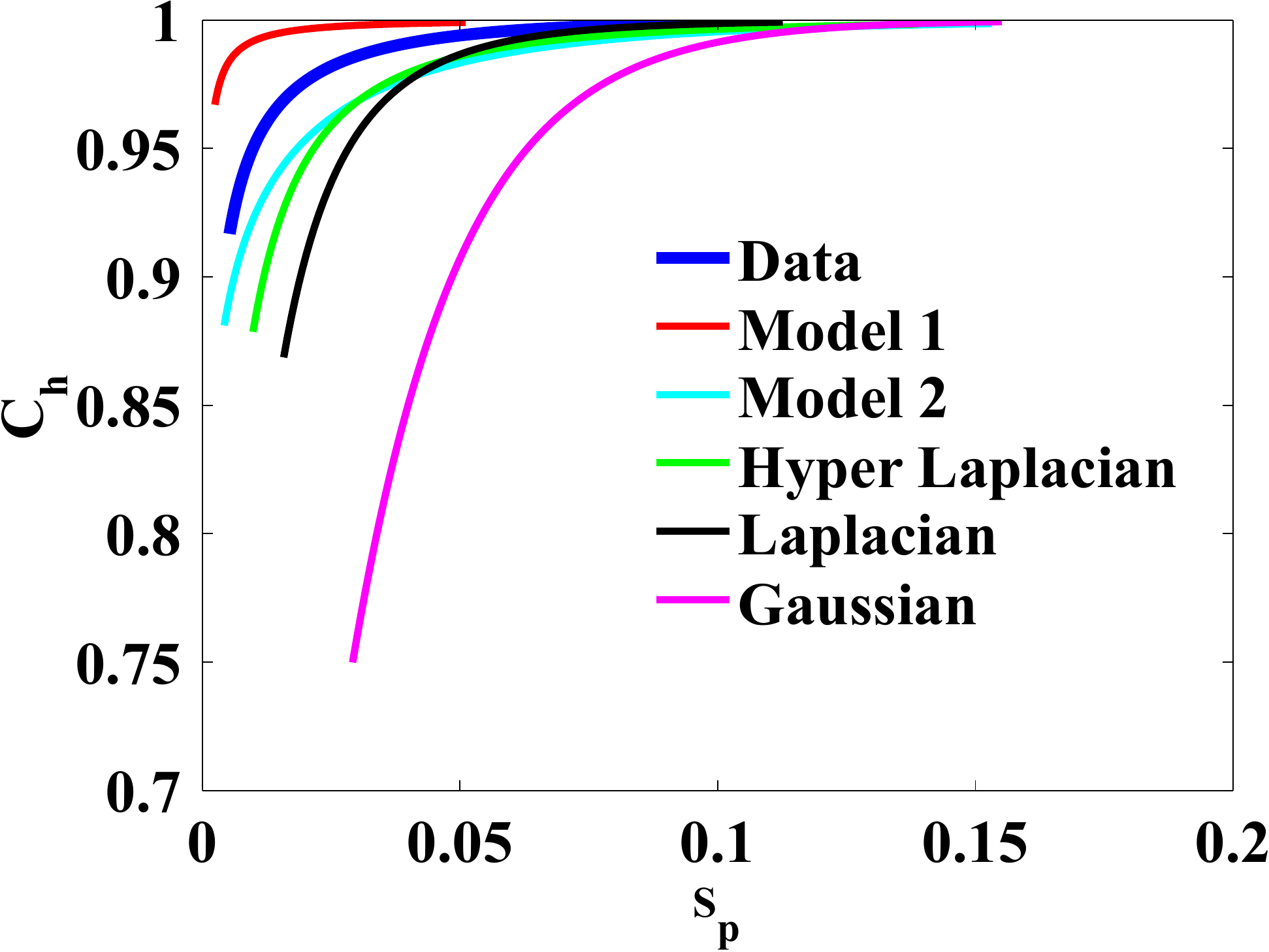}}
  \caption{Sparsity comparison for different models.}
  \label{fig:sparsityM} %% label for entire figure
\end{figure}

The information of a signal is closely related to its entropy, describing the macroscopic behavior of the system as discussed above. The entropy of a 2D gradient distribution is defined as:
\begin{equation}
\label{eq:entropy}
E(p) = -\iint p(\vec{G})\log(p(\vec{G}))\,\mathrm{d}G^x\mathrm{d}G^y\, .
\end{equation} 
Since the entropy is entirely determined by the gradient distribution, imposing a gradient prior implies imposing an entropy prior. The entropy distribution of the natural-scene images from the training dataset is show in Fig.~\ref{fig:entropy}. It is normally distributed. The average entropy and the entropies of different parametric models are given in Table~\ref{table:entropy}.

\begin{figure}[h]
  \centering
  {\includegraphics[width=0.7\linewidth]{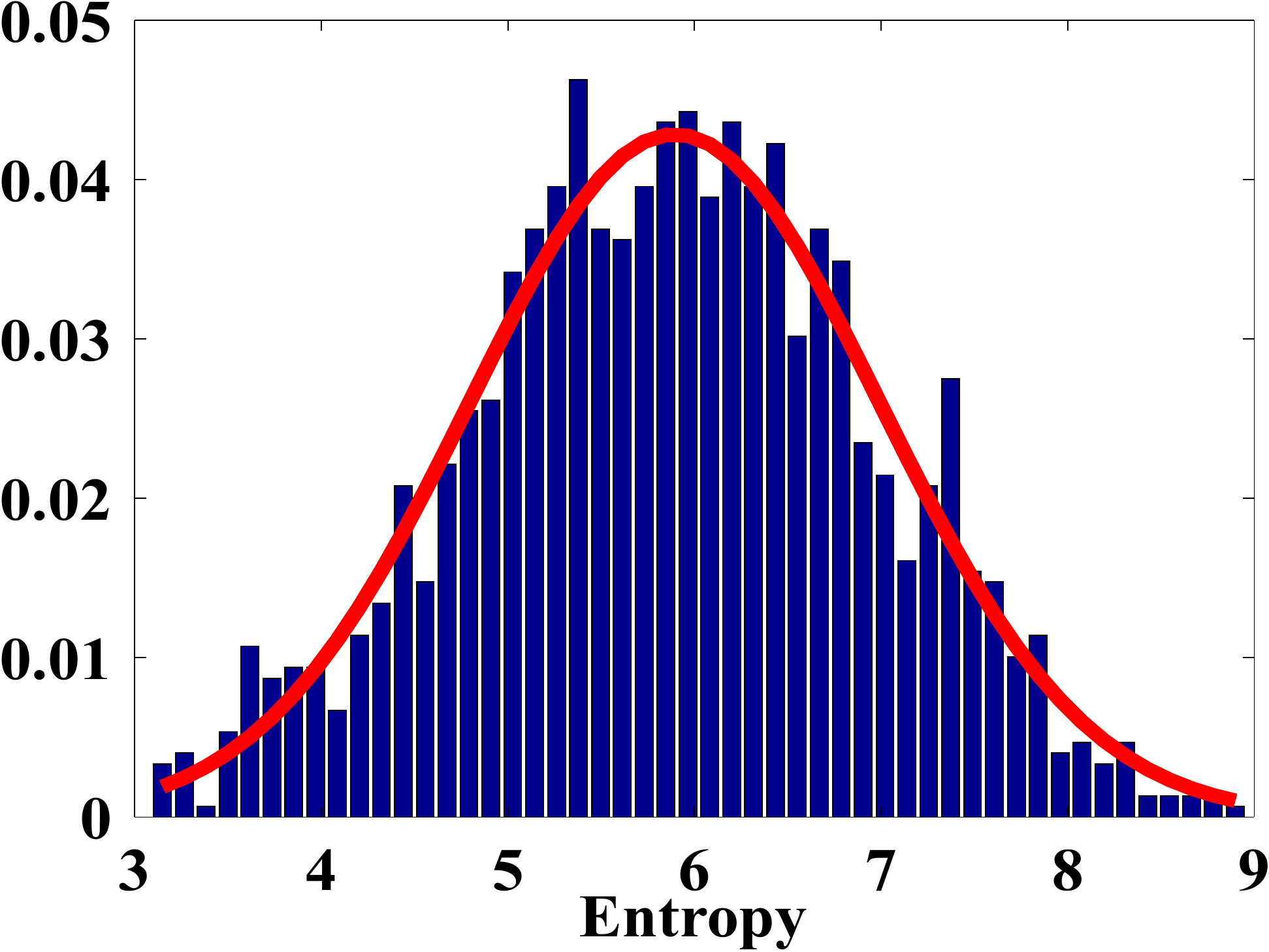}}
  \caption{Entropy distribution of the natural-scene images from the training set. The average is 5.88. The red line is the Gaussian $0.043\exp\left[{-(\frac{E-5.89}{1.56})^2}\right]$.}
  \label{fig:entropy} %% label for entire figure
\end{figure} 

\begin{table}[h]
\scriptsize
\centering  % used for centering table
\begin{tabular}{c|c|c|c|c|c|c} 
\hline\hline
  & data & Eq.~\ref{eq:2Ddistr} & Eq.~\ref{eq:2Ddistr2} & HyperLap & Laplace & Gaussian\\
\hline
Entropy  & 5.88 & 9.05 & 2.05 & 2.94 & 31.2 & 223\\
\hline
\end{tabular}

\caption{Entropy of the training data and of different parametric models.} % title of Table
\label{table:entropy} % is used to refer this table in the text
\end{table}

\section{Correlation with Image Quality}
\label{sec:correlation}
\begin{figure*}[!htb]
\centering
\includegraphics[width=.32\linewidth ]{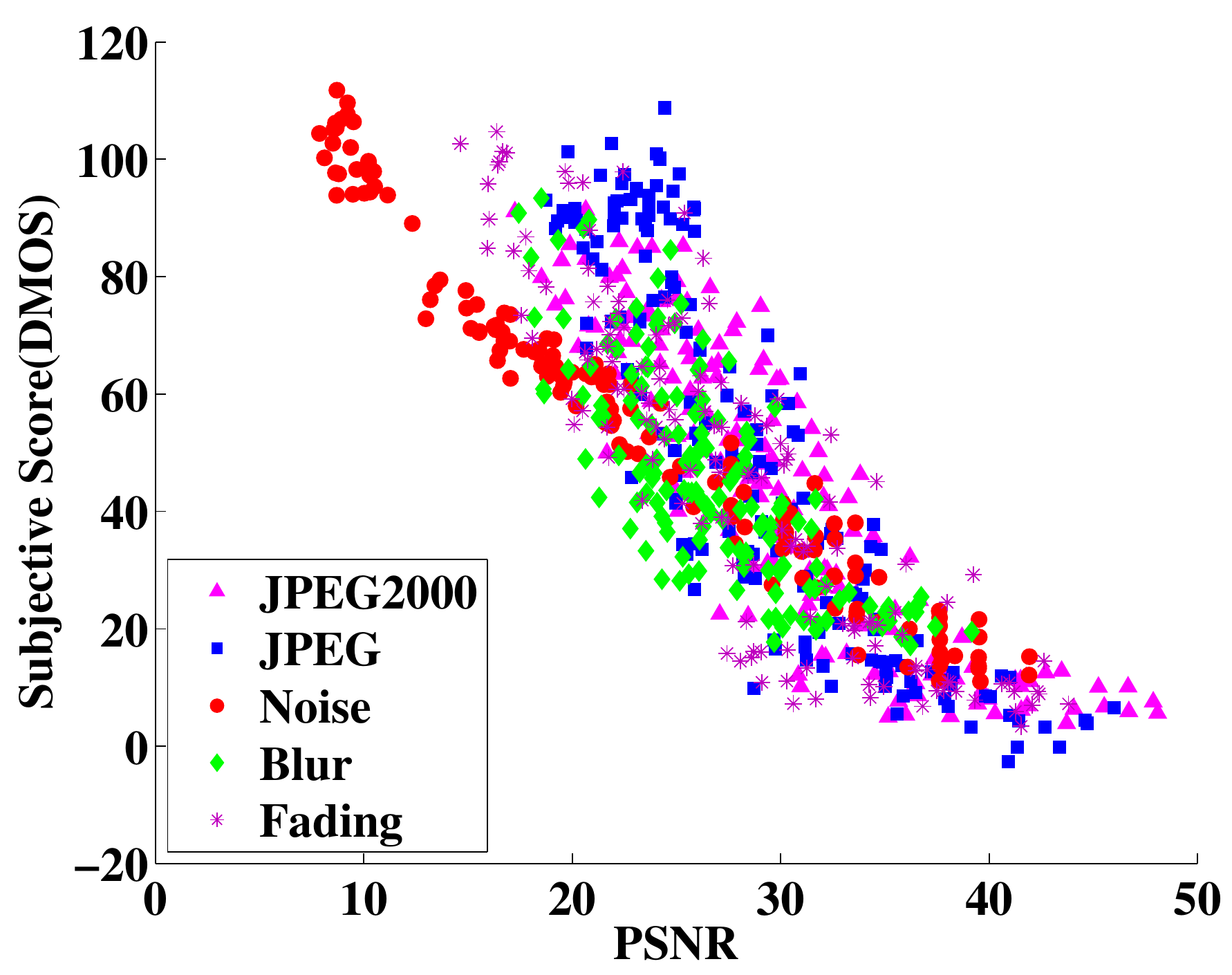}
\includegraphics[width=.32\linewidth]{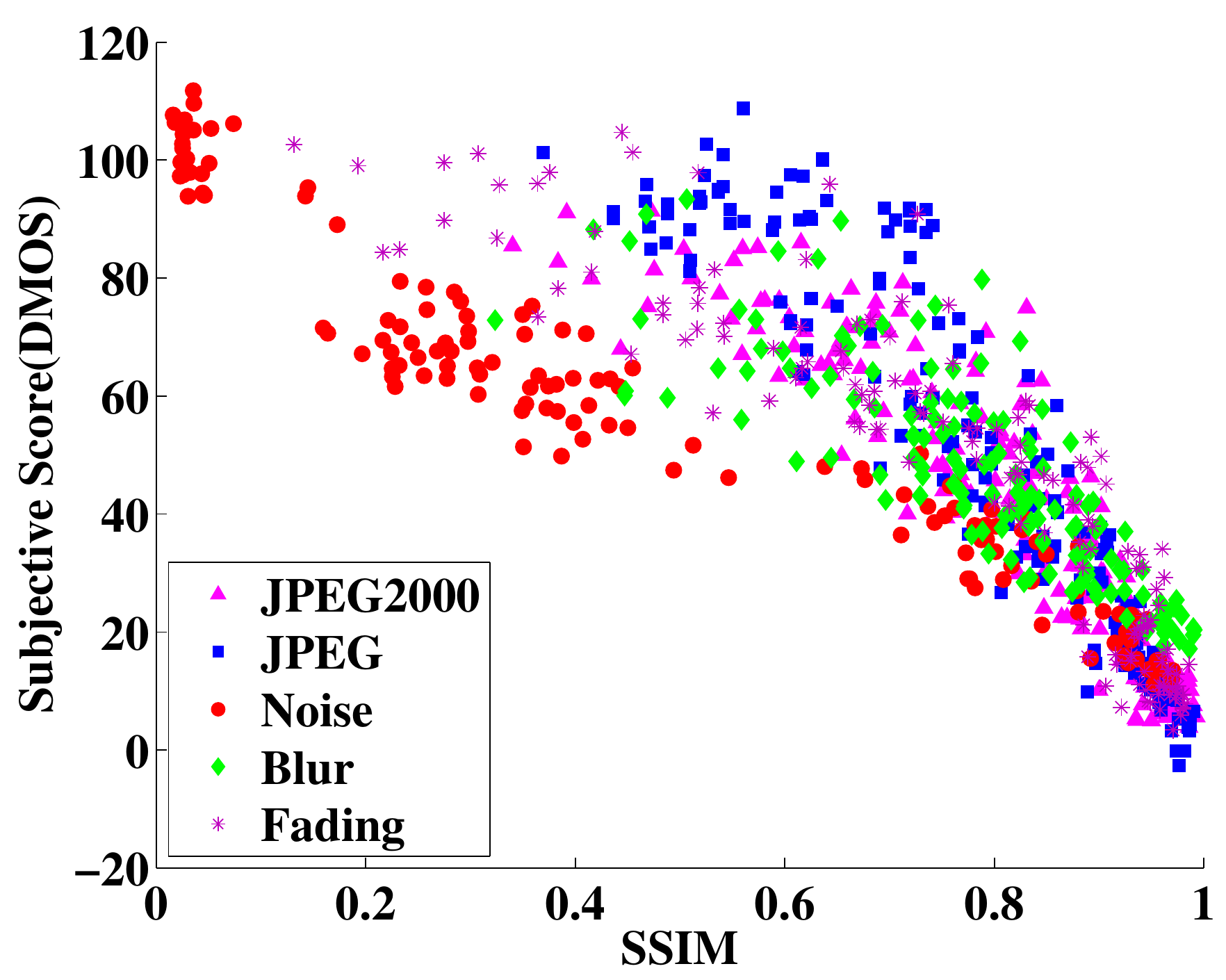}
\includegraphics[width=.32\linewidth]{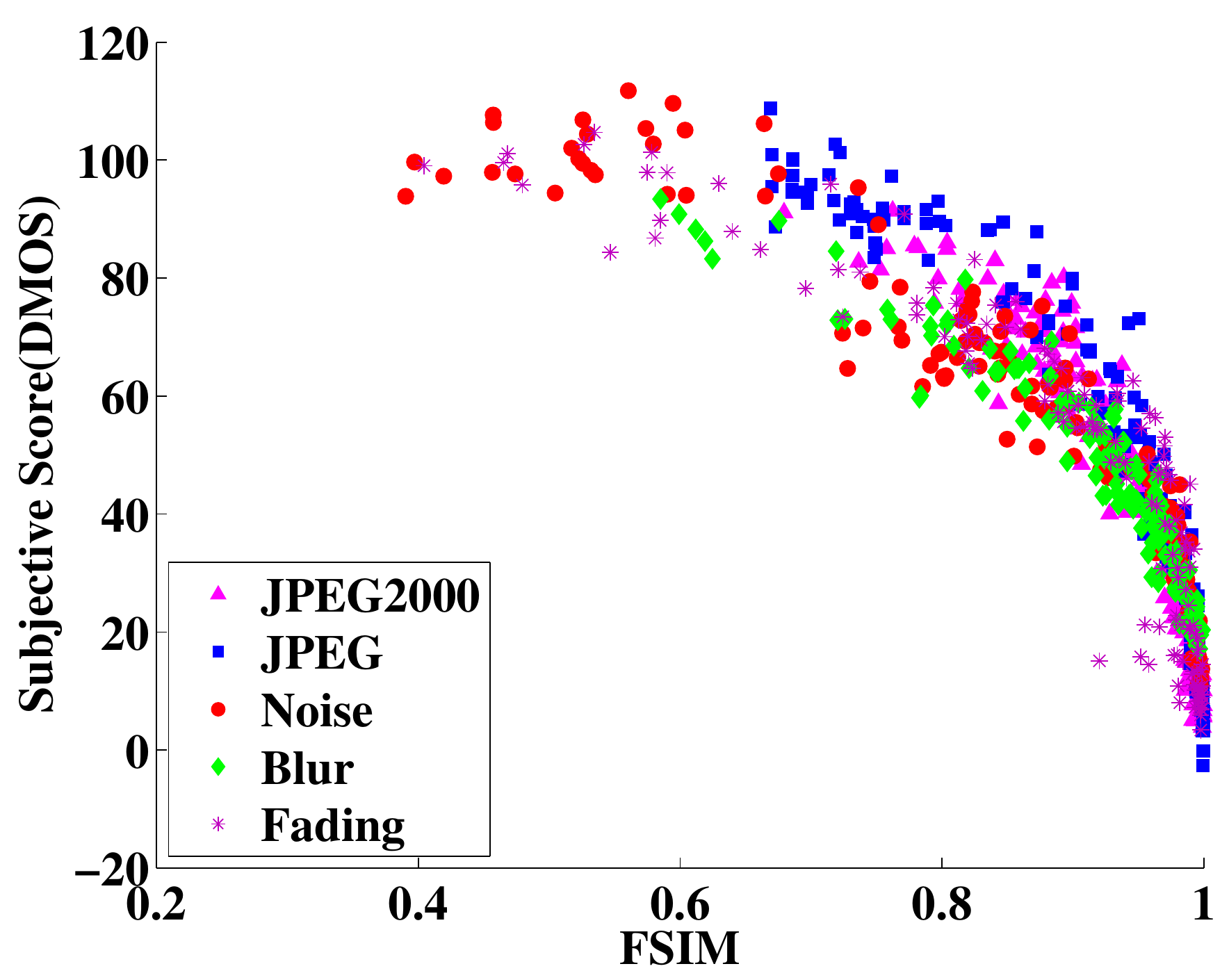}

\includegraphics[width=.32\linewidth]{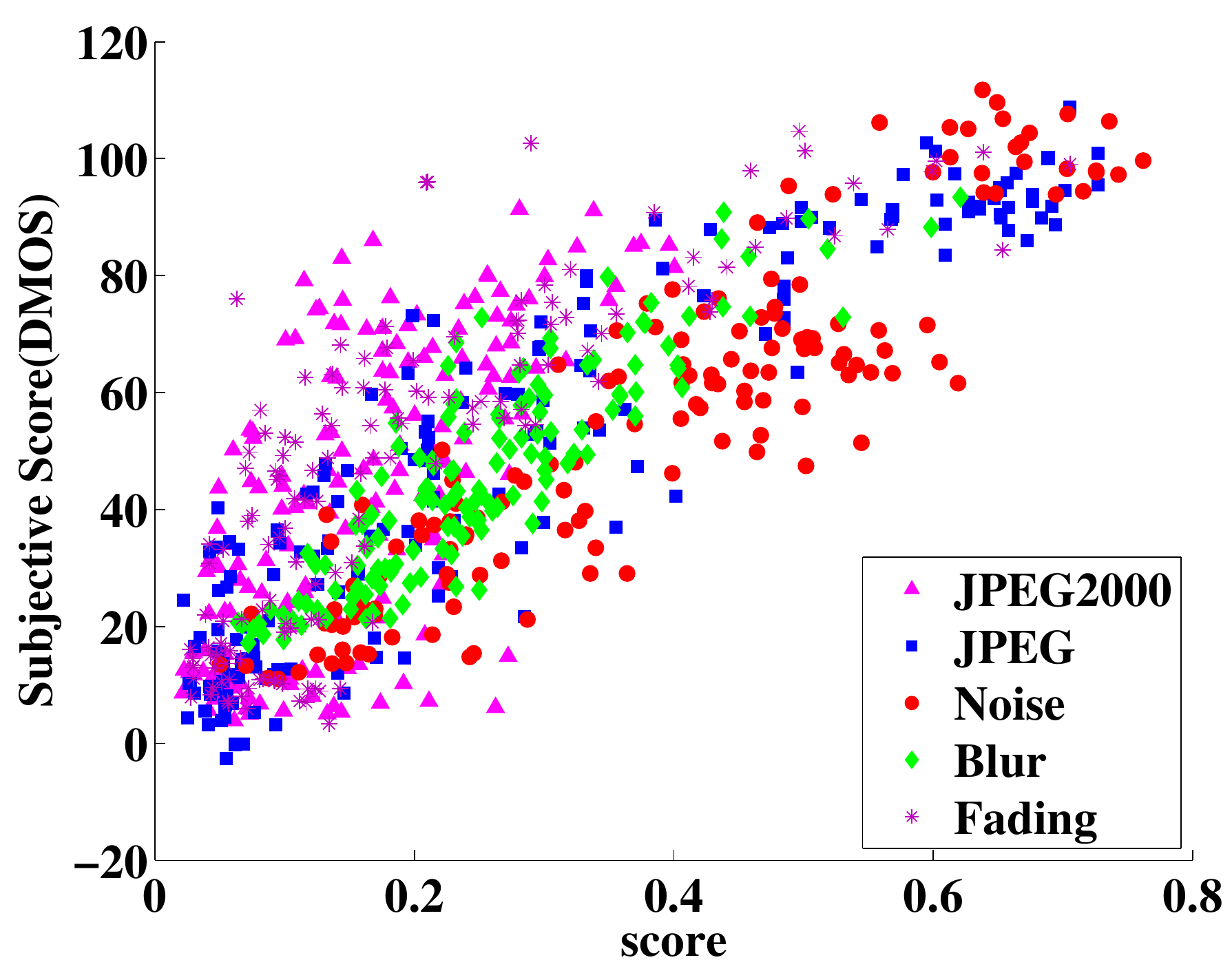}
\includegraphics[width=.32\linewidth]{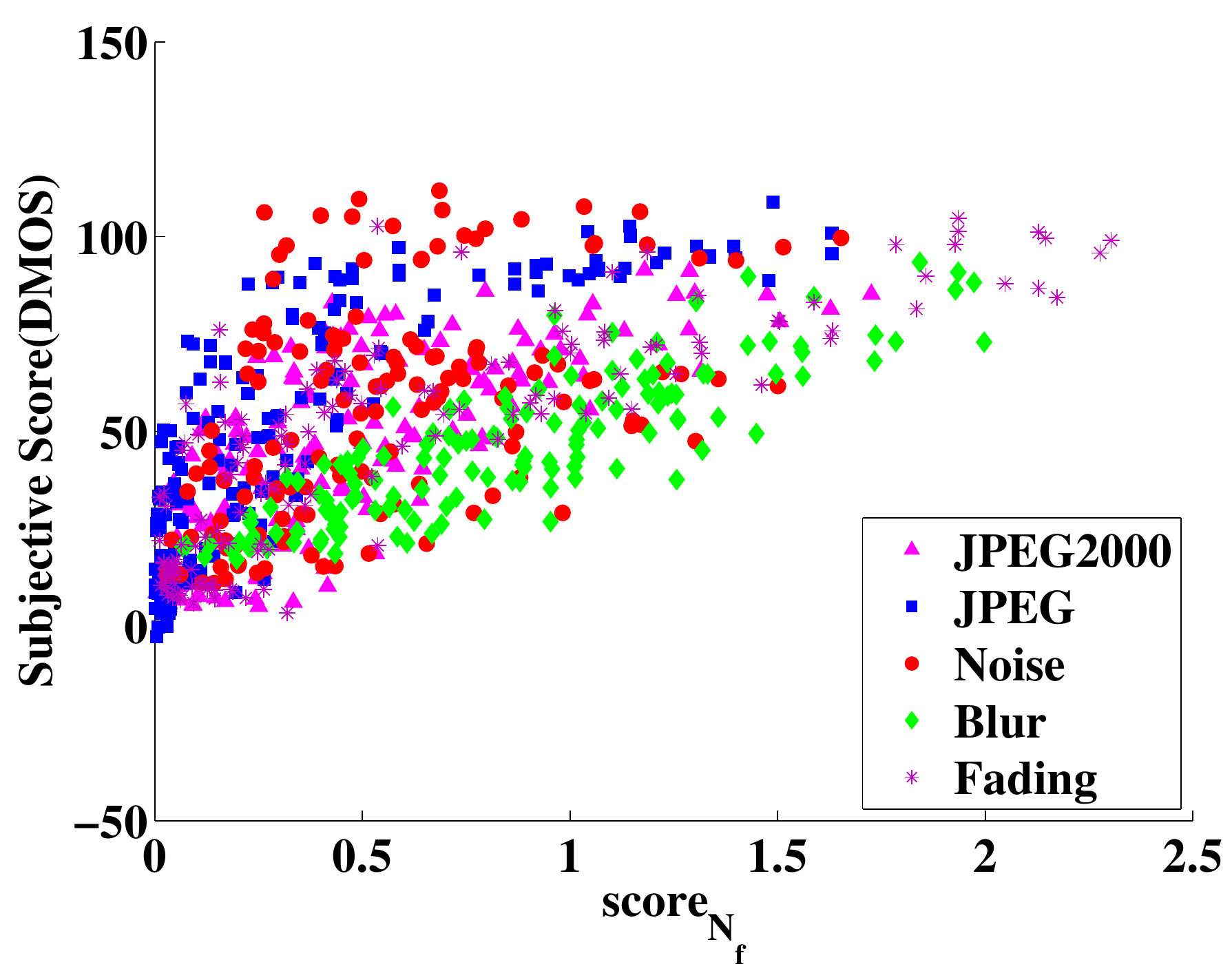}
\includegraphics[width=.32\linewidth]{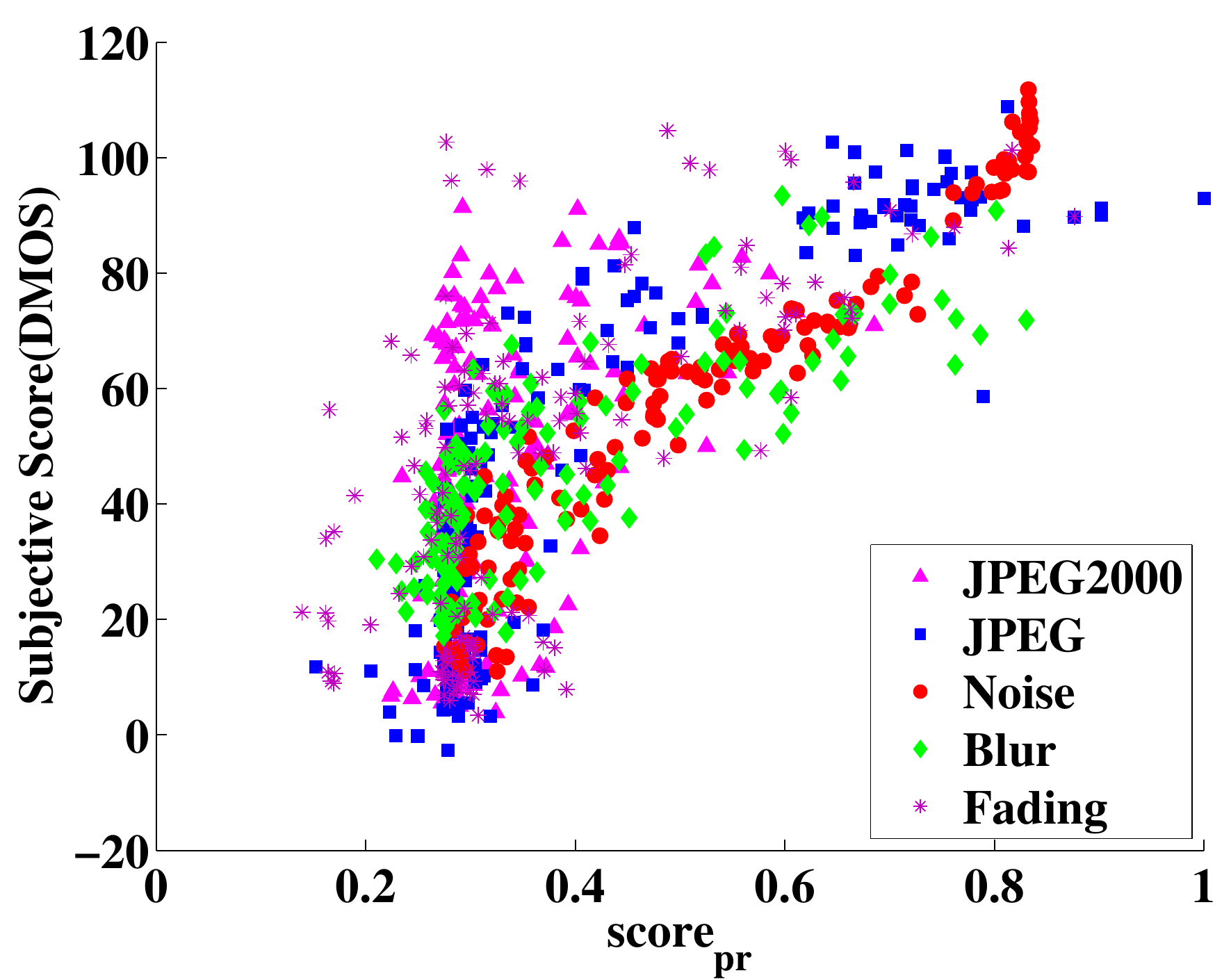}
\caption{Image quality assessment results on the LIVE benchmark. We show the scatter plots of different objective image-quality measures (from left to right: PSNR, SSIM, FSIM; $score$, $score_{pr}$, $N_{\!f}$) with subjectively perceived image quality as quantified by the DMOS score. The different colors correspond to the different distortion in the LIVE benchmark, as given in the inset legends.}
 \label{fig:quality} %% label for entire figure
\end{figure*}
As a second key requirement for an image-processing prior, besides stability, we show next that the GDP is highly correlated with subjectively perceived image quality. We show that the distance between the gradient distribution of any given image and the GDP correlates with image quality. We use the standard LIVE benchmark dataset for image quality assessment~\citep{SSIM}. In order to measure the distance between two gradient distributions, we test the $\ell_2$ norm, $\ell_1$ norm, cosine distance, the Earth Mover Distance (EMD), $\chi^2$ distance, and the Hellinger distance. Using this distance, we form the objective image-quality score:
\begin{eqnarray}
\label{eq:score}
score = Distance(p(\nabla I^{\mathrm{true}}),p(\nabla I^{\mathrm{distorted}})).
\end{eqnarray} 
This can be further simplified to only use $N_{\!f}$:
\begin{eqnarray}
\label{eq:scoreNf}
score_{N_{\!f}} = |N_{\!f}^{\mathrm{true}}-N_{\!f}^{\mathrm{distorted}}|.
\end{eqnarray}
If the ground-truth image is unknown (i.e., in a real-world application rather than a benchmark setting), the score is defined with respect to the GDP:
\begin{eqnarray}
\label{eq:scorePr}
score_{pr} = Distance(p^{\mathrm{pr}},p(\nabla I^{\mathrm{distorted}}))\, .
\end{eqnarray} 

A measure of subjectively perceived image quality is provided by the LIVE benchmark's DMOS (difference mean opinion score). Different correlations between DMOS and our objective score are reported in Tab.~\ref{table:correlation}. 
\begin{table}[!htb]
\scriptsize
\centering
\begin{tabular}{c|cccccc}
\hline\hline
     & $\ell_2$ & $\ell_1$ & cos & EMD & $\chi^2$ & Hellinger\\
\hline
PCC & 0.6193 & 0.7926 & 0.6277 & 0.5172 & 0.7662 & \textbf{0.8687}\\
SCC & 0.6434 & 0.7773 & 0.6114 & 0.7576 & 0.7977 & \textbf{0.8630}\\
KCC & 0.4588 & 0.5822 & 0.4355 & 0.5639 & 0.6027 & \textbf{0.6745}\\
\hline
\end{tabular}
\caption{Correlations between subjectively perceived image quality (DMOS from LIVE benchmark~\citep{SSIM}) and our objective score using different distance metrics. The following correlations are reported: Pearson's linear correlation coefficient (PCC), Spearman's rank-order correlation coefficient (SCC), and Kendall's rank-order correlation coefficient (KCC). In all cases, the Hellinger distance between the gradient distributions shows the best correlation.}
\label{table:correlation}
\end{table}    

In all cases, the Hellinger distance between the gradient distributions shows the best correlation. This suggests that using this distance metric, the GDP can directly be used to construct a novel image-quality measure. We compare this new measure (i.e., the score of Eq.~\ref{eq:scorePr}) with other stat-of-art image quality assessment methods, such as PSNR (peak signal-to-noise ratio), SSIM~\citep{SSIM}, and FSIM~\citep{FSIM}. The results are shown in Fig.~\ref{fig:quality} and Table~\ref{table:quality}. The Hellinger distance between the gradient distribution of an image and the ground-truth GDP shows high linearity with DMOS, rendering the GDP a favorable prior for image processing. When used as an image-quality metric, however, the score without knowing ground truth ($score_{pr}$) is less good than specialized metrics like SSIM~\citep{SSIM} and FSIM~\citep{FSIM} (Table~\ref{table:quality}). This is entirely expected, and it is not out aim to propose a new image-quality metric. However, since differences in the gradient distribution correlate with image quality, imposing the GDP is expected to improve an image's quality. Together with its stability, this renders the GDP a good prior for practical applications, as illustrated in Sec.~\ref{sec:app}.

\begin{table}
\scriptsize
\centering
\begin{tabular}{c|cccccc}
\hline\hline
      & PSNR & SSIM & FSIM & $score$ & $score_{N_f}$ & $score_{pr}$
      \\
\hline
PCC  & -0.8585 & -0.8252 & -0.8586  & \color{ForestGreen}{\textbf{0.8687}} & 0.669 & 0.761 
\\
SCC  & -0.8756 & -0.9104 & -0.9634  & 0.8630 & 0.712 & 0.706 
\\
KCC  & -0.6865 & -0.7311 & -0.8337  & 0.6745 & 0.512 & 0.522 
\\
\hline
\end{tabular}
\caption{Image quality assessment results on the LIVE benchmark.}
\label{table:quality}
\end{table}    
%With the gradient distribution prior, a single image's quality can be evaluated without any reference image. For example, $N_f$ in Fig~\ref{fig:naturalized} indicates the image quality, which is consist with human visual perception. 

\section{Imposing the GDP in Variational Problems}
\label{sec:remap}
In a variational framework, there are two ways to impose a prior: as a {\em hard constraint} or as a {\em soft constraint}. Both are possible for the GDP. For a hard constraint, the GDP is imposed by gradient remapping. The mapped gradient field is then used to reconstruct the output image by solving a Poisson equation with proper boundary conditions. For a soft constraint, the GDP can be imposed as a regularization term, leading to a minimization problem. As shown in Sec.~\ref{sec:app}, the decision between using a soft or hard constraint depends on the specific application.

\subsection{As a hard constraint}
We impose the GDP prior as a hard constraint by gradient-field remapping. The idea is to map the original gradient field, using a linear or nonlinear mapping function, into a new gradient field that exactly satisfies the prior. From this remapped gradient field, the output image is then reconstructed by solving a Poisson equation. In the special case of a linear mapping function, the reconstruction simplifies to rescaling the image pixel values.

While it is possible to directly map the input gradient field to the desired distribution~\citep{ExactHist, ExactHistVariation}, such non-parametric mappings lead to numerical ambiguity due to discretization of the distribution into bins. This approach also does not guarantee the output gradient field to be integrable. We hence instead propose the use of parametric mapping functions. At the expense of some accuracy, they guarantee integrability of the result and lead to well-posed reconstruction problems.

\subsubsection{Gradient field remapping}
\label{sec:GDS}
Let $Map$ remap the gradient field to a new field $\vec{G}_n$, which satisfies the GDP:
\begin{equation}
\label{eq:map}
\vec{G}_n = Map(\vec{G}),\ \ s.t.\ \ p(\vec{G}_n) = p^{\mathrm{pr}}.
\end{equation}

In general, $Map$ can be non-parametric, parametric, nonlinear, or linear. A parametric nonlinear mapping leads to an integrable field, and the final image can be obtained by solving Poisson equation (Eq.~\ref{eq:Poisson}), as outlined below. A linear parametric mapping leads to a simple rescaling of the pixel intensities and no Poisson equation needs to be solved. Fig.~\ref{fig:Poisson} illustrates the effects of different types of remapping. Linear remapping amounts to a simple rescaling of the intensities such that the gradient distribution fits the GDP in average. However, the fit obtained by nonlinear remapping is much better, but requires solving a Poisson equation. The gradient field entropies after remapping are 6.03 (original), 5.86 (reconstruction without remapping), 6.25 (linear remapping), and 5.79 (non-linear remapping), respectively. As expected, remapping makes the entropy of $\vec{G}_n$ closer to 5.88, the average entropy of natural-scene images. 

\subsubsection{Image reconstruction} 
\label{sec:reconstruction}
Reconstructing the output image from the remapped gradient field amounts to minimizing the following $q\mbox{-}Dirichlet$ energy:
\begin{equation}
\label{eq:reconstruction}
\begin{split}
& \arg \min_{I_n} \left\{\|\nabla I_n - \vec{G}_n \|_{q}\right\}\\
& s.\, t.~~ I_n\in Lip(\Omega)\, ,
\end{split}
\end{equation}
where $q>1$, $\|\cdot \|_q$ is the standard $\ell_q$ norm, and $Lip(\Omega)$ is the space of Lipschitz-continuous functions on domain $\Omega$. 

\begin{figure}[H]
\centering
\subfigure[Original]{
\includegraphics[width=0.48\linewidth]{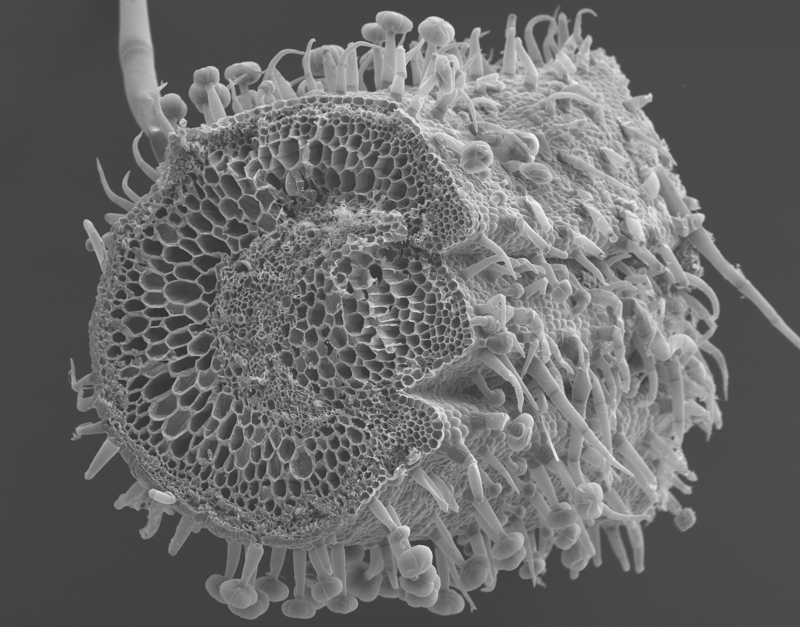}}
\subfigure[Original]{
\includegraphics[width=0.48\linewidth]{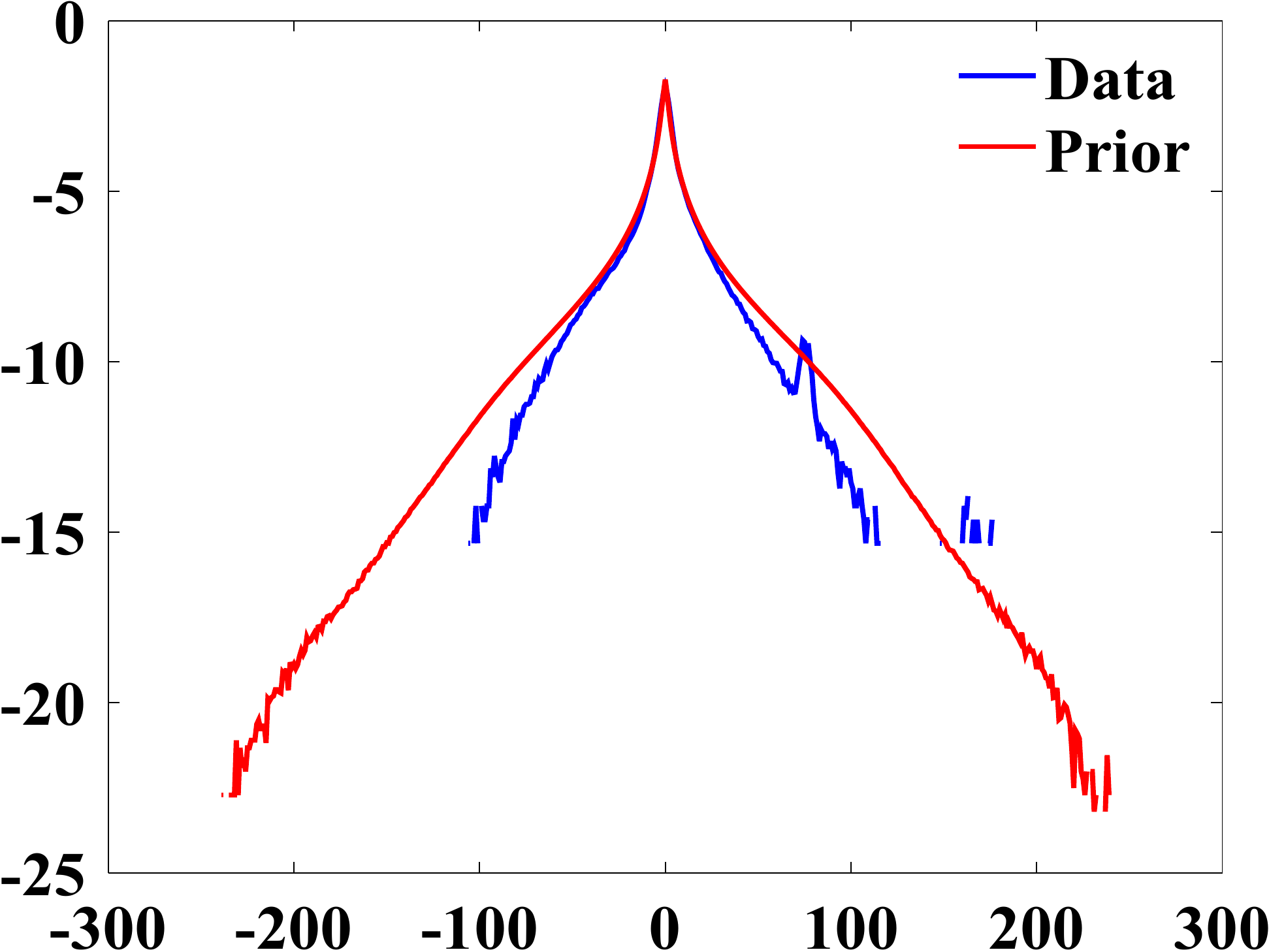}}
\subfigure[Reconstruction w/o map]{
\includegraphics[width=0.48\linewidth]{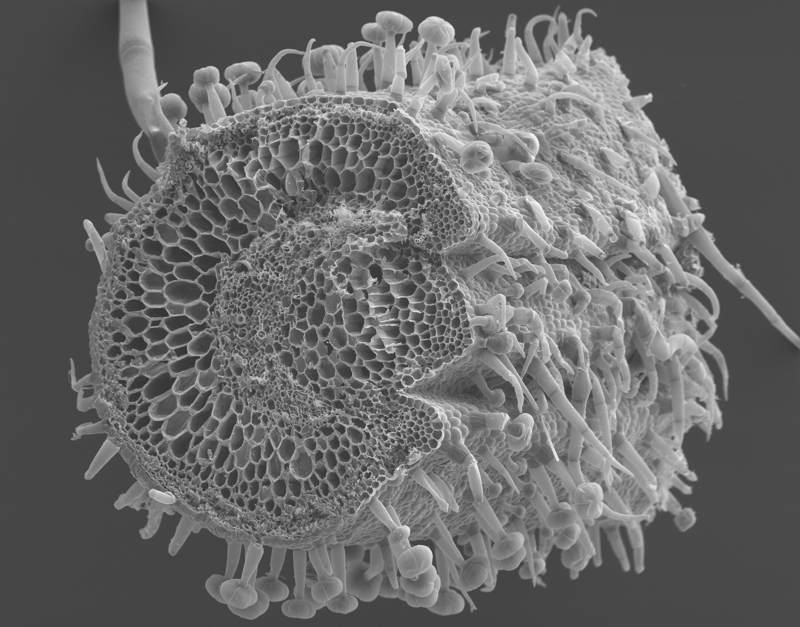}}
\subfigure[Reconstruction w/o map]{
\includegraphics[width=0.48\linewidth]{PoissonResultNoMap.pdf}}
\subfigure[Result with linear map]{
\includegraphics[width=0.48\linewidth]{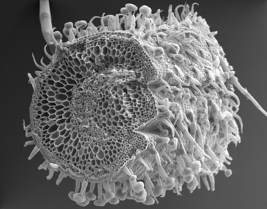}}
\subfigure[Result with linear map]{
\includegraphics[width=0.48\linewidth]{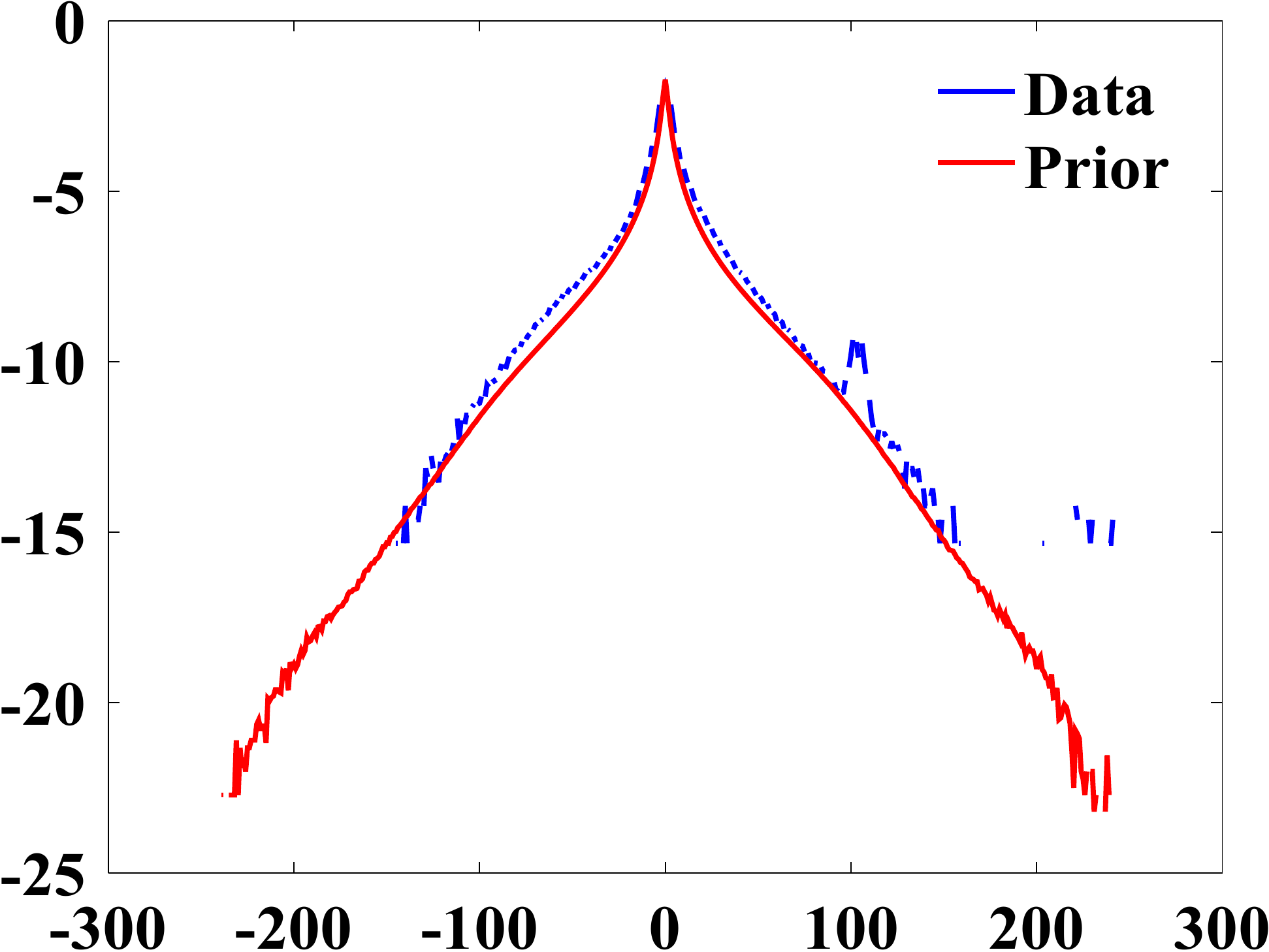}}
\subfigure[Result with nonlinear map]{
\includegraphics[width=0.48\linewidth]{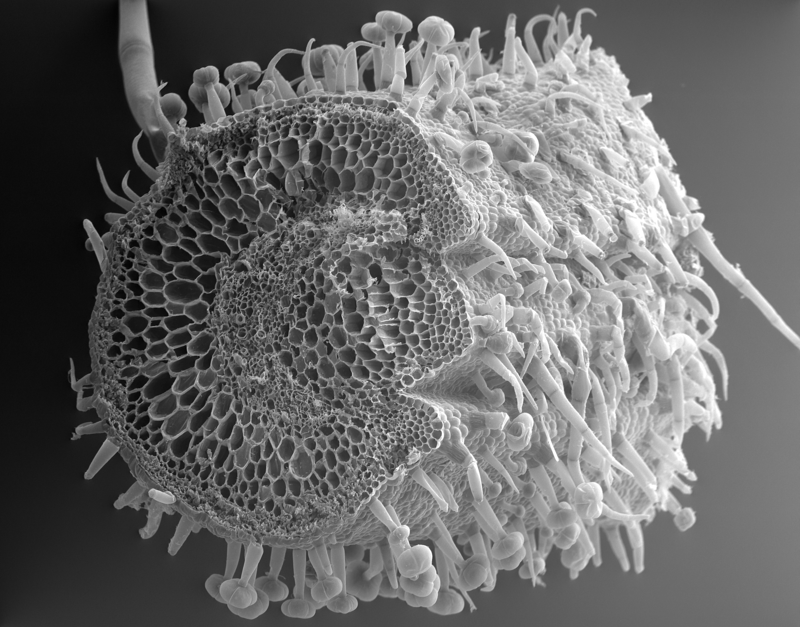}}
\subfigure[Result with nonlinear map]{
\includegraphics[width=0.48\linewidth]{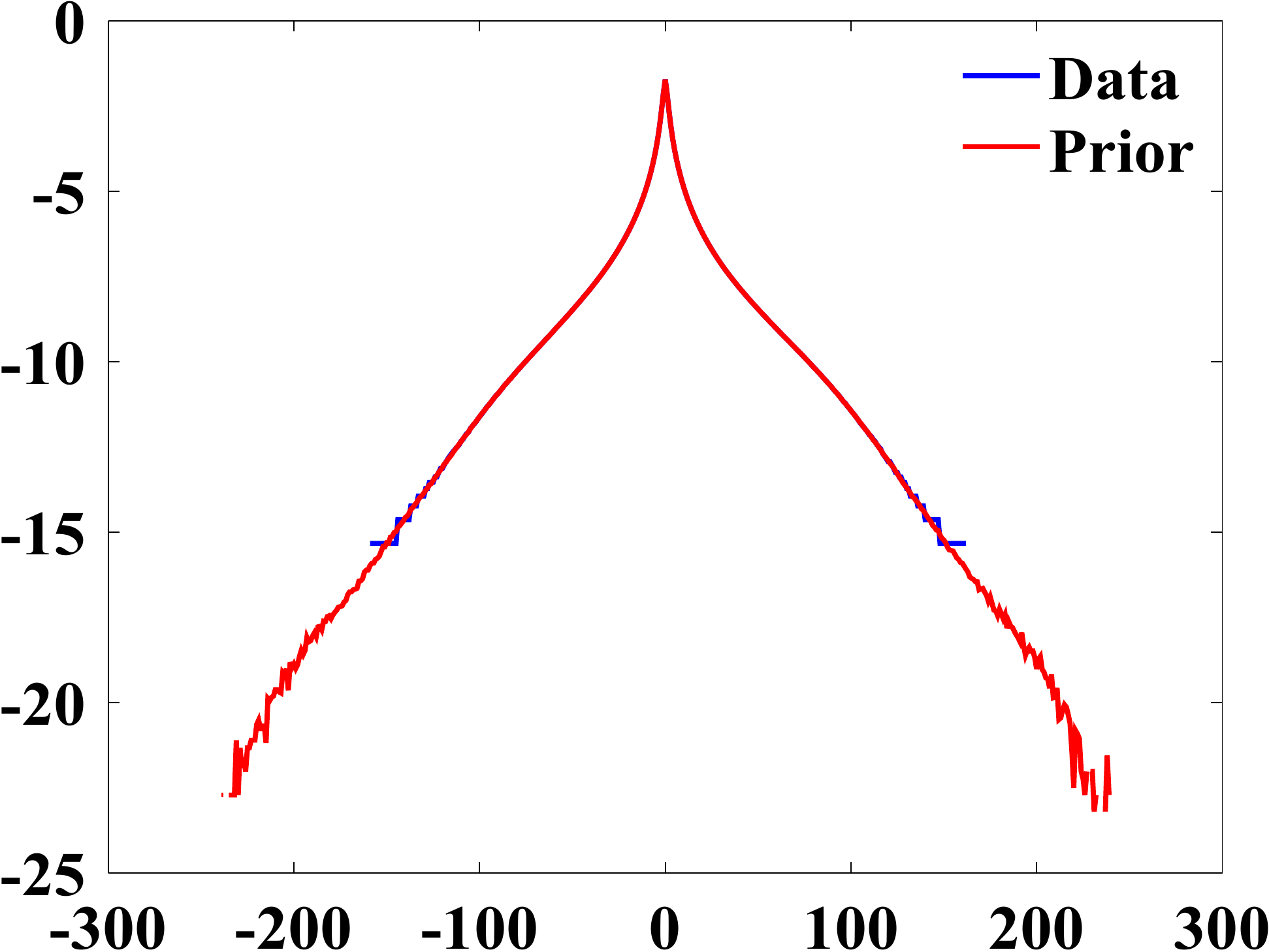}}
\subfigure[Linear and nonlinear mapping functions used]{
\includegraphics[width=0.6\linewidth]{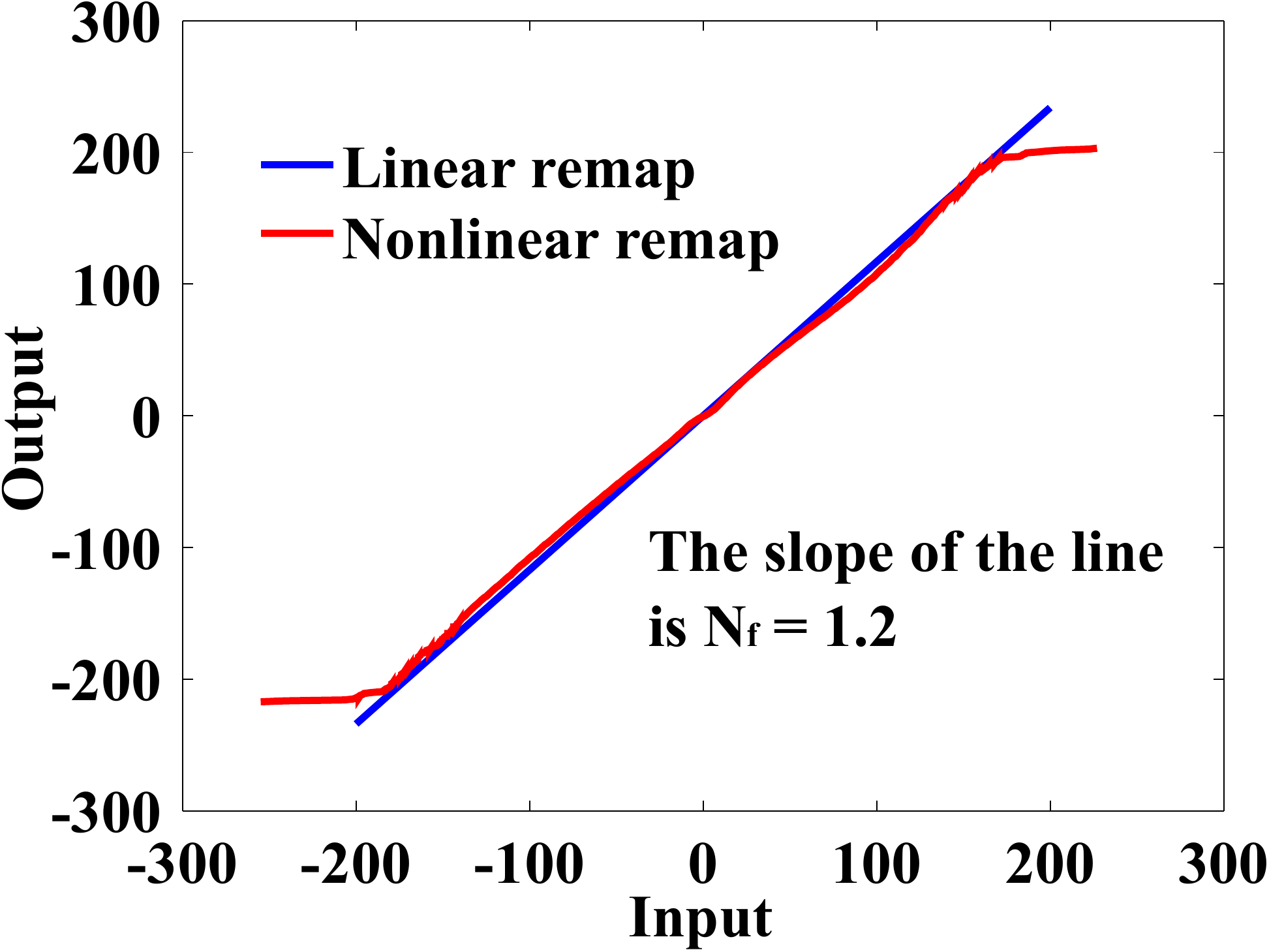}}
\caption{Comparison of different gradient field remapping methods: original image and its gradient distribution (a,b), image reconstructed from original gradient field without any remapping (c,d), with linear remapping (e,f), and with nonlinear remapping (g,h)~\citep{ExactHist}. The absolute RMS of the reconstructions are 0, 2.0, 23, and 33, respectively, with respect to the original image. The corresponding gradient distributions after remapping are shown to the right of the images. The linear and nonlinear remapping functions used are shown in (i).}
\label{fig:Poisson} 
\end{figure}

Existence and uniqueness of the solution of Eq.~\ref{eq:reconstruction} have been proven~\citep{boccardo:1996}. Commonly used norms are $\ell_2$ ($q=2$) and $\ell_1$ ($q=1$), which correspond to reducing measurement errors (unspecific) and gross errors (outliers), respectively. 

Taking the $\ell_2$ norm in Eq.~\ref{eq:reconstruction}, we recover the output image $I_n$ from the remapped gradient field $\vec{G}_n$ by solving the Poisson equation:
\begin{equation}
\label{eq:Poisson}
\Delta I_n = \nabla\cdot \vec{G}_n\, .
\end{equation}

This equation can be solved efficiently, e.g., by FFT-based algorithms or wavelet solvers. A short summary of available Poisson solvers is given in Table~\ref{table:Poisson}.

Reconstructing an image from its gradient field is accurate. An example is shown in Fig.~\ref{fig:Poisson}(c). The original image (a) is an 8-bit grayscale image. The absolute RMS of the reconstruction without remapping (b) is 2.032 with an average intensity value of 104.9. The size of the image is $1881\times2400$ pixels. Reconstruction using the wavelet solver in Matlab takes about 3.5 seconds on an Apple MacBook Pro (early 2011).
 
\begin{table}[h]
\scriptsize
\centering  % used for centering table
\begin{tabular}{c|cccc} 
\hline\hline
Solver  & Cholesky\footnotemark[8] & Jacobi&	Gauss-Seidel & SOR  \\
\hline
Type & direct & iterative & iterative	&	iterative\\
\hline
Complexity & $(mn)^3$ & $(mn)^2$ & $(mn)^2$ & $(mn)^{3/2}$\\

\hline \hline
Solver  & Cholesky\footnotemark[9] & {\color{ForestGreen}FFT} & Multigrid & {\color{ForestGreen}Wavelet}\\
\hline
Type & direct & direct&	iterative&	direct\\
\hline
Complexity & $(mn)^{3/2}$ & $(mn)log(mn)$ & $(mn)$&	$(mn)$\\
\hline
\end{tabular}
\caption{Summary of Poisson solvers. The FFT and Wavelet-based solvers are implemented in our software package.} % title of Table
\label{table:Poisson} % is used to refer this table in the text
\end{table}
\footnotetext[8]{dense Cholesky decomposition}  
\footnotetext[9]{sparse Cholesky decomposition} 

\subsection{As a soft constraint}
\label{sec:diffusion}
Imposing the GDP as a soft constraint is done by using the GDP as a regularization term. For a variational function $\mathcal{E}(\hat{U})$ this can be done by evolving the PDE
\begin{equation}
\frac{\partial \hat{U}}{\partial t} = -\frac{\partial \mathcal{E}(\hat{U})}{\partial \hat{U}}
\end{equation}
over pseudo-time $t$ (i.e., the iterations of the algorithm). Since the energy is $\mathcal{E}$ non-convex in general, minimization should be performed in a multi-scale space in order to avoid local minima and accelerate the computation. This can for example be done using multi-scale anisotropic diffusion, similar to the Perona-Malik model~\citep{PM1990}. More details about this procedure are given in Sec.~\ref{sec:denoise}.

When using our parametric Model 2 for the GDP, the minimization problem further simplifies. In this case, the variational energy is the difference of two convex functions, and the minimization problem can efficiently be solved using algorithms based on D.C.~programming~\citep{DCPBregman}. Then, the following decomposition holds:
\begin{equation}
\label{eq:DC}
\mathcal{E}(U)=\mathcal{E}_1(U)-\mathcal{E}_2(U)\, ,
\end{equation} where $\mathcal{E}_1(U)=\int_{\vec{x}\in \Omega}(\frac{1}{2}\|U-I\|_* +\frac{\lambda}{2}T_{\mathrm{pr}}^2\|\nabla U\|_2^2)\mathrm{d}\vec{x}$ and $\mathcal{E}_2(U)=-\frac{\lambda}{2}\int_{\vec{x}\in\Omega}\log(b_{\mathrm{pr}}+\|\nabla U\|_2^2)\mathrm{d}\vec{x}$ are differentiable convex functions, and $b_{\mathrm{pr}}$ is the GDP value of parameter $b_2$ of Model 2.

One way to solve such problems is to use Bregman splitting techniques~\citep{DCPBregman}. In this case, one needs to choose a Bregman function $\phi$, the choice of which however does not matter much to the algorithm performance~\citep{DCPBregman}. Then, Eq.~\ref{eq:DC} can be minimized using Algorithm~\ref{algo:DC}. The convergence proof can be found in~\citep{DCPBregman, Gasso:2009}. In the special case when $\phi$ is chosen to be a quadratic function, Algorithm~\ref{algo:DC} reduces to the standard proximal point algorithm.

\begin{algorithm}
\caption{Minimization using D.C.~programming}
\label{algo:DC}
\begin{algorithmic}[1]
\REQUIRE $\mathcal{E}_1$, $\mathcal{E}_2$, $\phi$, step size $\delta t>0$, $\epsilon>0$
\WHILE{$\|(U_{i} - U_{i-1})\|_\infty >\epsilon$}
\STATE{$U_{i+1}=(\nabla \phi +\delta t\nabla \mathcal{E}_1)^{-1}(\nabla \phi(U_i)+\delta t\nabla \mathcal{E}_2(U_i))$ } 
\ENDWHILE
\end{algorithmic}
\end{algorithm} 
%\section{Quasi convexity, Global unique optimal solution!}

%\subsubsection{$\ell_1$ norm}
%We recover the naturalized image $I_n$ by solving following model:
%\begin{equation}
%\label{eq:reconstruction2}
%\min E(I_n) = \int_{\vec{x}\in\Omega}\!\!\!\|\nabla I_n - \vec{G}_n \|_{1}\,\mathrm{d}\vec{x}\, .
%\end{equation}
%This energy function can be solved by iteratively reweighted least squares(IRLS) method or split Bregman method. Both of them are based on Poisson solvers. We have following equivalent function:
%\begin{equation}
%\label{eq:reconstruction2SP}
%\min_{\vec{d},I_n} E(\vec{d},I_n) = \int_{\vec{x}\in\Omega}\!\!\!\|\vec{d} \|_{1}+\frac{\lambda}{2}\|\vec{d}-\nabla I_n + \vec{G}_n\|_2^2\,\mathrm{d}\vec{x}\, .
%\end{equation}  
%For $\vec{d}$, classical $Shrink$ operator is performed; for $I_n$, Poisson equation is solved. %This energy is specially useful when the mapped field is not integrable.

\subsection{Implementation details}
For hard GDP constraints, we implemented nonparametric remapping based on exact histogram specification~\citep{coltuc:2006} and remapping based on our parametric Model 2. For image reconstruction, we implemented FFT(DST, DCT)-based and wavelet-based Poisson solvers. 

For soft GDP constraints, we implemented the multi-scale diffusion (Perona-Malik model~\citep{PM1990}), which is valid for all priors. Specifically for our Model 2, we also implemented the D.C.~programming of Algorithm~\ref{algo:DC}. 

The source code is available from our website as Matlab code, C++ code included in the OpenCV library, and Java code as an ImageJ/Fiji plugin.

%When parametric linear remapping is considered, direct finding the parameter $T$ in CDF modeling is computational expensive(left of Fig.~\ref{fig:Perform}). But it is possible to search this parameter in the marginal distribution(right of Fig.~\ref{fig:Perform}) by  ternary search algorithm because it is a single parameter. We run both of searching technique 100 times and report the performance in Fig.~\ref{fig:Perform}. The difference of the two searching results is smaller than the searching stop criteria. Our C++ version of parameter search(without using SSE) is much faster and can achieve 868.4 MPixel/Second for 8-bit single channel image.  
%\begin{figure}[h]
%\centering
%\includegraphics[width=0.49\linewidth,height=0.3\linewidth]{PerformanceT.pdf}
%\includegraphics[width=0.49\linewidth,height=0.3\linewidth]{PerformanceTfast.pdf}
%\caption{Parameter search in Matlab(MacBook Pro(2011 early))}
% \label{fig:Perform} %% label for entire figure
%\end{figure}

%\subsection{Parallelization}
%The computing of gradient distribution can be paralleled by spatially decomposing the image.
%\clearpage
\section{Example Applications}
\label{sec:app}
Priors play a central role in many image processing tasks. We exemplify this here by showing the use of the present GDP, the parametric models, and the solvers in a wide variety of image-processing tasks, ranging from contrast enhancement, to noise level estimation, denoising, deconvolution, zooming (super resolution), and dehazing.

\subsection{Image naturalization}
\begin{figure}[h]
  \centering
  \includegraphics[width=\linewidth]{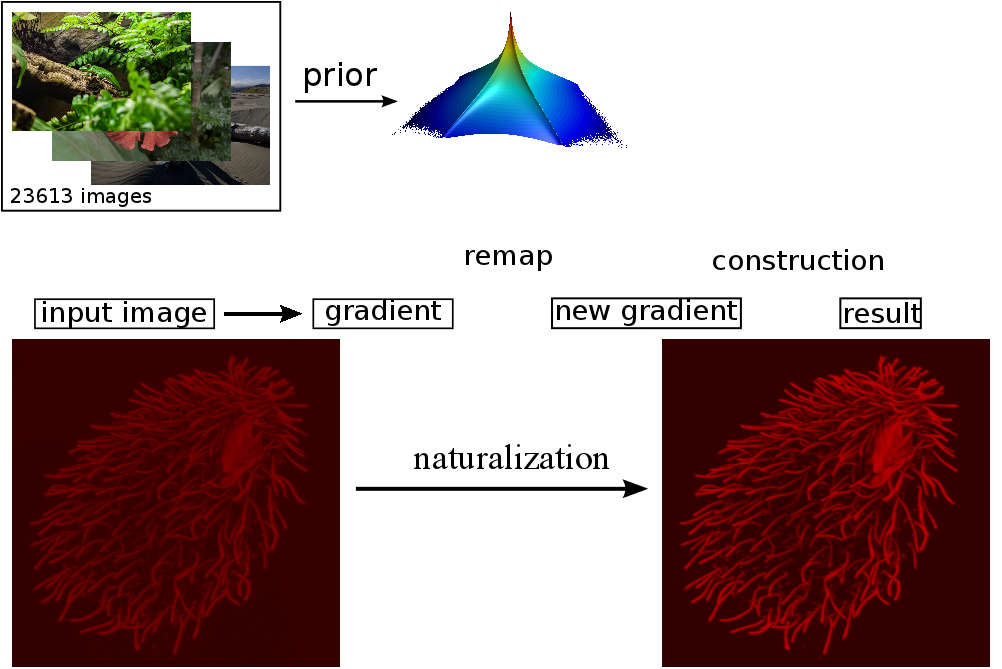}
  \caption{Naturalization.}
  \label{fig:flowN} %% label for entire figure
\end{figure}

\begin{figure*}[!htb]
\centering
\subfigure[$N_{\!f}=1.07$ {\footnotesize{(beyondthehumaneye.blogspot.de)}}]{\includegraphics[width=.24\linewidth, height=.08\linewidth ]{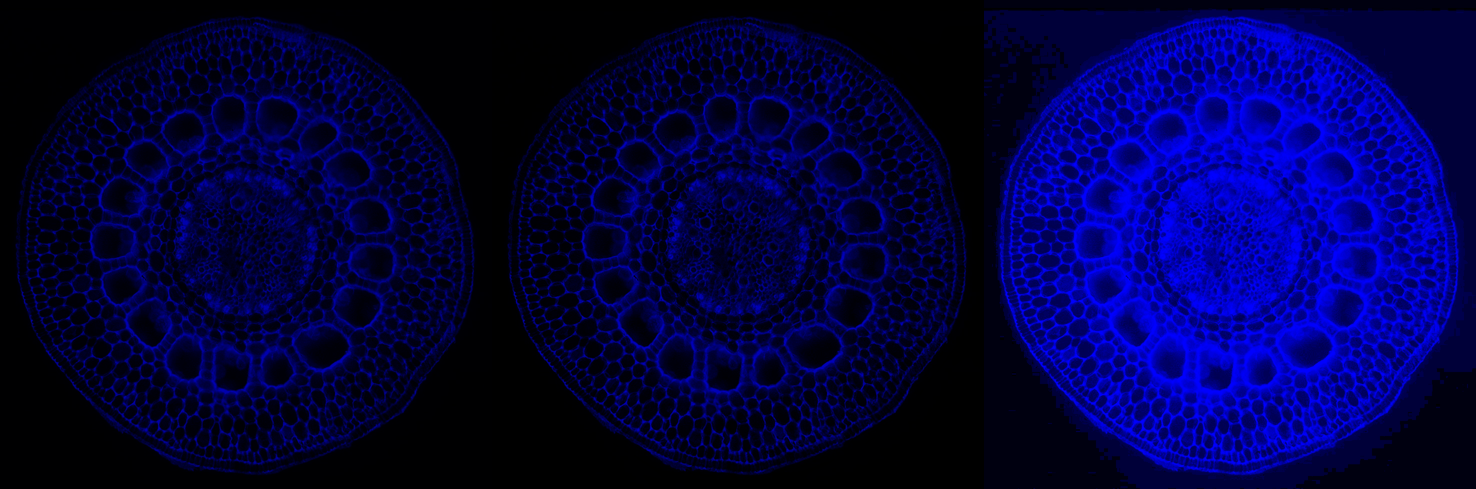}}
\subfigure[$N_{\!f}=1.51$ {\footnotesize{(Maryann Martone, CCDB)}}]{\includegraphics[width=.24\linewidth, height=.08\linewidth ]{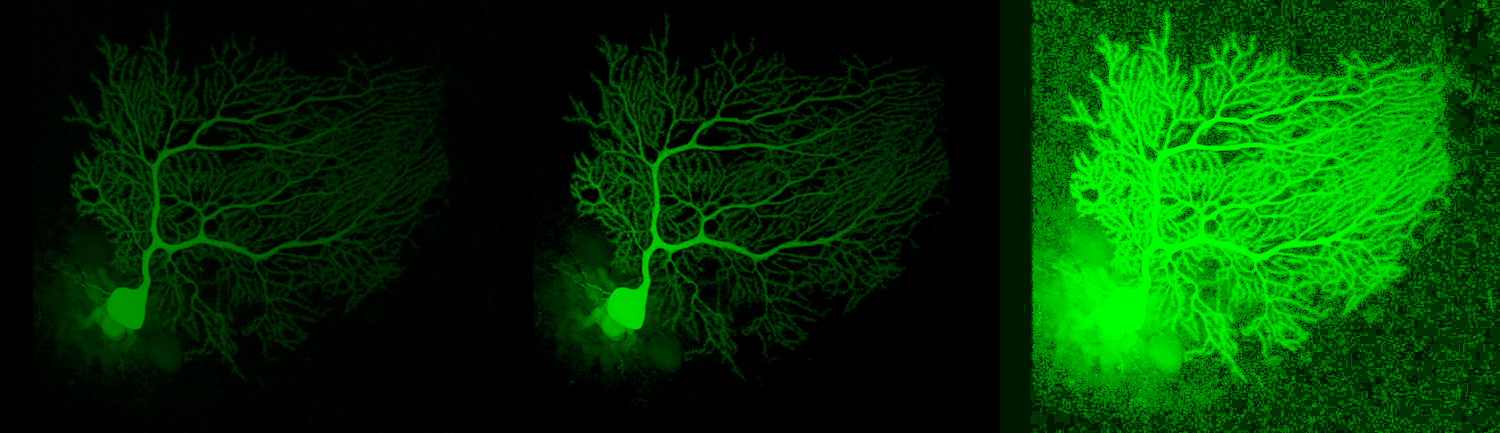}}
\subfigure[$N_{\!f}=1.81$ {\footnotesize{(Lee \& Matus, U Hawaii, amicros.org)}}]{\includegraphics[width=.24\linewidth, height=.08\linewidth ]{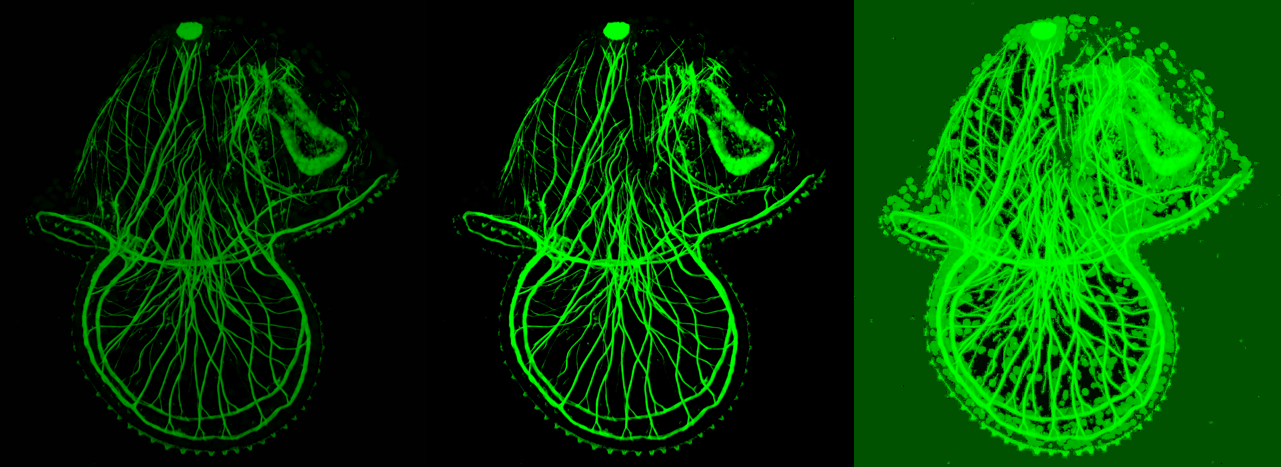}}
\subfigure[$N_{\!f}=2.00$ {\footnotesize{(Gaertig lab, U Georgia, bmc.uga.edu)}}]{\includegraphics[width=.24\linewidth, height=.08\linewidth ]{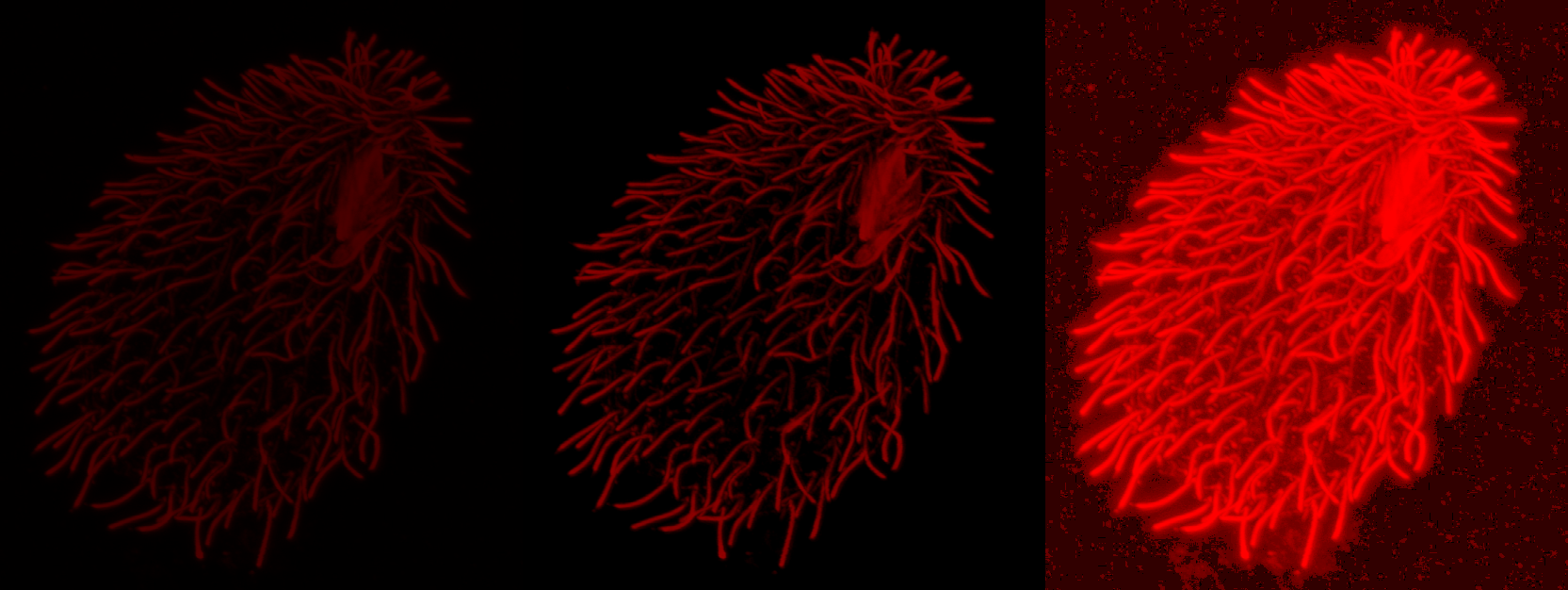}}
%\subfigure[$N_{\!f}=3.31$ {\footnotesize{(bme.case. edu/Capadona/Photos)}}]{\includegraphics[width=.19\linewidth, height=.1\linewidth ]{some5.png}}

\subfigure[$N_{f}=1.03$ {\footnotesize{(Dartmouth EM Gallery)}}]{\includegraphics[width=.24\linewidth, height=.08\linewidth ]{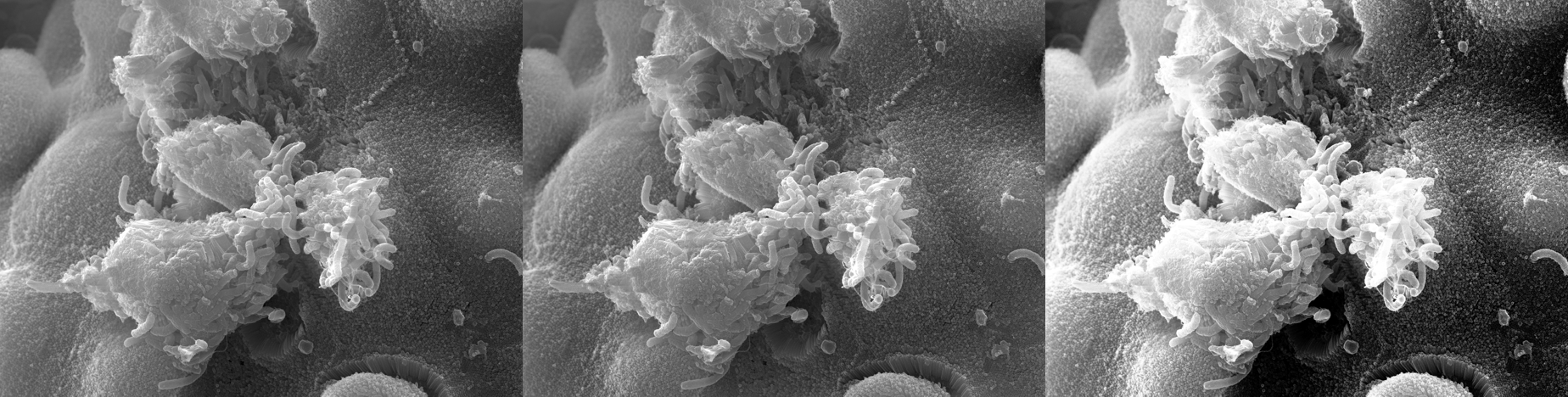}}
\subfigure[$N_{\!f}=1.27$ {\footnotesize{(Dartmouth EM Gallery)}}]{\includegraphics[width=.24\linewidth, height=.08\linewidth ]{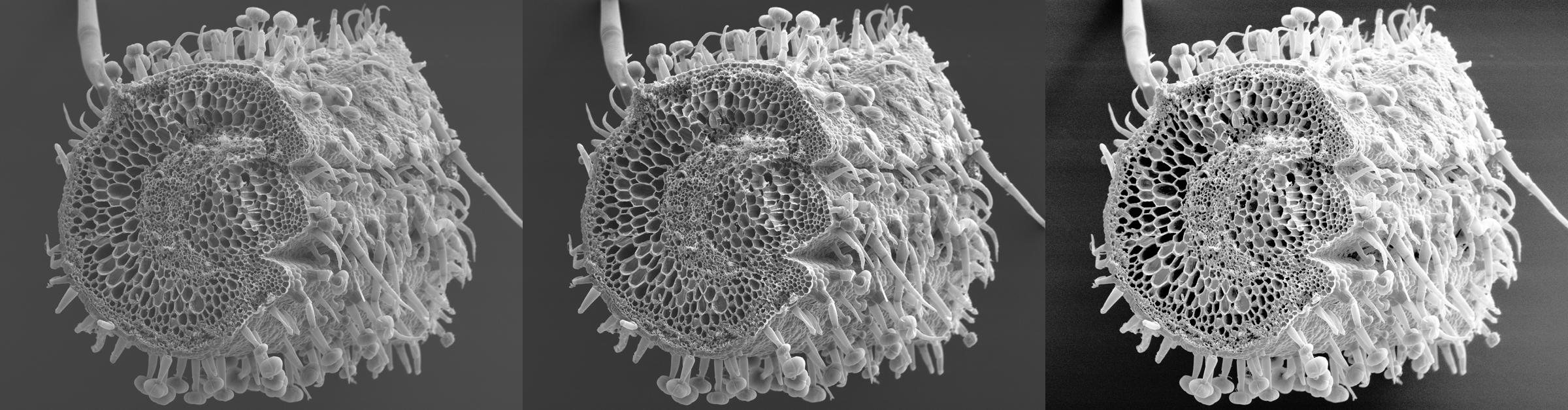}}
\subfigure[$N_{\!f}=1.45$ {\footnotesize{(Dartmouth EM Gallery)}}]{\includegraphics[width=.24\linewidth, height=.08\linewidth ]{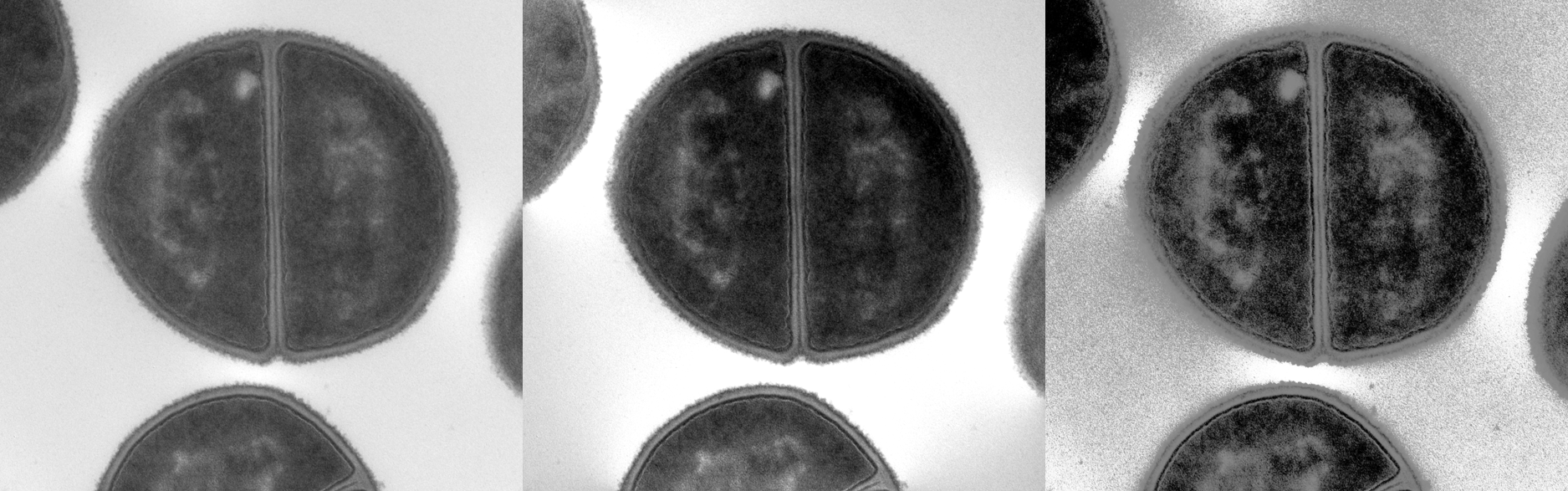}}
\subfigure[$N_{\!f}=2.29$ {\footnotesize{(K.~Ushakov, U Geneva)}}]{\includegraphics[width=.24\linewidth, height=.08\linewidth ]{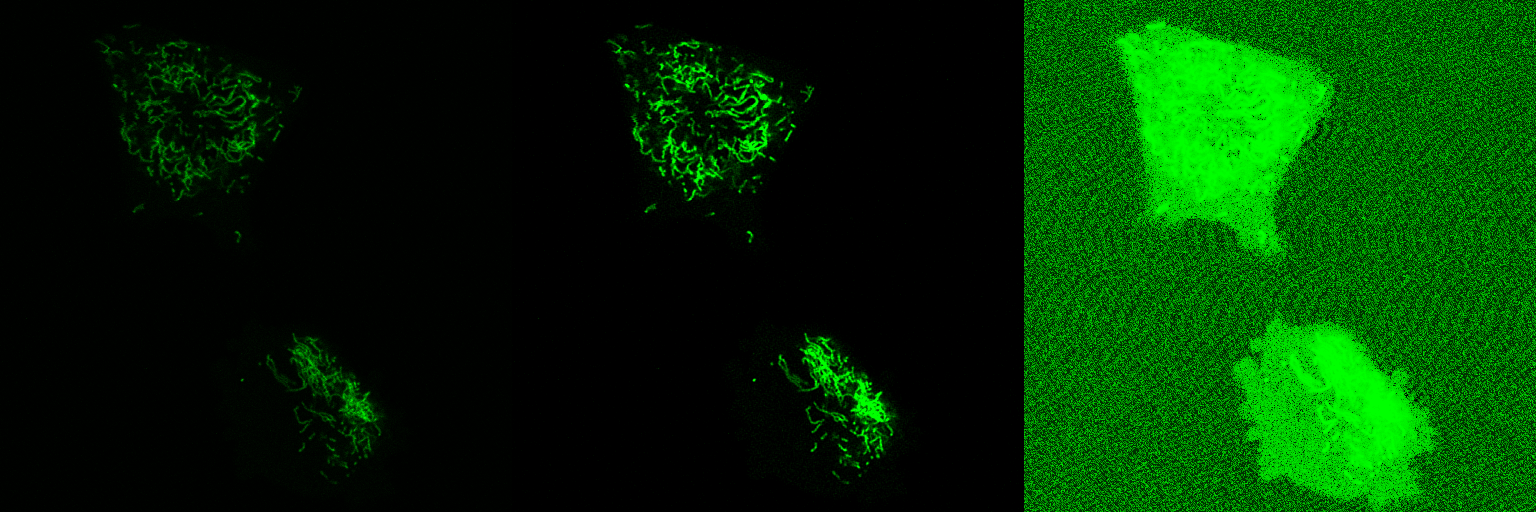}}
%\subfigure[$N_{\!f}=3.02$ {\footnotesize{(K.~Ushakov, U Geneva)}}]{\includegraphics[width=.12\linewidth, height=.08\linewidth ]{mito2.png}}

%\subfigure[$N_{\!f}=1.87$ {\footnotesize{(OpenSPIM.org)}}]{\includegraphics[width=0.98\linewidth, height=.1\linewidth ]{fish.png}}
\caption{Examples of original (left), naturalized (middle) and histogram equalized (right) microscopy images.}
 \label{fig:naturalized} %% label for entire figure
\end{figure*}

Remapping the gradient field of any image to match the GDP of natural-scene images, and then reconstructing the output image is called {\em image naturalization}, because the output image will have a gradient distribution that matches the one of natural-scene images. The naturalized output image hence looks more ``natural''. The workflow is shown in Fig.~\ref{fig:flowN}.
Since the GDP correlates with image quality, this makes the image look more appealing. We hence propose to use image naturalization as an alternative to histogram equalization when displaying images to a human observer. Image naturalization enhances contrast by solving Eq.~\ref{eq:spectral} with hard GDP constraint. Some examples of microscopy images (left tile of each panel) and their naturalized versions (middle tiles) are shown in Fig.~\ref{fig:naturalized} along with the naturalness factor of the original image. The histogram-equalized images are shown in right tiles for comparison. The first row shows four fluorescence-microscopy images. The second row shows three electron-microscopy images and one fluorescence image. All images were collected from publicly accessible web pages; credits are in parentheses. Here we use the simple linear $Map$ function, amounting to a straightforward rescaling of the image, albeit with a good, ``natural'' scale factor as determined by the GDP. In all cases. the naturalized image looks more appealing than the histogram-equalized image, and suffers from less background artifacts. 

\subsection{Noise level estimation}
Traditional denoising methods heavily rely on having an estimate of the noise level to adjust their parameters. In practical applications, however, the true noise level is unknown. We show how the GDP can be used to robustly estimate the noise level of an image. As shown by the noise case in Fig.~\ref{fig:CDF}, the parameter $T$ of Model 2 is sensitive to noise. This can be exploited to estimate the noise level by relating the fitted parameter $T$ of any given image to noise level through a calibration curve. We construct such a calibration curve ($T$ vs.~true noise level) by randomly choosing seven images from our natural-scene training dataset and adding to them Gaussian noise of varying $\sigma=[0.02:0.02:0.8]$. The dependence of $T$ on $\sigma$ shows a distinct characteristic, which is almost independent of image content (left panel of Fig.~\ref{fig:NoiseCurve}, 7 curves with different symbols). We fit this dependence using the mixture of exponentials:
\begin{equation}
\widetilde{\sigma} = \sum_{i=0}^{i=N}q_i\exp\!{\left\{s_iT\right\} },
\end{equation}
where $q_i>0$ and $s_i<0$ are parameters to be determined. For our dataset, we find the best fit $N=2$, $q_{\{1,2\}}=\{772.6,0.9538\}$, $s_{\{1,2\}}=\{-5321,-931.2\}$. The goodness of fit is: SSE=0.2741, RMSE=0.034, $R^2$=0.979. The model is shown by the solid blue line in the left panel of Fig.~\ref{fig:NoiseCurve}.

We test the model by adding Gaussian noise of different, known $\sigma$ to a disjoint randomly selected set of seven image (the test set for cross validation). The differences between the noise levels $\widetilde{\sigma}$ estimated by our model and the ground truth are shown in the right panel of Fig.~\ref{fig:NoiseCurve}. 87\% of the predictions have an accuracy $|\widetilde{\sigma}-\sigma|<0.04$. We find similar results also for other random image sets tested. 

This suggests that the parameter $T$ can provide accurate and robust noise-level estimation. A particularly favorable property of this estimator is its high sensitivity to changes in $\sigma$ for low noise levels ($\sigma<0.2$). Correctly estimating low noise levels is particularly hard for traditional, pixel-based estimators. 

\begin{figure}[h]
\centering
\setlength{\abovecaptionskip}{0cm}
\setlength{\belowcaptionskip}{0cm}
\setlength{\subfigbottomskip}{0cm}
\setlength{\subfigcapskip}{0cm}
\setlength{\subfigtopskip}{0cm}
\subfigure{\includegraphics[width=0.49\linewidth]{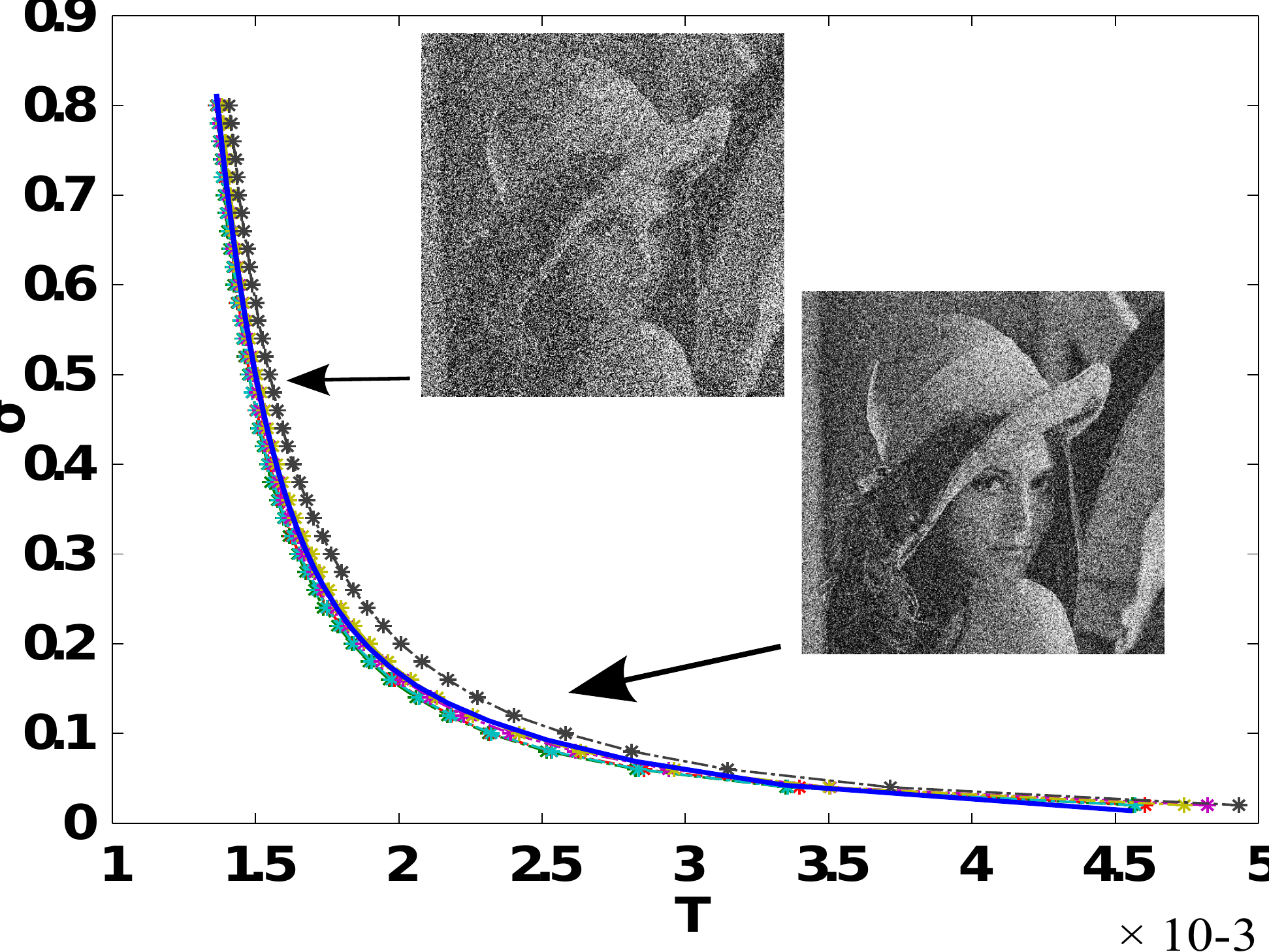}}
\subfigure{\includegraphics[width=0.49\linewidth]{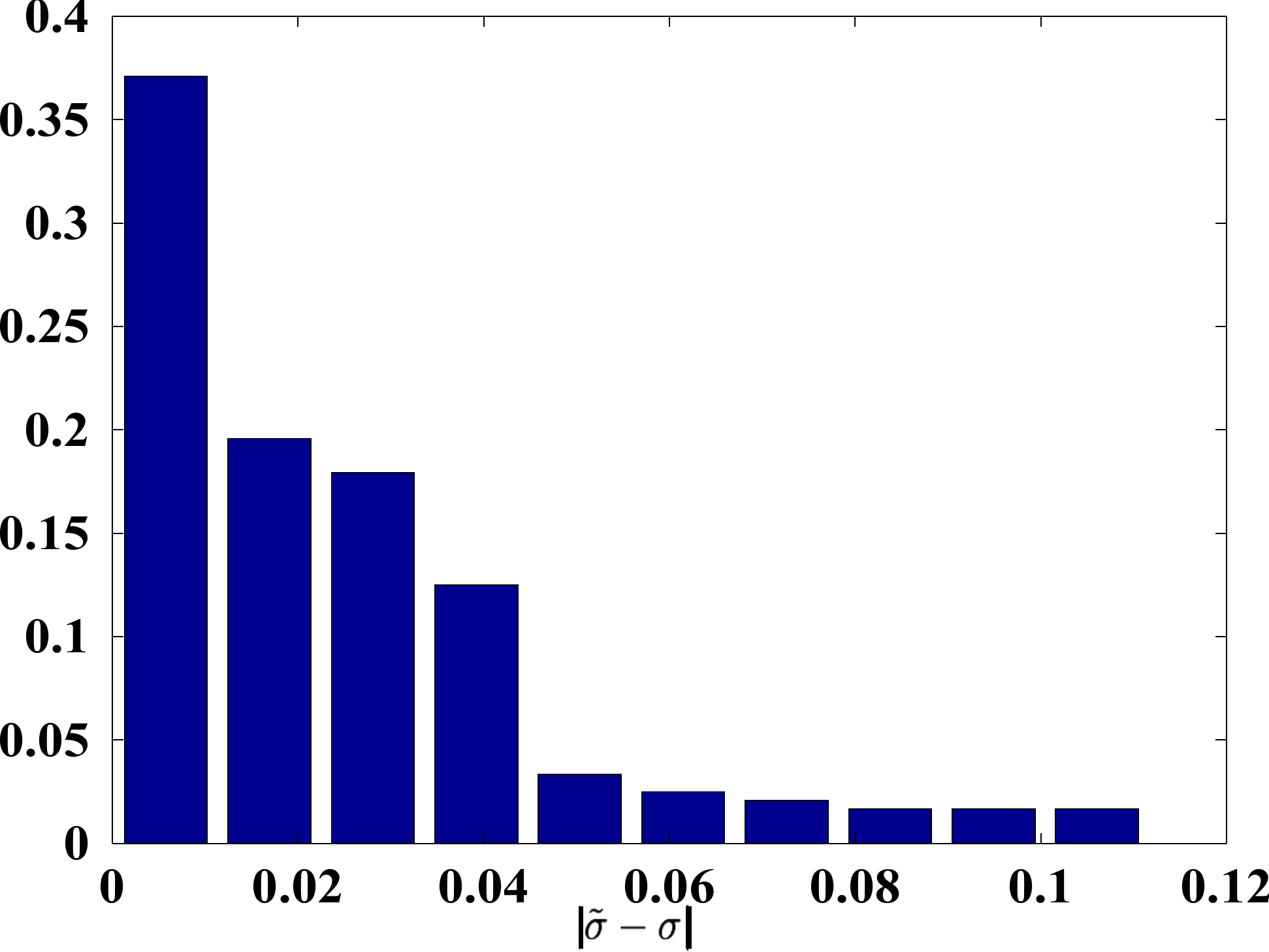}}
\caption{Noise-level model and its prediction-error distribution.}
 \label{fig:NoiseCurve} %% label for entire figure
\end{figure}

\subsection{Denoising}
\label{sec:denoise}
The high sensitivity of the cumulative gradient distribution to small levels of noise makes the GDP a good prior for image denoising. Small non-zero gradients play a key role to recover image details. Traditional denoising methods with spatial regularization (such as TV and its variants, GC, etc.) remove both noise and small signal gradients. In contrast, the GDP can be used to distinguish between noise and small signal gradients. This is compatible with many researchers' observation that split-Bregman solvers for TV-$\ell_1$ models achieve better results in the sense of PSNR~\citep{splitBregman,Paul:2013}. This is because the auxiliary variable introduced in Bregman splitting changes the model to allow small gradients. These small gradients improve the result. Another example is non-local TV, using spatially repeated patterns to allow for small signal gradients~\citep{denoise_NTV:2014}. A third example is stochastic (Monte Carlo) denoising~\citep{wong:2011}. While all of these methods allow for small image gradients, distinguishing them from noise is mostly {\em ad hoc} and arbitrary. Here, the GDP can provide additional information. This has recently been demonstrated in a MAP Bayesian framework~\citep{cho:2012}, which also has the capability of recovering image details. Using our novel parametric GDP model, we introduce:
\begin{eqnarray}
\label{eq:denoisePre}
\begin{split}
\mathcal{E}(U)  = &\int_{\vec{x}\in\Omega} \frac{1}{2}\|U-I\|_2^2 +\\
& \frac{\lambda}{2}(T_{\mathrm{pr}}^2\|\nabla U\|_2^2+\log(b_{\mathrm{pr}}+\|\nabla U\|_2^2))\mathrm{d}\vec{x}. 
\end{split}
\end{eqnarray}

This denoising model is differentiable with respect to $U$ and can be efficiently solved by gradient descent (Algorithm~\ref{algo:denoise}). Using the G\^{a}teaux derivative, this model can be interpreted as an anisotropic diffusion and inverse diffusion process:
\begin{eqnarray}
\label{eq:denoise}
\begin{split}
\frac{\partial U}{\partial t} & = -\frac{\partial \mathcal{E}}{\partial U} \\
& = I - U +\lambda\left(T^2_{\mathrm{pr}} + \frac{b_{\mathrm{pr}}-\|\nabla U\|_2^2}{(b_{\mathrm{pr}}+\|\nabla U\|_2^2)^2}\right)\Delta U.
\end{split}
%\frac{\partial U}{\partial t} = -\frac{\partial E}{\partial U} = I - U +\lambda\left(T^2_{\mathrm{pr}} + \frac{1}{(b_{\mathrm{pr}}+\|\nabla U\|_2^2)^2}\right)\Delta U.
\end{eqnarray}
This equation can also be derived from the Euler-Lagrange equation of the variational form. 

The regularization term is a hybrid of diffusion and inverse diffusion, which is fundamentally different from traditional approaches that only depend on one of them. For example, the traditional anisotropic diffusion $(1+\|\nabla U\|_2^2)^{-1}$ (Perona Malik model~\citep{PM1990}) only tries to smooth the image, while inverse diffusion only enhances the image~\citep{Calder:2011}. The behavior of the diffusion coefficient $W$ in Algorithm~\ref{algo:denoise} is illustrated in Fig.~\ref{fig:TheW}. It is clear that $U$ gets enhanced (inverse diffusion, $W<0$) or smoothed (diffusion, $W>0$) depending on the gradient magnitude. Even though this behavior is similar to forward-backward diffusion~\citep{Guy:2002}, the fundamental difference is that our model is derived from a distribution prior, rather than from the gradient itself. As a result, the parameters $T_{\mathrm{pr}}$ and $b_{\mathrm{pr}}$ are learned from datasets and do not need to be manually adjusted, as in forward-backward diffusion~\citep{Guy:2002}.

An example of using this model for denoising as given in Algorithm~\ref{algo:denoise} is shown in Fig.~\ref{fig:denoiseResult} and Table~\ref{table:denoise}. The present model achieves state-of-the-art PSNR with significantly better image quality, quantified by the SSIM quality measure in Table~\ref{table:denoise}.

\begin{algorithm}
\begin{algorithmic}[1]
\REQUIRE $I$, $\lambda$, step size $\delta t$, $T_{\mathrm{pr}}$, $b_{\mathrm{pr}}$, $\epsilon$
\WHILE{$\|(U_{i} - U_{i-1})\|_\infty >\epsilon$}
\STATE{$W=T_{\text{pr}}^2+\frac{b_{\mathrm{pr}}-\|\nabla U\|_2^2}{(b_{\mathrm{pr}}+\|\nabla U\|_2^2)^2}$}
\STATE{$U_{i+1}=\frac{U_i}{1+\delta t} +\frac{\delta t}{1+\delta t}I+\frac{\delta t}{1+\delta t}\lambda W\Delta U_i$ } 
\ENDWHILE
\end{algorithmic}
\caption{Denoising with GDP}
\label{algo:denoise}
\end{algorithm} 

\begin{lem}
\label{lem:w}
In Algorithm~\ref{algo:denoise}, let $v=\|\nabla U\|_2^2$, if~$T^2_{\mathrm{pr}}b_{\mathrm{pr}}\geq \frac{1}{8}$, then $W\geq 0$. If ~$T^2_{\mathrm{pr}}b_{\mathrm{pr}}< \frac{1}{8}$, then there are two fixed points $v_L$ and $v_U$ ($v_L<v_U$) such that 
\begin{eqnarray}
\begin{cases}
W \leq 0      & if ~~v_L\leq v\leq v_U \\
W >0   & else. \\
\end{cases}
\end{eqnarray}
\end{lem}

This is illustrated in Fig.~\ref{fig:TheW}. The proof is given in Appendix A.

\begin{figure}[h]
\centering
\includegraphics[width=0.7\linewidth]{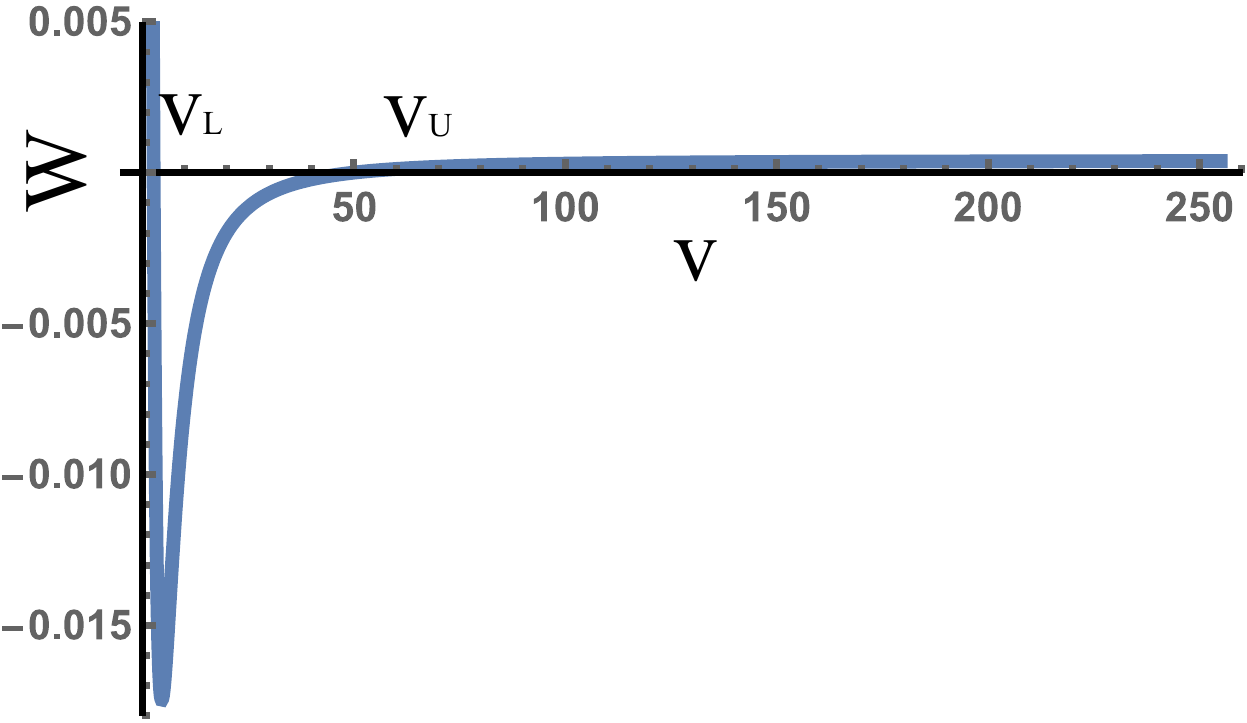}
\caption{Behavior of the diffusion coefficient $W$ versus the square gradient magnitude $v=\| \nabla U \|_2^2$ for the $T^2_{\mathrm{pr}}$ and $b_{\mathrm{pr}}$ values of the GDP.}
 \label{fig:TheW}
\end{figure}

\begin{figure}[h]
\centering
\includegraphics[width=\linewidth]{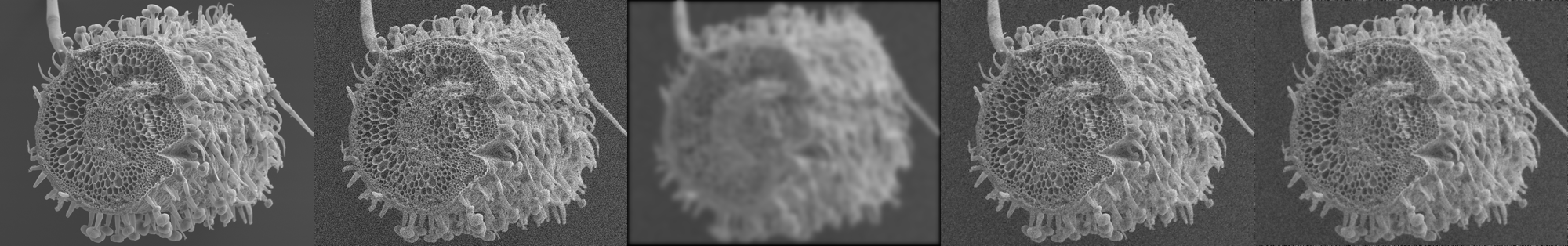}
\includegraphics[width=\linewidth]{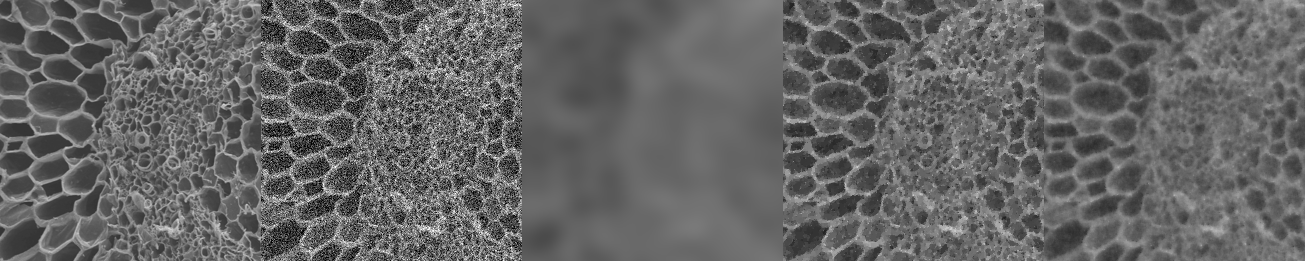}
\caption{Denoising example. From left to right: original image, noisy image, solution of the Perona-Malik model~\citep{PM1990}, solution of the TV model, solution of our GDP model. A magnified patch is shown under each image.}
 \label{fig:denoiseResult}
\end{figure}

\begin{table}[h]
\scriptsize
\centering  % used for centering table
\begin{tabular}{c|cccc} 
\hline\hline
  & noisy & Perona-Malik & TV & Present  \\
\hline
PSNR & 17.12 & 20.72& 26.39  & 26.39\\
\hline
SSIM & 0.2087 & 0.5534& 0.6441  & 0.7196\\
\hline
\end{tabular}
\caption{Quality comparison of denoising results from different models.} % title of Table
\label{table:denoise} % is used to refer this table in the text
\end{table}

\subsection{Blind Deconvolution}
Biomedical images are often blurred, e.g., due to object motion during exposure, or due to light diffraction in the detector optics. The latter is particularly common in microscopy, where the imaged objects are of similar length scale as the wavelength of the light used. The image is then significantly blurred by the point-spread function (PSF, or impulse-response function) of the optics, the Fourier transform of which is the optical transfer function of the imaging equipment. Since this blurring is an artifact of the imaging method, one often seeks to undo it to the extent possible. In fluorescence microscopy, the imaging process is linear, and the blurring is accurately described by a convolution of the original image with the PSF of the microscope. The task of {\em deconvolution} is to estimate the perfect latent image $U$ from the observed blurred image $I$. If the PSF (blur kernel) $K$ is unknown and to be estimated along, the problem is referred to as {\em blind deconvolution}. This is a typical ill-posed inverse problem. Imposing a prior can render the problem well posed. 

$K$ and $U$ can be estimated either in the spatial and/or the gradient domain. We provide here an algorithm for blind deconvolution using GDP. The algorithm is inspired by Fig.~\ref{fig:example}, showing that auto-correlation is significantly reduced in the gradient domain, which is a favorable property for (blur-)kernel estimation. The latent image, however, is better estimated in the spatial domain, where the auto-correlation signal can be exploited. Different combinations of spatial/gradient-domain deconvolution have already been priorly presented (see Table~\ref{table:deblur}). The present algorithm, however, is the first one to combine gradient-domain kernel estimation with spatial-domain image estimation, which we believe to be a particularly good combination. 

\begin{table}[h]
\scriptsize
\centering  % used for centering table
\begin{tabular}{c|c|c|c} 
\hline\hline
 & Kernel $K$ & Image $U$ & Typical Method \\
\hline
\multirow{4}{*}{\rotatebox{90}{domain}} & spatial & spatial & \citep{levin2007}\\
& spatial & gradient & \citep{fergus:2006}\\
& gradient & gradient & \citep{chen:2010}\\
& gradient & spatial & {\color{ForestGreen}present}\\
\hline 
\end{tabular}
\caption{Summary of blind deconvolution algorithms.} % title of Table
\label{table:deblur} % is used to refer this table in the text
\end{table}

Besides the working domain, the prior (or regularizer) used is of key importance. In general, sparsity of the kernel and TV of the latent image are imposed for deconvolution~\citep{Chan:1998,Marquina:2009,KrishnanTF11,LiLGT:2012}. However, it is known that a GDP on the latent image provides a better choice, removing less image detail than TV~\citep{fergus:2006,krishnan2009fast,cho:2009,chen:2010,shan:2008a}. Here, we use the present parametric GDP model as a prior for the latent image, but impose no prior on the kernel. This renders our methods generic to a wide variety of different blur kernels that do not have to be priorly known.

We use alternating minimization to estimate the kernel $K$ and the latent image $U$ by minimizing:
\begin{equation}
\label{eq:blurK}
\begin{split}
\mathcal{E}_k(K)&= \|\nabla U_i\otimes K -\nabla I \|_{2}^2\ \\
& s.t.~~  \|K\|_2=1 \, , K\geq 0\, .
\end{split}
\end{equation}
\begin{equation}
\label{eq:blurU}
\begin{split}
&\mathcal{E}_u(U) = 
\frac{1}{2}\| U\otimes K_{i+1} -I \|_{2}^2 \\
 & +\frac{\lambda}{2}\left(T_{\mathrm{pr}}^2\|\nabla U\|_2^2+\log(b_{\mathrm{pr}}+\|\nabla U\|_2^2)\right)\,,
\end{split}
\end{equation}  
where $i$ is the iteration number of the alternating minimization scheme. Equation~\ref{eq:blurK} is a convex function with convex constraints, guaranteeing a globally optimal solution. In Algorithm~\ref{algo:deblur} we hence solve this part of the problem analytically by projection. Equation~\ref{eq:blurU} is not convex, but can be solved by a diffusion process. Algorithm~\ref{algo:deblur} summaries the resulting overall blind deconvolution process, which is performed in a multi-scale fashion to avoid local minima and accelerate computation. A notable implementation detail is that we only compute on the interior pixels in order to avoid the boundary issue, instead of padding the image as done in previous methods. 

An example with a complicated blur kernel is shown in Fig.~\ref{fig:deblur} and Table~\ref{table:deblurResult}. The ground-truth image (Fig.~\ref{fig:deblur}(a)) is blurred with a known kernel (Fig.~\ref{fig:deblur}(b)). Figure~\ref{fig:deblur}(c,d) show the reconstructed images using two different {\em non-blind} deconvolution methods with the ground-truth kernel provided to them. Figure~\ref{fig:deblur}(e,f) show the results of two blind deconvolution methods along with the estimated kernels (insets). Figure~\ref{fig:DeblurKernel} shows the estimated $\hat{K}$ at different scales of the multi-scale process used in the present method. As evident from Table~\ref{table:deblurResult}, the result from the present GDP method achieves higher image quality (as measured by SSIM) than the comparison algorithms. 

\begin{algorithm}
\caption{Blind Deconvolution with GDP}
\label{algo:deblur}
\begin{algorithmic}[1]
\REQUIRE $I$, $\lambda$, $T_{\mathrm{pr}}$, $b_{\mathrm{pr}}$,$\epsilon$
\STATE ${\bf F}_I= FFT(\nabla I)$
%\STATE ${F}_I= FFT(I)$
%\STATE ${\bf D}= [FFT([-1\, 1]),FFT([-1\, 1]^T)]^T$
\WHILE{$\|\nabla (U_{i} - U_{i-1})\|>\epsilon$}
\STATE{${\bf F}_U = FFT(\nabla U_i)$} \STATE{$\hat{K}_{i+1}=FFT^{-1}\left(\frac{\overline{{\bf F}_U}^T\circ {\bf F}_I }{\overline{{\bf F}_U}^T\circ {\bf F}_U}\right)$} 
\STATE{$K_{i+1}\leftarrow \hat{K}_{i+1}>0,\int_{\vec{x}\in W_K}\hat{K}_{i+1}\mathrm{d}\vec{x}=1$}
\STATE{$U_{i+1} \leftarrow $ Anisotropic Diffusion $U_i$ with $ K_{i+1}$ (Eq.~\ref{eq:denoise})}

%direct remapping
%\STATE{$\nabla\tilde{U}_{i+1} = Map(\nabla U_i)$}
%\STATE{${\bf F}_{\tilde{U}} = FFT(\nabla \tilde{U}_{i+1})$}
%\STATE{$F_{K} = FFT(K_{i+1})$}
%\STATE{$U_{i+1} = FFT^{-1}(\frac{\overline{F_K}^T\circ F_I + \lambda \overline{{\bf F}_{\tilde{U}}}^T\circ {\bf D} }{\overline{F_K}^T\circ F_K + \lambda  \overline{\bf D}^T\circ \bf D})$}
\ENDWHILE
\end{algorithmic}
\end{algorithm} 

\begin{figure}[!htb]
\centering
\subfigure[ground truth image]{\includegraphics[width=0.48\linewidth]{deblur_source.png}
}
\subfigure[blurred image with $K$]{\includegraphics[width=0.48\linewidth]{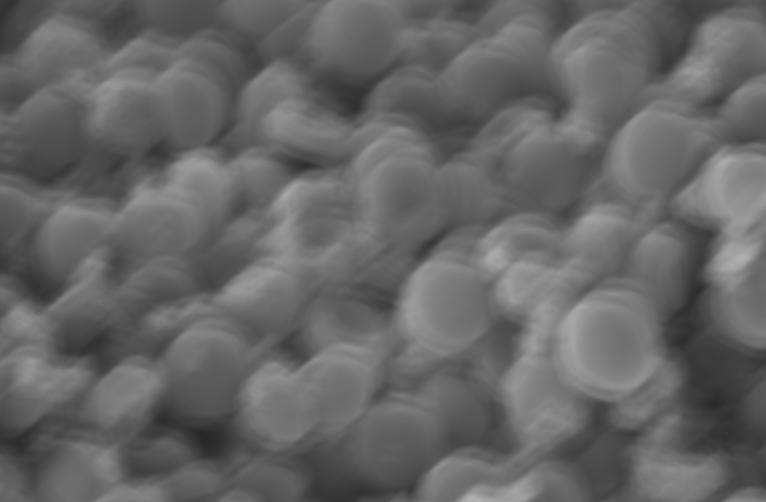}
\hspace{-0.23\linewidth}
\includegraphics[width=0.2\linewidth]{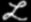}
\label{fig:blurF}
}
\subfigure[Lucy-Richardson result]{\includegraphics[width=0.48\linewidth]{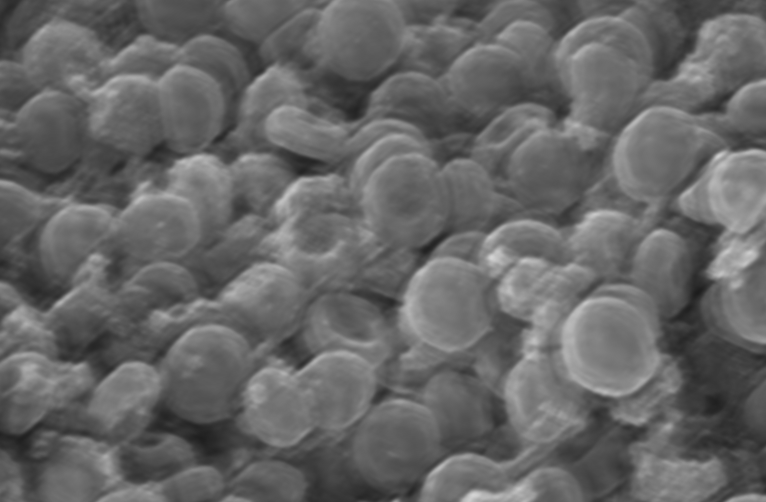}
\label{fig:DeblurR}
}
\subfigure[Hyper Laplace Prior result]{\includegraphics[width=0.48\linewidth]{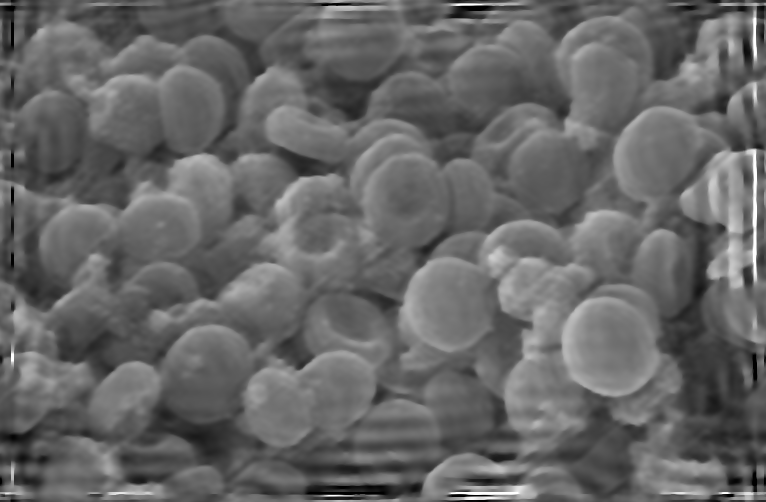}
\label{fig:DeblurLap}
}

\subfigure[normalized sparsity result]{\includegraphics[width=0.48\linewidth]{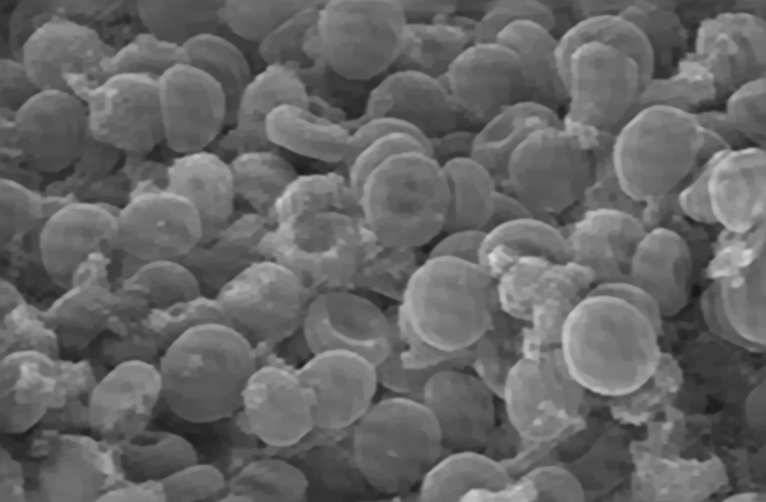}
\hspace{-0.23\linewidth}
\includegraphics[width=0.2\linewidth]{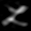}
\label{fig:DeblurL1L2}
}
\subfigure[present GDP result]{\includegraphics[width=0.48\linewidth]{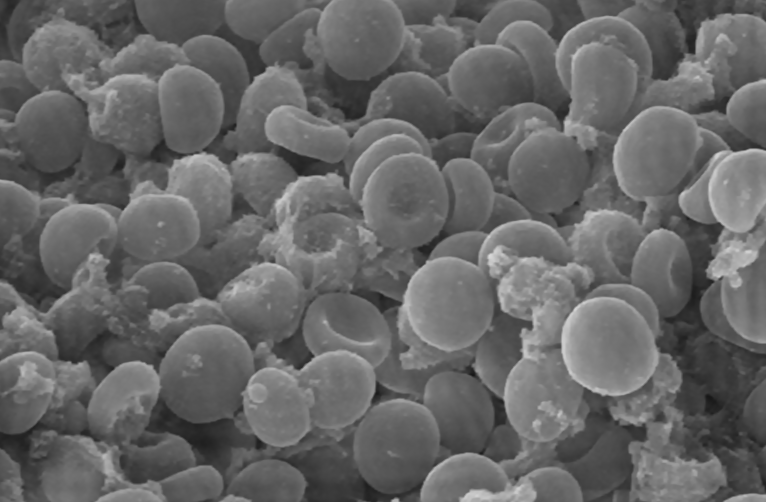}
\hspace{-0.22\linewidth}
\includegraphics[width=0.19\linewidth]{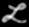}
\label{fig:DeblurF}
}
\subfigure[estimated $\hat{K}$ at different scales of our method]{\includegraphics[width=\linewidth]{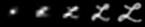}
\label{fig:DeblurKernel}
}
\caption{Image deconvolution example. (a) Ground truth image. (b) Input image to the deconvolution methods, obtained by blurring the image in (a) with the kernel shown in the inset. (c,d) Results from two non-blind deconvolution methods with the ground truth blur kernel given; the classical Lucy-Richardson algorithm~\citep{biggs1997acceleration} and the hyper-Laplace method~\citep{krishnan2009fast}. (e,f) Results from two blind deconvolution methods (\citep{KrishnanTF11} and our GDP) along with the estimated blur kernels (insets). (g) Blur kernel estimated by our method on different levels of the multi-scale process.}
 \label{fig:deblur}
\end{figure}

\begin{table}[h]
\scriptsize
\centering  % used for centering table
\begin{tabular}{c|cc} 
\hline\hline
  & PSRN & SSIM  \\
\hline
Blurred input image (Fig.~\ref{fig:blurF}) & 23.85 & 0.58\\
\hline
Lucy-Richardson (Fig.~\ref{fig:DeblurR}) & 27.44 & 0.72\\
\hline
Hyper-Laplacian (Fig.~\ref{fig:DeblurLap}) & 23.87 & 0.68\\
\hline
Normalized sparsity (Fig.~\ref{fig:DeblurL1L2}) & 27.34 & 0.73\\
\hline
Present method (Fig.~\ref{fig:DeblurF}) & 33.69 & 0.92\\
\hline 
\end{tabular}
\caption{Quality comparison of deconvolution results.} % title of Table
\label{table:deblurResult} % is used to refer this table in the text
\end{table}

A recent development in deconvolution is to use image patches instead of the whole image to accelerate kernel estimation~\citep{hu:eccv2012,patchmosaic}. This can easily be adopted also in our framework, provided the patches are large enough for the GDP to be valid within them (see Section~\ref{sec:DFW}).   

\subsection{Zooming and Super Resolution}
Zooming or super-resolution (SR) is the process of resampling an image (or a part of an image) onto a larger grid of pixels. Increasing the number of pixels in the image while keeping the field of view the same hence increases the image resolution. The interesting question is then how to interpolate the image information onto the finer pixel grid where no information is available on the course input grid. We show here how the same algorithm as for deconvolution can also be used for zooming. The only change is that we use an up-sampled $U$ (i.e., $U$ has more pixels than $I$) and a known Gaussian kernel $K(\vec{x}) = \frac{1}{\sqrt{ 2\pi}\sigma}e^{-\|\vec{x}\|^2/\sigma^2}$ in Eq.~\ref{eq:blurU}. We do not need to iterate Eq.~\ref{eq:blurK}, because the kernel is known in this application. An example is shown in Fig.~\ref{fig:SR}. For fun, we compare the resulting zoomed image with an image of the same sample acquired by a true super-resolution microscopy technique (here: PALM microscopy). While zooming with the present algorithm renders the image crisper (due to the deconvolution kernel) and better resolved (due to the finer pixel grid), it does not actually improve the optical resolution of the microscope. This can nicely be observed when two filaments cross. In the zoomed image there is a gap at the crossing point, whereas the PALM microscopy image properly resolves both filaments crossing.

\begin{figure}[H]
\centering
\setlength{\abovecaptionskip}{0cm}
\setlength{\belowcaptionskip}{0cm}
\setlength{\subfigbottomskip}{0cm}
\setlength{\subfigcapskip}{0cm}
\setlength{\subfigtopskip}{0cm}
\subfigure[original $128\times 128$]{\includegraphics[width=.32\linewidth]{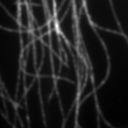}}
\subfigure[zoomed $512\times 512$ ]{\includegraphics[width=.32\linewidth]{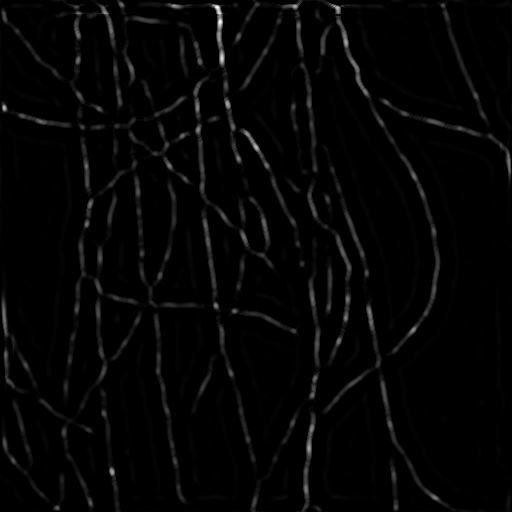}}
\subfigure[PALM $512\times 512$]{\includegraphics[width=.32\linewidth]{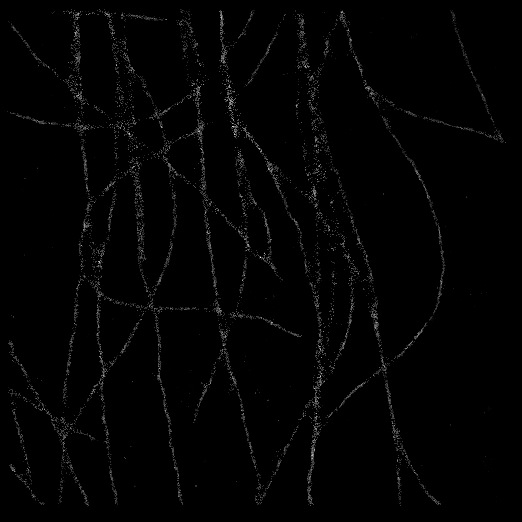}}
  \caption{Zooming using the GDP. Panel (b) shows the zoomed version (up-sampling factor 4) of the fluorescently labeled microtubules in (a) as computed using the present method. Panel (c) shows a real super-resolution PALM image of the same scene for comparison. {((a)\&(c) from: EPFL Collection of Reference Datasets, bigwww.epfl.ch/smlm/datasets/index.html?p=real-hd)}}
  \label{fig:SR} 
\end{figure} 

\subsection{Scatter Light Removal and Dehazing}
Scatter light is a common nuisance in light microscopy when imaging thick samples. The light propagating though the sample is scattered (Rayleigh  and Mie scattering), similarly to how fog or haze scatters light in a natural-scene photograph. The resulting image is the superposition of the scatter light and the latent image. In classical dehazing methods, the observed image $I$ is modeled as~\citep{DarkPrior}:
\begin{equation}
\label{eq:dehaze}
I(\vec{x}) = U(\vec{x})t(\vec{x}) + A(1-t(\vec{x}))\, ,
\end{equation} 
where $U$ is the latent image, $t(\vec{x})=e^{-\beta d(\vec{x})}$ is the unknown transmission map, and $A$ is the environment light constant. The unknown parameter $\beta$ is a material constant (scattering coefficient), and $d(\vec{x})$ is the distance from the scene to the camera. Solving this model for $U$ is ill-posed. A popular prior to regularize the problem in the spatial domain is the {\em dark-channel prior}~\citep{DarkPrior}. Alternatively, the problem can be regularized in a Bayesian framework~\citep{nishino:2012}. Here, we impose the GDP for the latent image and TV for the transmission map as hard constraints:
\begin{equation}
\label{eq:dehazeOur}
\begin{split}
\mathcal{E}(U,t)= &\frac{1}{2}\left\|U(\vec{x})t(\vec{x})+ A(1-t(\vec{x}))-I(\vec{x})\right\|_2^2\\
&s.t.~~p(U(\vec{x}))=p^{\mathrm{pr}},~p(t(\vec{x}))=p^{\mathrm{pr}}.
\end{split}
\end{equation}

Unlike previous works, Eq.~\ref{eq:dehaze} does not hold anymore in our model. 
Instead, our model can be written as: 
\begin{equation}
\label{eq:dehazeGDP}
\begin{split}
& \mathcal{E}(U,t)=\frac{1}{2}\left\|U(\vec{x})t(\vec{x})+ A(1-t(\vec{x}))-I(\vec{x})\right\|_2^2+\frac{\lambda}{2}\\
&\left[T^2\|\nabla U(\vec{x})\|_2^2+\log(b+\|\nabla U(\vec{x})\|^2_2)+\alpha \left\|\frac{\nabla t(\vec{x})}{t(\vec{x})}\right\|_1\right].
\end{split}
\end{equation} 
We use alternating minimization over $U$ and $t$ to obtain the final result. An example is shown in Fig.~\ref{fig:dehaze}. Using simple Otsu thresholding on the original image does not allow detecting any objects in the image. Instead, they are fused to one large blob by the scatter light. After dehazing using the present method, the same simple Otsu thresholding allows object segmentation.

\begin{figure}[H]
\centering
\setlength{\abovecaptionskip}{0cm}
\setlength{\belowcaptionskip}{0cm}
\setlength{\subfigbottomskip}{0cm}
\setlength{\subfigcapskip}{0cm}
\setlength{\subfigtopskip}{0cm}
\subfigure[Original]{\includegraphics[width=0.49\linewidth]{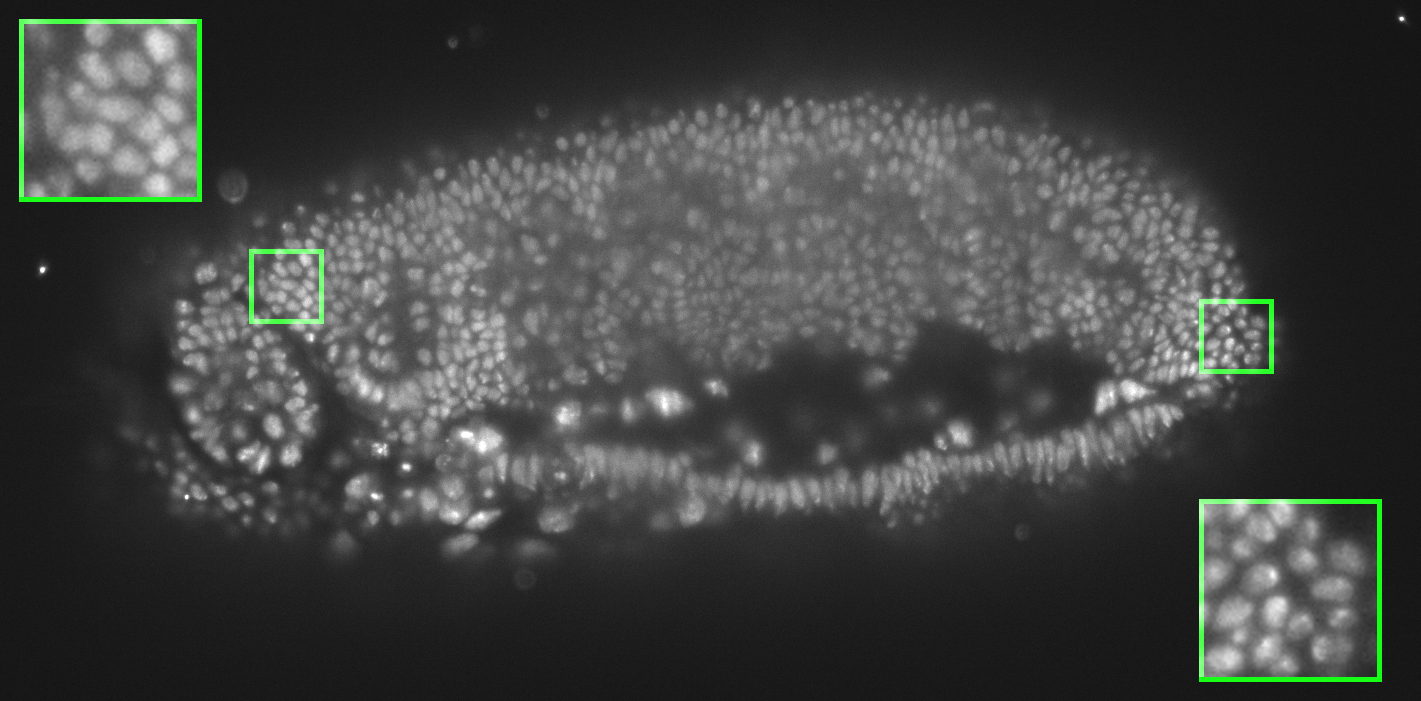}}
\subfigure[Present dehazing result]{\includegraphics[width=0.49\linewidth]{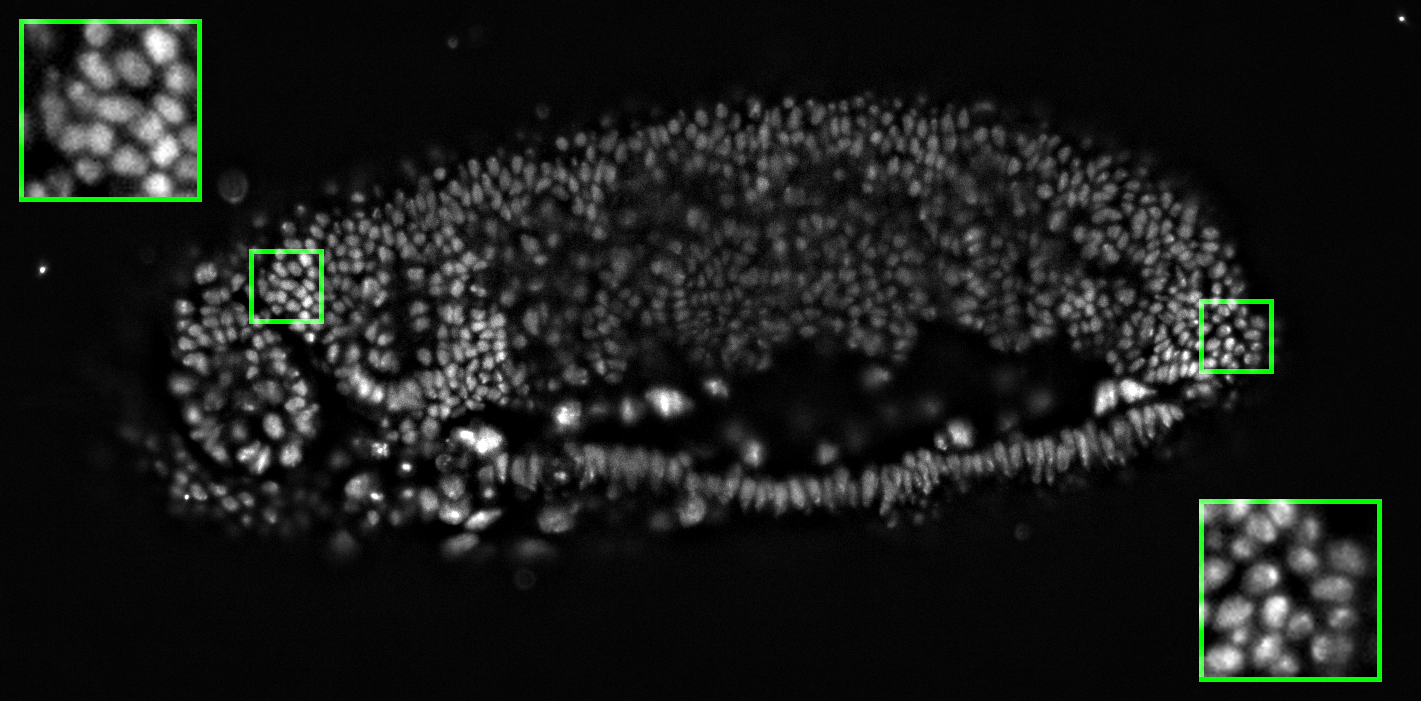}}
\subfigure[Otsu thresholding of (a)]{\includegraphics[width=0.49\linewidth]{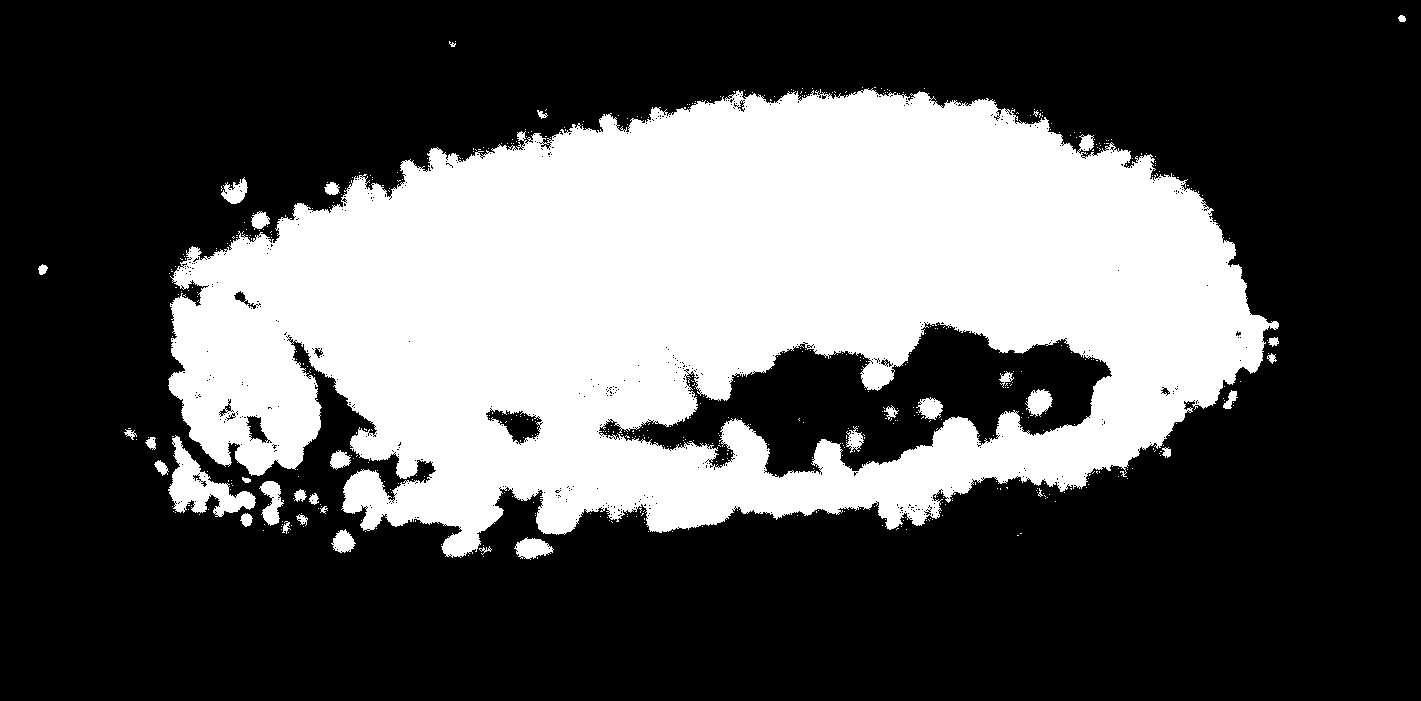}}
\subfigure[Otsu thresholding of (b)]{\includegraphics[width=0.49\linewidth]{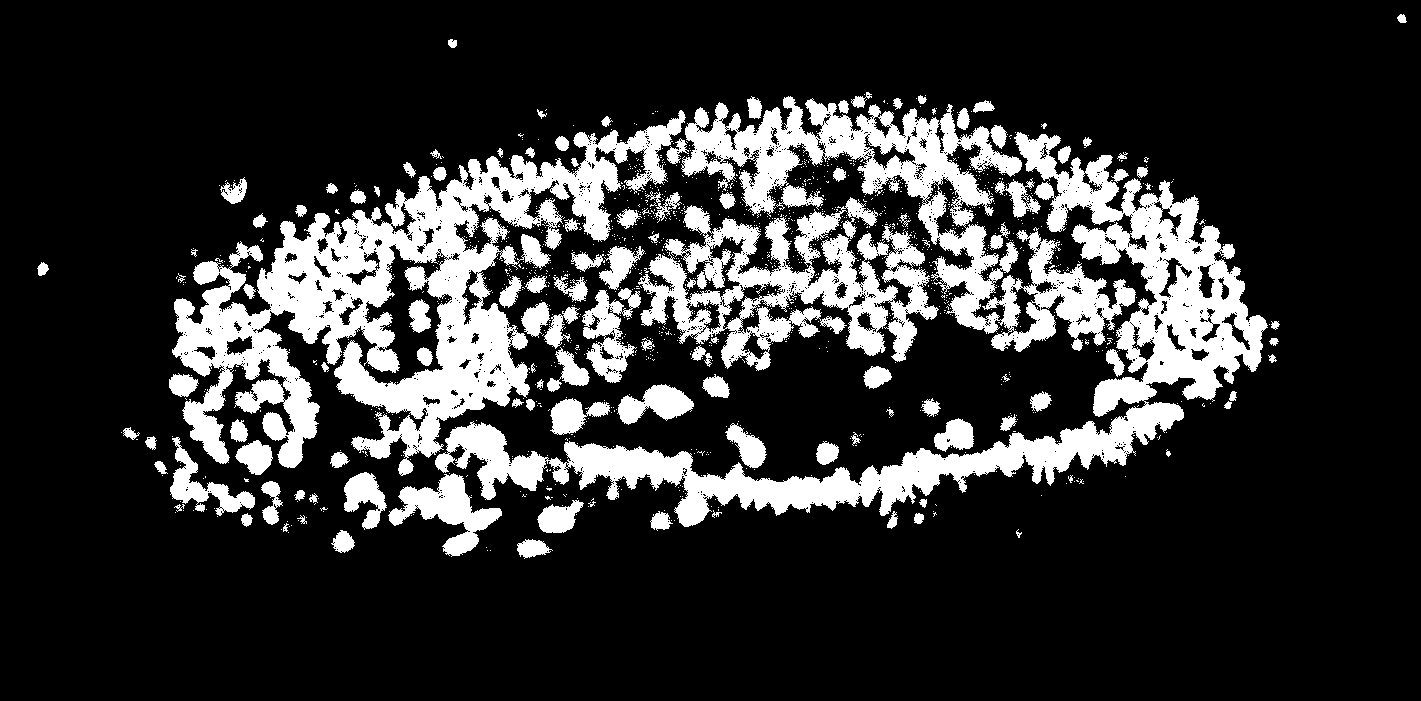}}
\caption{Scattering light removal in a SPIM microscopy image of a whole {\em Drosophila} embryo with the Nuclei labeled by fluorescence. (a) Original image as recorded by SPIM microscopy (source: Tomancak lab, MPI-CBG).
Due to the thickness of the sample, there is significant scatter light, prohibiting object segmentation using thresholding (c). (b) Result from the present dehazing method, enabling object thresholding (d). Insets show zoomed details as indicated by the green boxes.}
 \label{fig:dehaze} %% label for entire figure
\end{figure} 

\subsection{The Naturalness Factor as an Image Feature}
The naturalness factor $N_{\!f}$ is a scalar number that is easy to compute. It can hence provide an interesting image feature, for example in classification or machine-learning frameworks when the naturalness of an image is to be quantified. We illustrate this by classifying transmitted light microscopy images of marine phytoplankton from natural-scene images. We collected a dataset of 1322 images of 45 different species of phytoplankton. Some example images are shown in Fig.~\ref{fig:sea}(a). Figure~\ref{fig:sea}(b) shows the histogram of $N_{\!f}$ of all images. Despite the fact that the images look visually diverse, they mostly have similar $N_{\!f}$, which moreover is clearly different from that of natural-scene images (see Fig.~\ref{fig:NfDistr}). The naturalness factor could hence be used as an important feature to classify these images.

\begin{figure}[h]
\subfigure[sample images]{\includegraphics[width=0.48\linewidth]{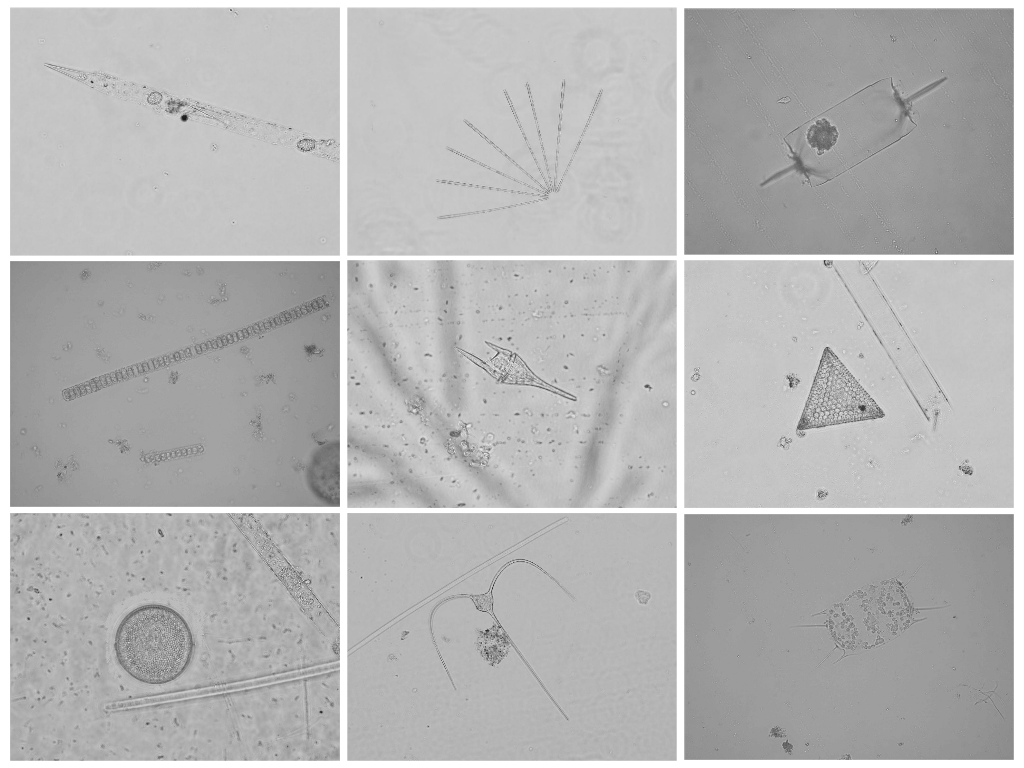}}
\subfigure[$N_{\!f}$ histogram]{\includegraphics[width=0.48\linewidth]{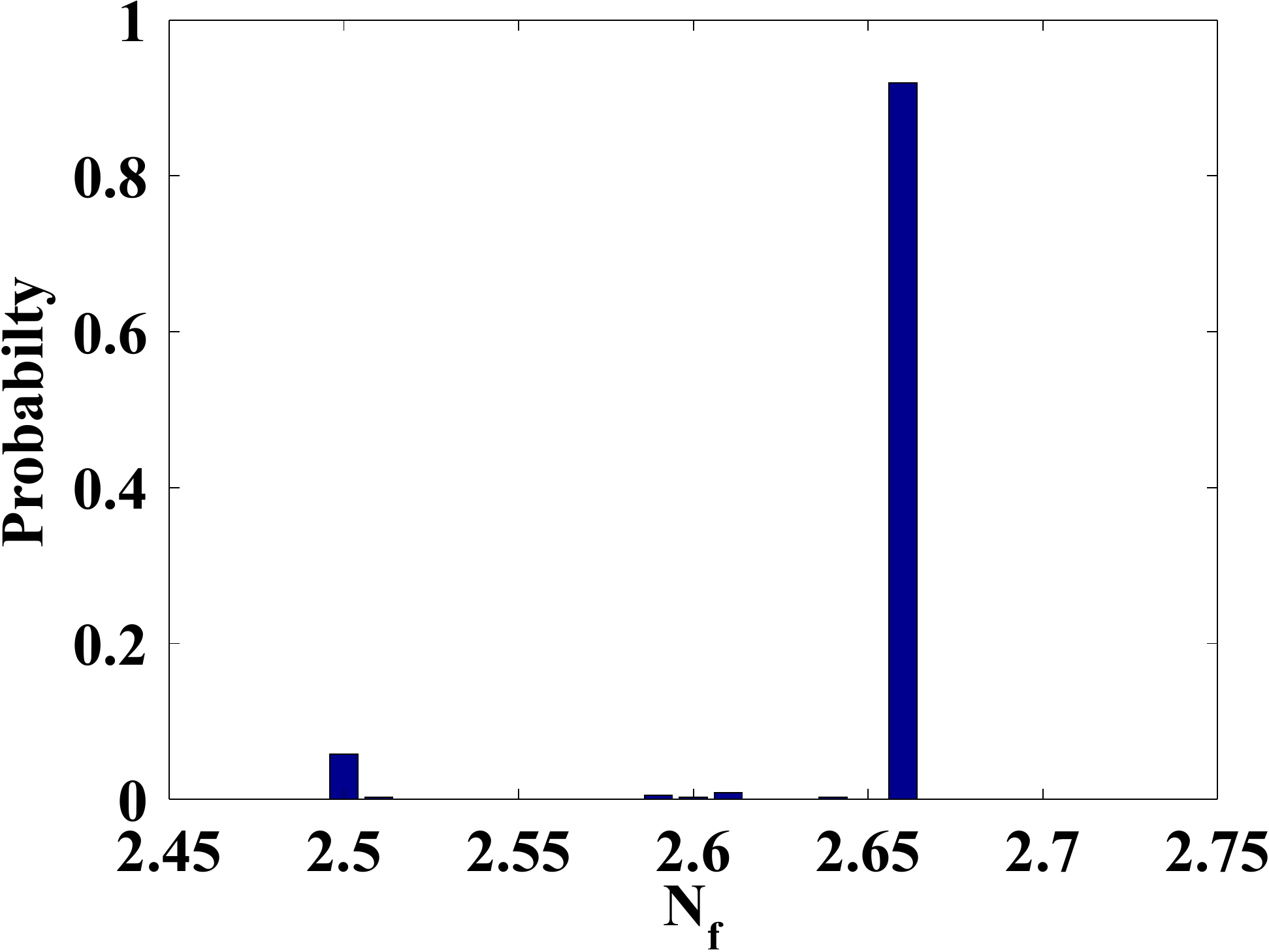}}
\caption{$N_{\!f}$ distribution of light-microscopy images of marine phytoplankton. Even though the images show plankton of very different morphologies, their $N_{\!f}$ is mostly the same, hence providing a classification feature for these images.}
\label{fig:sea}
\end{figure}
 
\section{Conclusion and Discussion}
\label{sec:DFW}
We proposed learning a gradient distribution prior (GDP) for biomedical images from natural-scene images. We provided different lines of argument why we believe this is worthwhile doing, and we have established that the resulting prior is stable and correlated with image quality. We have provided novel parametric models for GDP. Our models are of different accuracy and complexity, some of them leading to very simple convex problems. We have illustrated this in various applications, ranging from image enhancement to denoising, deconvolution, and dehazing. In all cases, the present GDP models led to results that were comparable or better than the existing state of the art in the respective field of application. We have further established a relationship between our GDP, traditional TV regularization, and anisotropic diffusion.

While we have exclusively focused on the image gradient here, the same work could also be done for higher-order derivatives, like the Laplacian. Spectral statistics of higher-order differential operators could provide additional regularization in the same framework. Of special interest could also be the mean or Gaussian curvature (GC) distributions~\cite{gong2013a}, as it directly relates to the geometry of cell membrane through the Willmore energy. As shown in Fig.~\ref{fig:secondOrder}, all of these second-order derivatives satisfy similar distributions as the gradient. Using GC as a prior is well known to better preserve edges in the image than mean curvature or the Laplacian. Figure~\ref{fig:secondOrder}, however, suggests the opposite. This needs more research. An interesting observation is that the naturalness factor derived from the distribution of the Laplacian is highly correlated with that derived from the gradient distribution (Fig.~\ref{fig:NfCorr}). This suggests that the naturalness factor is a universal image feature that does not depend on the order of the statistic over which it is defined. Confirming this, however, is still outstanding. 

\begin{figure}[h]
  \centering
  {\includegraphics[width=0.9\linewidth]{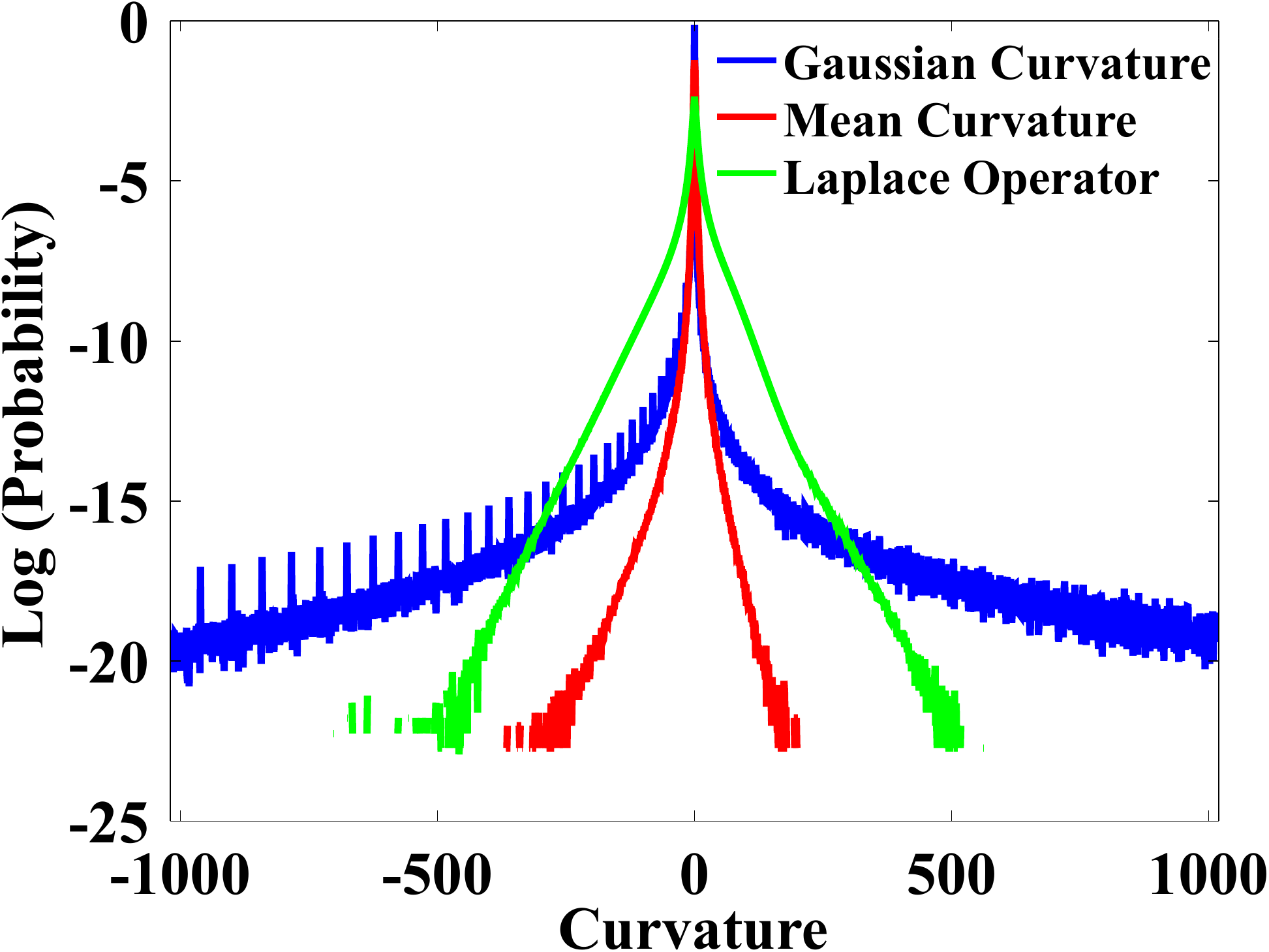}}
  \caption{Average distribution of Gaussian curvature, mean curvature, and Laplace operator response across all training images of our natural-scene image dataset.}
  \label{fig:secondOrder} 
\end{figure}

\begin{figure}[h]
  \centering
  {\includegraphics[width=0.9\linewidth]{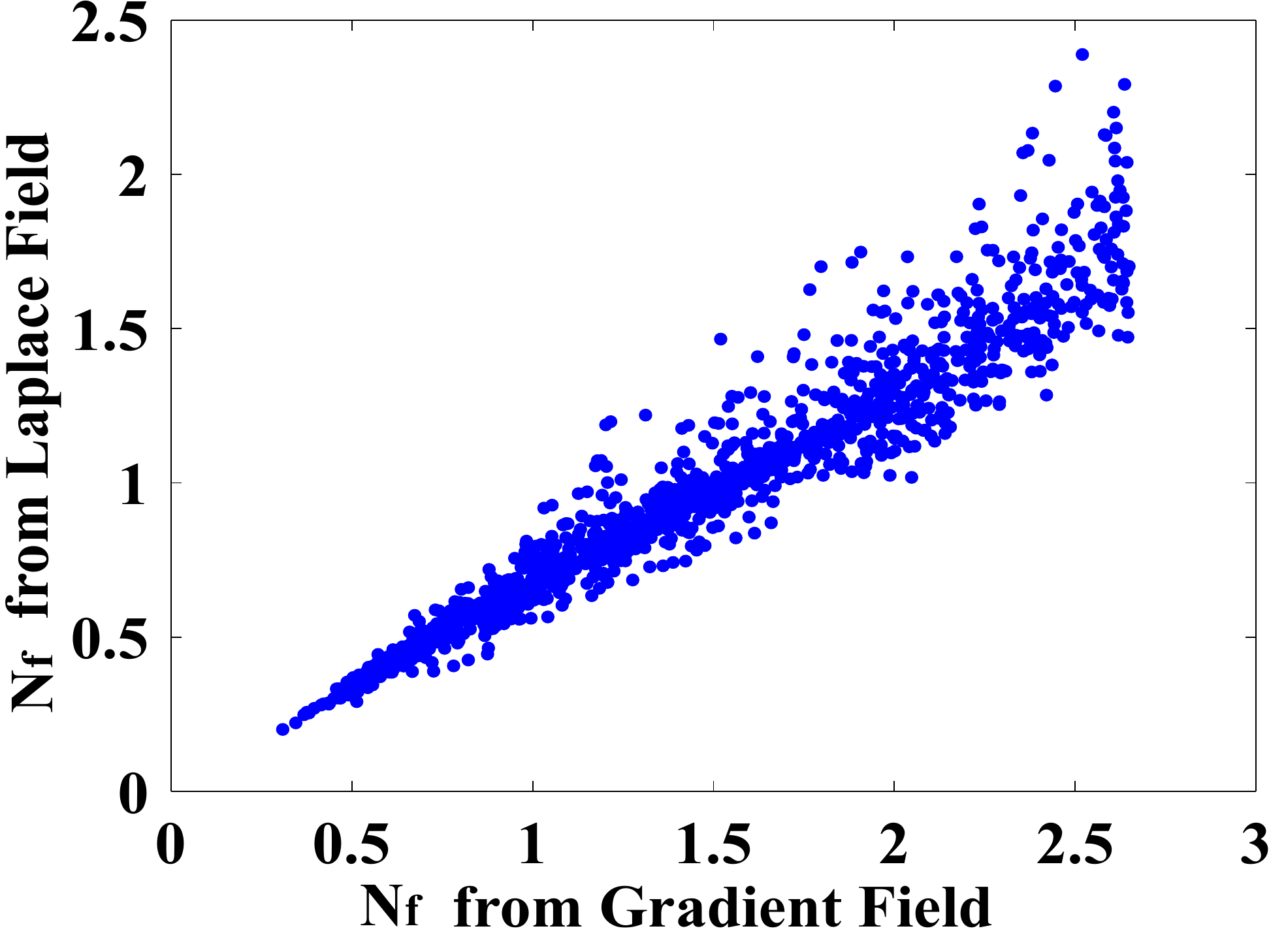}}
  %IFS: insert missing spaces between N_f and next word in both axes labels.
  \caption{the naturalness factors computed from the gradient and the Laplacian distributions are highly correlated.}
  \label{fig:NfCorr} 
\end{figure}

The present work can also be extended to higher-dimensional images. Constructing and using, for example, a GDP for 3D biomedical images is straightforward. The parametric models presented here are simple sums or products of 1D models and can hence trivially be extended to higher dimensions. As an example, Fig.~\ref{fig:3Dguess} shows the predicted gradient distribution in 3D using our Model 1. However, it is not easy to verify this result, because there are (almost) no 3D natural-scene images, and 3D biomedical images are corrupted by noise and blur.

\begin{figure}[h]
  \centering
  {\includegraphics[width=0.9\linewidth]{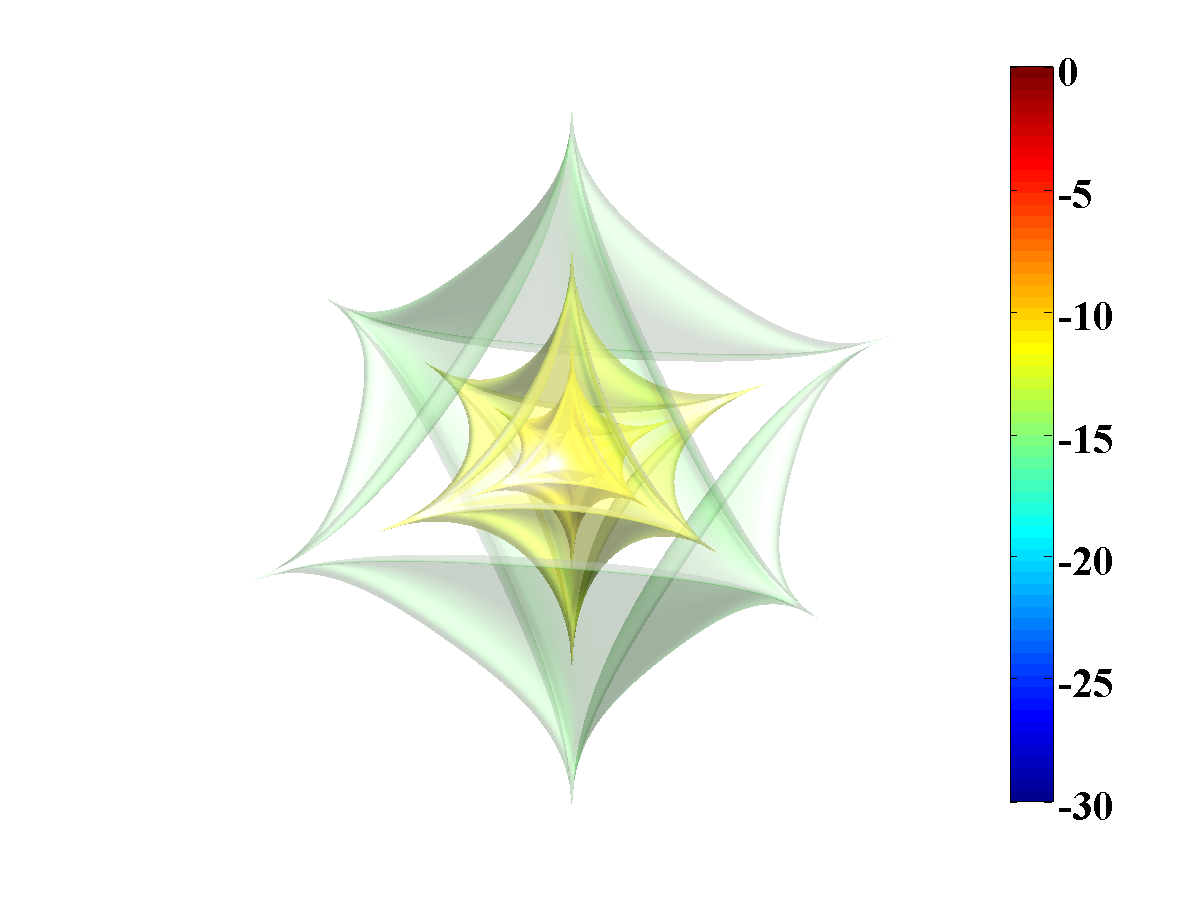}}
  \caption{Predicted 3D gradient PDF in log-scale from our Model 1. Isosurfaces are shown for -14, -12, -10, and -8. }
  \label{fig:3Dguess} %% label for entire figure
\end{figure}

Since the GDP is stable with image contents, it is also valid on sub-images and image regions. As shown in Fig.~\ref{fig:local}(a,b), the gradient distribution is insensitive to the position, size, and shape of the patch. This is confirmed for an electron-microscopy image in Fig.~\ref{fig:local}(c,d). The image shows a transmission electron micrograph (ssTEM) of the {\em Drosophila} first instar larva ventral nerve cord (VNC) with a resolution of $4\times 4\times 50$\,nm/pixel~\citep{cardona:2010}. The gradient distributions within manually segmented mitochondria and synapses (all regions pooled) are almost identical, as shown in Fig.~\ref{fig:local}(d). This suggests that the characteristic gradient distribution is maybe more of a function of the imaging process than of the imaged objects. It also shows how the GDP can be straightforwardly extended to multi-region methods. 

\begin{figure}[h]
  \centering
  \subfigure[Different local image patches.]{
  \includegraphics[width=0.48\linewidth,height=0.36\linewidth]{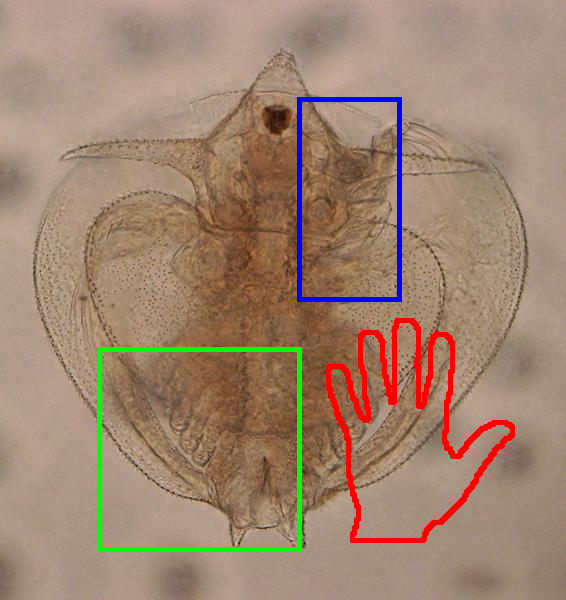}}
  \subfigure[Gradient distributions in these patches.]{
    \includegraphics[width=0.48\linewidth]{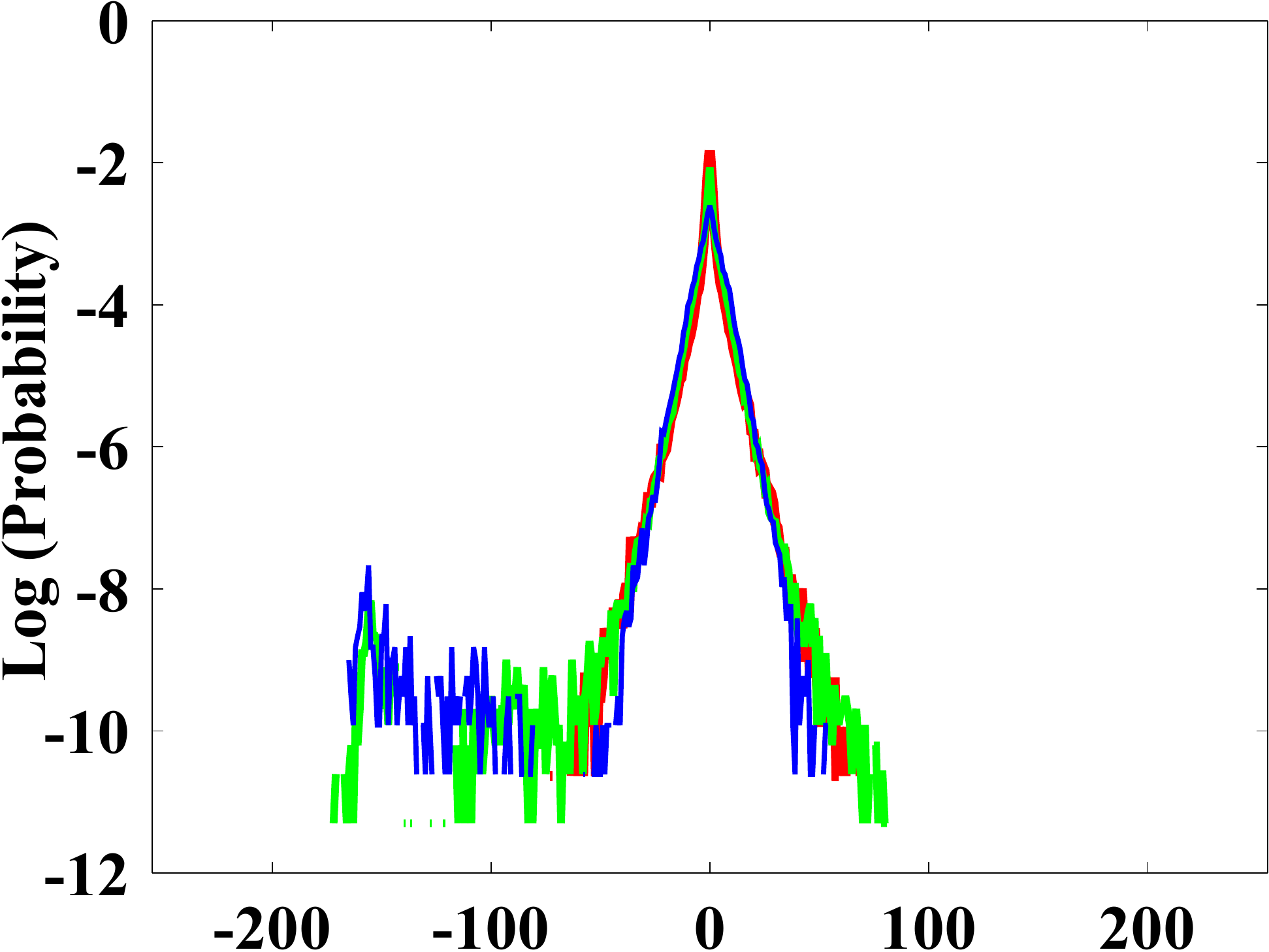}}
    \subfigure[Mitochondria (red) and synapses (blue) imaged by electron microscopy.]{
  \includegraphics[width=0.48\linewidth,height=0.36\linewidth]{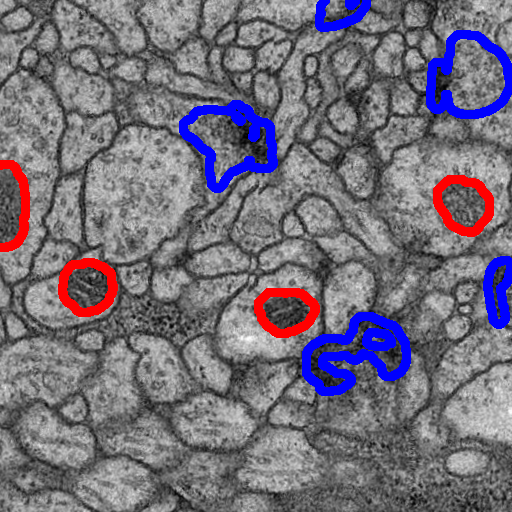}}
  \subfigure[Average gradient distributions.]{
    \includegraphics[width=0.48\linewidth]{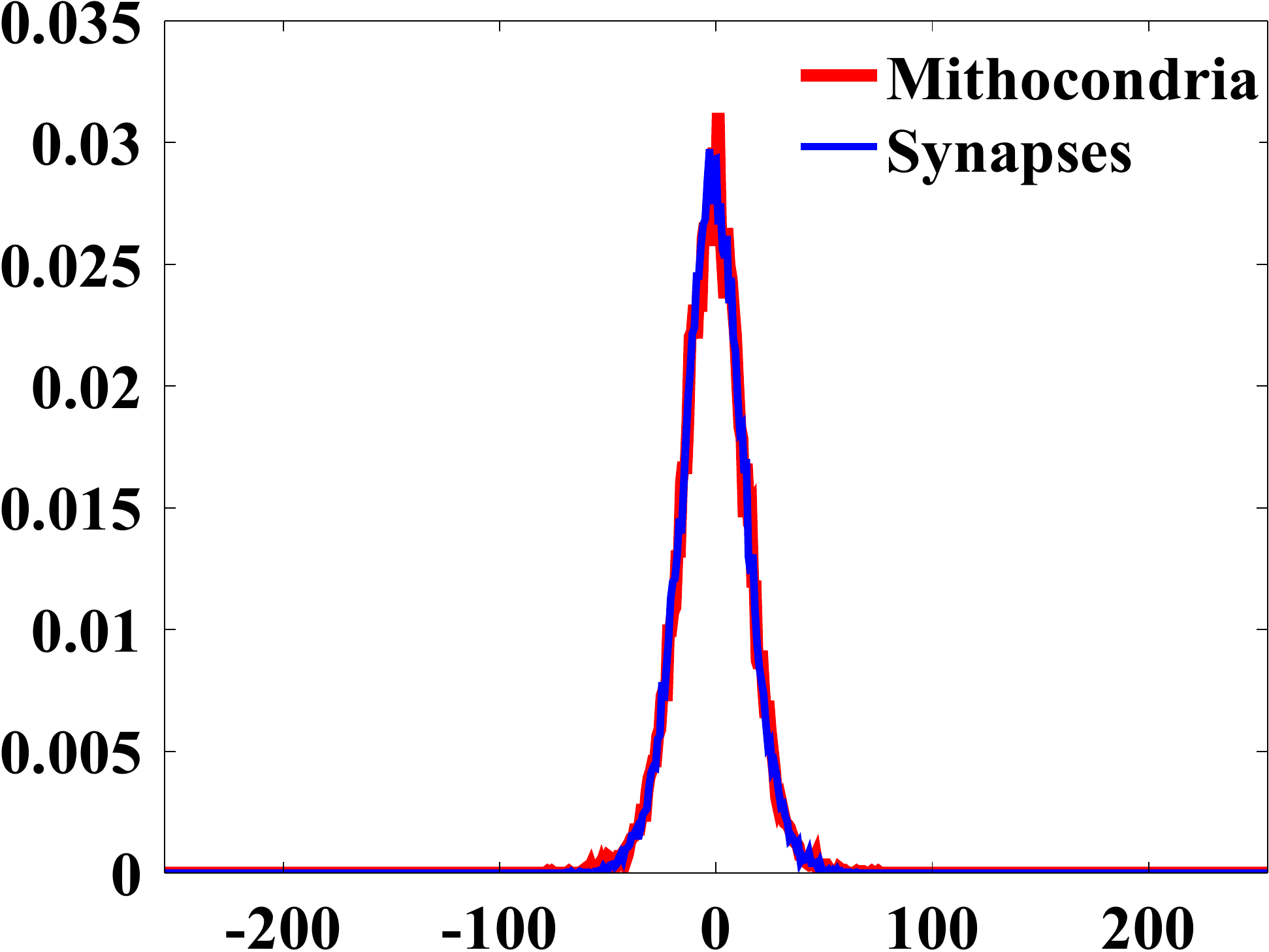}}
\caption{Gradient distributions in local image regions are invariant, provided the regions are large enough.}
 \label{fig:local} 
\end{figure}

Clearly, the GDP loses its validity when applied to small image patches that contain little or no internal structure. As the size of a local window decreases, we transition from a macroscopic view (entropy) to a microscopic view (pixel histogram). Since the GDP is a macroscopic quantity, it is only valid for large-enough image patches. But how large is large enough? Unfortunately, there is no sharp transition.  
To quantitatively see this, we define the naturalness map for a local window of edge length $w$ as:
\begin{equation}
N_w(x,y) = \iint p(\vec{G})\log \left(\frac{p(\vec{G})}{p^{\mathrm{pr}}}\right)\mathrm{d}G^{\hat{x}}\mathrm{d}G^{\hat{y}},
\end{equation} 
where $\hat{x}\in[x-w,x+w],\hat{y}\in[y-w,y+w]$. This quantifies the distance (KL-divergence) between the GDP and the gradient distribution in each local window. Computed for every image patch, this provides a map of how the image naturalness varies across patches. Two examples are shown in Fig.~\ref{fig:localChangeWindow}. When the window size decreases from 60 to 8, the average and median value of $N_w$ across all patches are plotted in Fig.~\ref{fig:localChangeWindow}(h,p). These plots show how the gradient distribution gradually diverges from the macroscopic GDP as the window size decreases. It seems that this behavior of prior invalidation is independent of image contents, as shown in Fig.~\ref{fig:localChangeWindowOther}. This is unexpected and requires further investigation. 

\begin{figure*}[!htb]
\centering
\subfigure[original]{\includegraphics[width=.24\linewidth]{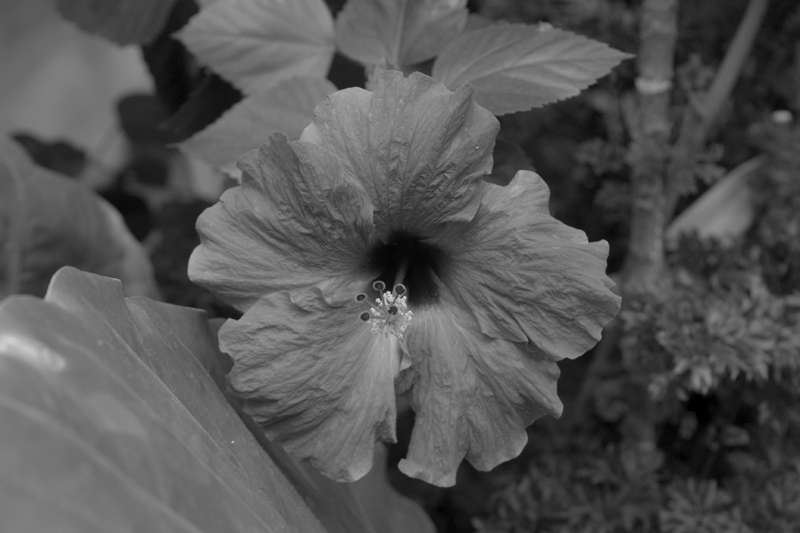}}
\subfigure[$w=56$]{\includegraphics[width=.24\linewidth]{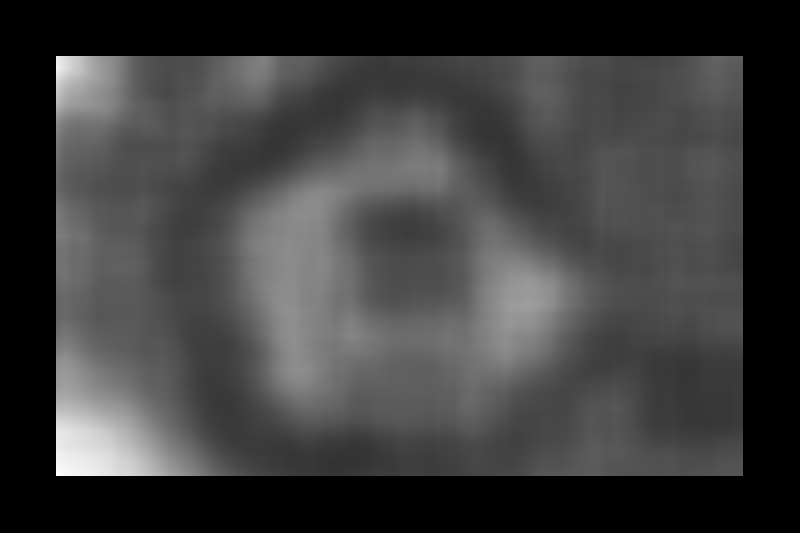}}
\subfigure[$w=48$]{\includegraphics[width=.24\linewidth]{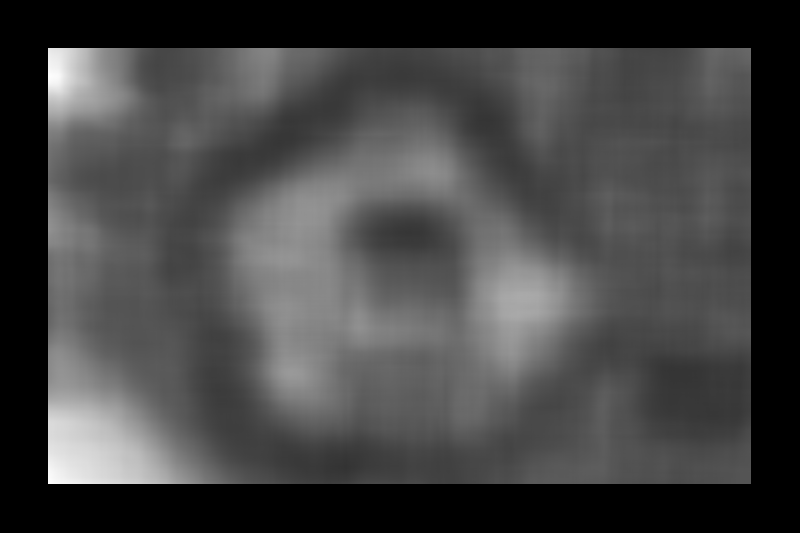}}
\subfigure[$w=40$]{\includegraphics[width=.24\linewidth]{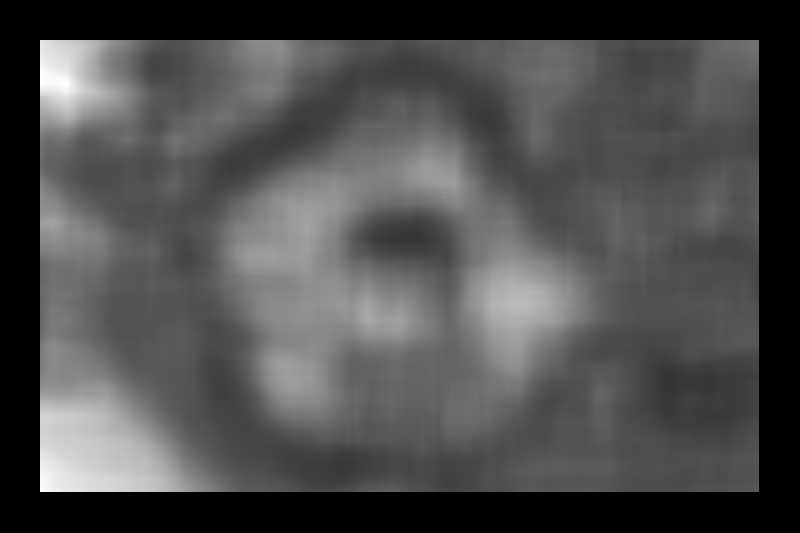}}
\subfigure[$w=32$]{\includegraphics[width=.24\linewidth]{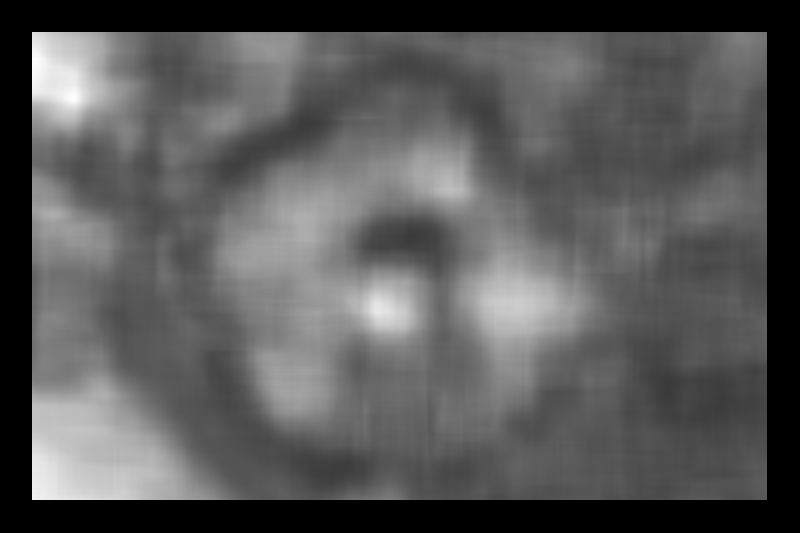}}
\subfigure[$w=16$]{\includegraphics[width=.24\linewidth]{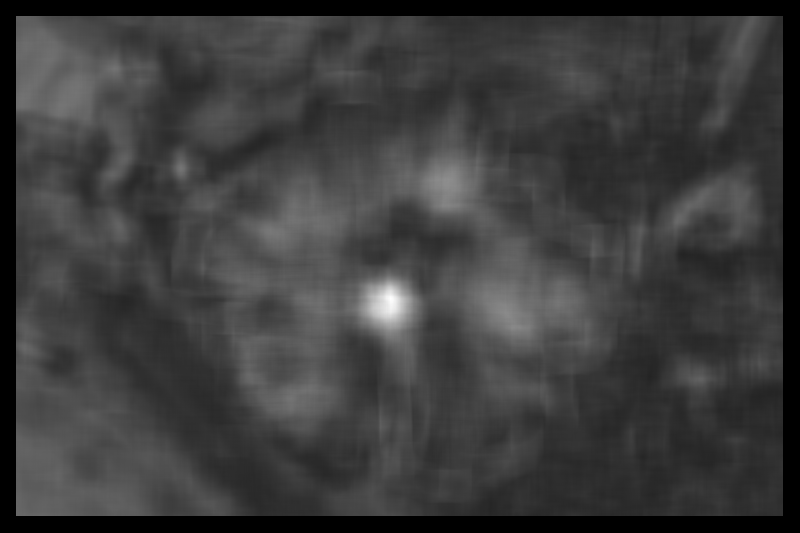}}
\subfigure[$w=8$]{\includegraphics[width=.24\linewidth]{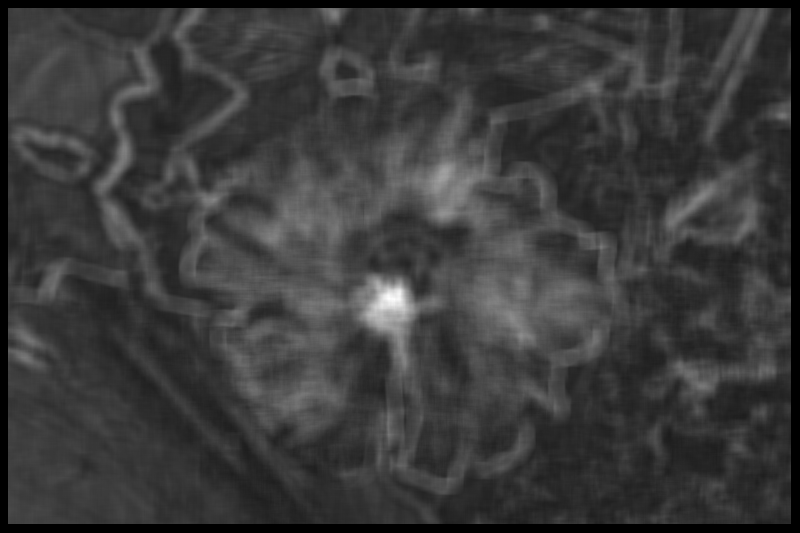}}
\subfigure[$N_w$ average and median]{\includegraphics[width=.24\linewidth, height=0.16\linewidth]{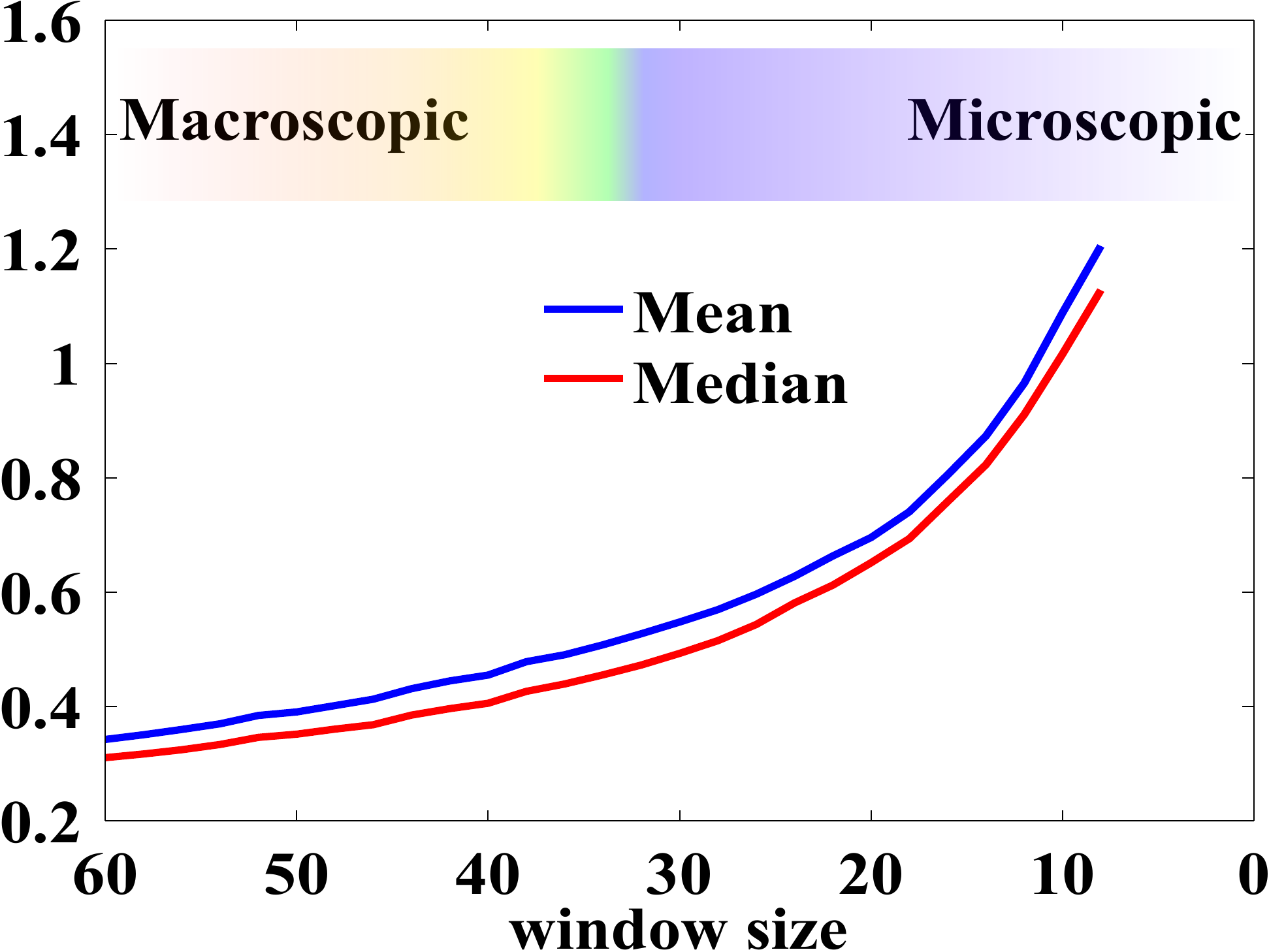}}

\subfigure[original]{\includegraphics[width=.24\linewidth]{deblur_source.png}}
\subfigure[$w=56$]{\includegraphics[width=.24\linewidth]{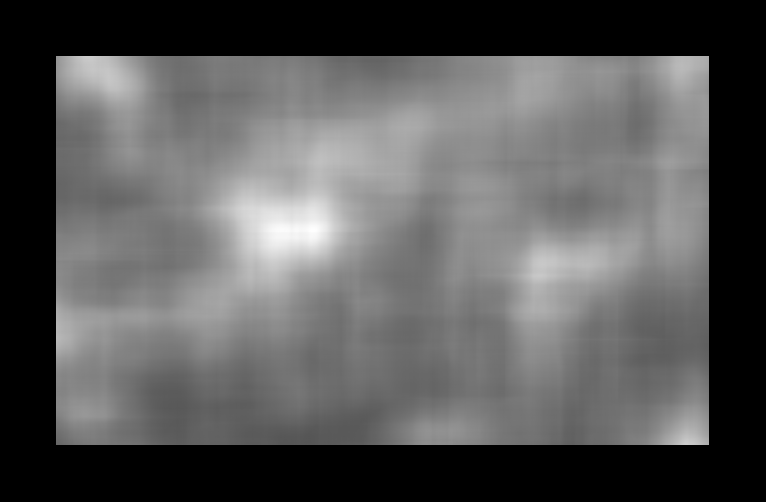}}
\subfigure[$w=48$]{\includegraphics[width=.24\linewidth]{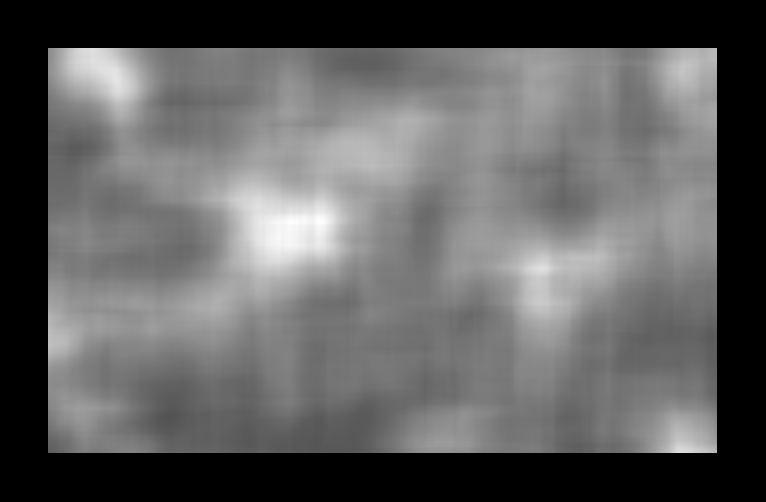}}
\subfigure[$w=40$]{\includegraphics[width=.24\linewidth]{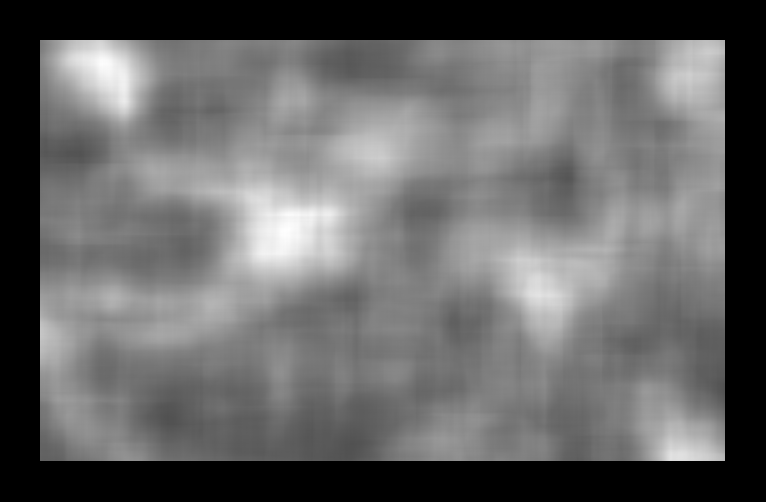}}
\subfigure[$w=32$]{\includegraphics[width=.24\linewidth]{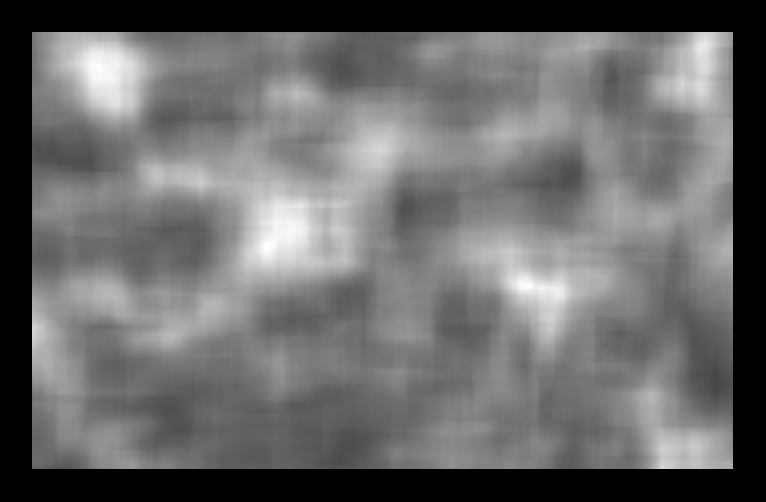}}
\subfigure[$w=16$]{\includegraphics[width=.24\linewidth]{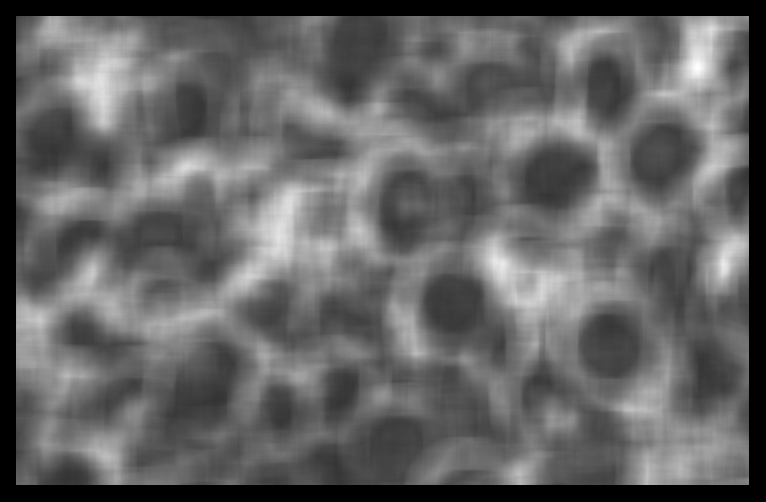}}
\subfigure[$w=8$]{\includegraphics[width=.24\linewidth]{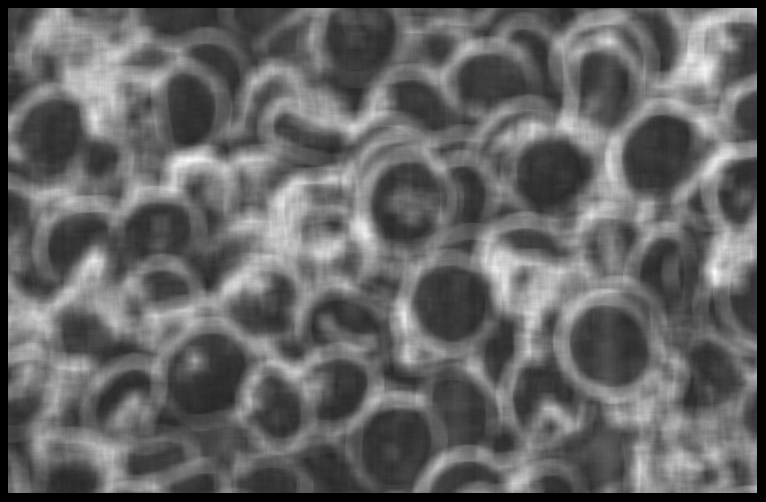}}
\subfigure[$N_w$ average and median]{\includegraphics[width=.24\linewidth, height=0.16\linewidth]{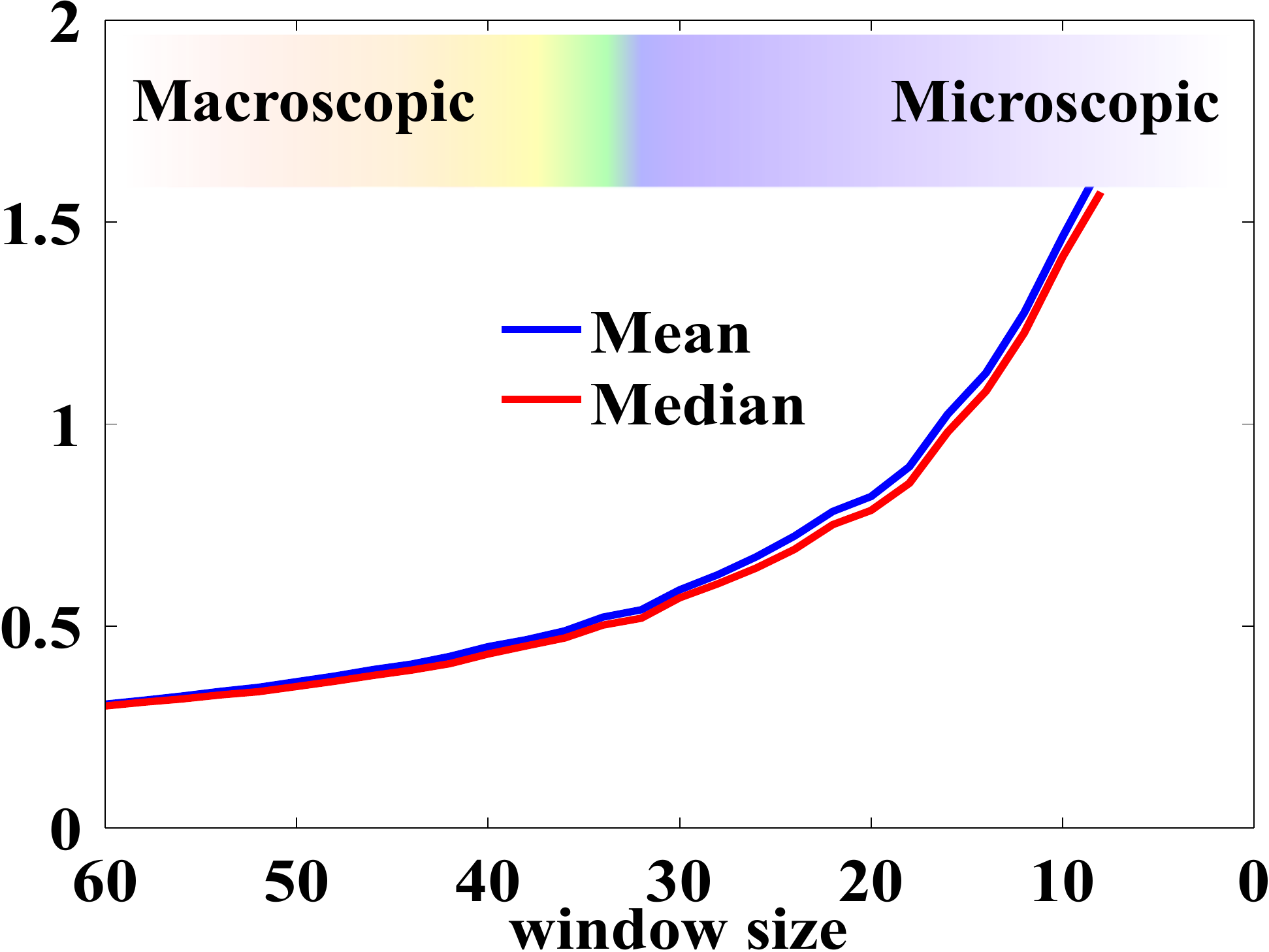}}
\caption{Decreasing the local window size $w$, the gradient distribution prior increasingly differs from the empirical distribution within the windows. The original image is shown followed by naturalness maps computed in increasingly smaller moving window sizes $w$. The last panel shows how both the mean and the median distance $N_w$ between the GDP and all local windows gradually grow with decreasing window size $w$.}
 \label{fig:localChangeWindow} %% label for entire figure
\end{figure*} 

\begin{figure*}[!htb]
\centering
\subfigure[parrots]{\includegraphics[width=.24\linewidth]{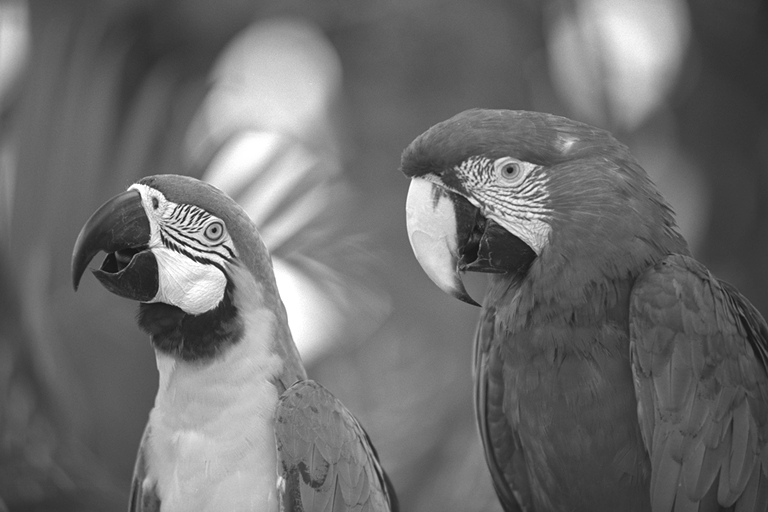}}
\subfigure[monarch]{\includegraphics[width=.24\linewidth]{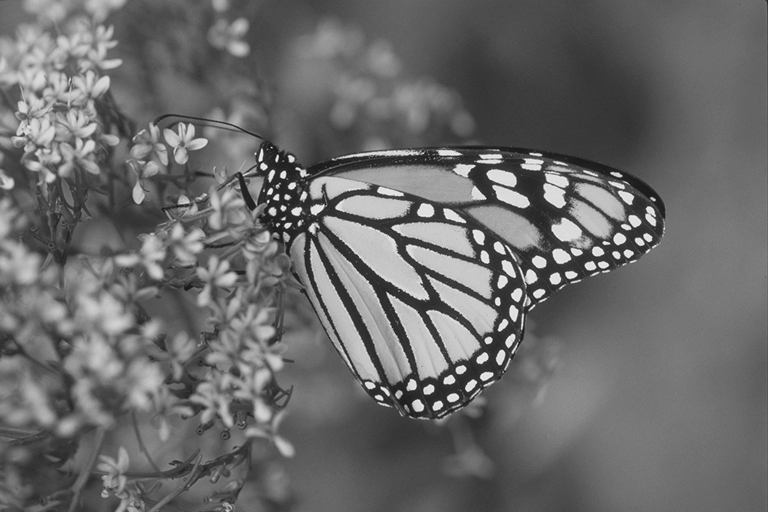}}
\subfigure[Lena]{\includegraphics[height=0.16\linewidth]{lena.jpg}}
\subfigure[$N_w$ average]{\includegraphics[width=.24\linewidth, height=0.16\linewidth]{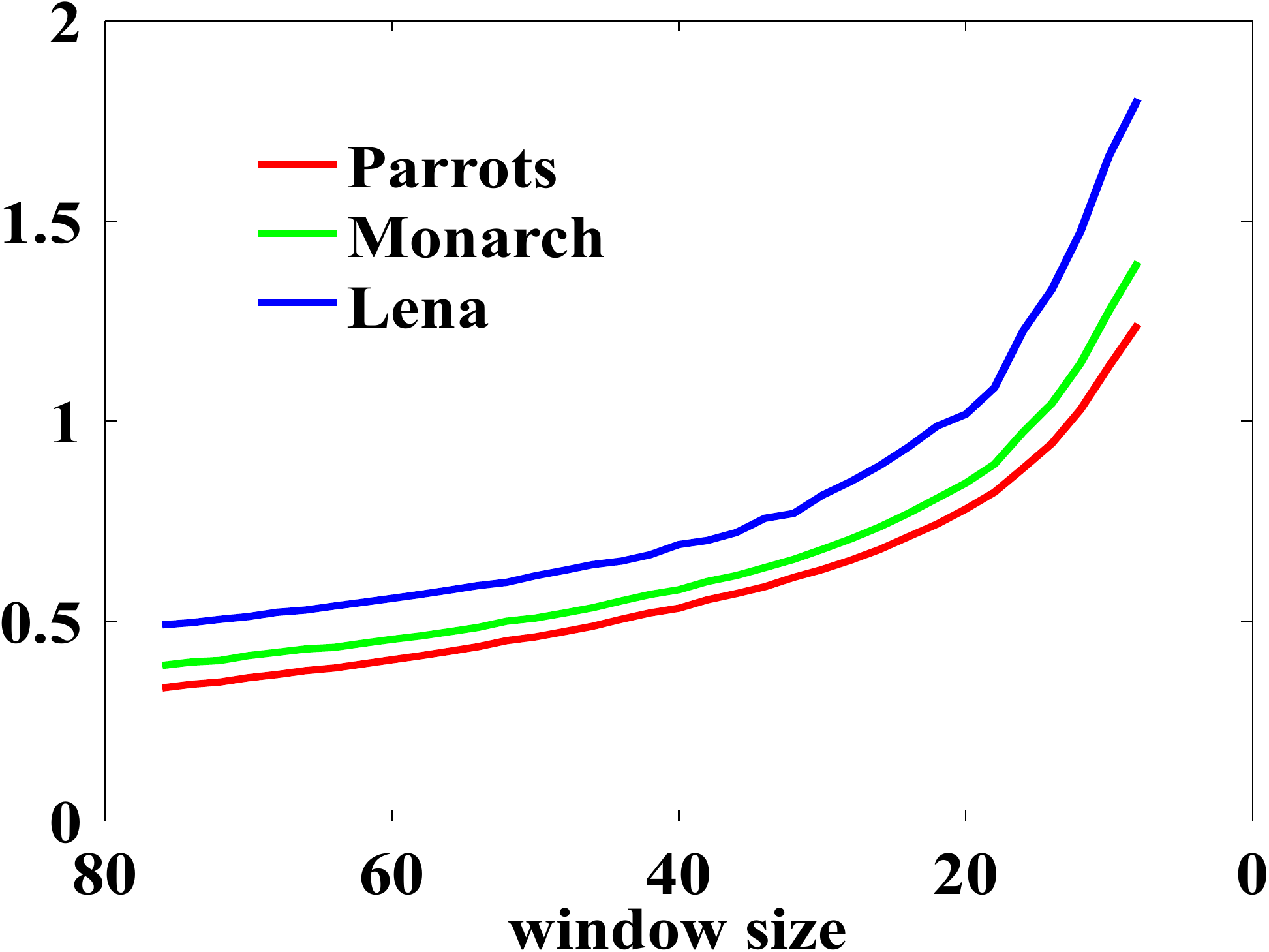}}
\caption{Three examples suggesting that the behavior of $N_w$ with patch size is independent of image contents.}
 \label{fig:localChangeWindowOther} %% label for entire figure
\end{figure*} 

%How the gradient distribution prior links the environment and human visual neurons is still an open research problem. On one hand, the retina neural circuit in human eyes is a complex system, composed by more than fifty types of cells which vary from shapes to arrangement~\citep{masland:2001,gollisch:2010}. And the function of each cell is not completely understood as well as how they synchronize their functionality. On the other hand, the environment contains different geometry (landscape, trees, flowers etc) and different light condition (sunshine, cloudy, sunset, scattering etc). There might be physical laws to obey. For example, the gradient distribution prior might indicate the visible light of black body radiation, covering particular range of temperature. But how those physical laws is related with gradient distribution prior still needs more investigation.     
Notwithstanding the many open questions and limitations of gradient distribution priors, they have repeatedly proven extremely useful and competitive in image processing. The results presented here add to this. We believe, for the arguments set out here, that efficient parametric models of GDP learned from natural-scene images are a powerful and well-founded tool for biomedical image processing. In order to make this available to the community, we provide open-source code of all models presented here on our MOSAIC Group web page. We provide implementations in Matlab, C++  (included in the OpenCV library), and Java (as an ImageJ/Fiji plugin).

\section*{Acknowledgements}
We thank all researchers from the MOSAIC Group for the many inspiring discussions. We also thank Prof.~Dr.~Carsten Rother (Computer Vision Lab Dresden) for his feedback on the manuscript. Y.G.~was funded by a grant from the Swiss National Science Foundation, grant CRSII3-132396/1, awarded to I.F.S. This work was supported in parts by the German Federal Ministry of Research and Education (BMBF) under funding code 031A099.

\appendix

\section{Proof of Lemma~\ref{lem:w}}
\begin{proof}
\nonumber
\notag
Let $v=\|\nabla U\|_2^2$ and $W=T^2_{\mathrm{pr}}+\frac{b_{\mathrm{pr}}-v}{(b_{\mathrm{pr}}+v)^2}=0$.
Then we have quadratic equation $T^2_{\mathrm{pr}} (b_{\mathrm{pr}}+v)^2+b_{\mathrm{pr}}-v=0$.

If $T^2_{\mathrm{pr}}b_{\mathrm{pr}}\geq \frac{1}{8}$, then $W\geq 0$. 

If ~$T^2_{\mathrm{pr}}b_{\mathrm{pr}}< \frac{1}{8}$, then $v_L = \frac{1-2T^2_{\mathrm{pr}}b_{\mathrm{pr}}-\sqrt{1-8T^2_{\mathrm{pr}}b_{\mathrm{pr}}}}{2T^2_{\mathrm{pr}}}$ and $v_U= \frac{1-2T^2_{\mathrm{pr}}b_{\mathrm{pr}}+\sqrt{1-8T^2_{\mathrm{pr}}b_{\mathrm{pr}}}}{2T^2_{\mathrm{pr}}}$ such that 
\begin{eqnarray}
\begin{cases}
W < 0      & when ~~v_L\leq v\leq v_U \\
W >0   & otherwise \\
\end{cases}
\end{eqnarray}
\end{proof}

\bibliographystyle{elsarticle-harv}
%\bibliography{../../../IP}
\bibliography{IP}

%% not want to use model2-names.bst.



%% \bibitem must have one of the following forms:
%%   \bibitem[Jones et al.(1990)]{key}...
%%   \bibitem[Jones et al.(1990)Jones, Baker, and Williams]{key}...
%%   \bibitem[Jones et al., 1990]{key}...
%%   \bibitem[\protect\citeauthoryear{Jones, Baker, and Williams}{Jones
%%       et al.}{1990}]{key}...
%%   \bibitem[\protect\citeauthoryear{Jones et al.}{1990}]{key}...
%%   \bibitem[\protect\astroncite{Jones et al.}{1990}]{key}...
%%   \bibitem[\protect\citename{Jones et al., }1990]{key}...
%%   \harvarditem[Jones et al.]{Jones, Baker, and Williams}{1990}{key}...
%%


% \end{thebibliography}

\end{document}